\documentclass[12pt, a4paper, twoside, openright]{ociamthesis}  

\usepackage[utf8]{inputenc}
\usepackage[english]{babel}
\usepackage{textgreek}

\usepackage[hidelinks]{hyperref}
\usepackage{breakurl}

\usepackage[
    style=alphabetic,
    sorting=nyt,
    backend=biber,
    maxbibnames=99,
    doi=true,
    isbn=false,
    url=true
]{biblatex}

\usepackage{fancyhdr}
\usepackage{extramarks}
\renewcommand{\chaptermark}[1]{\markboth{#1}{}}

\DeclareSourcemap{
  \maps[datatype=bibtex]{
    \map{
      \step[fieldsource=doi,final]
      \step[fieldset=url,null]
      \step[fieldset=urldate,null]
    }
    \map{
      \step[fieldsource=eprint,final]
      \step[fieldset=url,null]
      \step[fieldset=urldate,null]
    }
  }
}

\DeclareBibliographyCategory{enumpapers}

\usepackage[T1]{fontenc}
\usepackage{appendix}

\usepackage{mathtools}
\usepackage{amsthm}
\usepackage{amsmath}
\usepackage{amssymb}
\usepackage{enumerate}
\usepackage{mathrsfs}
\usepackage{braket}
\usepackage{physics}
\usepackage{bussproofs}

\usepackage[frozencache=true, cachedir=minted-cache]{minted}
\setminted{frame=lines,fontseries=mono,fontsize=\footnotesize}
\setmintedinline{fontsize=auto}
\BeforeBeginEnvironment{minted}{\vspace{-16pt}}
\AfterEndEnvironment{minted}{\vspace{-8pt}}
\newcommand{\py}[1]{{\color{brown}\mintinline{python}{#1}}}

\usepackage{tikz-cd}
\usepackage{tikz}
\usepackage{tikzit}

\tikzstyle{triangle}=[fill=white, draw=black, regular polygon, regular polygon sides=3, scale=0.5]
\tikzstyle{black node}=[fill=black, draw=black, shape=circle, scale=0.3]
\tikzstyle{white node}=[fill=white, draw=black, shape=circle, scale=0.6]
\tikzstyle{hadamard}=[fill=yellow, draw=black, shape=rectangle, scale=0.5]
\tikzstyle{big white node}=[fill=white, draw=black, shape=circle, scale=0.6]
\tikzstyle{red node}=[fill=red, shape=circle, scale=0.4]
\tikzstyle{big black}=[fill=black, draw=black, shape=circle, scale=1.2]
\tikzstyle{green node}=[fill=green, draw=green, shape=circle, scale=0.4]
\tikzstyle{new style 1}=[fill=white, draw=black, shape=circle]
\tikzstyle{white-box}=[fill=white, draw=black, shape=rectangle, scale=0.4]
\tikzstyle{blue-box}=[fill=bue, draw=black, shape=rectangle, scale=0.4]

\tikzstyle{mid-arrow}=[-]
\tikzstyle{dashs}=[-, dashed, line width=0.15mm]
\tikzstyle{thick}=[-, line width=0.5mm]
\tikzstyle{arrow}=[->]
\tikzstyle{invisible}=[-, draw=none]
\tikzstyle{functor}=[-, fill={rgb,255: red,240; green,240; blue,240}]
\tikzstyle{boxedge}=[-, fill=white]
\tikzstyle{red-edge}=[-, color={red!120}, line width=0.4mm]
\tikzstyle{arrow}=[->]
\tikzstyle{blue-edge}=[-, color={blue!80}, line width=0.4mm]

\tikzstyle{mid-arrow}=[postaction={decorate, decoration={markings, mark=at position 0.5 with {\arrow{>}}}}]
\usetikzlibrary{matrix,arrows,decorations.pathmorphing}
\usetikzlibrary{decorations.markings}
\tikzset{ decoration={
    markings,
    mark=at position 0.5 with {\arrow{>}}}}

\usepackage{titlesec}

\newcommand{\s}{\enspace}
\newcommand{\sub}{\subseteq}
\newcommand{\size}[1]{\left\vert{#1}\right\vert}
\renewcommand{\tt}[1]{\mathtt{#1}}
\renewcommand{\bf}[1]{\mathbf{#1}}

\renewcommand{\cal}[1]{\mathcal{#1}}
\newcommand{\bb}[1]{\mathbb{#1}}
\renewcommand{\phi}{\varphi}
\newcommand{\xto}[1]{\xrightarrow{#1}}
\newcommand{\injects}{\xhookrightarrow{}}

\newcommand{\graph}[2]{#1 \rightrightarrows #2}

\newcommand{\N}{\bb{N}}

\newcommand{\signature}[2]{#2 \xleftarrow{\tt{dom}} #1 \xrightarrow{\tt{cod}} #2}

\newtheorem{definition}{Definition}[section]
\newtheorem{proposition}[definition]{Proposition}
\newtheorem{theorem}[definition]{Theorem}
\newtheorem{conjecture}[definition]{Conjecture}
\newtheorem{lemma}[definition]{Lemma}
\newtheorem{corollary}[definition]{Corollary}
\newtheorem{example}[definition]{Example}
\newtheorem{remark}[definition]{Remark}
\newtheorem{python}[definition]{Listing}

\usepackage{tabularx,environ}
\makeatletter
\newcommand{\problemtitle}[1]{\gdef\@problemtitle{#1}}
\newcommand{\probleminput}[1]{\gdef\@probleminput{#1}}
\newcommand{\problemoutput}[1]{\gdef\@problemoutput{#1}}
\NewEnviron{problem}{
  \problemtitle{}\probleminput{}\problemoutput{}
  \BODY

   \@problemtitle
  \par\addvspace{.5\baselineskip}
  \noindent
  \begin{tabularx}{\textwidth}{@{\hspace{\parindent}} l X c}
    \textbf{Input:} & \@probleminput \\
    \textbf{Output:} & \@problemoutput
  \end{tabularx}
  \par\addvspace{.5\baselineskip}
}

\newcommand\G{\mathcal{G}}

\usepackage{tikz}
\usepackage{varwidth}
\usetikzlibrary{decorations.markings}
\tikzset{->-/.style={decoration={markings, mark=at position #1 with {\arrow{>}}}, postaction={decorate}},->-/.default=0.5}
\tikzset{game/.style={rectangle, minimum height = #1, minimum width = .8 cm, draw},game/.default=1.2cm}
\tikzset{dot/.style={circle, scale=.5, fill=#1, draw},dot/.default=black}
\tikzset{player/.style={isosceles triangle, isosceles triangle apex angle=90, inner sep=0pt, minimum width=2cm, shape border rotate=#1, draw},player/.default=180}
\newcommand{\diset}[2]{\binom{#1}{#2}}
\newcommand{\game}[6]{#1 \colon  \diset{#2}{#3} \overset{#6} \nrightarrow \diset{#4}{#5}}

\newcommand{\Pf}{\pi}
\newcommand{\Cf}{\kappa}
\newcommand{\Ef}{E}

\newcommand{\amax}{\operatorname*{argmax}}

\newcommand\teacher{\mathcal{T}}
\newcommand\student{\mathcal{S}}
\newcommand\marker{\mathcal{M}}
\newcommand\dist{d}

\title{Categorical Tools for Natural Language Processing}

\author{Giovanni de Felice}             
\college{Wolfson College}  
\degree{Doctor of Philosophy}     
\degreedate{Michaelmas 2022}         

\begin{document}

\maketitle
\cleardoublepage
\begin{abstract}
This thesis develops the translation between category theory and
computational linguistics as a foundation for natural language processing.
The three chapters deal with syntax, semantics and pragmatics.
First, string diagrams provide a unified model of syntactic structures in formal grammars.
Second, functors compute semantics by turning diagrams into logical, tensor, neural or quantum computation.
Third, the resulting functorial models can be composed to form games where equilibria are the solutions of language processing tasks.
This framework is implemented as part of DisCoPy, the Python library for computing with string diagrams.
We describe the correspondence between categorical, linguistic and computational structures,
and demonstrate their applications in compositional natural language processing.
\end{abstract}

\cleardoublepage

\begin{acknowledgements}
\addcontentsline{toc}{chapter}{Acknowledgements}

I would like to thank my supervisor Bob Coecke for introducing me to the
wonderland of string diagrams, for supporting me throughout my studies and
always encouraging me to switch topics and pursue my ideas.
This thesis is the fruit of a thousand discussions, board sessions,
smokes and beers with Alexis Toumi.
I am grateful to him for having the patience to teach Python to a mathematician,
for his loyal friendship and the continual support he has given me
in both personal matters and research.
I want to thank my examiners Aleks Kissinger and Pawel Sobocinski for their
detailed feedback on the first version of this thesis, and their suggestion to
integrate the passage from categories to Python code.
Thanks also to Andreas Petrossantis, Sebastiano Cultrera and Dimitri Kartsaklis
for valuable comments on this manuscript,
to Dan Marsden and Prakash Panangaden for providing guidance in my early research,
and to Samson Abramsky for prompting me to search into the deeper history of
applied category theory.

Among fellow collaborators who have shared their wisdom and passion,
for many insightful discussions, I would like to thank Amar Hadzihasanovic,
Rui Soares Barbosa, David Reutter, Antonin Delpeuch, Stefano Gogioso, Konstantinos Meichanetzidis, Mario Roman, Elena Di Lavore and Richie Yeung.
Among my friends, who have been there for me in times of sadness and of joy,
and made me feel at home in Oxford, Rome and Sicily,
special thanks to Tommaso Salvatori, Pinar Kolancali, Tommaso Battistini,
Emanuele Torlonia, Benedetta Magnano and Pietro Scammacca.
Finally, a very special thanks to Nonna Miti for hosting me in her garden
and to my mother and father for their loving support.

\end{acknowledgements}

\tableofcontents
\chapter*{Introduction}
\markboth{Introduction}{}
\addcontentsline{toc}{chapter}{Introduction}

Since the very beginnings of human inquiry into language, people have investigated the natural processes by which we learn, understand and produce linguistic meaning.
Only recently, however, the field of linguistics has become an autonomous
scientific discipline.
The origins of this modern science are closely interlinked with
the birth of mathematical logic at the end of the nineteenth century.
In the United States, Peirce founded ``semiotics'' --- the science of signs and their interpretation --- while developing graphical calculi for logical inference.
At around the same time, in the United Kingdom, Frege and Russell developed formal languages for logic in the search for a Leibnizian ``characteristica universalis''
while discussing the sense and reference of linguistic phrases.

These mathematical origins initiated a formal and computational
approach to linguistics, often referred to as the symbolic tradition,
which aims at characterising language understanding in terms of structured logical
processes and the automatic manipulation  of symbols.
On the one hand, it led to the development of mathematical theories of syntax,
such as the categorial grammars stemming from the Polish school of logic
\cite{ajdukiewiz1935, lambek1958} and Chomsky’s influential generative grammars \cite{Chomsky57}.
On the other hand, it allowed for the development of formal approaches to semantics such as Tarski's theory of truth \cite{tarski1936, tarski1943}, which
motivated the work of Davidson \cite{davidson1967a} and Montague \cite{montague1970}
in extracting the logical form of natural language sentences.
From the technological perspective, these theories enabled the design of programming
languages, the construction of large-scale databases for storing structured knowledge
and linguistic data, as well as the implementation of expert computer systems
driven by formal logical rules to reason about this accrued knowledge.
Since the 1990s, the symbolic tradition has been challenged by a series of
new advances motivated by the importance of context and ambiguity in language
use \cite{barwise1983}.
With the growing amount of data and large-scale corpora available on the internet,
statistical inference methods based on $n$-grams, Markov models or Bayesian
classifiers allowed for experiments to tackle new problems such as speech
recognition and machine translation \cite{manning1999}.
The distributional representation of meaning in vector spaces \cite{sparckjones1997} was found suitable for disambiguating words in context \cite{schutze1998} and computing synonymity \cite{turney2010}.
Furthermore, connectionist models based on neural networks have produced impressive results in the last couple of decades, outperforming previous models
on a range of tasks such as language modelling \cite{bengio2003, mikolov2010, vaswani2017}, word sense disambiguation \cite{navigli2009},
sentiment analysis \cite{socher2013a, chan2022}, question answering \cite{jurafsky2008, laskar2020} and machine translation \cite{bahdanau2014, edunov2018}.

Driven by large-scale industrial applications, the focus gradually shifted
from theoretical enquiries into linguistic phenomena to the practical concern
of building highly parallelizable connectionist code for beating
state-of-the-art algorithms. Recently, a transformer neural network with billions
of parameters (GPT-3) \cite{brown2020} wrote a Guardian article on why humans have nothing to fear from AI. The reasons for how and why GPT-3 ``chose'' to
compose the text in the way that it did is a mystery and the structure of
its mechanism remains a ``black box''.
Connectionist models have shown the importance of the distributional aspect of language and the effectiveness of machine learning techniques in NLP.
However, their task-specificity and the difficulty in analysing the underlying processes which concur in their output are limits which need to be addressed.
Recent developments in machine learning have shown the importance of taking structure into account when tackling scientific questions in network science \cite{wu2020}, chemistry \cite{kearnes2016}, biology \cite{jumper2021, zhou2021}.  NLP would also benefit from the same grounding in order to analyse and interpret the growing ``library of Babel'' of natural language data.

Category theory can help build models of language amenable to both linguistic reasoning and numerical computation.
Its roots are the same as computational linguistics, as categories were used to
link algebra, logic and computation \cite{Lawvere63, lambek1986a, lambek1986}.
Category theory has since followed a solid thread of applications, from the semantics of programming languages \cite{saraswat1991, abramsky1995} to the modelling of a wide range of computational systems, including knowledge-based \cite{spivak2012}, quantum \cite{abramsky2007, coecke2017a}, dataflow \cite{bonchi2014a},
statistical \cite{vakarmatthijs2019} and differentiable \cite{abadimartin2019} processes. In the Compositional Distributional models of Coecke et al. \cite{DisCoCat08,DisCoCat11, sadrzadeh2013} (DisCoCat), categories are used to design models of language in which the meaning of sentences is derived by composition from the distributional embeddings of words.
Generalising from this work, language can be viewed as a syntax for arranging symbols together with a functor for interpreting them.
Specifically, syntactic structures form a free category of string diagrams, while meaning is computed in categories of numerical functions. Functorial models can then be learnt in data-driven tasks.

The aim of this thesis is to provide a unified framework of mathematical tools
to be applied in three important areas of computational linguistics: syntax, semantics and pragmatics.
We provide an implementation of this framework in object-oriented Python,
by translating categorical concepts into classes and methods.
This translation has lead to the development of DisCoPy \cite{defelice2020b},
an open-source Python toolbox for computing with string diagrams and functors.
We show the potential of this framework for reasoning about compositional models of language and building structured NLP pipelines.
We show the correspondence between categorical and linguistic notions and we
describe their implementation as methods and interfaces in DisCoPy.
The library is available, with an extensive documentation and testing suite, at:
\begin{center}
    \url{https://github.com/oxford-quantum-group/discopy}
\end{center}

In Chapter 1, on syntax, we use the theory of free categories and
string diagrams to formalise Chomsky's regular, context-free and unrestricted
grammars \cite{chomsky1956}. With the same tools, the categorial grammars of Ajdiuciewicz \cite{ajdukiewiz1935}, Lambek \cite{lambek1958} and Steedman \cite{steedman2000}, as well as Lambek's pregroups \cite{lambek1999a} and Tesniere's dependency grammars \cite{tesniere1959, gaifman1965}, are formalised.
We lay out the architecture of the syntactic modules of DisCoPy, with interfaces
for the corresponding formal models of grammar and functorial reductions
between them.
The second chapter deals with semantics. We use Lawvere's concept of functorial semantics \cite{Lawvere63} to define several NLP models, including logical, distributional and connectionist approaches. By varying the target semantic category, we recover knowledge-based, tensor-based and quantum models of language,
as well as Montague's logical semantics \cite{montague1970} and connectionist models based on neural networks. The implementation of these models in Python is obtained by defining semantic classes that carry out concrete computation. We describe the implementation of the main semantic modules of DisCoPy and their interface with high-performance libraries for numerical computation.
This framework is then applied to the study of the pragmatic aspects of language use
in context and the design of NLP tasks in Chapter 3. To this end, we use the
recent applications of category theory to statistics \cite{cho2019, fritz2020},
machine learning \cite{fong2019e, cruttwell2021} and game theory \cite{GhaniHedges18}
to develop a formal language for modelling compositions of NLP models into games
and pragmatic tasks.

The mathematical framework developed in this thesis provides a structural
understanding of natural language processing, allowing for both interpreting
existing NLP models and building new ones.
The proposed software is expected to contribute to the design of language
processing systems and their implementation using symbolic, statistical,
connectionist and quantum computing.

\addcontentsline{toc}{section}{Contributions}
\section*{Contributions}

The aim of this thesis is to provide a framework of mathematical tools for computational linguistics. The three chapters are related to syntax, semantics and pragmatics, respectively.

In Chapter 1, a unified theory of formal grammars is provided in terms of free categories and functorial reductions and implemented in object-oriented Python. This work started from discussions with Alexis Toumi, Bob Coecke, Mernoosh Sadrzadeh, Dan Marsden and Konstantinos Meichanetzidis about logical and distributional models of
natural language \cite{coecke2018c, defelice2020, coecke2022}. It was a driving force in the development of DisCoPy \cite{defelice2020b}.
The main contributions are as follows.

\begin{enumerate}
    \item A unified treatment of \emph{formal grammars} is provided
    using the theory of string diagrams in free monoidal categories. We cover
    Chomsky's regular \ref{section-path}, context-free \ref{section-trees} and
    unrestricted grammar \ref{section-diagrams}, corresponding
    to free categories, free operads and free monoidal categories respectively.
    This picture is obtained by aggregating the results of Walters and Lambek \cite{walters1989a, walters1989b, lambek1988, shiebler2020}.
    Using the same tools, we formalise categorial grammars \ref{sec-closed},
    as well as pregroups and dependency grammars \ref{sec-compact} in terms of
    biclosed and rigid categories.
    The links between pregroups and rigid categories are known since  \cite{preller2007b}, those between categorial grammar and biclosed categories were previously discussed by Lambek \cite{lambek1988} but not fully worked out, while the categorical formalisation of dependency grammar is novel.
    We also introduce the notion of pregroup with coreference for discourse
    representation in \ref{sec-hypergraph} \cite{coecke2018c} to offer an alternative to the DisCoCirc framework of Coecke \cite{coecke2020b} which can be implemented with readily available tools.
    To the best of our knowledge, this is the first time that models from the Chomskyan and categorial traditions appear in the same framework, and that the full correspondence between linguistic and categorical notions is spelled out.
    \item \emph{Functorial reductions} between formal grammars are introduced as a convenient intermediate notion between weak and strong equivalences, and used to compute normal forms of context-free grammars in \ref{section-reductions}. We use this notion in the remainder of the chapter to capture the relationship between:
    i) CFGs and categorial grammar (Propositions \ref{prop-ab-cfg} and \label{prop-cfg-ab}),
    ii) categorial grammar and biclosed categories (Propositions \ref{prop-ab-biclosed}, \ref{prop-lambek-biclosed} and \ref{prop-crossed-biclosed}),
    iii) categorial and pregroup grammars (Propositions \ref{prop-categorial-to-pregroup} and \ref{prop-crossed-pregroup}) and
    iv) pregroups, dependency grammars and CFGs in \ref{section-dependency}.
    The latter yields a novel result showing that dependency grammars are the structural intersection of pregroups and CFGs (Theorem \ref{theorem-dependencies}).
    \item The previously introduced categorical definitions are implemented in
    \emph{object-oriented Python}.
    The structure of this chapter follows the architecture of the syntactic modules of DisCoPy, as described at the end of this section.
    Free categories and syntactic structures are implemented by subclassing \py{cat.Arrow} or \py{monoidal.Diagram}, the core data structures of DisCoPy.
    Functorial reductions are implemented by calling the corresponding \py{Functor} class. We interface DisCoPy with linguistic tools for
    large-scale parsing.
\end{enumerate}

In Chapter 2, functorial semantics is applied to the study of natural language
processing models. Once casted in this algebraic framework, it becomes possible
to prove complexity results and compare different NLP pipelines, while also implementing these models in DisCoPy.
Section \ref{sec-rel-model} is based on joint work with Alexis Toumi and Konstantinos Meichanetzidis on relational semantics \cite{defelice2020}.
Sections \ref{sec-tensor-network}, \ref{sec-discopy} and \ref{sec-quantum-model} are based on the recent quantum models for NLP introduced with Bob Coecke, Alexis Toumi and Konstantinos Meichanetzidis
\cite{meichanetzidis2020, meichanetzidis2020a} and further developed in \cite{kartsaklis2021, toumi2022}.
We list the main contributions of this chapter.

\begin{enumerate}
    \item \emph{NLP models} are given a unified theory, formalised as functors
    from free syntactic categories to concrete categories of numerical structures. These include i) knowledge-based relational models \ref{sec-rel-model}
    casted as functors into the category of relations, ii)
    tensor network models \ref{sec-tensor-network} seen as functors into the category of matrices and including factorisation models for knowledge-graph embedding covered in \ref{sec-embedding} iii) quantum NLP models \ref{sec-quantum-model} casted as functors into the category of quantum circuits, iv)
    Montague semantics \ref{sec-functional} given by functors into cartesian closed categories, and v) recurrent and recursive neural network models
    \ref{sec-neural-net} which appear as functors from grammar to the category of functions on euclidean spaces.
    \item We prove \emph{complexity results} on the evaluation of these functorial models and related NLP tasks.
    Expanding on \cite{defelice2020}, we use relational models to
    define $\tt{NP}$-complete entailment and question-answering problems
    (Propositions \ref{prop-rel-np}, \ref{prop-entailment} and \ref{prop-qa}). We show that the evaluation of tensor network models is in general $\#\tt{P}$-complete (Proposition \ref{prop-tensor-semantics}) but that it becomes tractable when the input structures come from a dependency grammar
    (Proposition \ref{prop-tensor-dep}).
    We show that the additive approximation of quantum NLP models is a
    $\tt{BQP}$-complete problem (Proposition \ref{prop-bqp}).
    We also prove results showing that Montague semantics is intractable in its general form (Propositions \ref{prop-elementary}, \ref{prop-pspace}).
    \item \emph{Montague semantics} is given a detailed formalisation in terms of
    free cartesian closed categories. This corrects a common misconception in the
    DisCoCat literature, and in particular in \cite{sadrzadeh2013, sadrzadeh2014},
    where Montague semantics is seen as a functor from pregroup grammars to relations.
    These logical models are studied in \ref{sec-rel-model}, but they are distinct from Montague semantics where the lambda calculus and higher-order types
    play an important role \ref{sec-functional}.
    \item We show how to implement \emph{functorial semantics in DisCoPy}.
    More precisely, the implementation of the categories of tensors and Python functions is described in \ref{sec-concrete}.
    We then give a concrete example of how to solve a knowledge-embedding task in DisCoPy by learning functors \ref{sec-discopy}.
    We define currying and uncurrying of Python functions \ref{sec-functional}
    and use it to give a proof-of-concept implementation of Montague semantics.
    We define sequential and parallel composition of Tensorflow/Keras models \cite{chollet2015keras}, allowing us to construct recursive neural networks with a DisCoPy functor \ref{sec-neural-net}.
\end{enumerate}

In Chapter 3, we develop formal diagrammatic tools to model pragmatic scenarios
and natural language processing tasks. This Chapter is based on our work
with Mario Roman, Elena Di Lavore and Alexis Toumi \cite{defelice2020a}, and provides a basis for the formalisation of monoidal streams in the stochastic setting \cite{dilavore2022}.
The contribution for this section is still at a preliminary stage, but the diagrammatic notation succeeds in capturing and generalising a range of approaches found in the literature as follows.

\begin{enumerate}
    \item \emph{Categorical probability} \cite{fritz2020} is applied to the study of discriminative and generative language models using notions from Bayesian statistics in \ref{section-tasks}. We investigate the category of \emph{lenses} \cite{riley2018a} over discrete probability distributions in
    \ref{section-tools} and use it to characterise notions of context, utility and reward for iterated stochastic processes \cite{dilavore2022} (see Propositions \ref{prop-lenses-combs} and \ref{prop-contexts-factor}).
    Finally, we apply open games \cite{bolt2019} the study of Markov decision processes and repeated games in \ref{section-agents}, while giving examples relevant to NLP.
    \item Three \emph{NLP applications} of the developed tools are provided in \ref{section-examples}:
    i) we discuss reference games \cite{frank2012, monroe2015, bergen2016} and give a diagrammatic proof that Bayesian inversion yields a Nash equilibrium (Proposition \ref{prop-nash-equilibrium}),
    ii) we define a question answering game between a teacher and a student and compute the Nash equilibria when the student's strategies are given by
    relational models \cite{defelice2020a} and
    iii) we give a compositional approach to word-sense disambiguation as a game
    between words where strategies are word-senses.
\end{enumerate}

The DisCoPy implementation, carried out with Alexis Toumi \cite{defelice2020b}, is described throughout the first two chapters of this thesis. We focused on
i) the passage from categorical definitions to object-oriented Python and
ii) the applications of DisCoPy to Natural Language Processing.
Every section of Chapter 1 corresponds to a \emph{syntactic module} in DisCoPy, as we detail.

\begin{enumerate}
    \item In \ref{section-path},  we view derivations of \emph{regular} grammars
    as arrows in the free category and show how these notions are implemented via the core DisCoPy class \py{cat.Arrow}.
    \item In \ref{section-trees}, we study \emph{context-free} grammars in terms
    of trees in the free operad and give an implementation of free operads and
    their algebras in DisCoPy. This is a new \py{operad} module which has been written for this thesis and features interfaces with NLTK \cite{loper2002}
    for CFG and SpaCy \cite{spacy} for dependencies.
    \item In \ref{section-diagrams}, we formalise Chomsky's \emph{unrestricted}
    grammars in terms of monoidal signatures and string diagrams and
    describe the implementation of the key DisCoPy class \py{monoidal.Diagram}.
    \item In \ref{sec-closed}, we formalise \emph{categorial} grammar in terms of free biclosed categories.
    We give an implementation of \py{biclosed.Diagram} as a monoidal diagram with
    \py{curry} and \py{uncurry} methods and show its interface with
    state-of-the-art categorial grammar parsers, as provided by Lambeq \cite{kartsaklis2021}.
    \item In \ref{sec-compact}, we show that \emph{pregroup} and \emph{dependency} structures are diagrams in free rigid categories.
    We describe the data structure \py{rigid.Diagram} and its interface with SpaCy \cite{spacy} for dependency parsing.
    \item In \ref{sec-hypergraph}, we introduce a notion
    of pregroup grammar with \emph{coreference} using hypergraphs to represent the syntax of text and discourse. We describe the \py{hypergraph.Diagram} data structure and show how to interface it with SpaCy's package for neural coreference.
\end{enumerate}

The models studied in Chapter 2 can all be implemented using one of the four \emph{semantic modules} of DisCoPy which we detail.

\begin{enumerate}
    \item The \py{tensor} module of DisCoPy implements the category of matrices, as described in \ref{sec-concrete} where we give its implementation in NumPy \cite{harris2020array}.
    We use it in conjunction with Jax \cite{jax2018github} in \ref{sec-discopy}
    to implement the models introduced in \ref{sec-rel-model} and \ref{sec-tensor-network}.
    \item The \py{quantum} module of DisCoPy implements quantum circuits \ref{sec-quantum-model} and supports interfaces with PyZX \cite{kissinger2019} and TKET \cite{sivarajah2020} for optimisation and compilation on quantum hardware. These features are described in our recent work \cite{toumi2022}.
    \item The \py{function} module of DisCoPy implements the category of functions on Python types, as described in \ref{sec-concrete}. We define currying and uncurrying of Python functions \ref{sec-functional}
    and use it to give a proof-of-concept implementation of Montague semantics.
    \item The \py{neural} module implements the category of neural networks on euclidean spaces. We describe it in \ref{sec-neural-net} as
    an interface between DisCoPy and TensorFlow/Keras \cite{chollet2015keras}.
\end{enumerate}

A schematic view of the modules in DisCoPy and their interfaces is summarized in Figure \ref{architecture}.

\begin{figure}
    \centering
    \scalebox{0.9}{\input{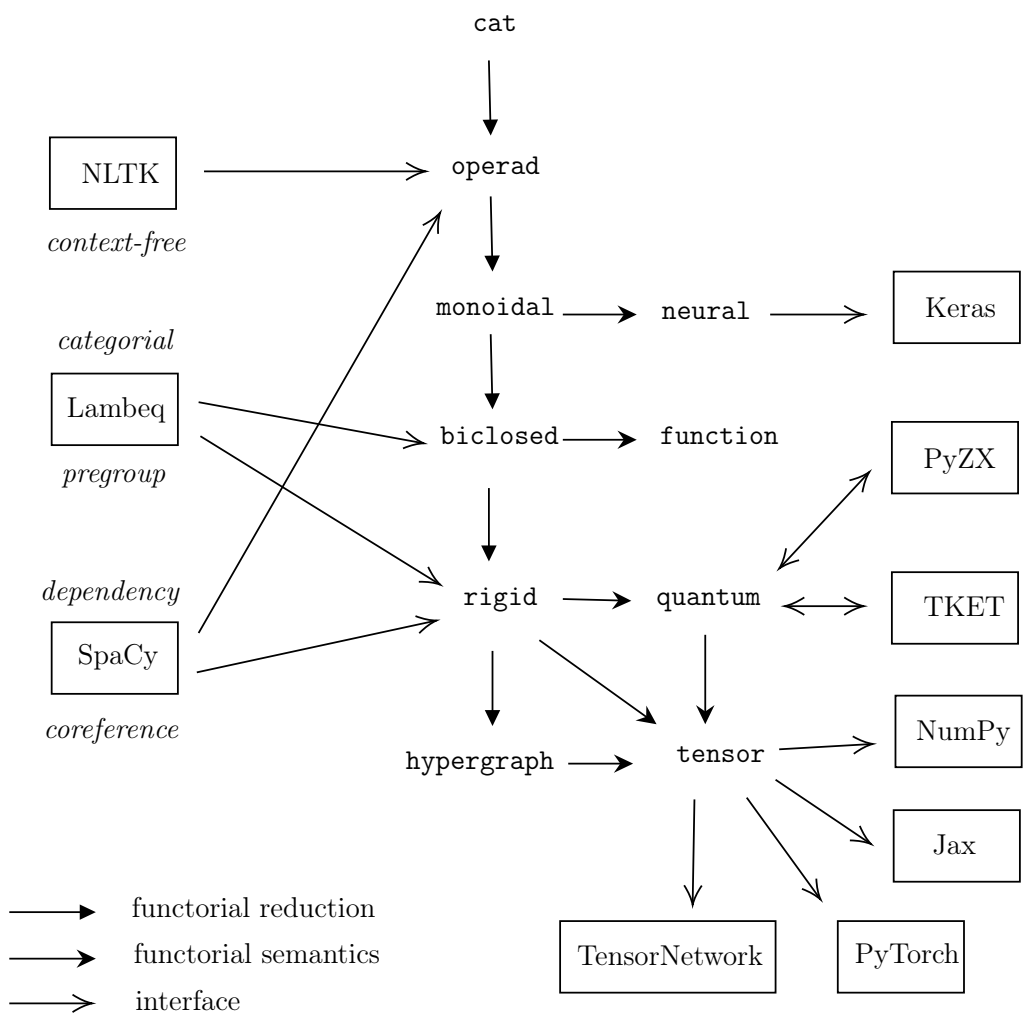}}
    \caption{DisCoPy: an interfaced compositional software for NLP}
    \label{architecture}
\end{figure}


\chapter{Diagrams for Syntax}
\pagestyle{fancy}
\renewcommand{\chaptermark}[1]{\markboth{#1}{}}
\renewcommand{\sectionmark}[1]{\extramarks{\thesection}{#1}}

The word ``grammar'' comes from the Greek
$\gamma \rho \acute{\alpha} \mu \mu \alpha$ (gramma), itself from
$\gamma\rho\acute{\alpha}\phi\epsilon\iota\nu$
(graphein) meaning both ``to write'' and ``to draw'', and we will represent
grammatical structure by drawing diagrams. A \emph{formal grammar} is usually defined
by a set of \emph{rewriting rules} on strings. The rewriting process, also
called \emph{parsing}, yields a  procedure for deciding whether a string of
words is grammatical or not.

These structures were studied in mathematics since the 1910s by Thue and later
by Post \cite{post1947} and Markov Jr. \cite{kushner2006}.
Their investigation was greatly advanced by Chomsky
\cite{Chomsky57}, who used them to \emph{generate} grammatical sentences from
some basic rewrite rules interpreted as \emph{productions}.
He showed that natural restrictions on the allowed production rules form a hierarchy,
from unrestricted to regular grammars, which corresponds to models of
computation of varying strengths, from Turing machines to deterministic finite
state automata \cite{chomsky1956}.
In parallel to Chomsky's seminal work, Lambek developed his syntactic calculus
\cite{lambek1958}, refining and unifying the \emph{categorial grammars} originated
in the Polish school of logic \cite{ajdukiewiz1935}.
These are different in spirit from Chomsky's grammars, but they also have tight
links with computation as captured by the Lambda calculus \cite{vanbenthem1987, steedman2000}.

In this chapter, we lay out the correspondence between free categorical structures
and linguistic models of grammar. Every level in this hierarchy is implemented with a corresponding class in DisCoPy.
In \ref{section-path},  we show that regular grammars
can be seen as graphs with a labelling homomorphism and their derivations as
arrows in the free category, implemented via the core DisCoPy class \py{cat.Arrow}.
In \ref{section-trees}, we show that context-free
grammars are operadic signatures and their derivations trees in the
free operad. We give an implementation of free operads as a class \py{operad.Tree},
interfaced with NLTK \cite{loper2002} for context-free parsing.
In \ref{section-diagrams},  we arrive at Chomsky's unrestricted
grammars, captured by monoidal signatures and string diagrams.
We discuss varying notions of reduction and normal form for these grammars,
and show the implementation of the key DisCoPy class \py{monoidal.Diagram}.
In \ref{sec-closed}, we show that categorial grammars such as the original
grammars of Ajdiuciewicz and Bar-Hillel, the Lambek calculus and
Combinatory Categorial Grammars (CCGs) can be seen as biclosed signatures and
their grammatical reductions as morphisms in free biclosed categories.
We give an implementation of \py{biclosed.Diagram} as a monoidal diagram with
\py{curry} and \py{uncurry} methods and show its interface with
state-of-the-art categorial grammar parsers, as provided by Lambeq \cite{kartsaklis2021}.
In \ref{sec-compact}, we show that pregroups and dependency grammars are both
captured by rigid signatures, and their derivations by morphisms in the free
rigid category. This leads to the data structure \py{rigid.Diagram} which we
interface with SpaCy \cite{spacy} for state-of-the-art dependency parsing.
Finally, in \ref{sec-hypergraph}, we introduce a notion
of pregroup grammar with coreference using hypergraphs to represent the syntax
of text and discourse, and give a proof-of-concept implementation in DisCoPy.

\section{Arrows}\label{section-path} 

In this section, we introduce three structures: categories, regular
grammars and \py{cat.Arrow}s. These cast light on a level-zero correspondence
between algebra, linguistics and Python programming.
We start by defining categories and their free construction from graphs.
Following Walters \cite{walters1989b}, regular grammars are defined as graphs
together with a labelling homomorphism and their grammatical sentences as labelled
paths, i.e. arrows of the corresponding free category.
This definition is very similar to the definition of a finite
state automaton, and we discuss the equivalence between Walters' notion,
Chomsky's original definition and finite automata.
We end by introducing the \py{cat} module, an implementation of free categories and
functors which forms the core of DisCoPy.


\subsection{Categories}

\begin{definition}[Simple signature / Directed graph]
    A simple signature, or directed graph, $G$ is a collection of vertices
    $G_0$ and edges $G_1$
    such that each edge has a domain and a codomain vertex
    $$\signature{G_1}{G_0}$$.
    A graph homomorphism $\phi: G \to \Gamma$ is a pair of functions
    $\phi_0 : G_0 \to \Gamma_0$ and $\phi_1: G_1 \to \Gamma_1$ such that the
    following diagram commutes:
    \begin{equation*}
        \begin{tikzcd}
        G_0 \arrow[d, "\phi_0"] & G_1 \arrow[l, "\tt{dom}"'] \arrow[r, "\tt{cod}"] \arrow[d, "\phi_1"] & G_0 \arrow[d, "\phi_0"]\\
        \Gamma_0 & \Gamma_1 \arrow[l, "\tt{dom}"] \arrow[r, "\tt{cod}"'] & \Gamma_0
        \end{tikzcd}
    \end{equation*}
    We denote by $G(a, b)$ the edges $f \in G_1$ such that
    $\tt{dom}(f) = a$ and $\tt{cod}(f) = b$.
    We also write $f: a \to b$ to denote an edge $f \in G(a, b)$.
\end{definition}

A category is a directed graph with a composition operation, in this context
vertices are called \emph{objects} and edges are called \emph{arrows} or
\emph{morphisms}.

\begin{definition}[Category]
    A category $\bf{C}$ is a graph $\bf{C}_1 \rightrightarrows \bf{C}_0$,
    where $\bf{C}_0$ is a set of \emph{objects}, and $\bf{C}_1$ a set of
    \emph{morphisms}, equipped with a composition operation
    $\cdot : \bf{C}(a, b) \times \bf{C}(b, c) \to \bf{C}(a, c)$
    defined for any $a, b, c \in \bf{C}_0$ such that:
    \begin{enumerate}
        \item for any $a \in \bf{C}_0$ there is an identity morphism
            $\tt{id}_a \in \bf{C}(a, a)$ (identities).
        \item for any $f: a \to b$,
        $f \cdot \tt{id}_a = f = \tt{id}_b \cdot f$ (unit law).
        \item whenever $a \xto{f} b \xto{g} c \xto{h} d$,
            we have $f \cdot (g \cdot h) = (f \cdot g) \cdot h$ (associativity).
    \end{enumerate}

    A \emph{functor} $F: \bf{C} \to \bf{D}$ is a graph homomorphism which respects
    composition and identities, i.e. for any $a \in \bf{C}_0$ $F(\tt{id}_a) = \tt{id}_{F(a)}$ and whenever $a \xto{g} b \xto{f} c$ in $\bf{C}$ we have $F(f \cdot g) = F(f) \cdot F(g)$.

    Given a pair of functors $F, G: \bf{C} \to \bf{D}$, a \emph{natural
    transformation} $\alpha: F \to G$ is a family of maps
    $\alpha_a : F(a) \to G(a)$ such that the following diagram commutes:
    \begin{equation}\label{eq-graph-homomorphism}
        \begin{tikzcd}
        F(a) \arrow[d, "F(f)"] \arrow[r, "\alpha_a"] & G(a) \arrow[d, "G(f)"]\\
        F(b) \arrow[r, "\alpha_b"] & G(b)
        \end{tikzcd}
    \end{equation}
    for any $f: a \to b$ in $\bf{C}$.
\end{definition}
\begin{remark}
    The symbol $\to$ appeared remarkably late in the history of symbols with
    the earliest use registered in Bernard Forest de Belidor's 1737
    \emph{L'architecture hydraulique}, where it is used to denote the direction
    of a flow of water. Arguably, it conveys more structured information then
    its predecessor: the Medieval manicule symbol. Its current mathematical
    use as the type of a morphism $f: x \to y $ appeared only at the beginning
    of the 20th century, the first extensive use being registered in Hausdorff
    \cite{hausdorff1935} to denote group homomorphisms.
\end{remark}
\begin{example}[Basic]
    Sets and functions form a category $\bf{Set}$.
    Monoids and monoid homomorphisms for a category $\bf{Mon}$.
    Graphs and graph homomorphisms form a category $\bf{Graph}$.
    Categories and functors form a category $\bf{Cat}$.
\end{example}

An arrow $f$ in a graph $G$ is a sequence of edges $f \in G_1^\ast$ such that
$\tt{cod}(f_{i}) = \tt{dom}(f_{i + 1})$, it can be represented graphically as
a sequence of arrows:
\begin{equation*}
    a_0 \xto{f} a_n = a_0 \xto{f_1} a_1 \dots \xto{f_n} a_n
\end{equation*}
Or as a sequence of vertices and edges:
\begin{equation*}
    \input{paths-as-nodes-edges.tikz}
\end{equation*}
Or as a sequence of boxes and wires:
\begin{equation*}
    \input{paths-as-boxes.tikz}
\end{equation*}
Two arrows with matching endpoints can be composed by concatenation.
\begin{equation*}
    a \xto{f} b \xto{g} c = a \xto{f \cdot g} c
\end{equation*}
Paths on $G$ in fact form a category, denoted $\bf{C}(G)$.
$\bf{C}(G)$ has the property of being the \emph{free category}
generated by $G$ \cite{MacLane71}.

The free category construction is the object part of functor
$\bf{C} : \bf{Graph} \to \bf{Cat}$,
which associates to any graph homomorphism $\phi : G \to V$ a functor
$\bf{C}(\phi): \bf{C}(G) \to \bf{C}(V)$ which relabels the vertices and
edges in an arrow. This free construction $\bf{C}$ is the \emph{left adjoint} of
the functor $U: \bf{Cat} \to \bf{Graph}$ which forgets the composition operation
on arrows. $\bf{C}$ is a left adjoint of $U$ in the sense that there is a natural
isomorphism:
$$\bf{Graph}(G, U(\bf{S})) \simeq \bf{Cat}(\bf{C}(G), \bf{S})$$
which says that specifying a functor $F: \bf{C}(G) \to \bf{S}$ is the same
as specifying an arrow in $\bf{S}$ for every generator in $G$. This will have important
consequences in the context of semantics.

A preorder $P$ is a category with at most one morphism between any two objects.
Given $a, b \in P_0$, the hom-set $P(a, b)$ is either the singleton
or empty, we write $a P b$ for the corresponding boolean value.
Identities and composition of the category, correspond to reflexivity and
transitivity of the preorder.
Following Lambek \cite{lambek1968} and \cite{strassurger2007}, we can interpret
a preorder as a \emph{logic}, by considering the underlying set as a set
of \emph{formulae} and the relation $\leq$ as a \emph{consequence} relation,
which is usually denoted $\vdash$ (entails).
Reflexivity and transitivity of the consequence relation
correspond to the following familiar rules of inference.

\begin{equation} \label{eq-reflexivity-and-transitivity}
    \AxiomC{}
    \RightLabel{$\qquad\qquad$}
    \UnaryInfC{$A \vdash A$}
    \DisplayProof
    \AxiomC{$A \vdash B$}
    \AxiomC{$B \vdash C$}
    \BinaryInfC{$A \vdash C$}
    \DisplayProof
\end{equation}

Given a graph $G$ we can build a preorder by taking the
reflexive transitive closure of the relation $G \sub G_0 \times G_0$ induced by the graph, $\leq = RTC(G) \sub G_0 \times G_0$.
We can think of the edges of the graph $G$ as a set of \emph{axioms}, then the
free preorder $RTC(G)$ captures the logical consequences of these axioms,
and, in this simple setting, we have that $a \implies b$ if and only if there
is an arrow from $a$ to $b$ in $G$.

This construction is analogous to the
free category construction on a graph, and is in fact part of a commuting
triangle of adjunctions relating topology, algebra and logic.
\begin{equation*}
    \begin{tikzcd}
     & \bf{Cat} \arrow[dl, "U"', bend right=25] \arrow[dr, "\exists"'] & \\
    \bf{Graph} \arrow[rr, "RTC"] \arrow[ur, "\bf{C}"'] & &
    \bf{Preord} \arrow[ul, "I"', bend right=25] \arrow[ll, "U", bend left=15]
    \end{tikzcd}
\end{equation*}
Where $\exists \bf{C}$ is the \emph{preorder collapse} of $\bf{C}$, i.e.
$a \exists \bf{C} b$ if and only if $\exists f \in \bf{C}(a, b)$.
The functor $RTC$ allows to define the following \emph{decision problem}.

\begin{definition}
    \begin{problem}
      \problemtitle{$\exists\tt{Path}$}
      \probleminput{$\graph{G}{B}$, $a, b \in N$}
      \problemoutput{$a \; RTC(G) \; b$}
    \end{problem}
\end{definition}

The problem $\exists\tt{Path}$ is also known as the graph accessibility
problem, which is complete for $\tt{NL}$ the complexity class of problems
solvable by a non-deterministic Turing machine in logarithmic space \cite{immerman1987}.
In particular, it is equivalent to $\tt{2SAT}$, the problem of satisfying a
formula in conjunctive normal form where each clause has two literals.
The reduction from $\tt{2SAT}$ to $\exists\tt{Path}$ is obtained by turning the
formula into its graph of implications.

The free construction $\bf{C}$ allows to define a \emph{proof-relevant} form of
the path problem, where what matters is not only whether $a$ entails $b$
but the way in which $a$ entails $b$.
\begin{definition}
    \begin{problem}
      \problemtitle{$\tt{Path}$}
      \probleminput{$G$, $a, b \in G_0$}
      \problemoutput{$f \in \bf{C}(G)(a, b)$}
    \end{problem}
\end{definition}
From these considerations we deduce that $\tt{Path} \in \tt{FNL}$, since
it is the \emph{function problem} corresponding to $\exists\tt{Path}$.
These problems correspond to parsing problems for regular grammars.

\subsection{Regular grammars}

We now show how graphs and categories formalise the notion of regular grammar.
Fix a finite set of words $V$, called the \emph{vocabulary}, and let us
use $V$ to label the edges in a graph $G$. The data for such a \emph{labelling}
is a function $L: G_1 \to V$ from edges to words, or equivalently
a graph homomorphism $L: G \to V$ where $V$ is seen as a graph
with one vertex and words as edges.
Fix a starting vertex $s_0$ and a terminal vertex $s_1$.
Given any arrow $f: s_0 \to s_1 \in \bf{C}(G)$, we can concatenate the labels
for each generator to obtain a string $L^\ast(f) \in \bf{C}(V) = V^\ast$
where $L^\ast = \bf{C}(L)$ is the function $L$ applied point-wise to arrows.
We say that a string $u \in V^\ast$ is \emph{grammatical} in $G$ whenever there
is an arrow $f: s_0 \to s_1$ in $G$ such that $L(f) = u$. We can think of the
arrow $f$ as a witness of the grammaticality of $u$, called a \emph{proof}
in logic and a \emph{derivation} in the context of formal grammars.

\begin{definition}[Regular grammar]
    A regular grammar is a finite graph $G$ equipped with a homomorphism
    $L: G \to V$, where $V$ is a set of words called the \emph{vocabulary},
    and two specified symbols $s_0, s_1 \in G_0$, the bottom (starting) and
    top (terminating) symbols.
    Explicitly it is given by three functions:
    \begin{equation*}
        \begin{tikzcd}
        G_0 & G_1 \arrow[l, "\tt{dom}"'] \arrow[r, "\tt{cod}"] \arrow[d, "L"] & G_0\\
         & V &
        \end{tikzcd}
    \end{equation*}
    The language generated by $G$ is given by the image of the labelling functor:
    $$\cal{L}(G) = L^\ast (\bf{C}(G)(s_0, s_1)) \sub V^\ast \, .$$
    A morphism of regular grammars $\phi: G \to H$ is a graph homomorphism
    such that the following triangle commutes:
    \begin{equation*}
        \begin{tikzcd}
        G \arrow[dr, "L"] \arrow[rr, "\phi"] &  & H \arrow[dl, "L"]\\
         & V &
        \end{tikzcd}
    \end{equation*}
    and such that $\phi(s_1) = s_1'$, $\phi(s_0) = s_0'$.
    These form a category of regular grammars $\bf{Reg}$ which is the
    slice or \emph{comma} category of the coslice over the points $\set{s_0, s_1}$
    of the category of signatures $\bf{Reg} = (2 \backslash \bf{Graph}) / V$.
\end{definition}

\begin{definition}[Regular language]
    A regular language is a subset of $X \sub V^\ast$ such that there is a regular grammar
    $G$ with $\cal{L}(G) = X$.
\end{definition}


\begin{example}[SVO]
    Consider the regular grammar generated by the following graph:
    \begin{equation*}
        \input{figures/family-relations.tikz}
    \end{equation*}
    An example of sentence in the language $\cal{L}(G)$ is ``A met B who met C.''.
\end{example}

\begin{example}[Commuting diagrams]
    Commuting diagrams such as \ref{eq-graph-homomorphism} can be understood using
    regular grammars. Indeed, a commuting diagram is a graph $G$ together with
    a labelling of each edge as a morphism in a category $\bf{C}$. Given a pair
    of vertices $x, y \in G_0$, we get a regular language $\cal{L}(G)$ given by
    arrows from $x$ to $y$ in $G$. Saying that the diagram $G$ commutes
    corresponds to the assertion that all strings in $\cal{L}(G)$ are equal as
    morphisms in $\bf{C}$.
\end{example}

We now translate from the original definition by Chomsky to the one above.
Recall that a regular grammar is a tuple $G = (N, V, P, s)$
where $N$ is a set of non-terminal symbols with a specified start symbol
$s \in N$, $V$ is a vocabulary and $P$ is a set of production rules
of the form $A \to aB$ or $A \to a$ or $A \to \epsilon$ where $a \in V$,
$A, B \in N$ and $\epsilon$ denotes the empty string.
We can think of the sentences generated by $G$ as arrows in a free category
as follows. Construct the graph $\Sigma = \graph{P}{(N + \set{s_1})}$
where for any production rule in $P$ of the form $A \to aB$ there is an edge
$A \xto{f} B$ with $L(f) = a$ and for any production rule $A \to a$ there is an edge
$A \xto{w} s_1$ with $L(w) = a$. The language generated by $G$ is the
image under $L$ of the set of labelled paths
$s \to s_1$ in $\Sigma$,
i.e. $\cal{L}(\Sigma) = L^\ast(\bf{C}(\Sigma)(s, s_1))$.

This translation is akin to the construction of a \emph{finite state automaton}
from a regular grammar. In fact, the above definition of regular grammar directly
is equivalent to the definition of \emph{non-deterministic} finite state automaton (NFA).
Given a regular grammar $(G, L)$, we can construct the span
$$ V \times G_0 \xleftarrow{L \times \tt{dom}} G_1 \xto{cod} G_0$$
which induces a relation $\tt{im}(G_1) \sub  V \times G_0 \times G_0$, which is
precisely the transition table of an NFA
with states in $G_0$, alphabet symbols $V$ and transitions in $\tt{im}(G_1)$.
If we require that the relation $\tt{im}(G_1)$ be a \emph{function} ---
i.e. for any  $(v, a) \in V \times G_0$ there is a unique $b \in G_0$ such that
$(v, a, b) \in \tt{im}(G_1)$ --- then this defines a \emph{deterministic}
finite state automaton (DFA).
From any NFA, one may build a DFA by blowing up the state space.
Indeed relations $X \sub V \times G_0 \times G_0$ are the same as
functions $V \times G_0 \to \cal{P}(G_0)$ where $\cal{P}$ denotes the powerset
construction. So any NFA $X \sub V \times G_0 \times G_0$
can be represented as a DFA $V \times \cal{P}(G_0) \to \cal{P}(G_0)$.

Now that we have shown how to recover the original definition of regular grammars,
consider the following folklore result from formal language theory.

\begin{proposition}
    Regular languages are closed under intersection and union.
\end{proposition}
\begin{proof}
   Suppose $G$ and $G'$ are regular grammars, with starting states $q_0, q_0'$
   and terminating states $q_1, q_1'$.

   Taking the cartesian product of the underlying graphs
   $G \times G' = \graph{G_1 \times G_1'}{G_0 \times G_0'}$
   we can define a regular grammar $G \cap G' \sub G \times G'$ with starting
   state $(q_0, q_0')$, terminating state $(q_1, q_1')$ and such that there is
   an edge between $(a, a')$ and $(b, b')$ whenever there are edges $a \xto{f} b$
   in $G$ and $a' \xto{f'} b'$ in $G'$ with the same label $L(f) = L'(f')$.
   Then an arrow from $(q_0, q_0')$ to  $(q_1, q_1')$ in $G \cap G'$ is the
   same as a pair of arrows $q_0 \to q_1$ in $G$ and $q_0' \to q_1'$ in $G'$.
   Therefore $\cal{L}(G \cap G') = \cal{L}(G) \cap \cal{L}(G')$.
   Proving the first part of the statement.

   Moreover, the disjoint union of graphs $G + G'$  yields a regular grammar
   $G \cup G'$ by identifying $q_0$ with $q_0'$ and $q_1$ with $q_1'$. Then
   an arrow $q_0 \to q_1$ in $G \cup G'$ is either an arrow $q_0 \to q_1$ in $G$
   or an arrow $q_0' \to q_1'$ in $G'$. Therefore
   $\cal{L}(G \cup G') = \cal{L}(G) \cup \cal{L}(G')$.
\end{proof}

The constructions used in this proof are canonical constructions in the
category of regular grammars $\bf{Reg}$. Note that
$\bf{Reg} = 2\backslash \bf{Graph} / V$ is both a slice and a coslice category.
Moreover, $\bf{Graph}$ has all limits and colimits.
While coslice categories reflect limits and slice categories reflect colimits,
we cannot compose these two statements to show that $\bf{Reg}$ has all limits
and colimits. However, we can prove explicitly that the constructions defined
above give rise to the categorical product and coproduct. We do not know whether
$\bf{Reg}$ also has equalizers and coequalizers which would yield all finite
limits and colimits.

\begin{proposition}
    The intersection $\cap$ of NFAs is the categorical product in $\bf{Reg}$.
    The union $\cup$ of NFAs is the coproduct in $\bf{Reg}$.
\end{proposition}
\begin{proof}
    To show the first part. Suppose we have two morphisms of regular grammars
    $f: H \to G$ and $g: H \to G'$. These must
    respect the starting and terminating symbols as well as the labelling homomorphism.
    We can construct a homomorphism of signatures $<f, g> : H \to G \cap G'$ where
    $G \cap G' = G \times_V G'$ with starting point $(q_0, q_0')$ and endpoint
    $(q_1, q_1')$ as defined above.
    $<f, g>$ is given on vertices by
    $<f, g>_0(x) = (f_0(x), g_0(x))$ and on arrows by
    $<f, g>_1(h) = (f_0(h), g_0(h))$. Since $L(f(h)) = L(h) = L(g(h))$,
    this defines a morphism of regular grammars $<f, g> : H \to G \cap G'$.
    There are projections $G \cap G' \to G, G'$ induced by the projections
    $\pi_0, \pi_1: G \times G' \to G, G'$, and it is easy to check that
    $\pi_0 \circ <f, g> = f$ and  $\pi_1 \circ <f, g> = g$ where $\pi_0$ and
    $\pi_1$ are the projections. Now, suppose that there is some $k: H \to G \cap G'$
    with $\pi_0 \circ k = f$ and $\pi_1 \circ k = g$, $k$ then the underlying function $H \to G \times G'$ must be equal to $<f, g>$ and thus also as
    morphisms of regular grammar $k = <f, g>$.
    Therefore $G \cap G'$ is the categorical product in $\bf{Reg}$.

    Similarly the union is the coproduct in $\bf{Reg}$. Given any pair of
    morphisms $f: G \to H$ and $g: G' \to H$, we may define
    $[f, g] : G \cup G' \to H$ on vertices by $[f, g](x) = f(x)$ if $x \in G$ and
    $[f, g](x) = g(x)$ if $x \in G'$ and on similarly on edges. We have that
    $[f, g](q_0) = f(q_0) = g(q_0') = q_1$ since $q_0$ and $q_0'$ are identified in
    $G \cup G'$ and $L([f, g](h))$ is either of these
    equal expressions $L(f(h))= L(h) = L(g(h))$, i.e. $[f, g]$ is a morphism
    of regular grammars. Let $i: G \to G \cup G'$ and $i': G' \to G \cup G'$
    be the injections into the union. We have $[f, g] \circ i = f$ and
    $[f, g] \circ i' = g$. Moreover, for any other morphism $k: G \cup G' \to H$
    with $k \circ i = f$ and $k \circ i' = g$, we must have that $k(h) = f(h)$
    whenever $h \in G \sub G \cup G'$ and $k(h) = g$ otherwise, i.e. $k = [f, g]$.
    Therefore $G \cup G'$ satisfies the universal property of the coproduct in
    $\bf{Reg}$.
\end{proof}

%

We now study the parsing problem for regular grammars.
First, consider the \emph{non-emptiness problem}, which takes as input a regular grammar
$G$ and returns ``no'' if $\cal{L}(G)$ is empty and ``yes'' otherwise.

\begin{proposition}
    The non-emptiness problem is equivalent to $\exists\tt{Path}$
\end{proposition}
\begin{proof}
    $\cal{L}(G)$ is non-empty if and only if there is an arrow from $s_0$ to
    $s_1$ in $G$.
\end{proof}

Now, consider the problem of recognizing the language of a regular grammar $G$,
also known as $\tt{Parsing}$. We define the proof-relevant version which, of
course, has a corresponding decision problem $\exists\tt{Parsing}$.

\begin{definition}
    \begin{problem}
      \problemtitle{$\tt{Parsing}$}
      \probleminput{$G$, $u \in V^\ast$}
      \problemoutput{$f \in \bf{C}(G)(s_0, s_1)$ such that $L^\ast(f) = u$.}
    \end{problem}
\end{definition}

Given a string $u \in V^\ast$, we may build the regular grammar $G_u$ given by
the path-graph with edges labelled according to $u$, so that $\cal{L}(G_u) = \set{u}$.
Then the problem of deciding whether there is a parse for $u$ in $G$ reduces
to the non-emptiness problem for the intersection $G \cap G_u$. This has the
following consequence.

\begin{proposition}
    $\exists\tt{Parsing}$ is equivalent to $\exists\tt{Path}$ and is thus
    $\tt{NL}$-complete. Similarly, $\tt{Parsing}$ is $\tt{FNL}$-complete.
\end{proposition}

At the end of the section, we give a simple algorithm for parsing regular grammars
based on the composition of arrows in a free category.
In order to model the higher levels of Chomsky's hierarchy, we need to equip
our categories with more structure.

\subsection{cat.Arrow}

We have introduced free categories and shown how they appear in formal language theory.
These structures have moreover a natural implementation in object-oriented Python,
which we now describe. In order to implement a free category in Python we need to
define three classes: \py{cat.Ob} for objects, \py{cat.Arrow} for morphisms
and \py{cat.Box} for generators.
Objects are initialised by providing a name.

\begin{python} \label{listing:cat.Ob}
{\normalfont Objects in a free category.}

\begin{minted}{python}
class Ob:
    def __init__(self, name):
        self.name = name
\end{minted}
\end{python}

Arrows, i.e. morphisms of free categories, are given by lists of boxes with matching
domains and codomains. In order to initialise a \py{cat.Arrow}, we provide a
domain, a codomain and a list of boxes.
The class comes with a method \py{Arrow.then} for composition and a static
method \py{Arrow.id} for generating identity arrows.

\begin{python} \label{listing:cat.Arrow}
{\normalfont Arrows in a free category.}

\begin{minted}{python}
class Arrow:
    def __init__(self, dom, cod, boxes, _scan=True):
        if not isinstance(dom, Ob) or not isinstance(cod, Ob):
            raise TypeError()
        if _scan:
            scan = dom
            for depth, box in enumerate(boxes):
                if box.dom != scan:
                    raise AxiomError()
                scan = box.cod
            if scan != cod:
                raise AxiomError()
        self.dom, self.cod, self.boxes = dom, cod, boxes

    def then(self, *others):
        if len(others) > 1:
            return self.then(others[0]).then(*others[1:])
        other, = others
        if self.cod != other.dom:
            raise AxiomError()
        return Arrow(self.dom, other.cod, self.boxes + other.boxes, _scan=False))

    @staticmethod
    def id(dom):
        return Arrow(self, dom, dom, [], _scan=False)

    def __rshift__(self, other):
        return self.then(other)

    def __lshift__(self, other):
        return other.then(self)
\end{minted}

    When \py{_scan == False} we do not check that the boxes in the arrow compose.
    This allows us to avoid checking composition multiple times for the same \py{Arrow}.
    The methods \py{__rshift__} and \py{__lshift__} allow to
    use the syntax \py{f >> g} and \py{g << f} for the composition of
    instances of the \py{Arrow} class.
\end{python}

Finally, generators of the free category are special instances of \py{Arrow},
initialised by a name, a domain and a codomain.

\begin{python} \label{listing:cat.Box}
{\normalfont Generators of a free category.}

\begin{minted}{python}
class Box(Arrow):
    def __init__(self, name, dom, cod):
        self.name, self.dom, self.cod = name, dom, cod
        Arrow.__init__(self, dom, cod, [self], _scan=False)
\end{minted}

    The subclassing mechanism in Python allows for \py{Box} to inherit all the
    \py{Arrow} methods, so that there is essentially no difference between a box and
    an arrow with one box.
\end{python}

\begin{remark}
    It is often useful to define dataclass methods such as
    \py{__repr__}, \py{__str__} and \py{__eq__} to represent, print and check
    equality of objects in a class.
    Similarly, other standard methods such as \py{__hash__} may be overwritten
    and used as syntactic gadgets.
    We set \py{self.name = name} although, in DisCoPy, this parameter
    is immutable and e.g. \py{Ob.name} is implemented as a \py{@property} method.
    In order to remain concise, we will omit these methods when defining further
    classes.
\end{remark}

We now have all the ingredients to compose arrows in free categories.
We check that the axioms of categories hold for \py{cat.Arrow}s on the nose.

\begin{python} \label{listing:cat.axioms}
{\normalfont Axioms of free categories.}

\begin{minted}{python}
x, y, z = Ob('x'), Ob('y'), Ob('z')
f, g, h = Box('f', x, y), Box('g', y, z), Box('h', z, x)
assert f >> Arrow.id(y) == f == Arrow.id(x) >> f
assert (f >> g) >> h == f >> g >> h == (f >> g) >> h
\end{minted}
\end{python}

A signature is a pair of lists, for objects and generators.
A homomorphism between signatures is a pair of Python dictionnaries.
Functors between the corresponding free categories are initialised by a pair
of mappings \py{ob}, from objects to objects, and \py{ar} from boxes to arrows.
The call method of \py{Functor} allows to evaluate the image of composite arrows.

\begin{python} \label{listing:cat.Functor}
{\normalfont Functors from a free category.}

\begin{minted}{python}
class Functor:
    def __init__(self, ob, ar):
        self.ob, self.ar = ob, ar

    def __call__(self, arrow):
        if isinstance(arrow, Ob):
            return self.ob[arrow]
        if isinstance(arrow, Box):
            return self.ar[arrow]
        if isinstance(arrow, Arrow):
            return Arrow.id(self(arrow.dom)).then(*map(self, arrow))
        raise TypeError()
\end{minted}

We check that the axioms hold.

\begin{minted}{python}
x, y = Ob('x'), Ob('y')
f, g = Box('f', x, y), Box('g', y, x)
F = Functor(ob={x : y, y: x}, ar={f : g, g: f})
assert F(f >> g) == F(f) >> F(g)
assert F(Arrow.id(x)) == Arrow.id(F(x))
\end{minted}
\end{python}

As a linguistic example, we use the composition method of \py{Arrow} to
write a simple parser for regular grammars.

\begin{python} \label{listing:cat.regular}
{\normalfont Regular grammar parsing.}

\begin{minted}{python}
from discopy.cat import Ob, Box, Arrow, AxiomError
s0, x, s1 = Ob('s0'), Ob('x'), Ob('s1')
A, B, C = Box('A', s0, x), Box('B', x, x), Box('A', x, s1)
grammar = [A, B, C]
def is_grammatical(string, grammar):
    arrow = Arrow.id(s0)
    bool = False
    for x in string:
        for box in grammar:
            if box.name == x:
                try:
                    arrow = arrow >> box
                    bool = True
                    break
                except AxiomError:
                    bool = False
        if not bool:
            return False
    return bool
assert is_grammatical("ABBA", grammar)
assert not is_grammatical("ABAB", grammar)
\end{minted}
\end{python}

So far, we have only showed how to implement \emph{free} categories and functors
in Python. However, the same procedure can be repeated. Implementing a
category in Python amounts to defining a pair of classes for objects and arrows
and a pair of methods for identity and composition.
In the case of \py{Arrow}, and the syntactic structures of this chapter,
we are able to check that the axioms hold in Python.
In the next chapter, these arrows will be mapped to concrete Python functions,
for which equality cannot be checked.

\section{Trees}\label{section-trees} 

Context-free grammars (CFGs) emerged from the linguistic work of Chomsky
\cite{chomsky1956} and are used in many areas of computer science.
They are obtained from regular grammars by allowing production rules to have
more than one output symbol, resulting in tree-shaped derivations.
Following Walters \cite{walters1989b} and Lambek \cite{lambek1999}, we formalise
CFGs as operadic signatures and their derivations as trees in the corresponding
free operad. Morphisms of free operads are trees with labelled nodes and edges.
We give an implementation -- the \py{operad} module of DisCoPy -- which satisfies
the axioms of operads on the nose and interfaces with the library NLTK \cite{loper2002}.


\subsection{Operads}

\begin{definition}[Operadic signature]
    An operadic signature is a pair of functions:
    $$G_0^\ast \xleftarrow{\tt{dom}} G_1 \xto{\tt{cod}} G_0$$
    where $G_1$ is the set of \emph{nodes} and $G_0$ a set of \emph{objects}.
    A morphism of operadic signatures $\phi: G \to \Gamma$ is a pair of functions
    $\phi_0: G_0 \to \Gamma_0$, $\phi_1: G_1 \to \Gamma_1$ such that the
    following diagram commutes:
    \begin{equation*}
        \begin{tikzcd}
        G_0^\ast \arrow[d, "\phi_0^\ast"] & G_1 \arrow[l, "\tt{dom}"'] \arrow[r, "\tt{cod}"] \arrow[d, "\phi_1"] & G_0 \arrow[d, "\phi_0"]\\
        \Gamma_0^\ast & \Gamma_1 \arrow[l, "\tt{dom}"] \arrow[r, "\tt{cod}"'] & \Gamma_0
        \end{tikzcd}
    \end{equation*}
    With these morphisms, operadic signatures form a category denoted $\bf{OpSig}$.
\end{definition}

A node or box in an operadic signature is denoted $b_0 \dots b_n \xto{f} a$.
Nodes of the form $a \xto{w} \epsilon$ for the empty
string $\epsilon$ are called \emph{leaves}.
\begin{equation*}
    \input{figures/node-leaf.tikz}
\end{equation*}

\begin{definition}[Operad]\label{def-operad}
    An operad $\bf{O}$ is an operadic signature equipped with a composition
    operation $\cdot : \prod_{b_i \in \vec{b}} \bf{Op}(\vec{c_i}, b_i) \times
    \bf{Op}(\vec{b}, a) \to \bf{Op}(\vec{c}, a)$
    defined for any $a \in \bf{O}_0$, $\vec{b}, \vec{c} \in \bf{O}_0^\ast$.
    Given $f: \vec{b} \to a$ and $g_i: \vec{c_i} \to b_i$, the composition of
    $\vec{g}$ with $f$ is denoted graphically as follows.
    \begin{equation}
        \input{figures/multicat-composition.tikz}
    \end{equation}
    We ask that they satisfy the following axioms:
    \begin{enumerate}
        \item for any $a \in \bf{O}_0$ there is an identity morphism
            $\tt{id}_a \in \bf{Op}(a, a)$ (identities).
        \item for any $f: \vec{b} \to a$,
        $f \cdot \tt{id}_a = f = \vec{\tt{id}_{b_i}} \cdot f$ (unit law).
        \item
        \begin{equation}\label{axiom-operad-interchanger}
            \input{figures/multicat-interchanger.tikz}
        \end{equation}
    \end{enumerate}
    An algebra $F: \bf{O} \to \bf{N}$ is a morphism of operadic signature which respects
    composition, i.e. such that whenever $\vec{c} \xto{\vec{g}} \vec{b} \xto{f} a$
    in $\bf{O}$ we have $F(\vec{g} \cdot f) = \vec{F(g)} \cdot F(f)$.
    With algebras as morphisms, operads form a category $\bf{Operad}$.
\end{definition}
\begin{remark}
    The diagrammatic notation we are using is not formal yet, but it will be made rigorous
    in the following section where we will see that operads are instances of monoidal
    categories and thus admit a formal graphical language of string diagrams.
\end{remark}

Given an operadic signature $G$, we may build the free operad over $G$, denoted
$\bf{Op}(G)$. Morphism of $\bf{Op}(G)$ are labelled trees with nodes from $G$ and
edges labelled by elements of the generating objects $B$.
Two trees are equivalent (or congruent) if they can be deformed continuously
into each other using the interchanger rules \ref{axiom-operad-interchanger}
repeatedly.
The free operad construction is part of the following free-forgetful
adjunction.
\begin{equation*}
    \bf{OpSig} \mathrel{\mathop{\rightleftarrows}^{\bf{Op}}_{U}} \bf{Operad}
\end{equation*}
This means that for any operadic signature $G$ and operad $\bf{O}$,
algebras $\bf{Op}(G) \to \bf{O}$ are in bijective correspondence with
morphisms of operadic signatures $G \to U(\bf{O})$.

\begin{example}[Peirce's alpha]\label{peirce-alpha}
    Morphisms in an operad can be depicted as trees, or equivalently as a nesting
    of bubbles:
    \begin{equation}
        \input{figures/trees-as-nestings.tikz}
    \end{equation}
    Interestingly, the nesting perspective was adopted by Peirce in his graphical
    formulation of propositional logic: the \emph{alpha graphs}
    \cite{peirce1906}. We may present Peirce's alpha graphs as an operad $\alpha$ with
    a single object $x$, variables as leaves $a, b, c, d : 1 \to x$, a binary
    operation $\land : xx \to x$ and a unary operation $\neg: x \to x$, together
    with equations encoding the associativity of $\land$ and $\neg \neg = \tt{id}_x$.
    In order to encode Peirce's rules for ``iteration'' and ``weakening'', we would
    need to work in a preordered operad, but we omit these rules here, see \cite{brady2000}
    for a full axiomatisation. The main purpose here is to note that the nesting
    notation is sometimes more practical than its tree counterpart. Indeed,
    since $\land$ is associative and it is the only binary operation in $\alpha$,
    when two symbols are put side by side on the page, it is unambiguous
    that one should take the conjunction $\land$. Therefore the nested
    notation simplifies reasoning and we may draw the propositional formula
    $$ \neg (a_0 \land \, \neg (a_1 \land  a_2) \, \land a_3)$$
    as the following diagram:
    \begin{equation}
        \input{figures/peirce-formula.tikz}
    \end{equation}
\end{example}

\subsection{Context-free grammars}

Given a finite operadic signature $G$, we can intepret the nodes in
$G$ as \emph{production rules}, the generating objects as \emph{symbols}, and
morphisms in the free operad $\bf{Op}(G)$ as \emph{derivations},
obtaining the notion of a context-free grammar.

\begin{definition}[Context-free grammar]\label{def-context-free-grammar}
    A CFG is a finite operadic signature of the following shape:
    $$(B + V)^\ast \leftarrow G \to B$$
    where $B$ is a set of non-terminal symbols with a specified sentence
    symbol $s \in B$, $V$ is a vocabulary (a.k.a a set of terminal symbols) and
    $G$ is a set of production rules.
    The language generated by $G$ is given by:
    $$ \cal{L}(G) = \set{u \in V^\ast \, \vert \, \exists g: u \to s \in \bf{Op}(G)}$$
    where $\bf{Op}(G)$ is the free operad of labelled trees with nodes
    from $G$.
\end{definition}
\begin{remark}
    Note that the direction of the arrows is the opposite of the usual direction
    used for CFGs, instead of seeing a derivation as a tree
    from the sentence symbol $s$ to the terminal symbols $u \in V^\ast$, we see
    it as a tree from $u$ to $s$. This, of course, does not change any of the
    theory.
\end{remark}

\begin{definition}[Context-free language]
    A context-free language is a subset $X \sub V^\ast$ such that $X = \cal{L}(G)$
    for some context-free grammar $G$.
\end{definition}

\begin{example}[Backus Naur Form]\label{ex-backus-naur-form}
    BNF is a convenient syntax for defining context-free languages recursively.
    An example is the following expression:
    $$ s \leftarrow s \land s \, \vert \, \neg s \, \vert \, a$$
    $$ a \leftarrow a_0 \, \vert \, a_1 \, \vert \, a_2 \, \vert \, a_3 $$
    which defines a CFG with seven production rules
    $\{s \leftarrow s \land s \, ,\, s \leftarrow \neg s \, , \, s \leftarrow a \,
    a \leftarrow a_0 \, a \leftarrow a_1 \, a \leftarrow a_2 \, a \leftarrow a_3\}$
    and such that
    trees with root $s$ are well-formed propositional logic formulae with variables
    in $\set{a_0, a_1, a_2, a_3}$. An example of a valid propositional formula is
    $$ \neg (a_0 \land a_1)\, \land\, \neg (a_2 \land a_3)$$
    as witnessed by the following tree:
    \begin{equation*}
        \input{figures/propositional-formula.tikz}
    \end{equation*}
    where we have omitted the types of intermediate wires for readability.
\end{example}
\begin{example}\label{ex-caesar-tree}
    As a linguistic example, let $B = \set{n, d, v, vp, np, s}$
    for nouns, prepositions, verbs, verb phrases and prepositional phrases, and let
    $G$ be the CFG defined by the following lexical rules:
    $$ \text{Caesar} \to n \quad \text{the} \to d \quad \text{Rubicon} \to n \quad \text{crossed} \to v$$
    together with the production rules $n \cdot v \to vp$,  $n \cdot d \to vp$,
    $vp \cdot pp \to s$. Then the following is a grammatical derivation:
    \begin{equation*}
        \input{figures/is-from-1.tikz}
    \end{equation*}
\end{example}
\begin{example}[Regular grammars revisited]\label{ex-regular-trees}
    Any regular grammar yields a CFG. The translation is given
    by turning paths into left-handed trees as follows:
    \begin{equation}
        \scalebox{0.9}{\input{figures/is-from-2.tikz}}
    \end{equation}
    Not all CFGs arise in this way. For example, the language of
    well-bracketed expressions, defined by the CFG with
    a single production rule $G = \set{s \leftarrow (s)}$, cannot be generated by a regular
    grammar. We can prove this using the \emph{pumping lemma}. Indeed, suppose
    there is a regular grammar $G'$ such that well-bracketed expressions are
    paths in $G'$. Let $n$ be the number of vertices in $G'$ and consider the
    grammatical expression $x = (\dots()\dots)$ with $n + 1$ open and
    $n+ 1$ closed brackets. If $G'$ parses $x$, then there must be a path $p_0$ in $G'$
    with labelling $(\dots($ and a path $p_1$ with labelling $)\dots)$ such that
    $p_0 \cdot p_1 : s_0 \to s_1$ in $G'$. By the \emph{pigeon hole}
    principle, the path $p_0$ must have a cycle of length $k \geq 1$.
    Remove this cycle from $p_0$ to get a new path $p_0'$. Then
    $p_0' \cdot p_1$ yields a grammatical expression $x' = (\dots()\dots)$
    with $n + 1 - k$ open brackets and $n + 1$ closed brackets. But then $x'$
    is not a well-bracketed expression. Therefore regular grammars cannot generate
    the language of well-bracketed expressions and we deduce that regular languages
    are \emph{strictly} contained in context-free languages.
\end{example}

We briefly consider the problem of parsing context-free grammars.

\begin{definition}
    \begin{problem}
      \problemtitle{$\tt{CfgParsing}$}
      \probleminput{$G$, $u \in V^\ast$}
      \problemoutput{$f \in \bf{Op}(G)(u, s)$}
    \end{problem}
\end{definition}

This problem can be solved using a \emph{pushdown automaton}, and in fact any
language recognized by a pushdown automaton is context-free \cite{chomsky1956}.
The following result was shown independently by several researchers at the end
of the 1960s.

\begin{proposition}\cite{younger1967, earley1970}
    Context-free grammars can be parsed in cubic time.
\end{proposition}

\subsection{operad.Tree}\label{subsec-trees}

Tree data structures are ubiquitous in computer science. They can implemented
via the inductive definition: a tree is a root together with a list of trees.
Implementing operads as defined in this Section, presents some extra difficulties
in handling types (domains and codomains) and identities. In fact, the concept
of ``identity tree'' is not frequent in the computer science literature.
Our implementation of free operads consists in the definition of classes
\py{operad.Tree}, \py{operad.Box}, \py{operad.Id} and \py{operad.Algebra},
corresponding to morphisms (trees), generators (nodes), identities and algebras of free operads, respectively.
A \py{Tree} is initialised by a \py{root}, instance of \py{Node}, together with
a list of \py{Tree}s called \py{branches}. Alternatively, it may be built from
generating Boxs using the \py{Tree.__call__} method, this allows for an intuitive
syntax which we illustrate below.

\begin{python} \label{listing:operad.Tree}
{\normalfont Tree in a free operad.}

\begin{minted}{python}
class Tree:
    def __init__(self, root, branches, _scan=True):
        if not isinstance(root, Box):
            raise TypeError()
        if not all([isinstance(branch, Tree) for branch in branches]):
            raise TypeError()
        if _scan and not root.cod == [branch.dom for branch in branches]:
            raise AxiomError()
        self.dom, self.root, self.branches = root.dom, root, branches

    @property
    def cod(self):
        if isinstance(self, Box):
            return self._cod
        else:
            return [x for x in branch.cod for branch in self.branches]

    def __repr__(self):
        return "Tree({}, {})".format(self.root, self.branches)

    def __str__(self):
        if isinstance(self, Box):
            return self.name
        return "{}({})".format(self.root.name,
                               ', '.join(map(Tree.__str__, self.branches)))

    def __call__(self, *others):
        if not others or all([isinstance(other, Id) for other in others]):
            return self
        if isinstance(self, Id):
            return others[0]
        if isinstance(self, Box):
            return Tree(self, list(others))
        if isinstance(self, Tree):
            lengths = [len(branch.cod) for branch in self.branches]
            ranges = [0] + [sum(lengths[:i + 1]) for i in range(len(lengths))]
            branches = [self.branches[i](*others[ranges[i]:ranges[i + 1]])
                        for i in range(len(self.branches))]
            return Tree(self.root, branches, _scan=False)
        raise NotImplementedError()

    @staticmethod
    def id(dom):
        return Id(dom)

    def __eq__(self, other):
        return self.root == other.root and self.branches == other.branches
\end{minted}
\end{python}

A \py{Box} is initialised by label \py{name}, a domain object \py{dom} and a
list of objects \py{cod} for the codomain.

\begin{python} \label{listing:operad.Box}
{\normalfont Node in a free operad.}

\begin{minted}{python}
class Box(Tree):
    def __init__(self, name, dom, cod):
        if not (isinstance(dom, Ob) and isinstance(cod, list)
                and all([isinstance(x, Ob) for x in cod])):
            return TypeError
        self.name, self.dom, self._cod = name, dom, cod
        Tree.__init__(self, self, [], _scan=False)

    def __repr__(self):
        return "Box('{}', {}, {})".format(self.name, self.dom, self._cod)

    def __hash__(self):
        return hash(repr(self))

    def __eq__(self, other):
        if isinstance(other, Box):
            return self.dom == other.dom and self.cod == other.cod \
                        and self.name == other.name
        if isinstance(other, Tree):
            return other.root == self and other.branches == []
\end{minted}
\end{python}

An \py{Id} is a special type of node, which cancels locally when composed with
other trees. The cases in which identities must be removed are handled in
the \py{Tree.__call__} method. The \py{Tree.__init__} method, as it stands,
does not check the identity axioms. We will however always use the
\py{__call__} syntax to construct our trees.

\begin{python} \label{listing:operad.Id}
{\normalfont Identity in a free operad.}

\begin{minted}{python}
class Id(Box):
    def __init__(self, dom):
        self.dom, self._cod = dom, [dom]
        Box.__init__(self, "Id({})".format(dom), dom, dom)

    def __repr__(self):
        return "Id({})".format(self.dom)
\end{minted}
\end{python}

We can check that the axioms of operads hold for \py{Tree.__call__}.

\begin{python} \label{listing:operad.axioms}
{\normalfont Axioms of free operads.}

\begin{minted}{python}
x, y = Ob('x'), Ob('y')
f, g, h = Box('f', x, [x, x]), Box('g', x, [x, y]), Box('h', x, [y, x])
assert Id(x)(f) == f == f(Id(x), Id(x))
left = f(Id(x), h)(g, Id(x), Id(x))
middle = f(g, h)
right = f(g, Id(x))(Id(x), Id(x), h)
assert left == middle == right == Tree(root=f, branches=[g, h])
\end{minted}
\end{python}

\begin{python} \label{listing:operad.caesar}
{\normalfont We construct the tree from Example \ref{ex-caesar-tree}.}

\begin{minted}{python}
n, d, v, vp, np, s = Ob('N'), Ob('D'), Ob('V'), Ob('VP'), Ob('NP'), Ob('S')
Caesar, crossed = Box('Caesar', n, []), Box('crossed', v, []),
the, Rubicon = Box('the', d, []), Box('Rubicon', n, [])
VP, NP, S = Box('VP', vp, [n, v]), Box('NP', np, [d, n]), Box('S', s, [vp, np])
sentence = S(VP(Caesar, crossed), NP(the, Rubicon))
\end{minted}
\end{python}

We define the \py{Algebra} class, which implements operad algebras as
defined in \ref{def-operad} and is initialised by a pair of mappings:
\py{ob} from objects to objects and \py{ar} from nodes to trees. These implement
functorial reductions and functorial semantics of CFGs, as defined in the next
section and chapter respectively.

\begin{python} \label{listing:operad.Algebra}
{\normalfont Algebra of the free operad.}

\begin{minted}{python}
class Algebra:
    def __init__(self, ob, ar, cod=Tree):
        self.cod, self.ob, self.ar = cod, ob, ar

    def __call__(self, tree):
        if isinstance(tree, Id):
            return self.cod.id(self.ob[tree])
        if isinstance(tree, Box):
            return self.ar[tree]
        return self.ar[tree.root](*[self(branch) for branch in tree.branches])
\end{minted}

    Note that we parametrised the class algebra over a codomain
    class, which by default is the free operad \py{Tree}. We may build any
    algebra from the free operad to an operad $A$ by providing
    a class \py{cod=A} with \py{A.id} and \py{A.__call__} methods.
    We will see a first example of this when we interface \py{Tree} with
    \py{Diagram} in the next section. Further examples will be given in Chapter 2.
\end{python}

We end by interfacing the \py{operad} module with the library NLTK \cite{loper2002}.

\begin{python} \label{listing:operad.nltk}
{\normalfont Interface between \py{nltk.Tree} and \py{operad.Tree}.}

\begin{minted}{python}
def from_nltk(tree):
    branches, cod = [], []
    for branch in tree:
        if isinstance(branch, str):
            return Box(branch, Ob(tree.label()), [])
        else:
            branches += [from_nltk(branch)]
            cod += [Ob(branch.label())]
    root = Box(tree.label(), Ob(tree.label()), cod)
    return root(*branches)
\end{minted}

This code assumes that the tree is generated from a lexicalised CFG. The \py{operad}
module of DisCoPy contains the more general version.
We can now define a grammar in NLTK, parse it, and extract an \py{operad.Tree}.
We check that we recover the correct tree for ``Caesar crossed the Rubicon''.

\begin{minted}{python}
from nltk import CFG
from nltk.parse import RecursiveDescentParser
grammar = CFG.fromstring("""
S -> VP NP
NP -> D N
VP -> N V
N -> 'Caesar'
V -> 'crossed'
D -> 'the'
N -> 'Rubicon'""")

rd = RecursiveDescentParser(grammar)
for x in rd.parse('Caesar crossed the Rubicon'.split()):
    tree = from_nltk(x)
assert tree == sentence
\end{minted}
\end{python}

\section{Diagrams}\label{section-diagrams} 

String diagrams in monoidal categories are the key tool that we use to
represent syntactic structures.
In this section we introduce \emph{monoidal grammars}, the equivalent of Chomsky's
unrestricted type-0 grammars. Their derivations are string diagrams in a
free monoidal category.
We introduce \emph{functorial reductions} as a structured way of comparing
monoidal grammars, and motivate them as a tool to reason about equivalence and
normal forms for context-free grammar.
String diagrams have a convenient \emph{premonoidal encoding} as lists of layers,
which allows to implement the class \py{monoidal.Diagram} as a subclass of \py{cat.Arrow}.
We give an overview of the \py{monoidal} module of DisCoPy and its
interface with \py{operad}.

\begin{center}
\begin{tabular}{ |c|c|c|c|c| }
 \hline
 Monoidal category & Type-$0$ grammar & Python\\
 \hline
 objects & strings & \py{Ty} \\
 generators & production rules & \py{Box} \\
 morphisms & derivations & \py{Diagram} \\
 functors & reductions & \py{Functor} \\
 \hline
\end{tabular}
\end{center}

\subsection{Monoidal categories}\label{sec-monoidal-categories}

\begin{definition}[Monoidal signature]
    A monoidal signature $G$ is a signature of the following form:
    $$\signature{G_1}{G_0^\ast}$$.
    $G$ is a finite monoidal signature if $G_1$ is finite.
    A morphism of monoidal signatures $\phi: G \to \Gamma$ is a pair of
    maps $\phi_0 : G_0 \to \Gamma_0$ and $\phi_1: G_1 \to \Gamma_1$
    such that the following diagram commutes:
    \begin{equation*}
        \begin{tikzcd}
        G_0^\ast \arrow[d, "\phi_0^\ast"] & G_1 \arrow[l, "\tt{dom}"'] \arrow[r, "\tt{cod}"] \arrow[d, "\phi_1"] & G_0^\ast \arrow[d, "\phi_0^\ast"]\\
        \Gamma_0^\ast & \Gamma_1 \arrow[l, "\tt{dom}"] \arrow[r, "\tt{cod}"'] & \Gamma_0^\ast
        \end{tikzcd}
    \end{equation*}
    With these morphisms, monoidal signatures form a category $\bf{MonSig}$.
\end{definition}

Elements $f: \vec{a} \to \vec{b}$ of $G_1$ are called \emph{boxes} and are denoted by
the following diagram, read from top to bottom, special cases are states and effects
with no inputs and outputs respectively.
\begin{equation*}
    \input{figures/monsig.tikz}
\end{equation*}

\begin{definition}[Monoidal category]
    A (strict) monoidal category is a category $\bf{C}$ equipped with a functor $\otimes : \bf{C} \times \bf{C} \to \bf{C}$ called the \emph{tensor} and
    a specified object $1 \in \bf{C}_0$ called the \emph{unit}, satisfying
    the following axioms:
    \begin{enumerate}
        \item $1 \otimes f = f = f \otimes 1$ (unit law)
        \item $(f \otimes g) \otimes h = f \otimes  (g \otimes h)$ (associativity)
    \end{enumerate}
    for any $f, g, h \in \bf{C}_1$.
    A (strict) monoidal functor is a functor that preserves the tensor product
    on the nose, i.e. such that $F(f \otimes g) = F(f) \otimes F(g)$.
    The category of monoidal categories and monoidal functors is denoted $\bf{MonCat}$.
\end{definition}

\begin{remark}
    The general (non-strict) definition of a monoidal category relaxes the
    equalities in the unit and associativity laws to the existence of natural
    isomorphisms, called unitors and associators. These are then required to satisfy some coherence conditions in the form of commuting diagrams, see MacLane \cite{MacLane71}.
    In practice, these natural transformations are not used in
    calculations. MacLane's coherence theorem ensures that any monoidal category
    is equivalent to a strict one.
\end{remark}

Given a monoidal signature $G$ we can generate the free monoidal category
$\bf{MC}(G)$, i.e. there is a free-forgetful adjunction:
\begin{equation*}
    \bf{MonSig} \mathrel{\mathop{\rightleftarrows}^{\bf{MC}}_{U}} \bf{MonCat}
\end{equation*}
This means that for any monoidal signature $G$ and monoidal category $\bf{S}$,
functors $\bf{MC}(G) \to \bf{S}$ are in bijective correspondence with
morphisms of signatures $G \to U(\bf{S})$.
The free monoidal category was first characterized by Joyal and Street
\cite{joyal1991geometry} who showed that morphisms in $\bf{MC}(G)$ are topological
objects called \emph{string diagrams}. We follow the formalisation of Delpeuch and
Vicary \cite{delpeuch2019b} who provided an equivalent combinatorial definition
of string diagrams.

Given a monoidal signature $G$, we can construct the signature of \emph{layers}:
$$ \signature{L(G) = G_0^\ast \times G_1 \times G_0^\ast}{G_0^\ast} $$
where for every layer $l = (u, f: x \to y, v) \in L(G)$ we define
$\tt{dom}(l) = uxv$ and $\tt{cod}(l) = uyv$. A layer $l \in L(G)$ is denoted as
follows:
\begin{equation*}
    \input{figures/layer.tikz}
\end{equation*}
The set of \emph{premonoidal diagrams} $\bf{PMC}(G)$ is the set of morphisms of
the free category generated by the layers:
$$ \bf{PMC}(G) = \bf{C}(L(G))$$
They are precisely morphisms of free premonoidal categories in the sense of \cite{power1997}.
Morphisms $d \in \bf{PMC}(G)$ are lists of layers $d = (d_1, \dots, d_n)$ such
that $\tt{cod}(d_i) = \tt{dom}(d_{i + 1})$. The data for such a diagram may be
presented in a more succint way as follows.

\begin{proposition}[Premonoidal encoding]\cite{delpeuch2019b}
    \label{prop-combinatorial-encoding}
    A premonoidal diagram $d \in \bf{PMC}(G)$ is uniquely defined by the
    following data:
    \begin{enumerate}
        \item a domain $\tt{dom}(d) \in \Sigma_0^\ast$,
        \item a codomain $\tt{cod}(d) \in \Sigma_0^\ast$,
        \item a list of boxes $\tt{boxes}(d) \in G_1^n$,
        \item a list of offsets $\tt{offsets}(d) \in \bb{N}^n$.
    \end{enumerate}
    Where $n \in \bb{N}$ is the \emph{length} of the diagram and the
    offsets indicate the number of boxes to the left of each wire. This data
    defines a valid premonoidal diagram if for $ 0 < i \leq n$ we have:
    $$ \tt{width}(d)_i \geq \tt{offsets}(d)_i + \size{\tt{dom}(b_i)}$$
    where the widths are defined inductively by:
    $$\tt{width}(d)_1 = \tt{size}(\tt{dom}(d)) \quad
    \tt{width}(d)_{i + 1} = \tt{width}(d)_i + \size{\tt{cod}(b_i)} - \size{\tt{dom}(b_i)}$$
    and $b_i = \tt{boxes}(d)_i$.
\end{proposition}

As an example consider the following diagram:
\begin{equation*}
    \input{figures/premon-diagram.tikz}
\end{equation*}
It has the following combinatorial encoding:
$$(\tt{dom} = abc, \tt{cod} = auyv, \tt{boxes} = [f, g, h], \tt{offsets} = [2, 1, 2])$$
where $f: c \to xv$, $g: b \to u$ and $h: x \to y$ and $a, b, c, x, u, y, v \in \G_0$.
This combinatorial encoding is the underlying data-structure of both the online
proof assistant Globular \cite{bar2018} and the Python implementation
of monoidal categories DisCoPy \cite{defelice2020b}.

Premonoidal diagrams are a useful intermediate step to define the
\emph{free monoidal category} $\bf{MC}(G)$ over a monoidal signature $G$.
Indeed, morphisms in $\bf{MC}(G)$ are equivalence classes of the quotient of
$\bf{PMC}(G)$ by the \emph{interchanger rules}, given by the following relation:
\begin{equation}\label{eq-interchnager-rules}
    \input{figures/monoidal-interchanger.tikz}
\end{equation}
The following result was proved by Delpeuch and Vicary \cite{delpeuch2019b}, who
showed how to transalte between the combinatorial definition of diagrams and the
definition of Joyal and Street as planar progressive graphs up to
planar isotopy \cite{joyal1991geometry}.
\begin{proposition}\cite{delpeuch2019b}
    $$\bf{MC}(G) \simeq \bf{PMC}(G) / \sim $$
\end{proposition}
Given two representatives $d, d': u \to v \in \bf{PMC}(G)$ we may check whether
they are equal in $\bf{MC}(G)$ in cubic time \cite{delpeuch2019b}. Assuming $d$ and $d'$
are boundary connected diagrams, this is done by turning $d$ and $d'$ into their
interchanger normal form, which is obtained by applying the interchanger rules
\ref{eq-interchnager-rules} from left to right repeatedly. For disconnected
diagrams the normalization requires more machinery but can still be performed
efficiently, see \cite{delpeuch2019b} for details.

\subsection{Monoidal grammars}

Monoidal categories appear in linguistics as a result of the following change
of terminology. Given a monoidal signature $G$, we may think of the
objects in $G_0^\ast$ as \emph{strings} of symbols, the generating arrows in $G_1$ as \emph{production rules} and morphisms in $\bf{MC}(G)$ as \emph{derivations}.
We directly obtain a definition of Chomsky's generative grammars, or string
rewriting system, using the notion of a monoidal signature.

\begin{definition}[Monoidal Grammar] \label{def-monoidal-grammar}
    A monoidal grammar is a finite monoidal signature $G$ of the following shape:
    $$ \signature{G}{(V + B)^\ast}$$
    where $V$ is a set of words called the \emph{vocabulary},
    and $B$ is a set of \emph{symbols} with $s \in B$ the sentence symbol.
    An utterance $u \in V^\ast$ is grammatical if there is a string diagram
    $g: u \to s$ in $\bf{MC}(G)$, i.e. the language generated by $G$ is given by:
    $$\cal{L}(G) = \set{u \in V^\ast \, \vert \, \exists f \in \bf{MC}(G)(u, s)}$$
    The free monoidal category $\bf{MC}(G)$ is called the \emph{category of
    derivations} of $G$.
\end{definition}
\begin{remark}
    Any context-free grammar as defined in \ref{def-context-free-grammar} yields
    directly a monoidal grammar.
\end{remark}

\begin{proposition}
    The languages generated by monoidal grammars are equivalent to the
    languages generated by Chomsky's unrestricted grammars.
\end{proposition}
\begin{proof}
    Unrestricted grammars are defined as finite relations $P \sub (V + B)^\ast \times
    (V + B)^\ast$ where $V$ is a set of terminal symbols, $B$ a set of non-terminals
    and $P$ is a set of prodution rules \cite{chomsky1956}. The only difference
    between this definition and monoidal grammars is
    that the latter allow more than one production rule between pairs of strings.
    However, a string $u$ is in the language if \emph{there exists} a derivation
    $g: u \to s$. Therefore the languages are equivalent.
\end{proof}

We define recursively enumerable languages as those generated by monoidal grammars,
or equivalently by Chomsky's unrestricted grammars, or equivalently those recognized
by Turing machines as discussed in the next paragraph.

\begin{definition}[Recursively enumerable language]
    A recursively enumerable language is a subset $X \sub V^\ast$ such that
    $X = \cal{L}(G)$ for some monoidal grammar $G$.
\end{definition}

\begin{example}[Cooking recipes]
    As an example of a monoidal grammar, let $V = \{ \text{aubergine},
    \text{tomato}, \text{parmigiana}\}$ be a set of cooking ingredients,
    $B = \set{\text{parmigiana}}$ be a set of dishes and let $G$ be a set of
    cooking steps, e.g.
    \begin{equation*}
        \tt{stack}: a \, p \to p \quad
        \tt{take}: t \to t \, t \quad
        \tt{spread}:  t \, p \to p \quad
        \tt{eat}:  t \to 1 .
    \end{equation*}
    Then the derivations $u \to \text{parmigiana}$ in $\bf{MC}(G)$ with
    $u \in V^\ast$ are cooking recipes to make parmigiana. For instance,
    the following is a valid cooking recipe:
    \begin{equation*}
        \input{figures/lasagna-recipe.tikz}
    \end{equation*}
\end{example}

Recall that a Turing machine is given by a tuple $(Q, \Gamma, \sharp, V, \delta, q_0, s)$
where $Q$ is a finite set of \emph{states}, $\Gamma$ is a finite set of
\emph{alphabet} symbols, $\sharp \in \Gamma$ is the \emph{blank} symbol,
$V \sub \Gamma \backslash \set{\sharp}$ is the set of \emph{input} symbols, $q_0 \in Q$
is the \emph{initial} state, $s \in Q$ is the \emph{accepting}
state and $\delta \sub ((Q \backslash \set{s}) \times \Gamma) \times (Q \times \Gamma)\times \set{L, R}$
is a \emph{transition} table, specifying the next state in $Q$ from a current state,
the symbol in $\Gamma$ to overwrite the current symbol pointed by the head and
the next head movement (left or right).
At the start of the computation, the machine is in state $q_0$ and the tape
contains a string of initial symbols $u \in V^\ast$ followed by the blank
symbol $\sharp$ indicating the end of the tape. The computation is then performed
by applying transitions according to $\delta$ until the accepting state $s$ is
reached. We assume that the transition $\delta$ doesn't overwrite the blank symbol
and that it leaves it at the end of the tape.

A Turing machine may be encoded in a monoidal grammar as follows. The set of
non-terminal symbols is $B = (\Gamma \backslash V) + Q$, the set of terminal
symbols is $V$ and the sentence type is $s \in B$. The production rules in $G$
are given by:
\begin{equation}
    \input{figures/turing-machine.tikz}
\end{equation}
for all $a, a', b, b' \in \Gamma \backslash \set{\sharp}$,
$q, q' \in Q$ such that $\delta((q, a), (q', a', R)) = 1$ and
$\delta((q, b), (q', b', L)) = 1$ and $\delta((q, \sharp), (q', \sharp, L)) = 1$.
Note that the last rule ensures that the blank
symbol $\sharp$ is always left at the end of the tape and never overwritten.
Then we have that morphisms in $\bf{MC}(G)(q_0\, w \, \sharp, u \, s \, v)$ are
terminating runs of the Turing machine. In order to express these runs as
morphisms $w \to s$ we may erase the content of the tape once we reach the
accepting state by adding a production rule $x s y \to s$ to $G$ for any
$x, y \in B$. Using this encoding we can prove the following proposition.

\begin{proposition}
    The parsing problem for monoidal grammars is undecidable.
\end{proposition}
\begin{proof}
    The encoding of Turing machines into monoidal grammars given above reduces
    the problem of parsing monoidal grammars to the Halting problem for
    Turing machines. Therefore it is an undecidable problem.
\end{proof}

\subsection{Functorial reductions}\label{section-reductions}

We now come to the question of reduction and equivalence for grammars.
Several definitions are available in the literature and
we introduce three alternative notions of varying strengths.
The most established notion of equivalence between CFGs
--- known as \emph{weak equivalence} --- judges a grammar from the
language it generates.

\begin{definition}[Weak reduction]
    Let $G$ and $G'$ be monoidal grammars over the same vocabulary $V$.
    We say that $G$ reduces weakly to $G'$, denoted $G \leq G'$, if
    $\cal{L}(G) \sub \cal{L}(G')$. $G$ is weakly equivalent to
    $G'$ if $\cal{L}(G) = \cal{L}(G')$.
\end{definition}

Even if two grammars are weakly equivalent, they may generate their sentences
in completely different ways. This motivates the definition of a stronger notion
of equivalence, which does not only ask for the generated languages
to be equal, but also for the corresponding derivations to be the same.
This notion has been studied in a line of work connecting context-free grammars
(CFGs) and algebraic signatures \cite{hatcher1976, goguen1977, buro2020}, which
we discuss below.

\begin{definition}[Strong reduction]\label{def-strong-reduction}
    Let $G$ and $G'$ be monoidal grammars over the same vocabulary $V$.
    A strong reduction from $G$ to $G'$ is a morphism of monoidal signatures
    $f: G \to G'$ such that $f_0 (v) = v$ for any $v \in V$ and $f_0(s) = s'$.
    With strong reductions as morphisms, monoidal grammars over $V$ form
    a subcategory of $\bf{MonSig}$ denoted $\bf{Grammar}_V$.
    We say that $G$ and $G'$ are strongly equivalent if they are isomorphic in
    $\bf{Grammar}_V$.
\end{definition}

Note first that strong reduction subsumes its weak counterpart. Indeed given a
morphism $f: G \to G'$, we get a functor $\bf{MC}(f) : \bf{MC}(G) \to \bf{MC}(G')$
mapping syntax trees of one grammar into syntax trees of the other.
Since $f_0 (v) = v$ for all $v \in V$ and $f_0(s) = s'$, there is an induced function
$\bf{MC}(f): \bf{MC}(G)(u, s) \to \bf{MC}(G')(u, s')$ for any $u \in V^\ast$, which
implies that $\cal{L}(G) \sub \cal{L}(G')$.

A strong reduction $f$ is a consistent relabelling of the nodes
and types of the underlying operadic signature. This often results in too strict of a notion,
since it relates very few grammars together.
We introduce an intermediate notion of reduction between grammars, which we
call \emph{functorial reduction}.

\begin{definition}[Functorial reduction]
    Let $G$ and $G'$ be monoidal grammars over the same vocabulary $V$.
    A functorial reduction from $G$ to $G'$ is a functor
    $F: \bf{MC}(G) \to \bf{MC}(G')$ such that $F_0(v) = v$ for all $v \in V$ and
    $F_0(s) = s'$. A functorial equivalence between $G$ and $G'$
    is a pair of functors $F: \bf{MC}(G) \to \bf{MC}(G')$ and
    $F': \bf{MC}(G) \to \bf{MC}(G')$. With functorial reductions as morphisms,
    monoidal grammars over $V$ form a category $\bf{Grammar}_V$.
\end{definition}

\begin{remark}
    The passage from strong to functorial reductions can be seen as
    as a Kleisli category construction.
    The free operad functor $\bf{MC}: \bf{MonSig} \to \bf{MonCat}$
    induces a monad $U \circ \bf{MC}: \bf{MonSig} \to \bf{MonSig}$.
    We can construct the Kleisli category $\bf{Kl}(U \circ \bf{MC})$ with objects given
    by operadic signatures and morphisms given by morphisms of signatures
    $f: G \to U\bf{MC}(G')$. Equivalently, a morphism $G \to G'$ in $\bf{Kl}(U\circ \bf{MC})$
    is a functor $\bf{MC}(G) \to \bf{MC}(G')$ since $\bf{MC} \dashv U$.
    $\bf{MC}$ to $\bf{Cfg}_V$ we still have an adjunction $\bf{MC} \dashv U$ and
    the following equivalence:
    $$ \bf{Grammar}_V \simeq \bf{Kl}(U \circ \bf{MC}).$$
\end{remark}

Functorial reductions can be computed in logarithmic space. We give a proof,
an alternative proof is given by the code for \py{Functor.__call__}.

\begin{proposition}\label{prop-functorial-reduction}
    Funtorial reductions can be computed in log-space $\tt{L}$.
\end{proposition}
\begin{proof}
    Let $G$ and $G'$ be monoidal grammars, a functorial reduction from $G$ to
    $G'$ is a functor $F: \bf{MC}(G) \to \bf{MC}(G')$. By the universal property
    of the free monoidal category $\bf{MC}(G)$ the data for such a functor is a
    finite homomorphism of signatures $G \to U(\bf{MC}(G'))$, i.e. a collection
    of morphisms $\set{F(f)}_{f \in G}$. Consider the problem
    which takes as input a diagram $g : x \to y \in \bf{MC}(G)$ and a functorial
    reduction $F: G \to U(\bf{MC}(G'))$ and outputs $F(g) \in \bf{MC}(G')$.
    We assume that we have premonoidal representations of $g$ and $F(f)$ for
    every production rule $f \in G$, i.e. they all come as a pair of lists for
    boxes and offsets. In order to compute $F(g)$ we run through the list of
    boxes and replace each box $f$ of $g$ by $F(f)$ adding the offset of $f$ to
    every offset in $F(f)$. This can be computed using a constant number of counters
    (one for the index of the box in the list, one for the offset and one for the
    pointer to $F(f)$) thus functorial reductions are in logspace.
\end{proof}

From the monoidal definition of weak, strong and functorial reduction we
derive the corresponding notions for regular and CFGs using
the following diagram.

\begin{equation*}
    \begin{tikzcd}
     \bf{Cat} \arrow[d, "U"', bend right=25] \arrow[r, hookrightarrow]
     & \bf{Operad} \arrow[d, "U"', bend right=25] \arrow[r, hookrightarrow]
     & \bf{MonCat} \arrow[d, "U"', bend right=25] & \\
    \bf{Graph} \arrow[u, "\bf{C}"', bend right=25] \arrow[r, hookrightarrow]
    & \bf{OpSig} \arrow[u, "\bf{Op}"', bend right=25] \arrow[r, hookrightarrow]
    & \bf{MonSig} \arrow[u, "\bf{MC}"', bend right=25]
    \end{tikzcd}
\end{equation*}

We can now compare regular, context-free and unrestricted grammars via
the notion of reduction.

\begin{proposition}
    Any regular grammar strongly reduces to a CFG and
    any CFG to a monoidal grammar.
\end{proposition}
\begin{proof}
    This is done by proving that the injections in the diagram above exist
    and make the diagram commute.
\end{proof}

The functorial notion of reduction sits in-between the weak and strong notions.
As shown above, any strong reduction $f: G \to G'$ induces a functorial
reduction $\bf{MC}(f): \bf{MC}(G) \to \bf{MC}(G')$ via the free construction
$\bf{MC}: \bf{MonSig} \to \bf{MonCat}$ and the existence of such a functor
induces an inclusion of the corresponding languages.
However, not all functorial reductions are strong. It is well-known that any
CFG can be lexicalised without losing expressivity, the resulting
grammar is functorially and not strongly equivalent to the original CFG.

\begin{definition}[Lexicalised CFG]
    A lexicalised CFG is an operadic signature of the following
    shape:
    $$B^\ast + V \leftarrow G \to B$$
    In other words, all production rules involving terminal symbols in $V$ are
    of the form $w \to b$ for $w \in V$ and $b \in B$. These are called
    \emph{lexical rules}.
\end{definition}

\begin{proposition}
    Any CFG is functorially (and not strongly) equivalent to a
    lexicalised CFG.
\end{proposition}
\begin{proof}
    For any CFG $(B + V)^\ast \leftarrow G \to B$ we can build
    a lexicalised CFG $B'^\ast + V \leftarrow G' \to B'$ where
    $B' = B + V$ and $G' = G + \set{v \to v}_{v \in V}$, where we distinguish
    between the two copies of $V$. There is a functor $F: \bf{MC}(G) \to \bf{MC}(G')$
    given on objects by $F_0(x) = x$ for $x \in B + V$
    and on arrows $f: \vec{y} \to x$ by $F_1(f) = \vec{g} \cdot f$ where
    $g_i : y_i \to y_i$ is the identity if $y_i \in B$ and is a lexical rule
    $y_i \to y_i$ if $y_i \in V$. Similarly for the other direction, there is
    a functor $F' : \bf{MC}(G') \to \bf{MC}(G)$ given on objects by $F_0'(x) = x$
    for all $x \in B + V$ and on arrows by $F_1'(f) = f$ for $f \in G$ and
    $F_1( v \to v) = \tt{id}_v$. Therefore $G$ and $G'$ are functorially equivalent.

    Note that even the $G'$ is \emph{not} strongly equivalent to $G$.
    Indeed strong equivalence would imply that there is a bijection between
    the underlying sets of symbols, i.e. $\size{B+ V} = \size{B + 2V}$
    which is only true for grammars over an empty vocabulary.
\end{proof}

We now study two useful normal forms for CFGs.

\begin{definition}[Chomsky normal form]
    A CFG $G$ is in Chomsky normal form if it has the following shape:
    $$ B^2 + V \leftarrow G \to B$$
    i.e. all production rules are of the form $w \to a$ or $bc \to a$ for
    $w \in V$ and $a, b, c \in B$.
\end{definition}

\begin{proposition}
    Any CFG $G$ is weakly equivalent to a CFG $G'$ in Chomsky
    normal form, such that the reduction from $G$ to $G'$ is functorial.
\end{proposition}
\begin{proof}
    Fix any CFG $G$. Without loss of generality, we may assume
    that $G$ is lexicalised $B^\ast + V \leftarrow G \to B$. In order to construct
    $G'$, we start by setting
    $G' = \set{f: \vec{a} \to b \in G \, \vert \, \size{\vec{a}} \leq 2}$.
    Then, for any production rule $f: \vec{a} \to b$ in $G$ such that
    $k = \size{\vec{a}} > 2$ we add $k - 1$ new rules $f_i$ and $k - 2$ new symbols
    $c_i$ to $G'$ given by $\set{ f_i : c_i a_{i + 1} \to c_{i + 1}}_{i=0}^{k - 2}$
    where $c_0 = a_0$ and $c_{k - 1} = b$. This yields a CFG
    $G'$ in Chomsky normal form.
    There is a functorial reduction from $G$ to $G'$ given by mapping production
    rules $f$ in $G$ to left-handed trees with $f_i$s as nodes, as in the
    following example:
    \begin{equation}
        \input{figures/chomsky-normal-form.tikz}
    \end{equation}
    This implies that $\cal{L}(G) \sub \cal{L}(G')$. Now suppose
    $\vec{u} \in \cal{L}(G')$, i.e. there is a tree $g: \vec{u} \to s$
    in $\bf{Op}(G')$. By construction, if some $f_i: c_i a_{i + 1} \to c_{i + 1}$
    appears as an node in $g$ then all the $f_i$s must appear as a sequence of
    nodes in $g$, therefore $g$ is in the image of a tree in
    $\bf{Op}(G)(\vec{u}, s)$ and $\vec{u} \in \cal{L}(G)$. Therefore
    $\cal{L}(G) = \cal{L}(G')$.
\end{proof}

Since the reduction from a CFG to its Chomsky normal form is functorial, the
translation can be performed in logspace. Indeed, we will show in the next
section that the problem of applying a functor between free monoidal categories
(of which operad algebras are an example) is in $\tt{NL}$. We end with an even
weaker example of equivalence between grammars.

\begin{definition}[Greibach normal form]
    A CFG $G$ is in Greibach normal form if it has the following shape:
    $$  V \times B^\ast \leftarrow G \to B$$
    i.e. every production rule is of the form $w b \to a$ for $a \in B$,
    $b \in B^\ast$ and $w \in V$.
\end{definition}

\begin{proposition} \cite{greibach1965}
    Any CFG is weakly equivalent to one in Greibach normal
    form and the conversion can be performed in poly-time.
\end{proposition}

We will use these normal forms in the next section, when we discuss functorial
reductions between categorial grammars and their relationship with context-free
grammars.

\subsection{monoidal.Diagram}

We now present \py{monoidal}, the key module of DisCoPy which allows to compute
with diagrams in free monoidal categories. We have defined free monoidal
categories via the concepts of layers and premonoidal diagrams. These have
a natural implementation in object-oriented Python, consisting in the definition
of classes \py{monoidal.Ty}, \py{monoidal.Layer} and \py{monoidal.Diagram},
for types, layers and diagrams in free premonoidal categories, respectively.
Types are tuples of objects, equipped with a \py{.tensor} method for the
monoidal product.

\begin{python}\label{listing:monoidal.Ty}
{\normalfont Types of free monoidal categories.}

\begin{minted}{python}
from discopy import cat

class Ty(cat.Ob):
    def __init__(self, *objects):
        self.objects = tuple(
            x if isinstance(x, cat.Ob) else cat.Ob(x) for x in objects)
        super().__init__(self)

    def tensor(self, *others):
        for other in others:
            if not isinstance(other, Ty):
                raise TypeError()
        objects = self.objects + [x for t in others for x in t.objects]
        return Ty(*objects)

    def __matmul__(self, other):
        return self.tensor(other)
\end{minted}
\end{python}

Layers are instances of \py{cat.Box}, initialised by triples $(u, f, v)$ for a
pair of types $u, v$ and a box $f$.

\begin{python}\label{listing:monoidal.Layer}
{\normalfont Layer in a free premonoidal category.}

\begin{minted}{python}
class Layer(cat.Box):
    def __init__(self, left, box, right):
        self.left, self.box, self.right = left, box, right
        dom, cod = left @ box.dom @ right, left @ box.cod @ right
        super().__init__("Layer", dom, cod)
\end{minted}
\end{python}

DisCoPy diagrams are initialised by a domain, a codomain, a list of boxes
and a list of offsets. It comes with methods \py{Diagram.then}, \py{Diagram.tensor}
and \py{Diagram.id} for composing, tensoring and generating identity diagrams.
As well as a method for \py{normal_form} which allows to check monoidal equality
of two premonoidal diagrams.

\begin{python}\label{listing:monoidal.Diagram}
{\normalfont Diagram in a free premonoidal category.}

\begin{minted}{python}
class Diagram(cat.Arrow):
    def __init__(self, dom, cod, boxes, offsets, layers=None):
        if layers is None:
            layers = cat.Id(dom)
            for box, off in zip(boxes, offsets):
                left = layers.cod[:off] if layers else dom[:off]
                right = layers.cod[off + len(box.dom):]\
                    if layers else dom[off + len(box.dom):]
                layers = layers >> Layer(left, box, right)
            layers = layers >> cat.Id(cod)
        self.boxes, self.layers, self.offsets = boxes, layers, tuple(offsets)
        super().__init__(dom, cod, layers, _scan=False)

    def then(self, *others):
        if len(others) > 1:
            return self.then(others[0]).then(*others[1:])
        other, = others
        return Diagram(self.dom, other.cod,
                       self.boxes + other.boxes,
                       self.offsets + other.offsets,
                       layers=self.layers >> other.layers)

    def tensor(self, other):
        dom, cod = self.dom @ other.dom, self.cod @ other.cod
        boxes = self.boxes + other.boxes
        offsets = self.offsets + [n + len(self.cod) for n in other.offsets]
        layers = cat.Id(dom)
        for left, box, right in self.layers:
            layers = layers >> Layer(left, box, right @ other.dom)
        for left, box, right in other.layers:
            layers = layers >> Layer(self.cod @ left, box, right)
        return Diagram(dom, cod, boxes, offsets, layers=layers)

    @staticmethod
    def id(dom):
        return Diagram(dom, dom, [], [], layers=cat.Id(dom))

    def interchange(self, i, j, left=False):
        ...

    def normal_form(self, normalizer=None, **params):
        ...

    def draw(self, **params):
        ...
\end{minted}

\py{Diagram}s always a carry a \py{cat.Arrow} called \py{layers}, which may be
thought of as a witness that the diagram is well-typed.
If no \py{layers} are provided, the \py{Diagram.__init__} method computes the
layers and checks that they compose.
A \py{cat.AxiomError} is returned when the layers do not compose.
The \py{Diagram.interchange} method allows to change the order of layers in a diagram when they commute, and returns an \py{InterchangerError} when they don't.
The \py{Diagram.normal_form} method implements the algorithm of \cite{delpeuch2019b},
see the \py{rewriting} module of DisCoPy.
The \py{Diagram.draw} method is implemented in the \py{drawing} module and allows
to render a diagram via matplotlib \cite{} as well as generating a TikZ \cite{}
output for academic papers.

\end{python}

Finally, a \py{Box} is initialised by a name together with domain and codomain types.

\begin{python}\label{listing:monoidal.Box}
{\normalfont Box in a free monoidal category.}

\begin{minted}{python}
class Box(cat.Box, Diagram):
    def __init__(self, name, dom, cod, **params):
        cat.Box.__init__(self, name, dom, cod, **params)
        layer = Layer(dom[0:0], self, dom[0:0])
        layers = cat.Arrow(dom, cod, [layer], _scan=False)
        Diagram.__init__(self, dom, cod, [self], [0], layers=layers)
\end{minted}
\end{python}

We check that the axioms for monoidal categories hold up to interchanger.

\begin{python}\label{listing:monoidal.axioms}
{\normalfont Axioms of free monoidal categories}

\begin{minted}{python}
x, y, z, w = Ty('x'), Ty('y'), Ty('z'), Ty('w')
f0, f1, f2 = Box('f0', x, y), Box('f1', z, w), Box('f2', z, w)
d = Id(x) @ f1 >> f0 @ Id(w)
assert f0 @ (f1 @ f2) == (f0 @ f1) @ f2
assert f0 @ Diagram.id(Ty()) == f0 == Diagram.id(Ty()) @ f0
assert d == (f0 @ f1).interchange(0, 1)
assert f0 @ f1 == d.interchange(0, 1)
\end{minted}
\end{python}

Functorial reductions are implemented via the \py{monoidal.Functor} class,
initialised by a pair of mappings: \py{ob} from objects to types and \py{ar}
from boxes to diagrams. It comes with a \py{__call__} method that scans through
a diagram a replaces each box and identity wire with its image under the mapping.

\begin{python}\label{listing:monoidal.Functor}
{\normalfont Monoidal functor.}

\begin{minted}{python}
class Functor(cat.Functor):
    def __init__(self, ob, ar, cod=(Ty, Diagram)):
        super().__init__(ob, ar)

    def __call__(self, diagram):
        if isinstance(diagram, Ty):
            return self.cod[0].tensor(*[self.ob[Ty(x)] for x in diagram])
        if isinstance(diagram, Box):
            return super().__call__(diagram)
        if isinstance(diagram, Diagram):
            scan, result = diagram.dom, self.cod[1].id(self(diagram.dom))
            for box, off in zip(diagram.boxes, diagram.offsets):
                id_l = self.cod[1].id(self(scan[:off]))
                id_r = self.cod[1].id(self(scan[off + len(box.dom):]))
                result = result >> id_l @ self(box) @ id_r
                scan = scan[:off] @ box.cod @ scan[off + len(box.dom):]
            return result
        raise TypeError()
\end{minted}

We check that the axioms hold.

\begin{minted}{python}
x, y, z = Ty('x'), Ty('y'), Ty('z')
f0, f1, f2 = Box('f0', x, y), Box('f1', y, z), Box('f2', z, x)
F = Functor(ob={x: y, y: z, z: x}, ar={f0: f1, f1: f2, f2: f0})
assert F(f0 >> f1) == F(f0) >> F(f1)
assert F(f0 @ f1) == F(f0) @ F(f1)
assert F(f0 @ f1 >> f1 @ f2) == F(f0) @ F(f1) >> F(f1) @ F(f2)
\end{minted}
\end{python}

Any \py{operad.Tree} can be turned into an equivalent \py{monoidal.Diagram}.
We show how this interface is built by overriding the \py{__call__} method of \py{operad.Algebra}.

\begin{python}\label{listing:monoidal.Functor}
{\normalfont Interface with \py{operad.Tree}.}

\begin{minted}{python}
from discopy import operad

class Algebra(operad.Algebra):
    def __init__(self, ob, ar, cod=Diagram, contravariant=False):
        self.contravariant = contravariant
        super().__init__(self, ob, ar, cod=cod)

    def __call__(self, tree):
        if isinstance(tree, operad.Id):
            return self.cod.id(self.ob[tree.dom])
        if isinstance(tree, operad.Box):
            return self.ar[tree]
        box = self.ar[tree.root]
        if isinstance(box, monoidal.Diagram):
            if self.contravariant:
                return box << monoidal.Diagram.tensor(
                    *[self(branch) for branch in tree.branches])
            return box >> monoidal.Diagram.tensor(
                *[self(branch) for branch in tree.branches])
        return box(*[self(branch) for branch in tree.branches])


ob2ty = lambda ob: Ty(ob)
node2box = lambda node: Box(node.name, Ty(node.dom), Ty(*node.cod))
t2d = Algebra(ob2ty, node2box, cod=Diagram)
node2box_c = lambda node: Box(node.name, Ty(*node.cod), Ty(node.dom))
t2d_c = Algebra(ob2ty, node2boxc, cod=Diagram, contravariant=True)

def tree2diagram(tree, contravariant=False):
    if contravariant:
        return t2dc(tree)
    return t2d(tree)
\end{minted}
\end{python}

\section{Categorial grammar}\label{sec-closed}

The \emph{categorial} tradition of formal grammars originated in the works of
Ajdukiewicz \cite{ajdukiewiz1935} and Bar-Hillel \cite{bar-hillel1953},
their formalisms are now known as AB grammars \cite{buszkowski2016}.
They analyse language syntax by assigning to every word a \emph{type}
generated from basic types using two operations: $\backslash$ (under) and $/$ (over).
Strings of types are then parsed according to the following basic reductions:
\begin{equation}\label{adjiukiewicz}
    \alpha \, (\alpha \backslash \beta) \to \beta \qquad
    (\alpha / \beta)\, \beta \to \alpha
\end{equation}

The slash notation $\alpha / \beta$, replacing the earlier fraction
$\frac{\alpha}{\beta}$, is due to Lambek who unified categorial grammars
in an algebraic foundation known as the \emph{Lambek calculus}, first presented
in his seminal 1958 paper \cite{lambek1958}.
With the advent of Chomsky's theories of syntax in the 1960s, categorial
grammars were disregarded for almost twenty years \cite{lambek1988}.
They were revived in the 1980s by several researchers such as Buszowski in
Poland, van Benthem and Moortgat in the Netherlands, as witnessed in the
1988 books \cite{oehrle1988, moortgat1988}.

One reason for this revival is the proximity between categorial grammars and
logic. Indeed, the original rewrite rules (\ref{adjiukiewicz}) can be seen as
versions of modus ponens in a Gentzen style proof system \cite{lambek1999}.
Another reason for this revival, is the proximity between categorial grammars,
the typed Lambda calculus \cite{church1940} and the semantic calculi
of Curry \cite{curry1961} and Montague \cite{Montague73}.
Indeed, one of the best intuitions for categorial grammars comes from
interpreting the slash type $\alpha / \beta$ as the type of a \emph{function}
with input of type $\beta$ and output of type $\alpha$,
and the reduction rules (\ref{adjiukiewicz}) as function application.
From this perspective, the syntactic process of recognizing
a sentence has the same form as the semantic process of understanding it
\cite{steedman2000}. We will see the implications of this philosophy in \ref{sec-functional}.

Although he did not mention categories in his original paper \cite{lambek1958},
Lambek had in mind the connections between linguistics and category theory all
along, as mentioned in \cite{lambek1988}.
Indeed the reductions in (\ref{adjiukiewicz}) and those introduced by Lambek,
correspond to the morphisms of \emph{free biclosed categories}, which admit a
neat description as deductions in a Gentzen style proof system.
This leads to a fruitful parallel between algebra, proof theory and
categorial grammar, summarized in the following table.

\begin{center}
\begin{tabular}{ |c|c|c|c| }
 \hline
 Categories & Logic & Linguistics & Python\\
 \hline
 Biclosed category & Proof system & Categorial grammar & DisCoPy\\
 \hline
 objects & formulas & types & \py{biclosed.Ty} \\
 generators & axioms & lexicon & \py{biclosed.Box} \\
 morphisms & proof trees & reductions & \py{biclosed.Diagram}\\
 \hline
\end{tabular}
\end{center}

We start by defining biclosed categories and a general notion of biclosed grammar.
Then the section will unfold as we unwrap the definition, meeting the three most
prominent variants of categorial grammars: AB grammars \cite{buszkowski2016},
the Lambek calculus \cite{lambek1958} and Combinatory Categorial
Grammars \cite{steedman2000}. We end by giving an implementation of free biclosed
categories as a class \py{biclosed.Diagram} with methods for currying and
uncurrying.

\subsection{Biclosed categories}

\begin{definition}[Biclosed signature]
    A biclosed signature $G$ is a collection of generators $G_1$ and basic
    types $G_0$ together with a pair of maps:
    $$ \signature{G_1}{T(G_0)}$$
    where $T(G_0)$ is the set of biclosed types, given by the following inductive
    definition:
    \begin{equation}\label{biclosed-types}
        T(B) \ni \alpha = a \in B \; \vert \; \alpha \otimes \alpha\;
        \vert \; \alpha \backslash \alpha\; \vert \; \alpha / \alpha
    \end{equation}
    A morphism of biclosed signatures $\phi : G \to \Gamma$ is a pair of
    maps $\phi_1 : G_1 \to \Gamma_1$ and $\phi_0 : G_0 \to \Gamma_0$ such that
    the diagram with the signature morphisms commutes.
    The category of biclosed signatures is denoted $\bf{BcSig}$
\end{definition}

\begin{definition}[Biclosed category]\label{def-biclosed}
    A biclosed monoidal category $\bf{C}$ is a monoidal category equipped with two
    bifunctors $-\backslash- : \bf{C}^{op} \times \bf{C} \to \bf{C}$ and
    $-/-: \bf{C} \times \bf{C}^{op} \to \bf{C}$ such that
    for any object $a$, $a \otimes - \dashv a \backslash -$ and
    $- \otimes a \dashv  - / a$. Explicitly, we have the following isomorphisms
    natural in $a, b, c \in \bf{C}_0$:
    \begin{equation}
        \bf{C}(a, c / b) \simeq \bf{C}(a \otimes b, c) \simeq \bf{C}(b, a \backslash c)
    \end{equation}
    These isomorphisms are often called \emph{currying} (when $\otimes$ is replaced
    by $\backslash$ or $/$) and \emph{uncurrying}.
    With monoidal functors as morphisms, biclosed categories form a category
    denoted $\bf{BcCat}$.
\end{definition}

Morphisms in free biclosed category can be described as the valid deductions
in a proof system defined as follows. The axioms of monoidal
categories may be expressed as the following rules of inference:
\begin{equation} \label{monoidal-axioms}
\begin{minipage}{0.3\linewidth}
    \begin{prooftree}
        \AxiomC{} \RightLabel{$\tt{id}$}
        \UnaryInfC{$a \to a$}
    \end{prooftree}
\end{minipage}
\begin{minipage}{0.3\linewidth}
\begin{prooftree}
    \AxiomC{$a \to b$}
    \AxiomC{$b \to c$} \RightLabel{$\circ$}
    \BinaryInfC{$a \to c$}
\end{prooftree}
\end{minipage}
\begin{minipage}{0.3 \linewidth}
\begin{prooftree}
    \AxiomC{$a \to b$}
    \AxiomC{$c \to d$} \RightLabel{$\otimes$}
    \BinaryInfC{$a \otimes c \to b \otimes d$}
\end{prooftree}
\end{minipage}
\end{equation}
Additionally, the defining adjunctions of biclosed categories may be expressed
as follows:
\begin{equation} \label{biclosed-iso}
    a \otimes b \to c \quad \text{iff} \quad a \to c / b\quad \text{iff} \quad b \to a \backslash c
\end{equation}
Given a biclosed signature $G$, the free biclosed category $\bf{BC}(G)$
contains all the morphisms that can be derived using the inference rules
\ref{biclosed-iso} and \ref{monoidal-axioms} from the signature
seen as a set of axioms for the deductive system.
$\bf{BC}$ is part of an adjunction relating biclosed signatures and biclosed
categories:
\begin{equation*}
    \bf{BcSig} \mathrel{\mathop{\rightleftarrows}^{\bf{BC}}_{U}} \bf{BcCat}
\end{equation*}
We can now define a general notion of biclosed grammar, the equivalent of
monoidal grammars in a biclosed setting.

\begin{definition}
    A biclosed grammar $G$ is a biclosed signature of the following shape:
    $$ \signature{G}{T(B + V)}$$
    where $V$ is a \emph{vocabulary} and $B$ is a set of \emph{basic types}.
    The language generated by a G is given by:
    $$ \cal{L}(G) = \set{u \in V^\ast \, \vert \, \exists g: u \to s \in \bf{BC}(G)}$$
    where $\bf{BC}(G)$ is the free biclosed category generated by $G$.
    We say that $G$ is \emph{lexicalised} when it has the following shape:
    $$ V \xleftarrow{\tt{dom}} G \xto{\tt{cod}} T(B) $$
\end{definition}

As we will see, AB grammars, Lambek grammars and Combinatory Categorial grammars
all reduce functorially to biclosed grammars. However, biclosed grammars can have
an infinite number of rules of inference, obtained by iterating over the
isomorphism (\ref{biclosed-iso}). Interestingly, these rules have been discovered
progressively throughout the history of categorial grammars.

\subsection{Ajdiuciewicz}\label{section-ab}

We discuss the classical categorial grammars of Ajdiuciewicz and Bar-Hillel,
known as AB grammars \cite{buszkowski2016}.
The types of the original AB grammars are given by the following inductive definition:
$$T_{AB}(B) \ni \alpha = a \in B \; |\; \alpha \backslash \alpha\; |\; \alpha / \alpha\; .$$
Given a vocabulary $V$, the \emph{lexicon} is usually defined as a relation
$\Delta \sub V \times T_{AB}(B)$ assigning a set of candidate types $\Delta(w)$
to each word $w \in V$.
Given an utterance $u = w_0 \dots w_n \in V^\ast$, one can prove that
$u$ is a grammatical sentence by producing a reduction
$t_0 \dots t_n \to s$ for some $t_i \in \Delta(w_i)$, generated by the basic
reductions in (\ref{adjiukiewicz}).

\begin{definition}[AB grammar]
    An AB grammar is a tuple $G = (V, B, \Delta, s)$ where $V$ is a vocabulary,
    $B$ is a finite set of basic types, and $\Delta \sub V \times T_{AB}(B)$
    is a finite relation, called the \emph{lexicon}.
    The \emph{rules} of AB grammars are given by the following monoidal signature:
    $$R_{AB} = \left\{\input{figures/ab-grammar.tikz}\right\}_{a, b \in T_{AB}(B)}.$$
    Then the language generated by $G$ is given by:
    $$\cal{L}(G) = \set{u \in V^\ast \, : \, \exists g \in \bf{MC}(\Delta + R_{AB})(u, s)}$$
\end{definition}
\begin{remark}
    Sometimes, categorial grammarians use a different notation where
    $a \backslash b$ is used in place of $b \backslash a$.
    We find the notation used here more intuitive: we write
    $a \otimes a \backslash b \to b$ instead of $a \otimes b \backslash a \to b$.
    Ours is in fact the notation used in the original paper by Lambek \cite{lambek1958}.
\end{remark}

\begin{example}[Application]
    As an example take the vocabulary
    $V = \{\text{Caesar}, \text{crossed}, \text{the}, \text{Rubicon} \}$ and basic
    types $B = \{ s, n, d \}$ for sentence, noun and determinant types. We define
    the lexicon $\Delta$ by:
    $$\Delta(\text{Caesar}) = \set{n} \quad \Delta(\text{crossed}) = \set{n\backslash (s / n)}
    \quad \Delta(\text{the}) = \set{d} \quad \Delta(\text{Rubicon}) = \set{d\backslash n}$$
    Then the sentence ``John wrote a dictionnary'' is grammatical as witnessed by the
    following reduction:
    \begin{equation*}
        \scalebox{0.8}{\input{figures/ab-grammar-thetis.tikz}}
    \end{equation*}
    %
    %
\end{example}

\begin{proposition}\label{prop-ab-biclosed}
    AB grammars reduce functorially to biclosed grammars.
\end{proposition}
\begin{proof}
    It is sufficient to show that the rules $R_{AB}$ can be derived from the
    axioms of biclosed categories,
    i.e. that they are morphisms in any free biclosed category.
    Let $a, b$ be objects in a free biclosed category. We derive
    the forward application rule as follows.
    \begin{prooftree}
    \AxiomC{} \RightLabel{$\tt{id}$}
    \UnaryInfC{$a / b \to a / b$} \RightLabel{\ref{biclosed-iso}}
    \UnaryInfC{$a / b \otimes b \to a$}
    \end{prooftree}
    Similarly, one may derive backward application
    $\tt{app}^< : b \otimes b \backslash a \to b$.
\end{proof}

AB grammars are weakly equivalent to context-free grammars \cite{retore2005} as
shown by the following pair of propositions.

\begin{proposition}\label{prop-ab-cfg}
    For any AB grammar there is a functorially equivalent context-free grammar
    in Chomsky normal form.
\end{proposition}
\begin{proof}
    The only difficulty in this proof comes from the fact that $R_{AB}$ is
    an infinite set, whereas context-free grammars must be defined over a
    finite set of symbols and production rules. Given an AB grammar
    $G = (V, B, \Delta)$, define $X = \{ x \in T_{AB}(B) \vert
    \exists (w, t) \in \Delta \text{such that} x \sub t \} $
    where we write $x \sub t$ to say that $x$ is a sub-type of $t$. Note that
    $X$ is finite. Now let $P  = \{ r \in R_{AB} \vert \tt{dom}(r) \in X \}$. Note that also $P$ is a finite set. Then
    $ X \leftarrow P + \Delta \rightarrow (X + V)^\ast$ forms a lexicalised
    context-free grammar where each production rule has at most arity $2$,
    i.e. $P + \Delta$ is in Chomsky normal form.
    There is a functorial reduction $\bf{MC}(P + \Delta) \to \bf{MC}(R_{AB} + \Delta)$
    induced by the injection $P \injects R_{AB}$, also there is a functorial
    reduction $\bf{MC}(R_{AB} + \Delta) \to \bf{MC}(P + \Delta)$ which sends
    all the types in $T_{AB}(B) \backslash X$ to the unit and acts as the identity
    on the rest. Therefore $G$ is functorially equivalent to $P + \Delta$.
\end{proof}
\begin{proposition}\label{prop-cfg-ab}
    Any context-free grammar in Greibach normal form reduces functorially to an
    AB grammar.
\end{proposition}
\begin{proof}
    Recall that a CFG is in Greibach normal form when it has the shape:
    $$ B \leftarrow G \to V \times B^\ast$$
    We can rewrite this as a signature of the following shape:
    $$ V \leftarrow G \rightarrow B \times B^\ast$$
    This yields a relation $G \sub V \times (B \times B^\ast)$.
    We define the lexicon $\Delta \sub V \times T(B)$ by
    $\Delta(w) = G(w)_0 / \tt{inv}(G(w)_1)$ where $\tt{inv}$ inverts the order of
    the basic types. Then there is a functorial reduction
    $\bf{MC}(G) \to \bf{MC}(\Delta + R_{AB})$ given on any production rule
    $w a_0 \dots a_k \to b$ in $G$ by:
    \begin{equation*}
        \input{figures/ab-grammar-greibach.tikz}
    \end{equation*}
\end{proof}

\subsection{Lambek}\label{section-lambek}

In his seminal paper \cite{lambek1958}, Lambek introduced two new types of rules
to the original AB grammars: \emph{composition} and \emph{type-raising}.
Several extensions have been considered since, including non-associative
\cite{moot2012g} and multimodal \cite{moot2012h} versions. Here, we focus
on the original (associative) Lambek calculus which is easier to cast
in a biclosed setting and already captures a large fragment of natural language.

The types of the Lambek calculus are precisely the biclosed types given in
\ref{biclosed-types}. This is not a coincidence since Lambek was aware of
biclosed categories when he introduced his calculus \cite{lambek1988}.

\begin{definition}[Lambek grammar]
    A Lambek grammar is a tuple $G = (V, B, \Delta, s)$ where $V$ is a
    vocabulary, $B$ is a finite set of basic types and
    $\Delta \sub V \times T(B)$ is a \emph{lexicon}.
    The rules of Lambek grammars are given by the following monoidal signature:
    $$R_{LG} = \left\{ \input{figures/lambek-grammar.tikz} \right\}_{a, b, c \in T(B)}.$$
    Then the language generated by $G$ is given by:
    $$\cal{L}(G) = \set{u \in V^\ast \, : \, \exists g \in \bf{MC}(\Delta + R_{LG})(u, s)}$$
\end{definition}

\begin{example}[Composition]\label{lambek-example}
    We adapt an example from \cite{steedman1987a}.
    Consider the following lexicon:
    $$\Delta(\text{I}) = \Delta(\text{Grandma}) = \set{n} \quad \Delta(\text{the}) = \set{d}
    \quad \Delta(\text{parmigiana}) = \set{d \backslash n}$$
    $$\Delta(\text{will}) = \Delta(\text{may}) = \set{(n\backslash s)/(n\backslash s)}
    \quad \Delta(\text{eat}) = \Delta(\text{cook}) = \set{n \backslash s / n}$$
    The following is a grammatical sentence, parsed using the composition rule:
    \begin{equation*}
        \scalebox{0.8}{\input{figures/lambek-long-winding.tikz}}
    \end{equation*}
    And the following sentence requires the use of type-raising:
    \begin{equation*}
        \scalebox{0.8}{\input{figures/lambek-type-raising.tikz}}
    \end{equation*}
    where $x = s/n$ and we have omitted the composition of modalities
    (will, may) with their verbs (cook, eat).
\end{example}

\begin{proposition}\label{prop-lambek-biclosed}
    Lambek grammars reduce functorially to biclosed grammars.
\end{proposition}
\begin{proof}
    It is sufficient to show that the rules of Lambek grammars $R_{LG}$ can be
    derived from the axioms of biclosed categories. We already derived forward
    and backward application in \ref{prop-ab-biclosed}.
    The following proof tree shows that the forward composition rule
    follows from the axioms of biclosed categories.
    \begin{prooftree}
    \AxiomC{} \RightLabel{$\tt{id}$}
    \UnaryInfC{$a / b \to a / b$}
    \AxiomC{} \RightLabel{$\tt{id}$}
    \UnaryInfC{$b / c \to b / c$} \RightLabel{\ref{biclosed-iso}}
    \UnaryInfC{$b / c \otimes c \to b$} \RightLabel{$\otimes$ and $\circ \tt{app}$}
    \BinaryInfC{$a / b \otimes b / c \otimes c \to a$} \RightLabel{\ref{biclosed-iso}}
    \UnaryInfC{$a / b \otimes b / c \to a / c$}
    \end{prooftree}
    A similar argument derives the backward composition rule
    $\tt{comp}^< : a / b \otimes a \backslash c \to  c / b$.
    Forward type-raising is derived as follows.
    \begin{prooftree}
    \AxiomC{} \RightLabel{$\tt{id}$}
    \UnaryInfC{$a \backslash b \to a \backslash b$} \RightLabel{\ref{biclosed-iso}}
    \UnaryInfC{$a \otimes a \backslash b \to b$} \RightLabel{\ref{biclosed-iso}}
    \UnaryInfC{$a \to b / (a \backslash b)$}
    \end{prooftree}
    And a similar argument derives backward type-raising
    $\tt{tyr}^< : a \to (b / a) \backslash b$.
\end{proof}

It was conjectured by Chomsky that any Lambek grammar is weakly equivalent
to a context-free grammar, i.e. that the language recognised by Lambek calculi
is context-free. This conjecture was proved in 1993 by Pentus \cite{Pentus93},
calling for a more expressive version of categorial grammars.

\subsection{Combinatory}\label{section-combinatory}

In the 1980s, researchers interested in the syntax of natural language started
recognizing that certain grammatical sentences naturally involve crossed
dependencies between words, a phenomenon that cannot be captured by context-free
grammars. These are somewhat rare in English, but they abound in Dutch for
instance \cite{bresnan1982}. An English example is the sentence
''I explained to John maths'' which is often allowed to mean
''I explained maths to John''.
Modeling cross-serial dependencies was one of the main motivations for the development
the Tree-adjoining grammars of Joshi \cite{joshi1985} and the
Combinatory Categorial grammars (CCGs) of Steedman \cite{steedman1987a, steedman2000}.
These were later shown to be weakly equivalent to linear indexed grammars
\cite{vijay-shanker1994}, making them all ``mildly-context sensitive'' according
to the definition of Joshi \cite{joshi1985}.

CCGs extend the earlier categorial grammars by adding a \emph{crossed composition}
rule which allows for controlled crossed dependencies within a sentence,
and is given by the following pair of reductions:
$$\tt{xcomp}^>: a / b \otimes c \backslash b \to c \backslash a
\qquad \tt{xcomp}^<: a / b \otimes a \backslash c \to  c / b$$
We start by defining CCGs as monoidal grammars, and then discuss how they relate
to biclosed categories.

\begin{definition}[Combinatory Categorial Grammar]
    A CCG is a tuple $G = (V, B, \Delta, s)$ where $V$ is a vocabulary,
    $B$ is a finite set of basic types and $\Delta \sub V \times T(B)$ is a \emph{lexicon}.
    The \emph{rules} of CCGs are given by the following monoidal signature:
    $$R_{CCG} = R_{LG} + \left\{\input{figures/combinatory.tikz}\right\}_{a, b, c \in T(B)}.$$
    Then the language generated by $G$ is given by:
    $$\cal{L}(G) = \set{u \in V^\ast \, : \, \exists g \in \bf{MC}(\Delta + R_{CCG})(u, s)}$$
\end{definition}

\begin{example}[Crossed composition]\label{ex-crossed-comp}
    Taking the grammar from Example \label{lambek-example} and adding two lexical
    entries we may derive the following grammatical sentences:
    \begin{equation*}
        \scalebox{0.8}{\input{figures/explained-maths.tikz}}
    \end{equation*}
    Note that the first one can be parsed already in an AB grammar, whereas the
    second requires the crossed composition rule.
\end{example}

It was first shown by Moortgat \cite{moortgat1988} that the crossed composition
rule cannot be derived from the axioms of the Lambek calculus. In our context,
this implies that we cannot derive $\tt{xcomp}$ from the axioms of biclosed
categories. Of course, CCGs may be seen as biclosed grammars by adding the
crossed composition rules $R_{CCG} - R_{LG}$ as generators in the signature.
However, it is interesting to note that these rules can be derived from
\emph{closed categories}, the \emph{symmetric} version of biclosed categories:.

\begin{definition}[Symmetric monoidal category]
    A symmetric monoidal category $\bf{C}$ is a monoidal category equipped with
    a natural transformation $\tt{swap}: a \otimes b \to b \otimes a$ satisfying:
    \begin{equation}\label{axioms-symmetry}
        \input{figures/symmetry.tikz}
    \end{equation}
    for any $a, b, c \in \bf{C}_0$ and $f: a \to b \in \bf{C}_1$.
\end{definition}
\begin{definition}[Closed category]
    A closed category is a symmetric biclosed category.
\end{definition}
\begin{proposition}\label{prop-crossed-biclosed}
    The crossed composition rule follows from the axioms of closed categories.
\end{proposition}
\begin{proof}
    Let $a, b, c$ be objects in the free closed category with no generators.
    The following proof tree shows that the backward crossed composition rule
    follows from the axioms of closed categories.
    \begin{prooftree}
    \AxiomC{} \RightLabel{$\tt{id}$}
    \UnaryInfC{$a / b \to a / b$}
    \AxiomC{} \RightLabel{$\tt{id}$}
    \UnaryInfC{$c \backslash b \to c \backslash b$} \RightLabel{\ref{biclosed-iso}}
    \UnaryInfC{$c \otimes c \backslash b \to b$} \RightLabel{$\otimes$ and $\circ \tt{app}$}
    \BinaryInfC{$a /b \otimes c \otimes c \backslash b \to a$} \RightLabel{$\circ \tt{swap}$}
    \UnaryInfC{$a / b \otimes c \backslash b \otimes c \to a$} \RightLabel{\ref{biclosed-iso}}
    \UnaryInfC{$a / b \otimes c \backslash b \to a \backslash c$}
    \end{prooftree}
    Note the crucial use of pre-composition with $\tt{swap}$ in the penultimate inference.
    A similar argument derives the forward crossed composition rule
    $a / b \otimes a \backslash c \to  c / b$.
\end{proof}

One needs to be very careful in adding permutations and symmetry in a formal grammar.
The risk is to lose any notion of word order.
It was initially believed that adding the crossed composition rule, coupled with
type-rasing, would collapse the calculus to permutation invariance \cite{williams2003}.
However, as argued by Steedman this is not the case:
out of the $n!$ possible permutations of a string of $n$ types,
CCGs only allow $S(n)$ permutations where $S(n)$ is the $n$th Large Schroder
number \cite{steedman2019}. Thus the crossed composition rule only permits limited crossed
dependencies within a sentence. This claim is also verified empirically by the
successful performance of large scale CCG parsers \cite{yoshikawa2019}.

If we consider the complexity of parsing categorial grammars, it turns out that
parsing the Lambek calculus is $\tt{NP}$-complete \cite{savateev2012}.
Even when the order of types is bounded the best parsing algorithms for the
Lambek calculus run in $O(n^5)$ \cite{fowler2008}. Parsing
CCGs is in $\tt{P}$ with respect to the size of the input sentence, but it becomes
exponential when also the grammar is taken in the input \cite{kuhlmann2018}.
This is in contrast to CFGs or the mildly context-sensitive tree-adjoining
grammars where parsing is computable in $O(n^3)$ in both the size of
input grammar and input sentence.
This called for a simpler version of the Lambek calculus, with a more efficient
parsing procedure, but still keeping the advantages of the categorial
lexicalised formalism and the proximity to semantics.

\subsection{biclosed.Diagram}

We wish to upgrade the \py{monoidal.Diagram} class to represent diagrams in free
biclosed categories. In order to achieve this, we first need to upgrade the \py{Ty}
class, to handle the types $T(B)$ of free biclosed categories. A \py{biclosed.Ty}
is an instance of \py{monoidal.Ty} with extra methods for under $\backslash$
and over $/$ operations. Recall the context-free grammar of biclosed types:
\begin{equation}\label{biclosed-types}
    T(B) \ni \alpha = a \in B \; \vert \; \alpha \otimes \alpha\;
    \vert \; \alpha \backslash \alpha\; \vert \; \alpha / \alpha
\end{equation}
Elements of $T(B)$ are trees built from leaves $B$ and binary nodes
$\otimes$, $\backslash$ and $/$.
It may be implemented in object-oriented programming, by subclassing
\py{monoidal.Ty} for the tensor and using the native tree structure of classes.
We give a working version of \py{biclosed.Ty}, detailing also some of the
standard methods.

\begin{python}\label{listing:biclosed.Ty}
{\normalfont Types in free biclosed categories}

\begin{minted}{python}
from discopy import monoidal

class Ty(monoidal.Ty):
    def __init__(self, *objects, left=None, right=None):
        self.left, self.right = left, right
        super().__init__(*objects)

    def __lshift__(self, other):
        return Over(self, other)

    def __rshift__(self, other):
        return Under(self, other)

    def __matmul__(self, other):
        return Ty(*(self @ other))

    @staticmethod
    def upgrade(old):
        if len(old) == 1 and isinstance(old[0], (Over, Under)):
            return old[0]
        return Ty(*old.objects)


class Over(Ty):
    def __init__(self, left=None, right=None):
        Ty.__init__(self, self)

    def __repr__(self):
        return "Over({}, {})".format(repr(self.left), repr(self.right))

    def __str__(self):
        return "({} << {})".format(self.left, self.right)

    def __eq__(self, other):
        if not isinstance(other, Over):
            return False
        return self.left == other.left and self.right == other.right


class Under(Ty):
    def __init__(self, left=None, right=None):
        Ty.__init__(self, self)

    def __repr__(self):
        return "Under({}, {})".format(repr(self.left), repr(self.right))

    def __str__(self):
        return "({} >> {})".format(self.left, self.right)

    def __eq__(self, other):
        if not isinstance(other, Under):
            return False
        return self.left == other.left and self.right == other.right
\end{minted}

The \py{Ty.upgrade} method allows for compatibility between the tensor
methods \py{__matmul__} of \py{biclosed} and \py{monoidal} types. The upgrading
mechanism is further desribed in the documentation of \py{monoidal}.

We illustrate the syntax of biclosed types.

\begin{minted}{python}
x = Ty('x')
assert x >> x << x == Over(Under(Ty('x'), Ty('x')), Ty('x'))
assert x >> (x << x) == Under(Ty('x'), Over(Ty('x'), Ty('x')))
x0, x1, y0, y1, m = Ty('x0'), Ty('x1'), Ty('y0'), Ty('y1'), Ty('m')
lens = (x0 >> m @ y0) @ ( m @ x1 >> y1)
assert lens == Ty(Under(Ty('x0'), Ty('m', 'y0')), Under(Ty('m', 'x1'), Ty('y1')))
\end{minted}
\end{python}

A \py{biclosed.Diagram} is a \py{monoidal.Diagram} with domain and
codomain \py{biclosed.Ty}s, together with a pair of static methods \py{curry()}
and \py{uncurry()} implementing the defining isomorphism of biclosed categories
(\ref{biclosed-iso}).
In fact, we can store the information of how a biclosed diagram is constructed
by using two special subclasses of \py{Box}, which record every application of
the currying morphisms. Thus a diagram in a biclosed category is a tree built
from generating boxes using \py{id}, \py{then}, \py{tensor}, \py{Curry} and
\py{UnCurry}.

\begin{python}\label{listing:biclosed.Diagram}
{\normalfont Diagrams in free biclosed categories}

\begin{minted}{python}
from discopy import monoidal

@monoidal.Diagram.subclass
class Diagram(monoidal.Diagram):

    def curry(self, n_wires=1, left=False):
        return Curry(self, n_wires, left)

    def uncurry(self):
        return UnCurry(self)

class Id(monoidal.Id, Diagram):


class Box(monoidal.Box, Diagram):


class Curry(Box):
    def __init__(self, diagram, n_wires=1, left=False):
        name = "Curry({})".format(diagram)
        if left:
            dom = diagram.dom[n_wires:]
            cod = diagram.dom[:n_wires] >> diagram.cod
        else:
            dom = diagram.dom[:-n_wires]
            cod = diagram.cod << diagram.dom[-n_wires or len(diagram.dom):]
        self.diagram, self.n_wires, self.left = diagram, n_wires, left
        super().__init__(name, dom, cod)


class UnCurry(Box):
    def __init__(self, diagram):
        name = "UnCurry({})".format(diagram)
        self.diagram = diagram
        if isinstance(diagram.cod, Over):
            dom = diagram.dom @ diagram.cod.right
            cod = diagram.dom.left
            super().__init__(name, dom, cod)
        elif isinstance(diagram.cod, Under):
            dom = diagram.dom.left @ diagram.dom
            cod = diagram.dom.right
            super().__init__(name, dom, cod)
        else:
            super().__init__(name, diagram.dom, diagram.cod)
\end{minted}

We give a simple implementation of free biclosed categories which does not impose
the axioms of free biclosed categories, \py{UnCurry(Curry(f)) == f} and naturality.
IT can be implemented on syntax by adding
\py{if} statement in the inits, and upgrading \py{Curry} and \py{UnCurry}
to subclasses of \py{biclosed.Diagram}..
\end{python}

Finally, a \py{biclosed.Functor} is a \py{monoidal.Functor} whose call
method has a predefined mapping for all structural boxes in \py{biclosed}.
It thus allows to map any \py{biclosed.Diagram} into a codomain class \py{cod}
equipped with \py{curry} and \py{uncurry} methods.

\begin{python}\label{listing:biclosed.Functor}
{\normalfont Functors from free biclosed categories}

\begin{minted}{python}
class Functor(monoidal.Functor):
    def __init__(self, ob, ar, cod=(Ty, Diagram)):
        self.cod = cod
        super().__init__(ob, ar, ob_factory=cod[0], ar_factory=cod[1])

    def __call__(self, diagram):
        if isinstance(diagram, Over):
            return self(diagram.left) << self(diagram.right)
        if isinstance(diagram, Under):
            return self(diagram.left) >> self(diagram.right)
        if isinstance(diagram, Ty) and len(diagram) > 1:
            return self.cod[0].tensor(*[
                self(diagram[i: i + 1]) for i in range(len(diagram))])
        if isinstance(diagram, Id):
            return self.cod[1].id(self(diagram.dom))
        if isinstance(diagram, Curry):
            n_wires = len(self(getattr(
                diagram.cod, 'left' if diagram.left else 'right')))
            return self.cod[1].curry(
                self(diagram.diagram), n_wires, diagram.left)
        if isinstance(diagram, UnCurry):
            return self.cod[1].uncurry(self(diagram.diagram))
        return super().__call__(diagram)
\end{minted}
\end{python}

We recover the rules of categorial grammars (in historical order) by constructing
them from identities in the free biclosed category with no generators.

\begin{python}\label{listing:biclosed.categorial}
{\normalfont Categorial grammars and the free biclosed category}

\begin{minted}{python}
# Adjiuciewicz
FA = lambda a, b: UnCurry(Id(a >> b))
assert FA(x, y).dom == x @ (x >> y) and FA(x, y).cod == y
BA = lambda a, b: UnCurry(Id(b << a))
assert BA(x, y).dom == (y << x) @ x and BA(x, y).cod == y

# Lambek
proofFC = FA(x, y) @ Id(y >> z) >> FA(y, z)
FC = Curry(proofFC, left=True)
assert FC.dom == (x >> y) @ (y >> z) and FC.cod == (x >> z)
BC = Curry(Id(x << y) @ BA(z, y) >> BA(y, x))
assert BC.dom == (x << y) @ (y << z) and BC.cod == (x << z)
TYR = Curry(UnCurry(Id(x >> y)))
assert TYR.dom == x and TYR.cod == (y << (x >> y))

# Steedman
Swap = lambda a, b: Box('Swap', a @ b, b @ a)
proofBX = Id(x << y) @ (Swap(z >> y, z) >> FA(z, y)) >> BA(y, x)
BX = Curry(proofBX)
assert BX.dom == (x << y) @ (z >> y) and BX.cod == (x << z)
proofFX = (Swap(x, y << x) >> BA(x, y)) @ Id(y >> z) >> FA(y, z)
FX = Curry(proofFX, left=True)
assert FX.dom == (y << x) @ (y >> z) and FX.cod == (x >> z)
\end{minted}
\end{python}

The assertions above are alternative proofs of Propositions
\ref{prop-ab-biclosed}, \ref{prop-lambek-biclosed} and \ref{prop-crossed-biclosed}.
We draw the proofs for forward composition (\py{proofFC}) and
backwards crossed composition (\py{proofBX}).
\begin{center}
    \includegraphics[scale=0.55]{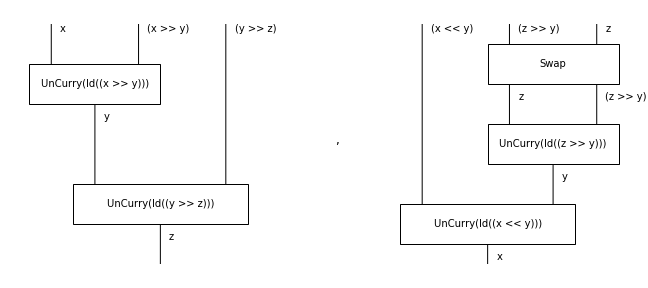}
\end{center}
%
%

\section{Pregroups and dependencies}\label{sec-compact}

In his 1897 paper ``The logic of relatives'' \cite{peirce1897}, Peirce makes an
analogy between the sentence ``John gives John to John'' and the molecule of ammonia.
\begin{equation*}
    \input{figures/peirce-analogy.tikz}
\end{equation*}
The intuition that words within a sentence are connected by ``bonds'',
as atoms in a molecule, is the basis of Peirce's diagrammatic approach to
logical reasoning, and of his analysis of the concept of \emph{valency} in
grammar \cite{askedal2019}.
This intuition resurfaced in the work of Lucien Tesni\'ere \cite{tesniere1959}
who analysed the valency of a large number of lexical items, in an approach to
syntax that became known as \emph{dependency grammar} \cite{hays1964, gaifman1965}.
These bear a striking resemblance to Lambek's pregroup grammars \cite{Lambek08}
and its developments in the DisCoCat framework of Coecke, Sadrzadeh et al.
\cite{DisCoCat08,sadrzadeh2013, sadrzadeh2014}.

In this section, we formalise both Pregroup Grammar (PG) and Dependency Grammar (DG)
in the language of \emph{free rigid categories}. Once casted
in this algebraic framework, the similarities between PG and DG become
apparent. We show that dependency grammars are structurally equivalent to
both pregroup grammar and context-free grammar, i.e. their derivations are
tree-shaped rigid diagrams. We end by describing the implementation of the class
\py{rigid.Diagram} and we interface it with SpaCy's dependency parser \cite{spacy}.

\subsection{Pregroups and rigid categories}\label{section-pregroup}

Pregroup grammars were introduced by Lambek in 1999 \cite{lambek1999a}.
Arising from a non-commutative fragment of Girard's linear logic
\cite{casadio2002} they refine and simplify the earlier Lambek calculus
discussed in Section \ref{section-lambek}.
As shown by Buszowski, pregroup grammars are weakly equivalent to context-free
grammars \cite{buszkowski2007}. However, the syntactic structures generated by
pregroup grammars differ from those of a CFG. Instead of trees,
pregroups parse sentences by assigning to them a nested pattern of
\emph{cups} or ``links'' as in the following example.
\begin{equation*}
    \input{figures/pregroup-gives.tikz}
\end{equation*}
In the strict sense of the word, a pregroup is a preordered monoid where each object
has a left and a right adjoint \cite{lambek1999a}. We formalise pregroup grammars
in terms of \emph{rigid categories} which categorify the notion of pregroup
by upgrading the preorder to a category. Going from inequalities in a preordered
monoid to arrows in a monoidal category allows both to reason about syntactic ambiguity
--- as discussed by Lambek and Preller \cite{preller2007b} --- as well as to define
pregroup semantics as a monoidal functor, an observation which lead to the
development of the compositional distributional models of Coecke et al.
\cite{DisCoCat11} and which will form the basis of the next chapter on semantics.

Given a set of basic types $B$, the set of pregroup types $P(B)$ is defined
as follows.
$$P(B) \ni t \, ::= \, b \in B \, | \, t^r \, | \, t^l \, | \, t \otimes t.$$
we can use it to define the notion of a rigid signature.

\begin{definition}[Rigid signature]
    A rigid signature is a graph $\Sigma = \graph{\Sigma_1}{P(\Sigma_0)}$.
    A morphism of rigid signatures $\Sigma \to \Gamma$ is a pair of maps
    $\Sigma_1 \to \Gamma_1$ and $\Sigma_0 \to \Gamma_0$ satisfying the obvious
    commuting diagram. Rigid signatures and their morphisms form a category
    denoted $\bf{RgSig}$.
\end{definition}
\begin{definition}[Rigid category]
    A rigid category $\bf{C}$ is a monoidal category such that each object $a$
    has a left adjoint $a^l$ and a right adjoint $a^r$. In other words there are
    morphisms $a^l \otimes a \to 1$, $1 \to a \otimes a^l$,
    $a \otimes a^r \to 1$ and $1 \to a^r \otimes a^l$, denoted as cups and caps
    and satisfying the snake equations:
    \begin{equation}\label{axioms-rigid}
        \input{figures/rigid.tikz}
    \end{equation}
    The category of rigid categories and monoidal functors is denoted $\bf{RigidCat}$.
\end{definition}
\begin{proposition}
    Rigid categories are biclosed, with $a \backslash b = a^r \otimes b$ and
    $b / a = b \otimes a^l$.
\end{proposition}

Given a rigid signature $\Sigma$ we can generate the \emph{free rigid category}
$$\bf{RC}(\Sigma) = \bf{MC}(\Sigma + \set{cups, caps}) / \sim_{snake}$$
where $\sim_{snake}$ is the equivalence relation on diagrams induced by the
snake equations (\ref{axioms-rigid}). Rigid categories are called
\emph{compact 2-categories} with one object by Lambek and Preller \cite{preller2007b},
who showed that $\bf{RC}$ defines an adjunction between rigid signatures
and rigid categories.
\begin{equation*}
    \bf{RgSig} \mathrel{\mathop{\rightleftarrows}^{\bf{RC}}_{U}} \bf{RgCat}
\end{equation*}

We start by defining a general notion of \emph{rigid grammar}, a subclass of
biclosed grammars.

\begin{definition}
    A rigid grammar $G$ is a rigid signature of the following shape:
    $$ \signature{G}{P(B + V)}$$
    where $V$ is a \emph{vocabulary} and $B$ is a set of \emph{basic types}.
    The language generated by $G$ is given by:
    $$ \cal{L}(G) = \set{u \in V^\ast \, \vert \, \exists g: u \to s \in \bf{RC}(G)}$$
    where $\bf{RC}(G)$ is the free rigid category generated by $G$.
\end{definition}

A pregroup grammar is a lexicalised rigid grammar defined as follows.

\begin{definition}[Pregroup grammar]
    A pregroup grammar is a tuple $G = (V, B, \Delta, I, s)$ where $V$ is a
    \emph{vocabulary}, $B$ is a finite set of basic types, $G \sub V \times P(B)$
    is a \emph{lexicon} assigning a set of possible pregroup types to each word,
    and $I \sub B \times B$ is a finite set of \emph{induced steps}.
    The language generated by $G$ is given by:
    $$\cal{L}(G) = \set{u \in V^\ast \, : \, \exists g \in \bf{RC}(G)(u, s)}$$
    where $\bf{RC}(G) : = \bf{RC}(\Delta + I)$.
\end{definition}

\begin{example}\label{ex:pair-of-lovers}
    Fix the basic types $B = \set{s, n, n_1, d, d_1}$ for sentence, noun, plural
    noun, determinant and plural determinant and consider the following pregroup lexicon:
    $$\Delta(\text{pair}) = \set{d^r \, n}, \: \Delta(\text{lives}) = \set{d_1^r \, n_1},
    \: \Delta(\text{lovers}) = \set{n_1},
    \: \Delta(\text{starcross}) = \set{n\, n^l},$$
    $$\Delta(\text{take}) = \set{n^r \, s \, n^l},
    \quad \Delta(\text{of}) = \set{n^r \, n \, n^l}, \quad
    \Delta(\text{A}) = \set{d}, \quad \Delta(\text{their}) = \set{d_1}.$$
    and one induced step $I = \set{n_1 \to n}$.
    Then the following is a grammatical sentence:
    \begin{equation*}
        \input{figures/pregroup-pair-of-lovers.tikz}
    \end{equation*}
    where we omitted the types for readability, and we denoted the induced step
    using a black node.
\end{example}

The tight connections between categorial grammars and pregroups were discussed
in \cite{buszkowski2007a}, they are evermore apparent from a categorical
perspective: since rigid categories are biclosed,
there is a canonical way of mapping the reductions of a categorial grammar
to pregroups.

\begin{proposition}\label{prop-categorial-to-pregroup}
    For any Lambek grammar $G$ there is a pregroup grammar $G'$ with a
    functorial reduction $\bf{MC}(G) \to \bf{RC}(G')$.
\end{proposition}
\begin{proof}
    The translation works as follows:
    \begin{equation*}
        \input{figures/categorial-to-pregroup.tikz}
    \end{equation*}
\end{proof}

\begin{example}
    Consider the categorial parsing of the sentence ``Grandma will cook the parmigiana''
    from Example \ref{lambek-example}.
    The categorial type of ``will'' is given by
    $(n \backslash s) / (n \backslash s)$ which translates to the pregroup
    type $(n^r s)(n^r s)^l = n^r s\, s^l n$, the transitive verb type
    $(n \backslash s) / n$ for ``cook'' translates to $n^r \, s \, n^l$,
    and similarly for the other types.
    Translating the categorial reduction according to the mapping above, we
    obtain the following pregroup reduction:
    \begin{equation*}
        \scalebox{0.8}{\input{figures/pregroup-from-categorial.tikz}}
    \end{equation*}
\end{example}

One advantage of pregroups over categorial grammars is that they can be parsed
more efficiently. This is a consequence of the following lemma, proved by Lambek
in \cite{lambek1999a}.

\begin{proposition}[Switching lemma]\label{prop-switching}
    For any pregroup grammar $G$ and any reduction $t \to s$ in
    $\bf{RC}(G)$, there is a type $t'$ such that
    $t \to s = t \to t' \to s$ and $t \to t'$ doesn't use contractions (cups),
    $t' \to s$ doesn't use expansions (caps).
\end{proposition}
\begin{remark}
    Note that the equivalent lemma for categorial grammars would state that
    all instances of the type-raising rule can appear after all instances of
    the composition and application rules. This is however not the case.
\end{remark}

A direct corollary of this lemma, is that any sentence $ u \in V^\ast$ may
be parsed using only contractions (cups). This drastically reduces the search space
for a reduction, with the consequence that pregroup grammars can be parsed
efficiently.

\begin{proposition}
    Pregroup grammars can be parsed in polynomial time \cite{degeilh2005, moroz2011}
    and in linear-time in linguistically justified restricted cases \cite{preller2007}.
\end{proposition}


As discussed in Section \ref{section-combinatory}, linguists have observed that
certain grammatical sentences naturally involve crossed dependencies between
words \cite{stabler2004}. Although the planarity assumption is justified in
several cases of interest \cite{cancho2006a}, grammars with crossed dependencies
allow for more flexibility when parsing natural languages.
In order to model these phenomena with pregroup grammars, we need to
step out of (planar) rigid categories and allow for (restricted) permutations.
These can be represented in the symmetric version of rigid categories, known as
\emph{compact closed} categories.

\begin{definition}[Compact-closed]\label{def-compact-closed}
    A compact-closed category is a rigid category (\ref{axioms-rigid}) that is
    also symmetric (\ref{axioms-symmetry}).
\end{definition}

Given a pregroup grammar $G$, we can generate the free compact-closed category
defined by:
$$\bf{CC}(G) = \bf{RC}(G + \tt{swap})/\sim_{\tt{sym}}$$
where $\sim_{\tt{sym}}$ is the congruence induced by the axioms for the symmetry
(\ref{axioms-symmetry}). Notice that in a compact-closed
category $a^r \simeq a^l$ for any object $a$, see e.g. \cite{heunen2019}.
Pregroup reductions in the free rigid category $\bf{RC}(G)$ can of course be
mapped in $\bf{CC}(G)$ via the canonical rigid functor which forgets the
$^r$ and $^l$ adjoint structure.

We cannot use free compact-closed categories directly to parse sentences.
If we only ask for a morphism $g: u \to s$ in $\bf{CC}(G)$ in order to show that
the string $u$ is grammatical, then any permutation of the words in $u$ would
also count as grammatical, and we would lose any notion of word order.
In practice, the use of the swap must be restricted to only special cases.
These were discussed in \ref{sec-closed}, where we saw that
the crossed composition rule of combinatory grammars is suitable for
modeling these restricted cases.

In recent work \cite{yeung2021}, Yeung and Kartsaklis introduced a translation
from CCG grammars to pregroup grammars which allows to build a diagram in the
free compact-closed category over a pregroup grammar from any derivation of a
CCG. This is useful in practical applications since it allows to turn the output
of state-of-the-art CCG parsers such as \cite{yoshikawa2017} into compact-closed
diagrams. The translation is captured by the following proposition.

\begin{proposition}\label{prop-crossed-pregroup}
    For any combinatory grammar $G$ there is a pregroup grammar $G'$ with a
    canonical functorial reduction $G \to \bf{CC}(G')$.
\end{proposition}
\begin{proof}
   The translation is the same as \ref{prop-categorial-to-pregroup}, together
   with the following mapping for the crossed composition rules:
   \begin{equation}
       \scalebox{0.8}{\input{figures/pregroup-cross-dep.tikz}}
   \end{equation}
\end{proof}

Representing the crossed composition rule in compact-closed categories, allows
to reason diagrammatically about equivalent syntactic structures.

\begin{example}
    As an example consider the pregroup grammar $G$ with basic types $B = \set{n, s}$
    and lexicon given by:
    $$\Delta(\text{cooked}) = \set{n^r s n^l} \quad
    \Delta(\text{me}) = \Delta(\text{Grandma}) = \Delta(\text{parmigiana}) = \set{n}
    \quad \Delta(\text{for}) = \set{s^l s n^l}$$
    Then using the grammar from Example \ref{ex-crossed-comp}, we can map
    the two CCG parses to the following compact-closed diagrams in $\bf{CC}(G)$,
    even though the second is not grammatical in a planar pregroup grammar.
    \begin{equation*}
        \input{figures/pregroup-cross-gave.tikz}
    \end{equation*}
    If we interpret the wires for words as the unit of the tensor, then these two
    diagrams are equal in $\bf{CC}(G)$ but not when seen in a biclosed category.
\end{example}

\subsection{Dependency grammars are pregroups}\label{section-dependency}

Dependency grammars arose from the work of Lucien Tesniere in the 1950s
\cite{tesniere1959}. It was made formal in the 1960s by Hays \cite{hays1964} and
Gaifman \cite{gaifman1965}, who showed that they have the same expressive power
as context-free grammars. Dependency grammars are very popular in NLP,
supported by large-scale parsing tools such as those provided by Spacy
\cite{spacy}. We take the formalisation of Gaifman \cite{gaifman1965} as a starting
point and show how the dependency relation may be seen as a diagram in a free
rigid category.

Let us fix a vocabulary $V$ and a set of symbols $B$, called categories in
\cite{gaifman1965}, with $s \in B$ the sentence symbol.

\begin{definition}[Dependency grammar \cite{gaifman1965}]\label{def-dependency-grammar}
    A dependency grammar $G$ consists in a set of rules $G \sub (B + V) \times B^\ast$
    of the following shapes:
    \begin{enumerate}[I]
        \item $(x , y_1 \dots y_l \star y_{l+1} \dots y_n)$ where $x, y_i \in B$,
        indicating that the symbol $x$ may depend on the symbols $y_1 \dots y_l$
        on the left and on the symbols $y_{l+1} \dots y_n$ on the right.
        \item $(w, x)$ for $w \in V$ and $x \in B$, indicating
        that the word $w$ may have type $x$.
        \item $(x, s)$ indicating that the symbol $x$ may govern a sentence.
    \end{enumerate}
\end{definition}

Following Gaifman \cite{gaifman1965}, we define the language $\cal{L}(G)$
generated by a dependency grammar $G$ to be the set of strings
$u = w_1 w_2 \dots w_n \in V^\ast$ such that there are
symbols $x_1, x_2 \dots x_n \in B$ and a binary \emph{dependency relation}
$d \sub X \times X$ where $X = \set{x_1, \dots, x_n}$ satisfying
the following conditions:
\begin{enumerate}
    \item $(w_i, x_i) \in G$ for all $i \leq n$,
    \item $(x_i, x_i) \notin RTC(d)$ where $RTC(d)$ is the reflexive transitive
    closure of $d$, i.e. the dependency relation is \emph{acyclic},
    \item if $(x_i, x_j) \in d$ an $(x_i, x_k) \in d$ then $x_j = x_k$,
    i.e. every symbol depends on at most one head, i.e. the dependency relation
    is single-headed or \emph{monogamous},
    \item if $ i \leq j \leq k$ and $(x_i, x_k) \in RTC(d)$ then
    $(x_i, x_j) \in RTC(d)$, i.e. the dependency relation is \emph{planar},
    \item there is exactly one $x_h$ such that $\forall j \, (x_h, x_j) \notin d$
    and $(x_h, s) \in G$, i.e. the relation is \emph{connected} and \emph{rooted}
    ($x_h$ is called the root and we say that $x_h$ governs the sentence).
    \item for every $x_i$, let $y_1, \dots, y_l \in X$ be the (ordered list of)
    symbols which depend on $x_i$ from the left and $y_{l+1}, \dots, y_n \in X$
    be the (ordered list of) symbols which depend on $x_i$ from the right,
    then $(x, y_1 \dots y_l \star y_{l+1} \dots y_n) \in G$, i.e. the dependency
    structure is allowed by the rules of $G$.
\end{enumerate}

\begin{example}
    Consider the dependency grammar with $V = \set{\text{Moses}, \text{crossed}, \text{the},
    \text{Red}, \text{Sea}}$, $B = \set{d, n, a, s, v}$ and
    rules of type I:
    $$(v, n \star n), (a, \star n), (d, \star), (n, \star), (n, a d \star)\, ,$$
    of type II:
    $$ (\text{Moses}, n), (\text{crossed}, v), (\text{the}, d), (\text{Red}, a),
    (\text{Sea}, n)$$
    and a single rule of type III $(v, s)$.
    Then the sentence ``She tied a plastic bag'' is grammatical as witnessed
    by the following dependency relation:
    \begin{equation*}
        \input{figures/dependency.tikz}
    \end{equation*}
\end{example}

This combinatorial definition of a dependency relation has an algebraic counterpart
as a morphism in a free rigid category.
Given a dependency grammar $G$, let $\Delta(G) \sub V \times P(B)$ be the
pregroup lexicon defined by:
\begin{equation}\label{dependency-lexicon}
    (w, y_1^r \dots y_l^r\, x\, y_{l+1}^l \dots y_n^l) \in \Delta(G) \iff (w, x) \in G
    \land (x, y_1 \dots y_l \,\star\, y_{l+1} \dots y_n) \in G
\end{equation}
also, let $I(G)$ be rules in $G$ of the form $(x, s)$ where $x \in B$ and
$s$ is the sentence symbol.

\begin{proposition}\label{prop-dependency-grammar}
    For any dependency grammar $G$, if a string of words is grammatical
    $u \in \cal{L}(G)$ then there exists a morphism $u \to s$ in
    $\bf{RC}(\Delta(G) + I(G))$.
\end{proposition}
\begin{proof}
    Fix a dependency grammar $G$, and suppose $w_1 \dots w_n \in \cal{L}(G)$
    then there is a list of symbols $X = \{ x_1, x_2 \dots x_n \}$ with $x_i \in B$
    and a dependency relation $d \sub X \times X$ such that the conditions
    $(1), \dots, (6)$ given above are satisfied. We need to show that $d$ defines
    a diagram $ w_1 \dots w_n \to s$ in $\bf{RC}(\Delta(G) + I(G))$. Starting
    from the type $ w_1 w_2 \dots w_n$, conditions (1) and (6) of the dependency
    relation ensure that there is an assignment of a single lexical entry in
    $\Delta(G)$ to each $w_i$. Applying these lexical entries we get a morphism
    $ w_1 w_2 \dots w_n \to T$ in $\bf{RC}(\Delta(G))$ where:
    $$T = \otimes_{i=1}^n (y_{i,1}^r \dots y_{i,l_i}^r\, x_i\, y_{i, l_i+ 1}^l \dots y_{i, n_i}).$$
    For each pair $(x_i, x_j) \in d$ with $i \leq j$, $x_i$ must appear as some
    $y_{j, k}$ with $k \leq l_j$ by condition (6). Therefore we can apply
    a cup connecting $x_i$ and $(y_{j, k})^r$ in $T$. Similarly for $(x_i, x_j)
    \in d$ with $j \leq i$, $x_i$ must appear as some $y_{j, k}$ with $k > l_j$ and
    we can apply a cup connecting $(y_{j, k})^l$ and $x_i$ in $T$.
    By monogamy (3) and connectedness (5) of the dependency relation, there is
    exactly one such pair for each $x_i$, except for the root $x_h$.
    Therefore we can keep applying cups until only $x_h$ is left.
    By planarity (4) of the dependency relation, these cups
    don't have to cross, which means that the diagram obtained is a valid morphism
    $ T \to x_h$ in $\bf{RC}(\Delta(G))$. Finally condition (5) ensures that
    there exists an induced step $x_h \to s \in I(G)$. Overall  we have built
    a morphism $ w_1 w_2 \dots w_n \to T \to x_h \to s$ in $\bf{RC}(\Delta(G) + I(G))$,
    as required.
\end{proof}

\begin{corollary}
    For any dependency grammar $G$ there is a pregroup grammar
    $\tilde{G} = (V, B, \Delta(G), I(G), s)$ such that
    $\cal{L}(G) \sub \cal{L}(\tilde{G})$.
\end{corollary}

\begin{example}
    An example of the translation defined above is the following:
    \begin{equation*}
        \scalebox{0.8}{\input{figures/dependency-to-pregroup.tikz}}
    \end{equation*}
\end{example}

The proposition above gives a structural reduction from dependency grammars
to pregroup grammars, where the dependency relation witnessing the grammaticality
of a string $u$ is seen as a pregroup reduction $u \to s$.
This leads to a first question: do all the pregroup reduction arise from
a dependency relation?
In \cite{preller2007}, Preller gives a combinatorial description of pregroup
reductions, which is strikingly similar to the definition of dependency relation.
In particular it features the same conditions for monogamy (3), planarity (4) and
connectedness (5).
However, pregroup reductions are in general \emph{not} acyclic, as the following
example shows:
\begin{equation}\label{ex-cyclic-pregroup}
    \input{figures/pregroup-acyclic.tikz}
\end{equation}
Therefore we do not expect that any given pregroup grammar reduces to a
dependency grammar. The question still remains for pregroup grammars with
restricted types of lexicons. Indeed, the cyclic example above uses a lexicon
which is not of the shape (\ref{dependency-lexicon}).
We define operadic pregroups to be pregroup grammars with lexicon of the shape
(\ref{dependency-lexicon}).

\begin{definition}[Operadic pregroup]
    An operadic pregroup is a pregroup grammar $G = (V, B, \Delta, s)$ such
    that for any lexical entry $(w, t) \in \Delta$ we have $t= y^r x \ z^l$
    for some $x \in B$ and $y, z \in B^\ast$.
\end{definition}

Using Delpeuch's autonomization of monoidal categories \cite{delpeuch2019},
we can show that reductions in an operadic pregroup always form a tree,
justifying the name ``operadic'' for these structures.

\begin{proposition}\label{prop-pregroup-trees}
     Every operadic pregroup is functorially equivalent to a CFG.
\end{proposition}
\begin{proof}
    It will be sufficient to show that the reductions of an operadic pregroup are trees.
    Fix an operadic pregroup $\tilde{G} = (V, B, \Delta, I, s)$.
    We now define a functor $F : \bf{RC}(\tilde{G}) \to \cal{A}(\bf{MC}(G))$
    where $\cal{A}$ is the free rigid (or autonomous) category construction on
    monoidal categories as defined by Delpeuch \cite{delpeuch2019}, and $G$ is a
    context-free grammar with basic symbols $B + V$ and production rules
    $y_1 \dots y_l w y_{l+1} \dots y_n \to x \in G$ whenever
    $(w, y_1^r \dots y_l^r\, x\, y_{l+1}^l \dots y_n^l) \in \tilde{G}$.
    $F$ is given by the identity on objects $F(x) = x$, and on lexical entries
    it is defined by:
    \begin{equation*}
        \input{figures/autonomize-1.tikz}
    \end{equation*}
    As shown by Delpeuch the embedding $\bf{MC}(G) \to \cal{A}(\bf{MC}(G))$ is full
    on the subcategory of morphisms $g: x \to y$ where $x, y \in (B + V)^\ast$.
    Given any pregroup reduction $g: u \to s$ in $\bf{RC}(\tilde{G})$ we may apply
    the functor $F$ to get a morphism $F(g): F(u) \to F(s)$. By definition
    of $F$, we have that $F(u) = u$ and $F(s) = s$ are both elements of
    $(B + V)^\ast$. By fullness of the embedding, all the cups and caps
    in $F(g)$ can be removed using the snake equation, i.e. $F(g) \in \bf{MC}(G)$.
    We give an example to illustrate this:
    \begin{equation*}
        \input{figures/autonomize-2.tikz}
    \end{equation*}
    It is easy to see that the induced monoidal diagram has the same connectivity
    as the pregroup reduction. Since it is a monoidal diagram it must be
    \emph{acyclic}, and since all the boxes have only one ouput it must be
    a \emph{tree}, finishing the proof.
\end{proof}

\begin{proposition}\label{prop-dep2preg}
    Every dependency grammar is struturally equivalent to an operadic pregroup.
\end{proposition}
\begin{proof}
    We need to show that given a dependency grammar $G$ and a string of words $u \in V^\ast$,
    dependency relations for $u$ are in one-to-one correspondence with reductions
    $u \to s$ for the operadic pregroup $\tilde{G} = (V, B, \Delta(G), I(G), s)$.
    The proof of Proposition \ref{prop-dependency-grammar}, gives an injection of
    dependency relations for $u$ to pregroup reductions $u \to s$ for
    $\tilde{G}$.  For the opposite direction note that, by the Lambek switching
    lemma \ref{prop-switching}, any pregroup reduction $u \to s$ can be obtained
    using lexical entries followed by a diagram made only of cups (contractions).
    This defines a relation on $u$ satisfying conditions (1) and (3-6) of
    dependency relations, see \cite{preller2007}. It remains to show that also
    condition (2) is satisfied, i.e. that reductions in an operadic pregroup
    are \emph{acyclic}, but this follows from Proposition \ref{prop-pregroup-trees}.
\end{proof}

\begin{theorem}\label{theorem-dependencies}
    Every dependency grammar is structurally equivalent to both a pregroup and
    a context-free grammar.
\end{theorem}
\begin{proof}
    Follows from Propositions \ref{prop-dep2preg} and \ref{prop-pregroup-trees}.
\end{proof}

Overall, we have three equivalent ways of looking at the structures induced by
dependency grammars a.k.a operadic pregroups.
First, we may see them as \emph{dependency relations} as
first defined by Gaifman \cite{gaifman1965} and reviewed above. Second, we may
see them as \emph{pregroup reductions} (i.e. patterns of cups)
as proven in Proposition \ref{prop-dependency-grammar}.
Third, we may see them as \emph{trees} as shown in Proposition \ref{prop-pregroup-trees}.

On the one hand, this new algebraic perspective will allow us to give functorial
semantics to dependency grammars. In \ref{sec-tensor-network}
we interpret them in rigid categories using their characterization as pregroup
grammars. In \ref{section-wsd}, we interpret them in a category
of probabilistic processes (where cups and caps are not allowed) using the
characterization of dependency relations as trees.
On the other, it allows us to interface DisCoPy with established dependency
parsers such as those provided by SpaCy \cite{spacy}.

\subsection{rigid.Diagram}

The \py{rigid} module of DisCoPy is often used as a suitable
intermediate step between any grammar and tensor-based semantics.
For example, the \py{lambeq} package
\cite{kartsaklis2021} provides a method for generating instances of
\py{rigid.Diagram} from strings parsed using a transformer-based
CCG parser \cite{clark2021}. We describe the implementation of the \py{rigid}
module and construct an interface with SpaCy's dependency parser \cite{spacy}.
A \py{rigid.Ob}, or basic type, is defined by a \py{name} and a winding number
\py{z}. It comes with property methods \py{.l} and \py{.r} for taking left
and right adjoints by acting on the winding integer \py{z}.

\begin{python}\label{listing:rigid.Ob}
{\normalfont Basic types and their iterated adjoints.}

\begin{minted}{python}
class Ob(cat.Ob):
    @property
    def z(self):
        """ Winding number """
        return self._z

    @property
    def l(self):
        """ Left adjoint """
        return Ob(self.name, self.z - 1)

    @property
    def r(self):
        """ Right adjoint """
        return Ob(self.name, self.z + 1)

    def __init__(self, name, z=0):
        self._z = z
        super().__init__(name)
\end{minted}
\end{python}

Types in rigid categories also come with a monoidal product,
We implement them by subclassing \py{monoidal.Ty} and providing the defining
methods of \py{rigid.Ob}, note that taking adjoints reverses the order of objects.

\begin{python}\label{listing:rigid.Ty}
{\normalfont Pregroup types, i.e. types in free rigid categories.}

\begin{minted}{python}
class Ty(monoidal.Ty, Ob):
    @property
    def l(self):
        return Ty(*[x.l for x in self.objects[::-1]])

    @property
    def r(self):
        return Ty(*[x.r for x in self.objects[::-1]])

    @property
    def z(self):
        return self[0].z

    def __init__(self, *t):
        t = [x if isinstance(x, Ob)
             else Ob(x.name) if isinstance(x, cat.Ob)
             else Ob(x) for x in t]
        monoidal.Ty.__init__(self, *t)
        Ob.__init__(self, str(self))
\end{minted}
\end{python}

Rigid diagrams are monoidal diagrams with special morphisms called \py{Cup}
and \py{Cap}, satisfying the snake equations (\ref{axioms-rigid}).
The requirement that the axioms are satisfied is relaxed to the availability of
a polynomial time algorithm for checking equality of morphisms. This is implemented
in DisCoPy, with the \py{rigid.Diagram.normal_fom} method, following the
algorithm of Dunn and Vicary \cite[Definition 2.12]{dunn2019}.
Rigid diagrams are also biclosed, i.e. they can be \py{curry}ed and \py{uncurry}ed.

\begin{python}\label{listing:rigid.Diagram}
{\normalfont Diagrams in free rigid categories.}

\begin{minted}{python}
@monoidal.Diagram.subclass
class Diagram(monoidal.Diagram):
    @staticmethod
    def cups(left, right):
        return cups(left, right)

    @staticmethod
    def caps(left, right):
        return caps(left, right)

    @staticmethod
    def curry(self, n_wires=1, left=False):
        return curry(self, n_wires=n_wires, left=left)

    @staticmethod
    def uncurry(self, n_wires=1, left=False):
        return uncurry(self, n_wires=n_wires, left=left)

    def normal_form(self, normalizer=None, **params):
        ...


class Box(monoidal.Box, Diagram):
    ...


class Id(monoidal.Id, Diagram):
    ...
\end{minted}

Note that currying and uncurrying correspond to transposition of wires in
the rigid setting. The class comes with its own \py{Box} instance which carry
a winding number $\py{_z}$ for their transpositions. The currying and
uncurrying functions are defined as follows.

\begin{minted}{python}
from discopy.rigid import Id, Ty, Box, Diagram

def curry(diagram, n_wires=1, left=False):
    if not n_wires > 0:
        return diagram
    if left:
        wires = diagram.dom[:n_wires]
        return Diagram.caps(wires.r, wires) @ Id(diagram.dom[n_wires:])\
            >> Id(wires.r) @ diagram
    wires = diagram.dom[-n_wires:]
    return Id(diagram.dom[:-n_wires]) @ Diagram.caps(wires, wires.l)\
        >> diagram @ Id(wires.l)

def uncurry(diagram, n_wires=1, left=False):
    if not n_wires > 0:
        return diagram
    if left:
        wires = diagram.cod[:n_wires]
        return Id(wires.l) @ diagram\
            >> Diagram.cups(wires.l, wires) @ Id(diagram.cod[n_wires:])
    wires = diagram.cod[-n_wires:]
    return diagram @ Id(wires.r)\
        >> Id(diagram.cod[:-n_wires]) @ Diagram.cups(wires, wires.r)
\end{minted}
\end{python}

We only showed the main methods available with rigid diagrams. The DisCoPy
implementation also comes with classes \py{Cup} and \py{Cap} for representing
the structural morphisms.
This allows to define \py{rigid.Functor} as a monoidal Functor with a
predefined mapping on instances of \py{Cup} and \py{Cap}.

\begin{python}\label{listing:rigid.Functor}
{\normalfont Functors from free rigid categories.}

\begin{minted}{python}
class Functor(monoidal.Functor):
    def __init__(self, ob, ar, cod=(Ty, Diagram)):
        super().__init__(ob, ar, cod=cod)

    def __call__(self, diagram):
        if isinstance(diagram, Ty):
            ...
        if isinstance(diagram, Cup):
            return self.cod[1].cups(
                self(diagram.dom[:1]), self(diagram.dom[1:]))
        if isinstance(diagram, Cap):
            return self.cod[1].caps(
                self(diagram.cod[:1]), self(diagram.cod[1:]))
        if isinstance(diagram, Box):
            ...
        if isinstance(diagram, monoidal.Diagram):
            return super().__call__(diagram)
        raise TypeError()
\end{minted}
\end{python}

We build an interface with the dependency parser of SpaCy \cite{spacy}. From a SpaCy
dependency parse we may obtain both an \py{operad.Tree} and a \py{rigid.Diagram}.

\begin{python}\label{listing:operad.SpaCy}
{\normalfont Interface between \py{spacy} and \py{operad.Tree}}

\begin{minted}{python}
from discopy import operad

def find_root(doc):
    for word in doc:
        if word.dep_ == 'ROOT':
            return word

def doc2tree(word):
    if not word.children:
        return operad.Box(word.text, operad.Ob(word.dep_), [])
    root = operad.Box(word.text, operad.Ob(word.dep_),
                      [operad.Ob(child.dep_) for child in word.children])
    return root(*[doc2tree(child) for child in word.children])

def from_spacy(doc):
    root = find_root(doc)
    return doc2tree(root)
\end{minted}
\begin{minted}{python}
import spacy
nlp = spacy.load("en_core_web_sm")
doc = nlp("Moses crossed the Red Sea")
assert str(from_spacy(doc)) = 'crossed(Moses, Sea(the, Red))'
\end{minted}
\end{python}

\begin{python}\label{listing:rigid.SpaCy}
{\normalfont Interface between \py{spacy} and \py{rigid.Diagram}}

\begin{minted}{python}
from discopy.rigid import Ty, Id, Box, Diagram, Functor

def doc2rigid(word):
    children = word.children
    if not children:
        return Box(word.text, Ty(word.dep_), Ty())
    left = Ty(*[child.dep_ for child in word.lefts])
    right = Ty(*[child.dep_ for child in word.rights])
    box = Box(word.text, left.l @ Ty(word.dep_) @ right.r, Ty(),
              data=[left, Ty(word.dep_), right])
    top = curry(curry(box, n_wires=len(left), left=True), n_wires=len(right))
    bot = Id(Ty()).tensor(*[doc2rigid(child) for child in children])
    return top >> bot

def doc2pregroup(doc):
    root = find_root(doc)
    return doc2rigid(root)
\end{minted}

We can now build pregroup reductions from dependency parses.
We check that the outputs of the two interfaces are functorially equivalent.

\begin{minted}{python}
def rewiring(box):
    if not box.data:
        return Box(box.name, box.dom, Ty())
    left, middle, right = box.data[0], box.data[1], box.data[2]
    new_box = Box(box.name, middle, left @ right)
    return uncurry(uncurry(new_box, len(left), left=True), len(right))

F = Functor(ob=lambda ob: Ty(ob.name), ar=rewiring)
assert repr(F(doc2pregroup(doc)).normal_form()) ==\
    repr(operad.tree2diagram(from_spacy(doc)))

drawing.equation(dep2pregroup(doc).normal_form(left=True),
                 operad.tree2diagram(from_spacy(doc)), symbol='->')
\end{minted}

\begin{center}
    \includegraphics[scale=0.55]{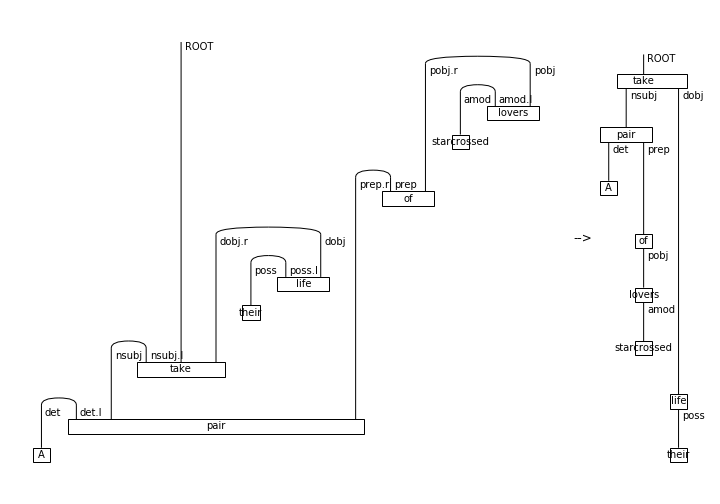}
\end{center}
\end{python}

\section{Hypergraphs and coreference} \label{sec-hypergraph} 

Coreference resolution is the task of finding all linguistic expressions, called
mentions, which refer to the same entity in a piece of text.
It has been a core research topic in
NLP \cite{elango2006}, including early syntax-based models of pronoun resolution
\cite{hobbs1978}, Bayesian and statistical approaches \cite{ge1998, clark2015}
as well as neural-based models \cite{clark2016, lee2017}. This is still a
very active area of research with new state-of-the-art models released every
year, and several open-source tools available in the web \cite{stanza2020}.

In the previous sections, we studied a range of formal grammars that capture
the syntactic structure of sentences. The aim of this section is
to cross the sentence boundary and move towards an analysis of text and
discourse. Assuming that the resolution process has been completed, we want a
suitable syntactic representation of text with coreference.
We can obtain it using a piece of structure known as a commutative
special Frobenius algebra, or more simply a ``spider''. These were introduced in
linguistics by Sadrzadeh et al. \cite{sadrzadeh2013, sadrzadeh2014} as a model
for relative pronouns, and have recently been used by Coecke \cite{coecke2020b}
to model the interaction of sentences within text.

In this section, we introduce pregroup grammars with coreference, a syntactic
model which allows to represent the grammatical and referential structure of text
diagrammatically. This is similar in spirit to the work of Coecke \cite{coecke2020b},
although our approach preserves the pregroup formalism and adds coreference as extra structure.
This makes our model suited for implementation since one can first parse
sentences with a pregroup grammar and then link the entities together using coreference
resolution tools. We show a proof-of-concept implementation using the \py{hypergraph}
module of DisCoPy.

\subsection{Hypergraph categories}\label{section-hypergraph}

The term ``hypergraph categories'' was introduced in 2018 by Fong and Spivak
\cite{fong2018a}, to refer to categories equipped with \emph{Frobenius algebras}
on every object. These structures were studied at least since Carboni and
Walters \cite{carboni1987} and have been applied to such diverse fields as
databases \cite{BonchiEtAl18}, control theory \cite{baez2014a},
quantum computing \cite{coecke2017a} and linguistics \cite{sadrzadeh2014}.
In a recent line of work \cite{bonchi2016d, bonchi2020, bonchi2021},
Bonchi, Sobocinski et al. developed a rewrite theory for
morphisms in these categories in terms of double-pushout hypergraph rewriting
\cite{bonchi2020}. This makes apparent the combinatorial nature of hypergraph
categories, making them particularly suited to implementation \cite{wilson2021a}.
We will not be interested here in the rewrite theory for these categories, but
rather in their power in representing the grammatical and coreferential structure of
language. They will also provide us with an intermediate step between syntax
and the semantics of Sections \ref{sec-rel-model} and \ref{sec-tensor-network}.

\begin{definition}[Hypergraph category]\cite{fong2018a}
    A hypergraph category is a symmetric monoidal category such that each
    object $a$ is equipped with a commutative special Frobenius algebra
    $\tt{Frob}_a = \{ \Delta_a, \epsilon_a, \nabla_a, \eta_a\}$
    satisfying the following axioms:
    \begin{equation}\label{axioms-hypergraph}
        \input{figures/frobenius.tikz}
    \end{equation}
    Where the unlabeled wire denotes object $a$.
\end{definition}
\begin{proposition}
    Hypergraph categories are self-dual compact-closed with cups and caps given by:
    \begin{equation}\label{def-frobenius-cups}
        \input{figures/frobenius-cups.tikz}
    \end{equation}
\end{proposition}

The axioms of commutative special Frobenius algebras (\ref{axioms-hypergraph}),
may be expressed in a more intuitive way as fusion rules of \emph{spiders}.
Spiders are defined as follows:
\begin{equation}
    \scalebox{0.8}{\input{figures/spiders.tikz}}
\end{equation}
They are the normal form of commutative special Frobenius algebras. More precisely,
using the axioms (\ref{axioms-hypergraph}), it can be shown that any connected
diagram built using the Frobenius generators $\tt{Frob}_a$
can be rewritten into the right-hand side above. This was shown in the context
of categorical quantum mechanics \cite{heunen2019} where Frobenius algebras,
corresponding to ``observables'', play a foundational role. The following
result is also proved in \cite{heunen2019}, and used in the context of the ZX calculus \cite{vandewetering2020}, it provides a more concise and intuitive way of reasoning
with commutative special Frobenius algebras.

\begin{proposition}[Spider fusion]\cite{heunen2019}
    The axioms of special commutative Frobenius algebras (\ref{axioms-hypergraph}),
    are equivalent to the spider fusion rules:
    \begin{equation}\label{spider-fusion}
        \scalebox{0.6}{\input{figures/spider-fusion.tikz}}
    \end{equation}
\end{proposition}

Given a monoidal signature $\Sigma$, the free hypergraph category is defined by:
$$\bf{Hyp}(\Sigma) = \bf{MC}(\Sigma + \set{\tt{Frob}_a}_{a \in \Sigma_0})/ \cong$$
where $\bf{MC}$ is the free monoidal category construction, defined in
\ref{section-diagrams} and $\cong$ is the equivalence relation generated
by the axioms of commutative special Frobenius algebras (\ref{axioms-hypergraph}),
or equivalently the spider fusion rules (\ref{spider-fusion}).
Without loss of generality, we may assume that the monoidal signature $\Sigma$
has trivial $\tt{dom}$ function. Formally we have that for any monoidal signature
$\signature{\Sigma}{B^\ast}$ there is a monoidal signature of the form
$\Sigma' \xto{\tt{cod}} B^\ast$ such that $\bf{Hyp}(\Sigma) \simeq \bf{Hyp}(\Sigma')$.
The signature $\Sigma'$ is defined by taking the \emph{name} of every generator
in $\Sigma$, i.e. turning all the inputs into outputs using caps as follows:
\begin{equation*}
    \input{figures/signature-bend-wires.tikz}
\end{equation*}
We define a hypergraph signature as a monoidal signature where the generators
have only output types.

\begin{definition}[Hypergraph signature]
    A hypergraph signature $\Sigma$ over $B$ is a set of hyperedge symbols $\Sigma$
    together with a map $\sigma : \Sigma \to B^\ast$.
\end{definition}

 Morphisms in $\bf{Hyp}(\Sigma)$ for a hypergraph signature $\Sigma$ have a
 normal form given as follows.

\begin{proposition}[Hypergraph normal form]\label{prop-hyp-normal-form}
    Let $\Sigma \to B^\ast$ be a hypergraph signature and $x, y \in B^\ast$.
    Any morphism $d : x \to y \in \bf{Hyp}(\Sigma)$ is equal to a diagram of
    the following shape:
    \begin{equation*}
        \scalebox{0.7}{\input{figures/hyp-normal-form.tikz}}
    \end{equation*}
    where $H(d) \in \bf{Hyp}(\varnothing)$ is a morphism built using only
    spiders, i.e. generated from $\cup_b\tt{Frob}_b$ for $b \in B$.
\end{proposition}

This normal form justifies the name ``hypergraph'' for these categories. Indeed
we may think of the spiders of a diagram $d$ as \emph{vertices}, and boxes with
$n$ ports as \emph{hyperedges} with $n$ vertices. Then the morphism
$H(d)$ defined above is the \emph{incidence graph} of $d$, indicating for each
vertex (spider) the hyperedges (boxes) that contain it.
Note however, that morphisms in $\bf{Hyp}(\Sigma)$ carry more data than a simple
hypergraph. First, they carry labels for every hyperedge and every vertex. Second
they are ``open'', i.e. they carry input-output information allowing to compose
them. If we consider morphisms $d: 1 \to 1$ in $\bf{Hyp}(\Sigma)$, these are
in one-to-one correspondence with hypergraphs labeled by $\Sigma$, as
exemplified in the following picture:

\begin{equation}
    \input{figures/hypergraph-to-incidence-graph.tikz}
\end{equation}

\subsection{Pregroups with coreference}\label{section-coreference}

Starting from a pregroup grammar $G = (V, B, \Delta, I, s)$, we can model
coreference by adding memory or \emph{reference types} for each lexical
entry, and rules that allow to swap, merge or discard these reference types.
Formally, this is done by fixing a set of reference types $R$ and allowing the
lexicon to assign reference types alongside pregroup types:
$$\Delta \sub V \times P(B) \times R^\ast$$
We can then represent the derivations for such a grammar with coreference in one
category $\bf{Coref}(\Delta, I)$ defined by:
$$\bf{Coref}(\Delta, I) = \bf{RC}(\Delta + I + \set{\tt{Frob}_r}_{r \in R}, \set{\tt{swap}_{r, x}}_{r \in R\, , \, x \in P(B) + R})$$
where $\tt{Frob}_r$ contains the generators of Frobenius algebras for every
reference type $r \in R$ (allowing to initialise, merge, copy and delete them),
$\tt{swap}_{r, x}$ allows to swap any reference type $r \in R$ to the right of
any other type $x \in P(B) + R$, i.e. reference types commute with any other object.
Note that we are not quotienting $\bf{Coref}$ by any axioms since we use it
only to represent the derivations.

\begin{definition}[Pregroup with coreference]
    A pregroup grammar with coreference is a tuple $G = (V, B, R, I, \Delta, s)$
    where $V$ is a vocabulary, $B \ni s$ is a set of basic types, $R$ is a set of
    reference types, $I \sub B \times B$ is a set of induced steps,
    $\Delta \sub V \times P(B) \times R^\ast$ is a lexicon,
    assigning to every word $w \in V$ a set of types $\Delta(w)$ consisting of a
    pregroup type in $P(B)$ and a list of reference type in $R^\ast$.
    An utterance $u = w_0 \dots w_n \in V^\ast$ is
    $k$-grammatical in $G$ if there are types
    $(t_i , r_i) \in \Delta(w_i) \sub P(B) \times R^\ast$ such that
    there is morphism
    $t_0 \otimes r_0 \otimes t_1 \otimes r_1 \dots t_n \otimes r_n \to s^k$ in
    $\bf{Coref}(\Delta, I) =: \bf{Coref}(G)$ for some $k \in \bb{N}$.
\end{definition}

\begin{example}
    We denote reference types with red wires. The following is a valid reduction
    in a pregroup grammar with coreference, obtained from Example \ref{ex:pair-of-lovers} by adding reference types for the words ``A'', ``whose'' and ``their''.
    \begin{equation*}
        \scalebox{0.7}{\input{figures/a-pair-of-lovers.tikz}}
    \end{equation*}
\end{example}

We can now consider the problem of parsing a pregroup grammar with coreference.

\begin{definition}
    \begin{problem}
      \problemtitle{$\tt{CorefParsing}$}
      \probleminput{$G$, $u \in V^\ast$, $k \in \bb{N}$}
      \problemoutput{$f \in \bf{Coref}(G)(u, s^k)$}
    \end{problem}
\end{definition}

Note that, seen as a decision problem, $\tt{CorefParsing}$ is equivalent to
the parsing problem for pregroups.

\begin{proposition}
    The problem $\exists\tt{CorefParsing}$ is equivalent to the parsing problem
    for pregroup grammars.
\end{proposition}
\begin{proof}
    First note that for any pregroup grammar with coreference
    $G = (V, B, R, I, \Delta, s)$ there is aa corresponding pregroup grammar
    $\tilde{G} = (V, B, I, \tilde{\Delta}, s^k)$ where $\tilde{\Delta}$ is obtained
    from $\Delta$ by projecting out the $R^\ast$ component. Suppose there is a
    reduction $d: u \to s^k$ in $\bf{Coref}(G)$, since there are no boxes making
    pregroup and reference types interact, we can split the diagram $d$ into
    a pregroup reduction in $\bf{RC}(G)$ and a coreference resolution
    $R^n \to 1$ where $n$ is th number of reference types used. Therefore $u$
    is also grammatical in $\tilde{G}$. Now, $\tilde{G}$ reduces functorially
    to $G$ since we can map the lexical entries in $\tilde{\Delta}$ to the corresponding entry in $\Delta$ followed by discarding the reference types with
    counit spiders. Therefore any parsing for $\tilde{G}$ induces a parsing
    for $\bf{Coref}(G)$, finishing the proof.
\end{proof}

Supposing we can parse pregroups efficiently, $\tt{CorefParsing}$ becomes only
interesting as a function problem: which coreference resolution should we choose
for a given input utterance $u$?
There is no definite mathematical answer to this question. Depending on
the context, the same utterance may be resolved in different ways. Prominent
current approaches are neural-based such as the mention-ranking coreference models
of Clark and Manning \cite{clark2016a}. SpaCy offers a \py{neuralcoref} package
implementing their model and we show how to interface it with the \py{rigid.Diagram}
class at the end of this section. Being able to represent the coreference resolution
in the diagram for a piece of text is particularly useful in semantics, for example
it allows to build arbitrary conjunctive queries as one-dimensional strings of words,
see \ref{sec-rel-model}.

The image of the parsing function for pregroups with coreference may indeed be
arbitrarily omplicated, in the sense that any hypergraph $d \in \bf{Hyp}(\Sigma)$
can be obtained as a $k$-grammatical reduction, where
the hyperedges in $\Sigma$ are seen as \emph{verbs} and the vertices in $B$
as \emph{nouns}.

\begin{proposition}\label{prop-svo-argument}
    For any finite hypergraph signature $\Sigma \xto{\sigma} B^\ast$, there is a
    pregroup grammar with coreference $G = (V, B, R, \Delta(\sigma), s)$
    and a full functor $J: \bf{Coref}(\Delta(\sigma)) \to \bf{Hyp}(\Sigma)$
    with $J(w) = J(s) = 1$ for $w \in V$.
\end{proposition}
\begin{proof}
    We define the vocabulary $V = \Sigma + B$ where the elements of $\Sigma$ are
    seen as verbs and the elements of $B$ as nouns. We also set the basic types
    to be $B + \set{s}$ and the reference types to be $R = B$.
    The lexicon $\Delta(\sigma) \sub V \times P(B + \set{s}) \times R^\ast$
    is defined by:
    $$\Delta(\sigma) = \set{(w, s\,\sigma(w)^l, \epsilon) \s\vert\s w \in \Sigma \sub V }
    +  \set{(w, w, w) \vert w \in B \sub V }$$
    where $\epsilon$ denotes the empty list.
    The functor $J : \bf{Coref}(\Delta(\sigma)) \to \bf{Hyp}(\Sigma)$ is given
    on objects by $J(w) = 1$ for $w \in V$, $J(x) = x$ for $x \in B + R$ and
    $J(s) = 1$ and on the lexical entries by $J(w) = w$ if $w \in \Sigma$
    and $J(w) = \tt{cup}_w$ if $w \in B$.

    In order to prove the proposition, we show that for
    any $d: 1 \to 1 \in \bf{Hyp}(\Sigma)$ there is an
    utterance $u \in V^\ast$ with a $k$-grammatical reduction $g: u \to s^k$ in
    $\bf{Coref}(\Delta(\sigma))$ such that $J(g) = d$, where $k$ is
    the number of boxes in $d$.
    Fix any diagram $d : 1 \to 1 \in \bf{Hyp}(\Sigma)$. By proposition
    \ref{prop-hyp-normal-form}, it can be put in normal form:
    \begin{equation*}
        \input{figures/svo-argument-1.tikz}
    \end{equation*}
    Where $f_1 \dots f_k$ are the boxes in $d$ and $H(d)$ is a morphism built
    only from the Frobenius generators $\tt{Frob}_b$ for $b \in B$.
    Let $\sigma(f_i) = b_{i0} \dots b_{in_i}$, then for every $i \in \set{1, \dots k}$,
    we may build a sentence as follows (leaving open the reference wires):
    \begin{equation*}
        \input{figures/svo-argument-2.tikz}
    \end{equation*}
    Tensoring these sentences together and connecting them with coreference
    $H(d)$ we get a morphism $g: u \to s^k$ in $\bf{Coref}(\Delta(\sigma))$
    where $u = f_1 \sigma(f_1) \dots f_k \sigma(f_k)$, and it is easy to check
    that $J(g) = d$.
\end{proof}

\subsection{hypergraph.Diagram}

The \py{hypergraph} module of DisCoPy, still under development, is an implementation
of diagrams in free hypergraph categories based on \cite{bonchi2020, bonchi2021}.
The main class \py{hypergraph.Diagram} is well documented and it comes with
a method for composition implemented using pushouts.
We give a high-level overview of the datastrutures of this module, with a
proof-of-concept implementation of dependency grammars with coreference.
Before introducing diagrams, recall that the types of free hypergraph categories
are \emph{self-dual} rigid types.

\begin{minted}{python}
class Ty(rigid.Ty):
    @property
    def l(self):
        return Ty(*self.objects[::-1])

    r = l
\end{minted}

We store a \py{hypergraph.Diagram} via its incidence graph. This is given by
a list of \py{boxes} and a list of integers \py{wires} of the same length as ports.

\begin{minted}{python}
port_types = list(map(Ty, self.dom)) + sum(
    [list(map(Ty, box.dom @ box.cod)) for box in boxes], [])\
    + list(map(Ty, self.cod))
\end{minted}

Note that \py{port_types} are the concatenation of the domain of the diagram, the pairs of
domains and codomains of each box, and the codomain od the diagram.
We see that \py{wires} defines a mapping from ports to spiders, which we store as a
list of \py{spider_types}.

\begin{minted}{python}
spider_types = {}
for spider, typ in zip(wires, port_types):
    if spider in spider_types:
        if spider_types[spider] != typ:
            raise AxiomError
    else:
        spider_types[spider] = typ
spider_types = [spider_types[i] for i in sorted(spider_types)]
\end{minted}

Thus a \py{hypergraph.Diagram} is initialised by a domain \py{dom}, a codomain
\py{cod}, a list of boxes \py{boxes} and a list of wires \py{wires} as characterised
above. The \py{__init__} method automatically computes the \py{spider_types},
and in doing so checks that the diagram is well-typed.

\begin{python}\label{listing:hypergraph.Diagram}
{\normalfont Diagrams in free hypergraph categories}

\begin{minted}{python}
class Diagram(cat.Arrow):
    def __init__(self, dom, cod, boxes, wires, spider_types=None):
        ...

    def then(self, other):
        ...

    def tensor(self, other=None, *rest):
        ...

    __matmul__ = tensor

    @property
    def is_monogamous(self):
        ...

    @property
    def is_bijective(self):
        ...

    @property
    def is_progressive(self):
        ...

    def downgrade(self):
        ...

    @staticmethod
    def upgrade(old):
        return rigid.Functor(
            ob=lambda typ: Ty(typ[0]),
            ar=lambda box: Box(box.name, box.dom, box.cod),
            ob_factory=Ty, ar_factory=Diagram)(old)

    def draw(self, seed=None, k=.25, path=None):
        ...
\end{minted}

The current \py{draw} method is based on a random spring layout algorithm of the hypergraph.
One may check whether a diagram is symmetric, compact-closed or traced, using methods
\py{is_progressive}, \py{is_bijective} and \py{is_monogamous} respectively.
The \py{downgrade} method turns any \py{hypergraph.Diagram} into a \py{rigid.Diagram}.
This is by no means an optimal algorithm.
There are indeed many ways to tackle the problem of extraction of rigid diagrams
from hypergraph diagrams.
\end{python}

The classes \py{hypergraph.Box} and \py{hypergraph.Id} are defined as usual,
by suclassing the corresponding \py{rigid} classes and \py{Diagram}.
Two special types of morphisms, \py{Spider} and \py{Swap}, can be defined
directly as subclasses of \py{Diagram}.

\begin{python}\label{listing:hypergraph.Spider}
{\normalfont Swaps and spiders in free hypergraph categories}

\begin{minted}{python}
class Swap(Diagram):
    """ Swap diagram. """
    def __init__(self, left, right):
        dom, cod = left @ right, right @ left
        boxes, wires = [], list(range(len(dom)))\
            + list(range(len(left), len(dom))) + list(range(len(left)))
        super().__init__(dom, cod, boxes, wires)


class Spider(Diagram):
    """ Spider diagram. """
    def __init__(self, n_legs_in, n_legs_out, typ):
        dom, cod = typ ** n_legs_in, typ ** n_legs_out
        boxes, spider_types = [], list(map(Ty, typ))
        wires = (n_legs_in + n_legs_out) * list(range(len(typ)))
        super().__init__(dom, cod, boxes, wires, spider_types)
\end{minted}
\end{python}

We now show how to build hypergraph diagrams from SpaCy's dependency parser and
the coreference information provided by the package \py{neuralcoref}.

\begin{python}\label{listing:hypergraph.coref}
{\normalfont Coreference and hypergraph diagrams}

\begin{minted}{python}
import spacy
import neuralcoref

nlp = spacy.load('en')
neuralcoref.add_to_pipe(nlp)
doc1 = nlp("A pair of starcross lovers take their life")
doc2 = nlp("whose misadventured piteous overthrows doth with \
            their death bury their parent's strife.")
\end{minted}

We use the interface with SpaCy from the \py{operad} module to extract dependency
parses and a \py{rigid.Functor} with codomain \py{hypergraph.Diagram} to
turn the dependency trees into hypergraphs.

\begin{minted}{python}
from discopy.operad import from_spacy, tree2diagram
from discopy.hypergraph import Ty, Id, Box, Diagram
from discopy import rigid

F = rigid.Functor(ob=lambda typ: Ty(typ[0]),
                  ar=lambda box: Box(box.name, box.dom, box.cod),
                  ob_factory=Ty, ar_factory=Diagram)

text = F(tree2diagram(from_spacy(doc1, lexicalised=True)))\
    @ F(tree2diagram(from_spacy(doc2, lexicalised=True)))
assert text.is_monogamous
text.downgrade().draw(figsize=(10, 7))
\end{minted}

We can now use composition in \py{hypergraph} to add coreference boxes that link
leaves according to the coreference \py{clusters}.

\begin{minted}{python}
from discopy.rigid import Ob

ref = lambda x, y: Box('Coref', Ty(x, y), Ty(x))

def coref(diagram, word0, word1):
    pos0 = diagram.cod.objects.index(word0)
    pos1 = diagram.cod.objects.index(word1)
    swaps = Id(Ty(word0)) @ \
        Diagram.swap(diagram.cod[pos0 + 1 or pos1:pos1], Ty(word1))
    coreference = swaps >> \
        ref(word0, word1) @ Id(diagram.cod[pos0 + 1 or pos1:pos1])
    return diagram >> Id(diagram.cod[:pos0]) @ coreference @ \
        Id(diagram.cod[pos1 + 1 or len(diagram.cod):])

def resolve(diagram, clusters):
    coref_diagram = diagram
    for cluster in clusters:
        main = str(cluster.main)[0]
        for mention in cluster.mentions[1:]:
            coref_diagram = coref(coref_diagram, Ob(main), Ob(str(mention)))
    return coref_diagram

doc = nlp("A pair of starcross lovers take their life, whose misadventured \
    piteous overthrows doth with their death bury their parent's strife.")
clusters = doc._.coref_clusters
resolve(text, clusters).downgrade().draw(figsize=(11, 9), draw_type_labels=False)
\end{minted}

\begin{center}
    \includegraphics[scale=0.65]{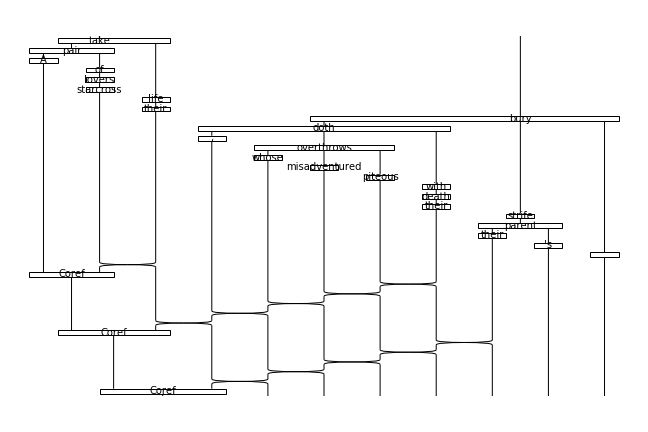}
\end{center}
\end{python}

Note that the only way we have so far of applying functors to a
\py{hypergraph.Diagram} is by first downgrading it to a \py{rigid.Diagram}.
An important direction of future work is the implementation of double
pushout rewriting for hypergraph diagrams \cite{bonchi2020, bonchi2021}.
In particular, this would allow to compute free functors directly on the hypergraph
representation as they are special instances of rewrites.

\chapter{Functors for Semantics}
\renewcommand{\chaptermark}[1]{\markboth{#1}{}}

The modern word ``semantics'' emerged from the linguistic turn of the end of the
19th century along with Peirce's ``semiotics'' and Saussure's ``semiology''.
It was introduced as ``s\'emantique'' by the French linguist Michel Breal,
and has its root in the greek word
$\sigma \eta \mu \breve{\alpha} \sigma \acute{\iota} \bar{\alpha}$
(semasia) which translates to ``meaning'' or ``signification''.
Semantics is the scientific study of language meaning. It thus presupposes
a definition of language as a system of signs, a definition of meaning
as a mental or computational process, and an understanding of how signs are
mapped onto their meaning.
The definition and analysis of these concepts is a problem
which has motivated the work of linguists, logicians and computer scientists
throughout the 20th century.

In his work on ``The Semantic Conception'' \cite{tarski1936, tarski1943},
Alfred Tarski proposed to found the science of semantics on the concept of
\emph{truth} relative to a \emph{model}.
As a mathematician, Tarski focused on the formal language of
logic and identified the conditions under which a logical formula
$\phi$ is true in a model $K$.
The philosophical intuition and mathematical tools
developed by Tarski and his collaborators had a great impact on
linguistics and computer science. They form the basis of Davidson's
\cite{davidson1967, davidson1967a} truth-conditional semantics, as well
as Montague's ``Universal Grammar'' \cite{montague1970, montague1970a, Montague73},
which translates natural language sentences into logical formulae.
They are also at the heart of the development of relational databases in the 1970s
\cite{Codd70, ChamberlinBoyce76, ChandraMerlin77}, where formulas are used as
queries for a structured storage of data.
These approaches adhere to the principle of \emph{compositionality}, which we
may sum up in Frege's words: ``The possibility of our understanding sentences
which we have never heard before rests evidently on this, that we can construct
the sense of a sentence out of individual parts which correspond to words''
\cite{frege1914}. In Tarski's approach, compositionality manifests itself in the notion of \emph{satisfaction} for a formula $\phi$ in a model $K$, which is defined
by induction over the formal grammar from which $\phi$ is constructed.
$$ \text{Formulas} \xto{\text{Model}} \text{Truth} $$

In his 1963 thesis \cite{Lawvere63}, Lawvere introduced the concept of
\emph{functorial semantics} as a foundation for universal algebra.
The idea is to represent syntax as a free category with products and semantics
as a structure-preserving functor computing the meaning of a compound algebraic
expression from the semantics of its basic operations.
The functorial approach to semantics is naturally compositional, but it
generalises Tarski's set-theoretic approach by allowing the semantic category
to represent processes in different models of computation.
For example, taking semantics in the category of complex vector spaces and linear
maps allows to build a structural understanding of quantum information protocols
\cite{abramsky2007}. The same principles are used in applications of category
theory to probability \cite{fong2013, cho2019}, databases \cite{spivak2012},
chemistry \cite{baez2017} and network theory \cite{baez2018}.
Of particular interest to us are the Categorical Compositional Distributional
(DisCoCat) models of Coecke et al. \cite{DisCoCat08, DisCoCat11} where functors
are used to compute the semantics of natural language sentences from the distributional
embeddings of their constituent words. As we will see, functors are useful
for both constructing new models of meaning and formalising already existing ones.
$$ \text{Syntax} \xto{\text{Functor}} \text{Semantics}$$

In this chapter, we use functorial semantics to characterise the expressivity and complexity of a number of NLP models, including logical, tensor network, quantum and neural network models, while showing how they are implemented in DisCoPy.
We start in \ref{sec-concrete} by showing how to implement semantic categories
in Python, while introducing the two main concrete categories we are interested in, $\bf{Mat}_\bb{S}$ and $\bf{Set}$, implemented respectively with the Python classes \py{Tensor} and \py{Function}.
In \ref{sec-functional}, we show that Montague semantics is captured by functors
from a biclosed grammar to the category of sets and functions $\bf{Set}$,
through a lambda calculus for manipulating first-order logic formulas.
We then give an implementation of Montague semantics by defining currying and uncurrying methods for \py{Function}.
In \ref{sec-neural-net}, we show that recurrent and recursive neural network
models are functors from regular and monoidal grammars (respectively) to a
category $\bf{NN}$ of neural network architectures. We illustrate this by
building an interface between DisCoPy and Tensorflow/Keras \cite{chollet2015keras}.
In \ref{sec-rel-model}, we show that functors from pregroup grammars to
the category of relations allow to translate natural language
sentences into conjunctive queries for a relational database.
In \ref{sec-tensor-network}, we formalise the relationship between DisCoCat
and tensor networks and use it to derive complexity results for DisCoCat models.
In \ref{sec-quantum-model}, we study the complexity of our recently proposed quantum models for NLP \cite{coecke2020a, meichanetzidis2020a}.
We show how tensor-based models are implemented in just a few lines of DisCoPy code, and we use them to solve a knowledge embedding task \ref{sec-discopy}.

\section{Concrete categories in Python}\label{sec-concrete}

We describe the implementation of the main semantic modules in DisCoPy: \py{tensor} and \py{function}. These consists in classes \py{Tensor} and \py{Function} whose methods carry out numerical computation. \py{Diagram} may then be evaluated using \py{Functor}.
We may start by considering the \emph{monoid}, a
set with a unit and a product, or equivalently a category with one object.
We can implement it as a subclass of \py{cat.Box} by overriding
\py{init}, \py{repr}, \py{then} and \py{id}.
In fact, it is sufficient to provide an additional \py{tensor} method and
we can make \py{Monoid} a subclass of \py{monoidal.Box}. Both \py{then}
and \py{tensor} are interpreted as multiplication, \py{id} as the unit.

\begin{python}
{\normalfont Delooping of a monoid as \py{monoidal.Box}.}

\begin{minted}{python}
from discopy import monoidal
from discopy.monoidal import Ty
from numpy import prod

class Monoid(monoidal.Box):
    def __init__(self, m):
        self.m = m
        super().__init__(m, Ty(), Ty())

    def __repr__(self):
        return "Monoid({})".format(self.m)

    def then(self, other):
        if not isinstance(other, Monoid):
            raise ValueError
        return Monoid(self.m * other.m)

    def tensor(self, other):
        return Monoid(self.m * other.m)

    def __call__(self, *others):
        return Monoid(prod([self.m] + [other.m for other in others]))

    @staticmethod
    def id(x):
        if x != Ty():
            raise ValueError
        return Monoid(1)

assert Monoid(2) @ Monoid.id(Ty()) >> Monoid(5) @ Monoid(0.1) == Monoid(1.0)
assert Monoid(2)(Monoid(1), Monoid(4)) == Monoid(8)
\end{minted}
\end{python}

\begin{remark}
We define semantic classes as subclasses of \py{monoidal.Box} in order for them
to inherit the usual DisCoPy syntax.
This can be avoided by explicitly providing
the definitions \py{__matmul__}, right and left \py{__shift__} as well as
the dataclass methods.
\end{remark}

A weighted context-free grammar (WCFG) is a CFG where every
production rule is assigned a \emph{weight}, scores are assigned to derivations
by multiplying the weights of each production rule appearing in the tree.
This model is equally expressive as probabilistic CFGs and has been
applied to range of parsing and tagging tasks \cite{smith2007}.
WCFGs are simply functors from \py{Tree} into the \py{Monoid} class!
Thus we can define weighted grammars as a subclass of \py{monoidal.Functor}.

\begin{python}\label{listing:functors.WeightedGrammar}
{\normalfont Weighted grammars as \py{Functor}.}

\begin{minted}{python}
from discopy.monoidal import Functor, Box, Id

class WeightedGrammar(Functor):
    def __init__(self, ar):
        ob = lambda x: Ty()
        super().__init__(ob, ar, ar_factory=Monoid)

weight = lambda box: Monoid(0.5)\
    if (box.dom, box.cod) == (Ty('N'), Ty('A', 'N')) else Monoid(1.0)

WCFG = WeightedGrammar(weight)
A = Box('A', Ty('N'), Ty('A', 'N'))
tree = A >> Id(Ty('A')) @ A
assert WCFG(tree) == Monoid(0.25)
\end{minted}
\end{python}

We can now generate trees with NLTK and evaluate them in a weighted CFG.

\begin{python}\label{listing:functors.WCFG}
{\normalfont Weighted context-free grammar.}

\begin{minted}{python}
from discopy.operad import from_nltk, tree2diagram
from nltk import CFG
from nltk.parse import RecursiveDescentParser
grammar = CFG.fromstring("""
S -> VP NP
NP -> D N
VP -> N V
N -> A N
V -> 'crossed'
D -> 'the'
N -> 'Moses'
A -> 'Red'
N -> 'Sea'""")

rd = RecursiveDescentParser(grammar)
parse = next(rd.parse('Moses crossed the Red Sea'.split()))
diagram = tree2diagram(from_nltk(parse))
parse2 = next(rd.parse('Moses crossed the Red Red Red Sea'.split()))
diagram2 = tree2diagram(from_nltk(parse2))
assert WCFG(diagram).m > WCFG(diagram2).m
\end{minted}
\end{python}

Functors into \py{Monoid} are degenerate examples of a larger class of models
called \py{Tensor} models. An instance of \py{Monoid} is in fact a
\py{Tensor} with domain and codomain of dimension $1$.
In \py{Tensor} models of language, words and production rules are upgraded from
just carrying a weight to carrying a tensor. The tensors are multiplied according
to the structure of the diagram.

\subsection{Tensor}

Tensors are multidimensional arrays of numbers that can be multiplied along their indices.
The \py{tensor} module of DisCoPy comes with interfaces with \py{numpy} \cite{harris2020array}, \py{tensornetwork} \cite{roberts2019} and
\py{pytorch} \cite{pytorch2019} for efficient tensor contraction as well as
\py{simpy} \cite{simpy} and \py{jax} \cite{jax2018github}
for computing gradients symbolically and numerically.
We describe the implementation of the semantic class \py{Tensor}.
We give a more in-depth example in \ref{sec-discopy}, after covering the
relevant theory in \ref{sec-rel-model} and \ref{sec-tensor-network}.

A semiring is a set $\bb{S}$ equipped with two binary operations $+$ and
$\cdot$ called addition and multiplication, and two specified elements
$0, 1$ such that $(\bb{S}, +, 0)$ is a commutative monoid, $(\bb{S}, \cdot, 1)$
is a monoid, the multiplication distributes over addition:
\begin{equation*}
    a \cdot (b + c) = a \cdot b + a \cdot c \qquad
    (a + b) \cdot c  = a \cdot c + b \cdot c
\end{equation*}
and multiplication by $0$ annihilates: $a \cdot 0 = 0 = 0 \cdot a$
for all $a, b, c \in \bb{S}$.
We say that $\bb{S}$ is commutative when $a \cdot b = b \cdot a$ for all
$a, b \in \bb{S}$.
Popular examples of semirings are the booleans $\bb{B}$, natural numbers
$\bb{N}$, positive reals $\bb{R}^+$, reals $\bb{R}$ and complex numbers $\bb{C}$.

The axioms of a semiring are the minimal requirements to define matrix
multiplication and thus a category $\bf{Mat}_\bb{S}$ with objects natural numbers
$n, m \in \bb{N}$ and arrows $n \to m$ given by $n \times m$ matrices with
entries in $\bb{S}$. Composition is given by matrix multiplication and identities
by the identity matrix.
For any commutative semiring $\bb{S}$ the category $\bf{Mat}_\bb{S}$ is monoidal
with tensor product $\otimes$ given by the kronecker product of matrices.
Note that $\bb{S}$ must be \emph{commutative} in order for $\bf{Mat}_\bb{S}$ to be monoidal, since otherwise the interchanger law wouldn't hold.
When $\bb{S}$ is non-commutative, $\bf{Mat}_\bb{S}$ is a premonoidal category \cite{power1997}.

\begin{example}
    The category of finite sets and relations $\bf{FRel}$ is isomorphic to
    $\bf{Mat}_\bb{B}$. The category of finite dimensional real vector spaces
    and linear maps is isomorphic to $\bf{Mat}_\bb{R}$.
    The category of finite dimensional Hilbert spaces and linear maps is
    isomorphic to $\bf{Mat}_\bb{C}$.
\end{example}

Matrices $f: 1 \to n_0 \otimes \dots \otimes n_k$ for objects $n_i \in \bb{N}$
are usually called \emph{tensors} with $k$ indices of dimensions $n_i$.
The word tensor emphasizes that this is a $k$ dimensional array of numbers,
while matrices are usually thought of as $2$ dimensional. Thus $\bf{Mat}_\bb{S}$
can be thought of as a category of tensors with specified input and output dimensions.
This gives rise to more categorical structure.
$\bf{Mat}_\bb{S}$ forms a hypergraph category with Frobenius structure
$(\mu, \nu, \delta, \epsilon)$ given by the ``generalised Kronecker delta''
tensors, defined in Einstein's notation by:
\begin{equation}
    \mu_{i, j}^k = \delta_i^{j, k} =
    \begin{cases*}
      1 & if $i = j = k$ \\
      0 & otherwise
    \end{cases*}
    \qquad
    \nu^i = \epsilon_i = 1
\end{equation}
In particular, $\bf{Mat}_\bb{S}$ is compact closed with cups and caps given
by $\mu\epsilon$ and $\delta\nu$. The transpose $f^\ast$ of a matrix $f: n \to m$
is obtained by pre and post composing with cups and caps.
When the semiring $\bb{S}$ is involutive, $\bf{Mat}_\bb{S}$ has moreover
a dagger structure, i.e. an involutive identity on objects contravariant endofunctor
(see the nlab). In the case when $\bb{S} = \bb{C}$ this is given by taking the
conjugate transpose, corresponding to the dagger of quantum mechanics.

The class \py{Tensor} implements the category of matrices in \py{numpy}
\cite{harris2020array},
with matrix multiplication as \py{then} and kronecker product as \py{tensor}.
The categorical structure of $\bf{Mat}_\bb{S}$ translates into methods
of the class \py{Tensor}, as listed below.
The types of the category of tensors are given by tuples
of dimensions, each entry corresponding to a wire. We can implement them as
a subclass of \py{rigid.Ty} by overriding \py{__init__}.

\begin{python}\label{listing:tensor.Dim}
{\normalfont Dimensions, i.e. types of Tensors.}

\begin{minted}{python}
class Dim(Ty):
    @staticmethod
    def upgrade(old):
        return Dim(*[x.name for x in old.objects])

    def __init__(self, *dims):
        dims = map(lambda x: x if isinstance(x, monoidal.Ob) else Ob(x), dims)
        dims = list(filter(lambda x: x.name != 1, dims))  # Dim(1) == Dim()
        for dim in dims:
            if not isinstance(dim.name, int):
                raise TypeError(messages.type_err(int, dim.name))
            if dim.name < 1:
                raise ValueError
        super().__init__(*dims)

    def __repr__(self):
        return "Dim({})".format(', '.join(map(repr, self)) or '1')

    def __getitem__(self, key):
        if isinstance(key, slice):
            return super().__getitem__(key)
        return super().__getitem__(key).name

    @property
    def l(self):
        return Dim(*self[::-1])

    @property
    def r(self):
        return Dim(*self[::-1])
\end{minted}
\end{python}

A \py{Tensor} is initialised by domain and codomain \py{Dim} types and an
array of shape \py{dom @ cod}. It comes with methods \py{then}, \py{tensor}
for composing tensors in sequence and in parallel. These matrix operations can
be performed using \py{numpy}, \py{jax.numpy} \cite{jax2018github} or
\py{pytorch} \cite{pytorch2019} as backend.

\begin{python}\label{listing:tensor.Tensor}
{\normalfont The category of tensors with \py{Dim} as objects.}

\begin{minted}{python}
import numpy as np


class Tensor(rigid.Box):
    def __init__(self, dom, cod, array):
        self._array = Tensor.np.array(array).reshape(tuple(dom @ cod))
        super().__init__("Tensor", dom, cod)

    @property
    def array(self):
        return self._array

    def then(self, *others):
        if self.cod != other.dom:
            raise AxiomError()
        array = Tensor.np.tensordot(self.array, other.array, len(self.cod))\
            if self.array.shape and other.array.shape\
            else self.array * other.array
        return Tensor(self.dom, other.cod, array)

    def tensor(self, others):
        dom, cod = self.dom @ other.dom, self.cod @ other.cod
        array = Tensor.np.tensordot(self.array, other.array, 0)\
            if self.array.shape and other.array.shape\
            else self.array * other.array
        source = range(len(dom @ cod))
        target = [
            i if i < len(self.dom) or i >= len(self.dom @ self.cod @ other.dom)
            else i - len(self.cod) if i >= len(self.dom @ self.cod)
            else i + len(other.dom) for i in source]
        return Tensor(dom, cod, Tensor.np.moveaxis(array, source, target))

    def map(self, func):
        return Tensor(
            self.dom, self.cod, list(map(func, self.array.flatten())))

    @staticmethod
    def id(dom=Dim(1)):
        from numpy import prod
        return Tensor(dom, dom, Tensor.np.eye(int(prod(dom))))

    @staticmethod
    def cups(left, right):
        ...

    @staticmethod
    def caps(left, right):
        ...

    @staticmethod
    def swap(left, right):
        array = Tensor.id(left @ right).array
        source = range(len(left @ right), 2 * len(left @ right))
        target = [i + len(right) if i < len(left @ right @ left)
                  else i - len(left) for i in source]
        return Tensor(left @ right, right @ left,
                      Tensor.np.moveaxis(array, source, target))

Tensor.np = np
\end{minted}
\end{python}

The compact-closed structure of the category of matrices is implemented via
static methods \py{cups} and \py{caps} and \py{swap}.
We check the axioms of compact closed categories (\ref{def-compact-closed}) on
a \py{Dim} object.

\begin{python}\label{listing:tensor.axioms}
{\normalfont Axioms of compact closed categories.}

\begin{minted}{python}
from discopy import Dim, Tensor
import numpy as np

x = Dim(3, 2)
cup_r, cap_r = Tensor.cups(x, x.r), Tensor.caps(x.r, x)
cup_l, cap_l = Tensor.cups(x.l, x), Tensor.caps(x, x.l)
snake_r = Tensor.id(x) @ cap_r >> cup_r @ Tensor.id(x)
snake_l =  cap_l @ Tensor.id(x) >> Tensor.id(x) @ cup_l
assert np.allclose(snake_l.array, Tensor.id(x).array, snake_r.array)

swap = Tensor.swap(x, x)
assert np.allclose((swap >> swap).array, Tensor.id(x @ x).array)
assert np.allclose((swap @ Tensor.id(x) >> Tensor.id(x) @ swap).array,
                   Tensor.swap(x, x @ x).array)
\end{minted}
\end{python}

Functors into \py{Tensor} allow to evaluate any DisCoPy diagram as a tensor network.
They are simply \py{monoidal.Functor}s with codomain \py{(Dim, Tensor)},
initialised by a mapping \py{ob} from \py{Ty} to \py{Dim} and a mapping \py{ar}
from \py{monoidal.Box} to \py{Tensor}.
We can use them to give an example of a DisCoCat model \cite{DisCoCat11} in DisCoPy, these are studied in \ref{sec-rel-model} and \ref{sec-tensor-network} and used for a concrete task in \ref{sec-discopy}.

\begin{python}\label{listing:functors.TTN}
{\normalfont DisCoCat model as \py{Functor}.}

\begin{minted}{python}
from lambeq import BobcatParser
parser = BobcatParser()
diagram = parser.sentence2diagram('Moses crossed the Red Sea')
diagram.draw()
\end{minted}

\begin{center}
    \includegraphics[scale=0.50]{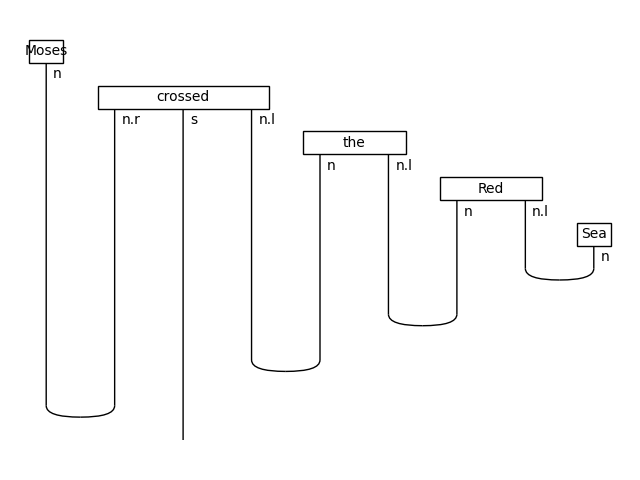}
\end{center}

\begin{minted}{python}
from discopy.tensor import Dim, Tensor, Functor
import numpy as np

def box_to_array(box):
    if box.name =='Moses':
        return np.array([1, 0])
    if box.name == 'Sea':
        return np.array([0, 1])
    if box.name == 'crossed':
        return np.array([[0, 1], [1, 0]])
    if box.name in ['the', 'Red']:
        return np.eye(2)
    raise NotImplementedError()

ob = lambda x: 2 if x == n else 1
F = Functor(ob, box_to_array)
assert F(diagram).array == [1.0]
\end{minted}
\end{python}

\subsection{Function}

\py{Function} is the semantic class that characterises the models
studied in \ref{sec-functional} and \ref{sec-neural-net}.
It consists in an implementation of the category $\bf{Set}$ of sets and functions,
or more precisely, the category of Python functions on tuples with \py{Ty} as objects. We describe the basic methods of \py{Function}, that arise from
the cartesian structure of the category of functions.
In \ref{sec-functional}, we also define \py{curry} and \py{uncurry} methods for \py{Function}, accounting for the biclosed structure of this category.
Alexis Toumi \cite{Toumi22} gives detailed implementation and examples for this
class.

$\bf{Set}$ is a monoidal category with the cartesian
product $\times : \bf{Set} \times \bf{Set} \to \bf{Set}$ as tensor.
It is moreover symmetric, there is a natural transformation
$\sigma_{A, B}: A \times B \to B \times A$ satisfying
$\sigma_{B, A} \circ \sigma_{A, B} = \tt{id}_{A \times B}$ and the axioms
\ref{axioms-symmetry}. $\bf{Set}$ is a \emph{cartesian} category.

\begin{definition}[Cartesian category]\label{def-cartesian}
    A cartesian category $\bf{C}$ is a symmetric monoidal category such that
    the product $\times$ satisfies the following properties:
    \begin{enumerate}
        \item there are projections $A \xleftarrow{\pi_1} A \times B \xrightarrow{\pi_2} B$
        for any $A, B \in \bf{C}_0$,
        \item any pair of function $f: C \to A$ and $g: C \to B$ induces a unique
        function $<f, g> : C \to A \times B$ with $\pi_1 \circ <f, g> = f$ and
        $\pi_2 \circ <f, g> = g$.
    \end{enumerate}
\end{definition}

The structure of the category of functions is orthogonal to the structure of tensors. This may be stated formally as in the following proposition,
shown in the context of the foundations of quantum mechanics.

\begin{proposition}\cite{abramsky2012}
    A compact-closed category that is also cartesian is trivial, i.e. there is
    at most one morphism between any two objects.
\end{proposition}

The main difference comes from the presence of the diagonal map $\tt{copy}$ in  $\bf{Set}$. This is a useful piece of structure which exists in any cartesian category.

\begin{proposition}[Fox \cite{fox1976}]
    In any cartesian category $\bf{C}$ there is a natural transformation
    $\tt{copy}_A: A \to A \otimes A$ satisfying the following axioms:
    \begin{enumerate}
        \item Commutative comonoid axioms:
        \begin{equation}
            \scalebox{0.8}{\input{figures/commutative-comonoid.tikz}}
        \end{equation}
        \item Naturality of copy:
        \begin{equation}
            \scalebox{0.8}{\input{figures/copy.tikz}}
        \end{equation}
    \end{enumerate}
\end{proposition}

This proposition may be used to characterise the \emph{free cartesian category}
$\bf{Cart}(\Sigma)$ from a generating monoidal signature $\Sigma$ as the
free monoidal category with natural comonoids on every object. These were
first studied by Lawvere \cite{Lawvere63} who used them to define algebraic theories as functors.

\begin{definition}[Lawvere theory]
    A Lawvere theory with signature $\Sigma$ is a product-preserving functor
    $F: \bf{Cart}(\Sigma) \to \bf{Set}$.
\end{definition}

We now show how to implement these concepts in Python.
A \py{Function} is initialised by a domain \py{dom} and a codomain \py{cod}
together with a Python function \py{inside} which takes tuples
of length \py{len(dom)} to tuples of length \py{len(cod)}.
The class comes with methods \py{id}, \py{then} and \py{tensor}
for identities, sequential and parallel composition of functions, as well as a
\py{__call__} method which accesses \py{inside}.

\begin{python}\label{listing:pyth.Function}
{\normalfont The category of Python functions on tuples.}

\begin{minted}{python}
class Function(monoidal.Box):
    def __init__(self, inside, dom, cod):
        self.inside = inside
        name = "Function({}, {}, {})".format(inside, dom, cod)
        super().__init__(name, dom, cod)

    def then(self, other):
        inside = lambda *xs: other(*tuple(self(*xs)))
        return Function(inside, self.dom, other.cod)

    def tensor(self, other):
        def inside(*xs):
            left, right = xs[:len(self.dom)], xs[len(self.dom):]
            result = tuple(self(*left)) + tuple(other(*right))
            return (result[0], ) if len(self.cod @ other.cod) == 1 else result
        return Function(inside, self.dom @ other.dom, self.cod @ other.cod)

    def __call__(self, *xs): return self.inside(*xs)

    @staticmethod
    def id(x):
        return Function(lambda *xs: xs, x, x)

    @staticmethod
    def copy(x):
        return Function(lambda *xs: (*xs, *xs), x, x @ x)

    @staticmethod
    def delete(x):
        return Function(lambda *xs: (), x, Ty())

    @staticmethod
    def swap(x, y):
        return Function(lambda x0, y0: (y0, x0), x @ y, y @ x)
\end{minted}
\end{python}

We can check the properties of diagonal maps and projections.

\begin{python}\label{listing:function.axioms}
{\normalfont Axioms of cartesian categories.}

\begin{minted}{python}
X = Ty('X')
copy = Function.copy(X)
delete = Function.delete(X)
I = Function.id(X)
swap = Function.swap(X, X)

assert (copy >> copy @ I)(54) == (copy >> I @ copy)(54)
assert (copy >> delete @ I)(46) == (copy >> I @ delete)(46)
assert (copy >> swap)('was my number') == (copy)('was my number')

f = Function(lambda x: (46,) if x == 54 else (54,), X, X)
assert (f >> copy)(54) == (copy >> f @ f)(54)
assert (copy @ copy >> I @ swap @ I)(54, 46) == Function.copy(X @ X)(54, 46)
\end{minted}
\end{python}

This is all we need in order to interpret diagrams as functions!
Indeed, it is sufficient to use an instance of \py{monoidal.Functor}, with
codomain \py{cod = (Ty, Function)}. We generate a diagram using the interface
with SpaCy and we evaluate its semantics with a \py{Functor}.

\begin{python}\label{listing:pyth.arithmetic}
{\normalfont Lawvere theory as a \py{Functor}}

\begin{minted}{python}
from discopy.operad import from_spacy, tree2diagram
import spacy
nlp = spacy.load("en_core_web_sm")
doc = nlp("Fifty four was my number")
diagram = tree2diagram(from_spacy(doc), contravariant=True)
diagram.draw()
\end{minted}

\begin{center}
    \includegraphics[scale=0.50]{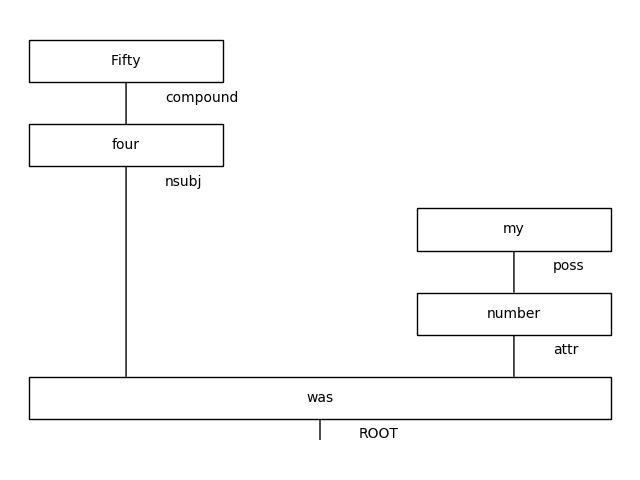}
\end{center}

\begin{minted}{python}
from discopy.monoidal import Ty, Id, Box, Functor

X = Ty('X')
ob = lambda x: X
def box_to_function(box):
    if box.name == 'was':
        return Function(lambda x, y: (x == y, ), X @ X, X)
    if box.name == 'number':
        return Function.id(X)
    if box.name == 'four':
        return Function(lambda x: (4 + x, ), X, X)
    if box.name == 'Fifty':
        return Function(lambda: (50, ), Ty(), X)
    if box.name == 'my':
        return Function(lambda: (54, ), Ty(), X)
    raise NotImplementedError()

F = Functor(ob, box_to_function, ob_factory=Ty, ar_factory=Function)
assert F(diagram)() == (True,)
\end{minted}
\end{python}

\section{Montague models}\label{sec-functional}
Montague's work appeared in three papers in the 1970s
\cite{montague1970, montague1970a, Montague73}.
In the first, he characterises his endeavour in formulating a mathematically
precise theory of natural language semantics:
``The basic aim of semantics is to characterize the notion of a true sentence
(under a given interpretation) and of entailment'' \cite{montague1970}.

On the syntactic side, Montague's grammar can be seen as an instance of the
categorial grammars studied in \ref{sec-closed}, see e.g. \cite{moot2012b}.
On the semantic side, he used a blend of lambda calculus and modal logic allowing
him to combine the logical meaning of individual words into the meaning of a sentence.
This work had an immediate influence on philosophy and linguistics
\cite{partee1976, halvorsen1979, sep-montague-semantics}.
It motivated much of the work on combinatory categorial grammars
(see \ref{section-combinatory}) and has been used in the implementation
of semantic parsers turning natural language into database queries
\cite{krishnamurthy2014, artzi2014}.
From the point of view of large-scale NLP, Montague models suffer from great complexity
issues and making them practical comes at the cost of restricting the possibility
of (exact) logical reasoning.

It was first proposed by Lambek \cite{lambek1988, lambek1999}, that Montague
semantics should be seen as a functor from categorial grammars to a cartesian
closed category. In this section, we make this idea precise, focusing on the
first-order logic aspects of Montague's translation and leaving the intensional
and modal aspects for future work.
We formulate Montague semantics as a pair of functors:
$$G \longrightarrow \bf{CCC}(\Gamma_\Sigma) \longrightarrow \bf{Set}$$
where $G$ is a categorial grammar, $\bf{CCC}(\Gamma_\Sigma)$ is a lambda calculus
for typed first-order-logic and $\bf{Set}$ is the category of sets and functions.
This factorization allows to distinguish between the syntactic translation from
sentences to formulae $G \to \bf{CCC}(\Gamma_\Sigma)$ and the evaluation of
the resulting formulae in a concrete model $\bf{CCC}(\Gamma_\Sigma) \to \bf{Set}$.

We start in \ref{section-lambda} by introducing the lambda calculus and
reviewing its relationship with cartesian closed categories.
In \ref{section-typed-logic}, we define a typed lambda calculus
$\bf{CCC}(\Gamma_\Sigma)$ for manipulating first-order logic formulae, and
in \ref{section-montague} we define Montague models and discuss their complexity.
Finally, in \ref{section-montague-discopy} we show how to implement Montague
semantics in DisCoPy, by defining \py{curry} and \py{uncurry} methods for
the \py{Function} class introduced in \ref{sec-concrete}.

\subsection{Lambda calculus}\label{section-lambda}

The lambda calculus was introduced in the 1930s by Alonzo Church as part of his
research into the foundations of mathematics \cite{church1932, church1936}.
It is a model of computation which can be used to simulate any Turing machine
\cite{turing1937}. Church also introduced a simply typed version in
\cite{church1940} which yields a weaker model of computation but allows to avoid
the infinite recursions of the untyped calculus.
There is a well known correspondence --- due to Curry and Howard --- between
the typed lambda calculus and intuitionistic logic, where types are seen as
formulae and terms as proofs. This correspondence was later extended by Lambek
\cite{lambek1986} who showed that the typed lambda calculus has a clear
characterisation as the equational theory of \emph{cartesian closed categories},
viewing programs (or proofs) as morphisms.

The rules of the lambda calculus emerge naturally from the structure of the category
of sets and functions $\bf{Set}$.
We have seen in \ref{sec-concrete} that $\bf{Set}$ is a monoidal category with
the cartesian product $\times : \bf{Set} \times \bf{Set} \to \bf{Set}$.
Moreover, for any pair of sets $A$ and $B$, the hom-set $\bf{Set}(A, B)$ is itself a
set, denoted $B^A$ or $A \to B$. In fact, there is a functor
$-^B: \bf{Set} \to \bf{Set}$ taking a set $A$ to $A^B$ and a function
$f: A \to C$ to a function $f^B : A^B \to C^B$ given by $f^B(g) = f \circ g$
for any $g \in A^B$.
This functor is the right adjoint of the cartesian product, i.e. there is a
natural isomorphism:
$$ \bf{Set}(A \times B, C) \simeq \bf{Set}(A, C^B)$$
This is holds in any cartesian closed category.

\begin{definition}[Cartesian closed category]\label{def-cartesian-closed}
    A cartesian closed category $\bf{C}$ is cartesian category equipped with a
    functor $-^A$ for any object $A \in \bf{C}_0$ which is the right adjoint
    of the cartesian product $A \times - \dashv -^A$. Explicitly, there is
    a natural isomorphism:
    \begin{equation}\label{eq-cartesian-closed}
        \bf{C}(A \times B, C) \simeq \bf{C}(A, C^B)
    \end{equation}
\end{definition}
\begin{proposition}
    A cartesian closed category is a cartesian category (Definition \ref{def-cartesian})
    which is also biclosed (Definition \ref{def-biclosed}).
\end{proposition}
\begin{remark}
    The categorial biclosed grammars studied in \ref{sec-closed}
    map canonically into the lambda calculus, as we will see in \ref{section-montague}.
\end{remark}

A functional signature is a set $\Gamma$ together with a function
$\tt{ty}: \Gamma \to TY(B)$ into the set of functional types defined inductively
by:
$$TY(B) \ni T, U\, =\, b\in B \;|\; T \otimes U\;|\; T \to U\;.$$
we write $x: T$ whenever $\tt{ty}(x) = T$ for $x \in \Gamma$.
Given a functional signature $\Gamma$, we can consider the free cartesian closed
category $\bf{CCC}(\Gamma)$ generated by $\Gamma$. As first shown by Lambek
\cite{lambek1986}, morphisms of the free cartesian closed category over $\Gamma$
can be characterised as the terms of the simply typed lambda calculus generated
by $\Gamma$.

We define the lambda calculus generated by the basic types $B$ and a functional
signature $\Gamma$. Its \emph{types} are given by $TY(B)$.
We define the set of \emph{terms} by the following inductive definition:
$$ TE \ni t, u \; = \; x \;\vert\; tu \;\vert\; \lambda x.t \;\vert\; <t, u> \;\vert\;
\pi_1 u \;\vert\; \pi_2 u$$
A \emph{typing context} is just a set of pairs of the form $x: T$, i.e.
a functional signature $\Gamma$. Then a \emph{typing judgement} is a triple:
$$ \Gamma \vdash t:T $$
consisting in the assertion that term $t$ has type $T$ in context $\Gamma$.
A \emph{typed term} is a term $t$ with a typing judgement $\Gamma \vdash t:T$
which is derivable from the following rules of inference:

\begin{equation}
    \begin{minipage}{0.3\linewidth}
        \begin{prooftree}
            \AxiomC{}
            \UnaryInfC{$ \Gamma, x: T \vdash x: T$}
        \end{prooftree}
    \end{minipage}
\end{equation}

\begin{equation} \label{lambda-rules}
\begin{minipage}{0.3\linewidth}
\begin{prooftree}
    \AxiomC{$\Gamma \vdash t: T$}
    \AxiomC{$\Gamma \vdash u : U$}
    \BinaryInfC{$\Gamma \vdash <t, u> : T \times U$}
\end{prooftree}
\end{minipage}
\begin{minipage}{0.3\linewidth}
\begin{prooftree}
    \AxiomC{$\Gamma \vdash v: T \times U$}
    \UnaryInfC{$\Gamma \vdash \pi_1v : T$}
\end{prooftree}
\end{minipage}
\begin{minipage}{0.3\linewidth}
\begin{prooftree}
    \AxiomC{$\Gamma \vdash v: T \times U$}
    \UnaryInfC{$\Gamma \vdash \pi_2v : U$}
\end{prooftree}
\end{minipage}
\end{equation}

\begin{equation}
    \begin{minipage}{0.3 \linewidth}
    \begin{prooftree}
        \AxiomC{$\Gamma, x: U \vdash t: T$}
        \UnaryInfC{$\Gamma \vdash \lambda x. t : U \to T$}
    \end{prooftree}
    \end{minipage}
    \begin{minipage}{0.3\linewidth}
    \begin{prooftree}
        \AxiomC{$\Gamma \vdash t: U \to T$}
        \AxiomC{$\Gamma \vdash u : U$}
        \BinaryInfC{$\Gamma \vdash tu : T$}
    \end{prooftree}
    \end{minipage}
\end{equation}

We define the typed lambda terms generated by $\Gamma$, denoted $\Lambda(\Gamma)$ as the
lambda terms that can be typed in context $\Gamma$.
In order to define equivalence of typed lambda terms we start by defining
$\beta$-reduction $\to_\beta$ which is the relation on terms generated by the
following rules:
\begin{equation}
    (\lambda x. t)u \to_\beta t[u/x] \qquad \pi_1<t, u> \to_\beta t
    \qquad \pi_2<t, u> \to_\beta u
\end{equation}
where $t[u/x]$ is the term obtained by substituting $u$ in place of $x$ in $t$,
see e.g. \cite{abramsky2010a} for an inductive definition.
We define $\beta$-conversion $\sim_\beta$ as the symmetric reflexive transitive
closure of $\to_\beta$. Next, we define $\eta$-conversion $\sim_\eta$ as the
symmetric reflexive transitive closure of the relation defined by:
\begin{equation}
    t \sim_\eta \lambda x. tx \qquad v \sim_\eta <\pi_1v, \pi_2v>
\end{equation}
Finally $\lambda$-conversion, denoted $\sim_\lambda$ is the transitive closure
of the union $\sim_\beta \cup \sim_\eta$. One may show that for any typed term
$\Gamma \vdash t: T$, if $t \to_\beta t'$ then $\Gamma \vdash t': T$, and similarly
for $\sim_\eta$, so that $\lambda$-equivalence is well defined on typed terms.
Moreover, $\beta$-reduction admits \emph{strong normalisation}, i.e.
ever reduction sequence is terminating and leads to a normal form without redexes.
For any lambda term $t$ we denote its normal form by $\tt{nm}(t)$, which
of course satisfies $t \sim_\lambda \tt{nm}(t)$. However normalising lambda
terms is a problem known to be \emph{not elementary recursive} in general
\cite{statman1979}! We discuss the consequences of this result for Montague
grammar at the end of the section.

We can now state the correspondence between cartesian closed categories and
the lambda calculus \cite{lambek1986}. For a proof, we refer to the lecture
notes \cite{abramsky2010a} where this equivalence is spelled out in detail
alongside the correspondence with intuitionistic logic.

\begin{proposition}\cite[Section 1.6.5]{abramsky2010a}\label{prop-lambda-ccc}
    The free cartesian closed category over $\Gamma$
    is equivalent to the lambda calculus generated by $\Gamma$.
    $$\bf{CCC}(\Gamma) \simeq \Lambda(\Gamma)/\sim_\lambda$$
\end{proposition}

Note that a morphism $f : x \to y$ in $\bf{CCC}(\Gamma)$ may have several
equivalent representations as a lambda term in $\Lambda(\Gamma)$.
In the remainder of this section, by $f \in \bf{CCC}(\Gamma)$ we will mean
any such representation, and we will write explicitly $\tt{nm}(f)$ when
we want its normal form.

\subsection{Typed first-order logic}\label{section-typed-logic}

Montague used a blend of lambda calculus and logic which allows to compose the
logical meaning of individual words into the meaning of a sentence. For
example to intepret the sentence ``John walks'', Montague would assign to
``John'' the lambda term $J = \lambda x . \text{John}(x)$ and to ``walks'' the
term $W = \lambda \phi . \exists x \cdot \phi(x) \land \text{walks}(x)$ so that their
composition results in the closed logic formula
$\exists x \cdot \text{John}(x) \land \text{walks}(x)$.
Note that $x$ and $\phi$ above are symbols of different type, $x$ is a variable
and $\phi$ a proposition.
The aim of this section is to define a typed lambda calculus for manipulating
first-order logic formulae --- akin to \cite{lappin2004} and \cite{baral2012} ---
which will serve as codomain for Montague's mapping.

We start by recalling the basic notions of first-order logic (FOL).
A  \emph{FOL signature} $\Sigma = \cal{C} + \cal{F} + \cal{R}$ consist in a set
of constant symbol $a, b \in \cal{C}$, a set of function symbols $f, g \in \cal{F}$
and a set of relational symbols $R, S \in \cal{R}$ together
with a function $\tt{ar}: \cal{F} + \cal{R} \to \bb{N}$ assigning an arity to
functional and relational symbols.
The terms of first-order logic over $\Sigma$ are generated by the following
context-free grammar:
$$FT(\Sigma) \ni t \s ::= \s x \s\vert\s a \s\vert\s f(\vec{x})$$
where $x \in \cal{X}$ for some countable set of variables $\cal{X}$, $a \in \cal{C}$,
$f \in \cal{F}$ and $\vec{x} \in (\cal{X} \cup \cal{C})^{\tt{ar}(f)}$.
The set of first-order logic formulae over $\Sigma$ is defined by the following
context-free grammar:
\begin{equation}\label{def-fol}
    FOL(\Sigma) \ni \phi \s ::= \s \top \s\vert\s t \s\vert\s x = x' \s\vert\s R(\vec{x})
    \s\vert\s \phi \land \phi \s\vert\s \exists \ x \cdot \phi \s\vert\s \neg \phi
\end{equation}
where $t, x, x' \in FT(\Sigma)$, $R \in \cal{R}$,
$\vec{x} \in FT(\Sigma)^{\tt{ar}(R)}$ and $\top$ is the truth symbol.
Let us denote the variables of $\phi$ by $\tt{var}(\phi) \sub \cal{X}$ and
its free variables by $\tt{fv}(\phi) \sub \tt{var}(\phi)$.
For any formula $\phi$ and $\vec{x} \in FT(\Sigma)^\ast$, we denote by
$\phi(\vec{x})$ the formula obtained by substituting the terms $x_1 \dots x_n$
in place of the free variables of $\phi$ where $n = \size{\tt{fv}(\phi)}$.

We can now build a typed lambda calculus over the set of FOL formulae.

\begin{definition}[Typed first-order logic]
    Given a FOL signature $\Sigma = \cal{C} + \cal{R}$, we define the typed first
    order logic over $\Sigma$ as the free cartesian closed category
    $\bf{CCC}(\Gamma_\Sigma)$ where $\Gamma_\Sigma$ is the functional signature
    with basic types:
    $$ B  = \set{ X, P}_{n \in \bb{N}}$$
    where $X$ is the type of terms and $P$ is the type of propositions,
    and entries given by:
    $$ \Gamma_\Sigma = \set{\phi : P \s\vert\s \phi \in FOL(\Sigma) - FT(\Sigma)}
    + \set{x : X \s\vert\s x \in FT(\Sigma)} $$
\end{definition}

In order to recover a FOL formula from a lambda term $f : T$ in $\bf{CCC}(\Gamma_\Sigma)$,
we need to normalise it.

\begin{proposition}
    For any lambda term $f : P$ in $\bf{CCC}(\Gamma_\Sigma)$, the normal form
    is a first order logic formula $\tt{nm}(f) \in FOL(\Sigma)$.
\end{proposition}
\begin{proof}

\end{proof}

Note that morphisms $\phi : P$ in $\bf{CCC}(\Gamma_\Sigma)$ are the same as
first-order logic formulae $\phi \in FOL(\Sigma)$, since we have adopted the convention
that a morphism in $\bf{CCC}(\Gamma)$ is the normal form of its representation as
a lambda term in $\Lambda(\Gamma)$.

\begin{example}
    Take $\Sigma = \set{\text{John}, \text{Mary}, \text{walks}, \text{loves}}$
    with $\tt{ar}(\text{John}) = \tt{ar}(\text{Mary}) = \tt{ar}(\text{walks}) = 1$
    and $\tt{ar}(\text{loves}) = 2$. Then the following are examples of well-typed
    lambda expressions in $\bf{CCC}(\Gamma_\Sigma)$:
    $$ \lambda \phi . \exists x \cdot \phi(x) \land \text{John}(x)\, :\, P \to P
    \qquad \lambda x \lambda y. \text{loves}(x, y)\, :\, X \times X \to P$$
\end{example}

A model for first-order logic formulae is defined as follows.

\begin{definition}[FOL model]
        A model $K$ over a FOL signature $\Sigma$, also called $\Sigma$-model,
        is given by a set $U$ called the \emph{universe} and
        an interpretation $K(R) \sub U^{\tt{ar}(R)}$ for every relational symbol
        $R \in \cal{R}$ and $K(a) \in U$ for every constant symbol $c \in \cal{C}$.
        We denote by $\cal{M}_\Sigma$ the set of $\Sigma$-models.
\end{definition}

Given a model $K \in \cal{M}_\Sigma$ with universe $U$, let
$$\tt{eval}(\phi, K) = \set{v \in U^{\tt{fv}(\phi)} \s \vert \s  (K, v) \vDash \phi}$$
where the satisfaction relation $(\vDash)$ is defined by induction over
(\ref{def-fol}) in the usual way.

\begin{proposition}
    Any model $K \in \cal{M}_\Sigma$ induces a monoidal functor
    $F_K: \bf{CCC}(\Gamma_\Sigma) \to \bf{Set}$ such that closed lambda terms
    $\phi : F_0$ are mapped to their truth value in $K$.
\end{proposition}
\begin{proof}
    By the universal property of free cartesian closed categories, it is sufficient
    to define $F_K$ on generating objects and arrows. On objects we define:
    $$F_K(X) = 1 \qquad F_K(P) = \coprod_{n=0}^{N}\cal{P}(U^n)$$
    where $N$ is the maximum arity of a symbol in $\Sigma$ and $\cal{P}(U^n)$
    is the powerset of $U^n$.
    On generating arrows $\phi : P$ and $x: X$ in $\Gamma$, $F_k$ is defined by:
    $$F_K(\phi) = \tt{eval}(\phi, K) \in F_K(P) \qquad F_K(x) = 1$$
\end{proof}

\subsection{Montague semantics}\label{section-montague}

We model Montague semantics as a pair of functors
$G \to \bf{CCC}(\Gamma_\Sigma) \to \bf{Set}$ for a grammar $G$.
We have already seen that functors  $\bf{CCC}(\Gamma_\Sigma) \to \bf{Set}$ can
be built from $\Sigma$-models or relational databases.
It remains to study functors $G \to \bf{CCC}(\Gamma_\Sigma)$,
which we call \emph{Montague models}.

\begin{definition}[Montague model]
    A Montague model is a monoidal functor $M: G \to \bf{CCC}(\Gamma_\Sigma)$
    for a biclosed grammar $G$ and a FOL signature $\Sigma$ such that $M(s) = P$
    and $M(w) = 1$ for any $w \in V \sub G_0$.
    The semantics of a grammatical sentence $g : u \to s$ in $\bf{BC}(G)$ is the
    first order logic formula $\tt{nm}(M(w)) \in FOL(\Sigma)$ obtained by
    normalising the lambda term $M(w) : P$ in $\bf{CCC}(\Gamma_\Sigma)$.
\end{definition}
\begin{remark}
    Note that since $\bf{CCC}(\Gamma_\Sigma)$ is biclosed, it is sufficient to
    define the image of the lexical entries in $G$ to obtain a Montague model.
    The structural morphisms $\tt{app}$, $\tt{comp}$ defined in \ref{section-lambek}
    have a canonical intepretation in any cartesian closed category given by:
    $$\tt{app}_{A, B} \mapsto \lambda f. (\pi_2f)(\pi_1f): (A \times (A \to B)) \to B$$
    $$\tt{comp}_{A, B, C} \mapsto \lambda f \lambda a . (\pi_2f)(\pi_1f)a : ((A \to B)
    \times (B \to C)) \to (A \to C)$$
\end{remark}
\begin{example}
    Let us consider a Lambek grammar $G$ with basic types $B = \set{n, n', s}$
    and lexicon given by:
    $$\Delta(\text{All}) = \set{s/(n\backslash s)/n'} \quad \Delta(roads) = \set{n'}
    \quad \Delta(\text{lead to}) = \set{(n\backslash s) / n} \quad
    \Delta(\text{Rome}) = \set{n}$$
    The following is a grammatical sentence in $\cal{L}(G)$:
    \begin{equation*}
        \input{figures/montague-all-roads.tikz}
    \end{equation*}
    We fix a FOL signature $\Sigma = \cal{R} + \cal{C}$ with one constant
    symbol $\cal{C} =\set{\text{Rome}}$ and two relational
    symbols $\cal{R} = \set{\text{road}, \text{leads-to}}$ with arity $1$ and $2$
    respectively. We can now define a Montague model $M$ on objects by:
    $$ M(n) = X \quad M(n') = X \to P \quad M(s) = P$$
    and on arrows by:
    $$M(\text{All}) = \lambda \phi \lambda \psi . \forall x (\phi(x) \implies \psi(x)):
    (X \to P) \to ((X \to P) \to P)$$
    $$M(\text{roads}) = \lambda x. \text{road}(x) : X \to P \qquad
    M(\text{Rome}) = \text{Rome} : X$$
    $$M(\text{lead to}) = \lambda x \lambda y. \text{leads-to}(x, y) : X \to (X \to P)$$
    Then the image of the sentence above normalises to the following lambda term:
    $$ M(\text{All roads lead to Rome} \to s) = \forall x \cdot (\text{road}(x)
    \implies \text{leads-to}(x, \text{Rome}))$$
\end{example}

We are interested in the problem of computing the Montague semantics of sentences
generated by a biclosed grammar $G$.

\begin{definition}
    \begin{problem}
    \problemtitle{$\tt{LogicalForm}(G)$}
    \probleminput{$g: u \to s$, $\Sigma$, $M: G \to \bf{CCC}(\Gamma_\Sigma)$}
    \problemoutput{$\tt{nm}(M(g)) \in FOL(\Sigma)$}
  \end{problem}
\end{definition}

We may split the above problem in two steps. First, given the sentence $g: u \to s$
we are asked to produce the corresponding lambda term $M(g) \in \Lambda(\Gamma_\Sigma)$.
Second, we are asked to normalise this lambda term in order to obtain the underlying
first-order logic formula in $FOL(\Sigma)$. The first step is easy and may in
fact be done in $\tt{L}$ since it is an example of a functorial reduction, see
Proposition \ref{prop-functorial-reduction}. The second step may however be
exceptionally hard. Indeed, Statman showed that deciding $\beta$ equivalence between
simply-typed lambda terms is not elementary recursive \cite{statman1979}.
As a consequence, the normalisation of lambda terms is itself not elementary recursive. We may moreover show that any simply typed lambda term can be
obtained in the image of some
$M: G \to \bf{CCC}(\Gamma_\Sigma)$, yielding the following result.

\begin{proposition}\label{prop-elementary}
    There are biclosed grammars $G$ such that $\tt{LogicalForm}(G)$ is not
    elementary recursive.
\end{proposition}
\begin{proof}
    Let $G$ be a biclosed grammar with one generating object $a$
    generating arrows $\tt{swap}: a \otimes a \to a \otimes a$,
    $\tt{copy}: a \to a \otimes a$ and $\tt{discard}: a \to 1$.
    These generators map canonically in any cartesian category. In fact, there
    is a full functor $F: \bf{BC}(G) \to \bf{CCC}$, since all the structural morphisms of cartesian closed categories can be represented in $\bf{BC}(G)$.
    Therefore for any morphism $h \in \bf{CCC}$ there is a morphism
    $f \in \bf{BC}(G)$  such that $F(f) = h$.
    Then for any typed lambda term $g \in \Lambda(\varnothing)$ there is a morphism
    $f \in \bf{BC}(G)$ such that $F(f)$ maps to $g$ under the translation
    \ref{prop-lambda-ccc}. Therefore normalising $g$ is equivalent to computing
    $\tt{nm}(F(f)) = \tt{LogicalForm}(G)$.
    Therefore the problem of normalising lambda terms
    reduces to $\tt{LogicalForm}$.
\end{proof}

Of course, this proposition is about the worst case scenario, and definitely
not about human language. In fact, small to medium scale semantic parsers exist
which translate natural language sentences into their logical form using Montague's
translation \cite{krishnamurthy2014, artzi2014}.
It would be interesting to show that for restricted choices of grammar $G$
(e.g. Lambek or Combinatory grammars), and possibly assuming that the order of
the lambda terms assigned to words is bounded (as in \cite{terui2012}),
$\tt{LogicalForm}$ becomes tractable.

Even if we are able to extract a logical formula efficiently
from natural language, we need to be able to evaluate it in a database in order
to compute its truth value. Thus, in order to compute the full-blown montague
semantics we need to solve the following problem.

\begin{definition}
    \begin{problem}
    \problemtitle{$\tt{Montague}(G)$}
    \probleminput{$g: u \to s$, $\Sigma$, $M: G \to \bf{CCC}(\Gamma_\Sigma)$,
                   $K \in \cal{M}_\Sigma$}
    \problemoutput{$F_K(M(g)) \sub U^{\tt{fv}(M(g))}$}
  \end{problem}
\end{definition}

Note that solving this problem does not necessarily require to solve
$\tt{LogicalForm}$. However, since we are dealing with first-order logic,
the problem for a general biclosed grammar is at least as hard as the
evaluation of FOL formulae.

\begin{proposition}\label{prop-pspace}
    There are biclosed grammars $G$ such that $\tt{Montague}$ is $\tt{PSPACE}$-hard.
\end{proposition}
\begin{proof}
    It was shown by Vardi that the problem of evaluating $\tt{eval}(\phi, K)$
    for a first-order logic formula $\phi$ and a model $K$ is $\tt{PSPACE}$-hard
    \cite{vardi1982}.
    Reducing this problem to $\tt{Montague}$ is easy. Take $G$ to be a categorial
    grammar with a single word $w$ of type $s$. Given any first-order logic formula $\phi$
    and model $K$, we define $M: G \to \bf{CCC}(\Gamma_\Sigma)$ by $M(s) = P$ and
    $M(w) = \phi$. Then we have that the desired output $F_K(M(w)) = F_K(\phi) =
    \tt{eval}(\phi, K)$ is the evaluation of $\phi$ in $K$.
\end{proof}

It would be interesting to look at restrictions of this problem, such
as fixing the Montague model $M$, fixing the input sentence $g: u \to s$ and
restricting the allowed grammars $G$. In general however, the proposition above shows
that working with full-blown first-order logic is inpractical for large-scale
industrial NLP applications. It also shows that models of this kind are the most
general and applicable if we can find a compromise between tractability and
logical expressivity.

\subsection{Montague in DisCoPy}\label{section-montague-discopy}

We implement Montague models as DisCoPy functors from \py{biclosed.Diagram}
into the class \py{Function}, as introduced in \ref{sec-closed} and
\ref{sec-concrete} respectively.
Before we can do this, we need to upgrade the \py{Function} class
to account for its biclosed structure. The only additional methods we need are
\py{curry} and \py{uncurry}.

\begin{python}\label{listing:Function.curry}
{\normalfont Curry and UnCurry methods for \py{Function}.}

\begin{minted}{python}
class Function:
    ...
    def curry(self, n_wires=1, left=False):
            if not left:
                dom = self.dom[:-n_wires]
                cod = self.cod << self.dom[-n_wires:]
                inside = lambda *xl: (lambda *xr: self.inside(*(xl + xr)),)
                return Function(inside, dom, cod)
            else:
                dom = self.dom[n_wires:]
                cod = self.dom[:n_wires] >> self.cod
                inside = lambda *xl: (lambda *xr: self.inside(*(xl + xr)),)
                return Function(inside, dom, cod)

    def uncurry(self):
        if isinstance(self.cod, Over):
            left, right = self.cod.left, self.cod.right
            cod = left
            dom = self.dom @ right
            inside = lambda *xs: self.inside(*xs[:len(self.dom)])[0](*xs[len(self.dom):])
            return Function(inside, dom, cod)
        elif isinstance(self.cod, Under):
            left, right = self.cod.left, self.cod.right
            cod = right
            dom = left @ self.dom
            inside = lambda *xs: self.inside(*xs[len(left):])[0](*xs[:len(left)])
            return Function(inside, dom, cod)
        return self
\end{minted}
\end{python}

We can now map biclosed diagrams into functions using \py{biclosed.Functor}.
We start by initialising a sentence with its categorial grammar parse.

\begin{python}\label{listing:montague.example}
{\normalfont Example parse from a categorial grammar}

\begin{minted}{python}
from discopy.biclosed import Ty, Id, Box, Curry, UnCurry

N, S = Ty('N'), Ty('S')
two, three, five = Box('two', Ty(), N), Box('three', Ty(), N), Box('five', Ty(), N)
plus, is_ = Box('plus', Ty(), N >> (N << N)), Box('is', Ty(), N >> S << N)

FA = lambda a, b: UnCurry(Id(a >> b))
BA = lambda a, b: UnCurry(Id(b << a))

sentence = two @ plus @ three @ is_ @ five
grammar = FA(N, N << N) @ Id(N) @ BA(N, N >> S) >> BA(N, N) @ Id(N >> S) >> FA(N, S)
sentence = sentence >> grammar
sentence.draw()
\end{minted}
\begin{center}
    \includegraphics[scale=0.40]{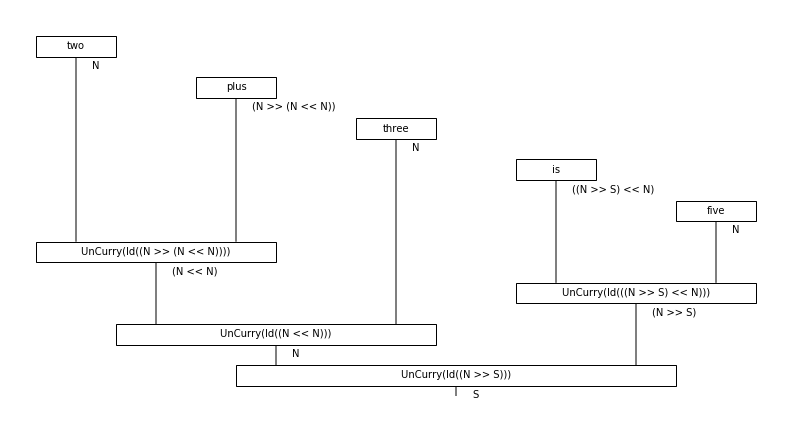}
\end{center}
\end{python}

We can now evaluate this sentence in a Montague model defined as a biclosed functor.

\begin{python}\label{listing:Function.curry}
{\normalfont Evaluating a sentence in a Montague model.}

\begin{minted}{python}
from discopy.biclosed import Functor

number = lambda y: Function(lambda: (y, ), Ty(), N)
add = Function(lambda x, y: (x + y,), N @ N, N)
equals = Function(lambda x, y: (x == y, ), N @ N, S)

ob = lambda x: x
ar = {two: number(2), three: number(3), five: number(5),
      is_: equals.curry().curry(left=True),
      plus: add.curry().curry(left=True)}
Montague = Functor(ob, ar, ob_factory=Ty, ar_factory=Function)
assert Montague(sentence)() == (True,)
\end{minted}
\end{python}

\section{Neural network models}\label{sec-neural-net}

In the last couple of decades, neural networks have become ubiquitous in
natural language processing.
They have been used successfully in a wide range of tasks including language
modelling \cite{demulder2015}, machine translation \cite{bahdanau2014, popel2020},
parsing \cite{jurafsky2008}, question answering \cite{zhou2017} and
sentiment analysis \cite{socher2013a, kalchbrenner2014}.

In this section we show that neural network models can be formalised as functors from
a grammar $G$ to the category $\bf{Set}_\bb{R}$ of Euclidean spaces and functions
via a category $\bf{NN}$ of neural network architectures. This includes feed-forward,
recurrent and recursive neural networks and we discuss how the recent
attention mechanisms could be formalised in this diagrammatic framework.
At each step, we show how these neural architectures are used to solve concrete
tasks such as sentence classification, language modelling, sentiment analysis and
machine translation. We keep a level of informality in describing the learning
process, an aspect which we will further explore in Chapter 3.
We end the section by building an interface between DisCoPy and Tensorflow/Keras
neural networks \cite{chollet2015keras, tensorflow2015-whitepaper}, by defining
a class \py{Network} that allows for composing and tensoring Keras models.

\subsection{Feed-forward networks}

The great advantage of neural networks comes from their ability
to simulate any function on Euclidean spaces, a series of results
known as universal approximation theorems \cite{tikk2003, ohn2019}.
We will thus give them semantics in the category $\bf{Set}_\bb{R}$ where morphisms
are functions acting on Euclidean spaces.
Note that $\bf{Set}_\bb{R}$ is monoidal with product $\oplus$ defined on objects
as the direct sum of Euclidean spaces $\bb{R}^n \oplus \bb{R}^m = \bb{R}^{n +m}$
and on arrows as the cartesian product $f \oplus g (x, y) = (f(x), g(y))$.
To our knowledge, $\bf{Set}_\bb{R}$ doesn't have much more structure than this.
We will in fact not be able to interpret

In a typical supervised machine learning problem one wants
to approximate an unknown function $f: \bb{R}^n \to \bb{R}^m$ given a
dataset of pairs $D = \{(x, f(x)) \vert x \in \bb{R}^n , \, f(x) \in \bb{R}^m \}$.
Neural networks are parametrized functions built from the following basic
processing units:
\begin{enumerate}
    \item sum $\tt{sum} : n \to 1$ for $n \in \bb{N}$,
    \item weights $\{ w : 1 \to 1 \} _{w \in W_0}$,
    \item biases $\{ r: 0 \to 1 \}_{r \in W_1}$
    \item and activation $\sigma : 1 \to 1$.
\end{enumerate}
where $W = W_0 + W_1$ is a set of variables.
These generate a free cartesian category
$$\bf{NN} = \bf{Cart}(W + \{\tt{sum}, \sigma\})$$
where morphisms are the diagrams of neural network architectures.
For example, a neuron with $n$ inputs, bias $w_0 \in W$ and weights $\vec{w} \in W^\ast$
is given by the following diagram in $\bf{NN}$.
\begin{equation}\label{eq-neuron}
    \input{figures/neuron.tikz}
\end{equation}
``Deep'' networks are simply given by composing these neurons in parallel
forming a \emph{layer}:
\begin{equation*}
    \input{figures/neural-layer.tikz}
\end{equation*}
and then stacking layers one on top of the other, as pasta sheets on a lasagna.


Fix a neural network architecture $K: n \to m$ in $\bf{NN}$.
Given a choice of parameters $\theta: W \to \bb{R}$, the network $K$ induces a
function $I_\theta(K) : \bb{R}^n \to \bb{R}^m$ called an \emph{implementation}.
$I_\theta$ is in fact a monoidal functor $\bf{NN} \to \bf{Set}_\bb{R}$ defined on objects by
$I_\theta(1)= \bb{R}$, and on arrows by:
\begin{enumerate}
    \item $I_\theta(\tt{sum}) = \{ (x_1, \dots, x_n) \mapsto \sum_{i=1}^n x_i \}$,
    \item $I_\theta(w) = \{ x \mapsto \theta(w) \cdot x \}$,
    \item $I_\theta(r) = \{() \mapsto r\}$
    \item $I_\theta(\sigma)$ is a non-linearity such as sigmoid.
\end{enumerate}
For example, the image of the neuron \ref{eq-neuron} is given by the
function $I_\theta(o(\vec{w})) : \bb{R}^n \to \bb{R}$ defined as follows:
$$ I_\theta(o(\vec{w}))(\vec{x}) = \sigma(w_0 + \sum_{i = 1}^n w_i x_i)$$

In order to learn $f: \bb{R}^n \to \bb{R}^m$, we choose an architecture $K: n \to m$
and a loss function $l: \bb{R}^W \to \bb{R}$ that evaluates the choice of parameters
against the dataset, such as the mean-squared error
$l(\theta) = \sum_i (I_\theta(K)(x_i) - f(x_i))^2$.
The aim is to minimize the loss function with respect to the parameters so that
$I_\theta(K)$ approximates $f$ as closely as possible.
The most reliable way of doing this is to compute the gradient of $l$ w.r.t.
$\theta$ and descending along the gradient. This boils down to computing the
gradient of $I_\theta(K)$ w.r.t $\theta$ which itself boils down to computing the gradient
of each component of the network. Assuming that $\sigma : \bb{R} \to \bb{R}$ is a
smooth function, the image of the functor $I_\theta$ lies in $\bf{Smooth}
\injects \bf{Set}_\bb{R}$, the category of smooth functions on Euclidean spaces.
The backpropagation algorithm provides an
efficient way of updating the parameters of the network step-by-step while
descending along the gradient of $l$, see \cite{cruttwell2021} for a recent
categorical treatment of differentiation.

Although neural networks are deterministic functions, it is often useful to think
of them as probabilistic models. One can turn the output of a neural network into
a probability distribution using the $\tt{softmax}$ function as follows.
Given a neural network $K: m \to n$ and a choice of parameters $\theta : W \to \bb{R}$
we can compose $I_\theta(K): \bb{R}^m \to \bb{R}^n$ with a softmax layer,
given by:
$$ \tt{softmax}_n(\vec{x})_i = \frac{e^{x_i}}{\sum_{i=1}^n e^{x_i}} $$
Then the output is a normalised vector of positive reals of length $n$,
which yields a distribution over $[n]$ the set with $n$ elements, i.e. softmax
has the following type:
$$\tt{softmax}_n : \bb{R}^n \to \cal{D}([n])$$
where $\cal{D}(A)$ is the set of probability distributions over $A$, as defined
in \ref{section-tasks}, where softmax is analysed in more detail.
In order to simulate a probabilistic process,
from a neural network $K: m \to n$, an element of $[n]$ is drawn at random from
the induced distribution $\tt{softmax}_n(I_\theta(K)(x)) \in \cal{D}([n])$.

Feed-forward neural networks can be used for \emph{sentence classification}:
the general task of assigning labels to pieces of written text.
Applications include spam detection, sentiment analysis and document classification.
Let $D \sub V^\ast \times X$ be a dataset of pairs $(u, x)$ for $u$ an utterance
and $x$ a label. The task is to find a function
$f: V^\ast \to \cal{D}(X)$ minimizing a loss function $l(f, D)$ which computes
the distance between the predicted labels and the expected ones given by $D$.
For example one may take the mean squared error $l(f, D) = \sum_{(u, x) \in D} (f(u) - \ket{x})^2$
or the cross entropy loss $-\frac{1}{\size{D}}\sum_{(u, x) \in D} \ket{x} \cdot \tt{log}(f(u))$,
where $\ket{x}$ is the one-hot encoding of $x \in X$ as a distribution $\ket{x} \in \cal{D}(X)$
and $\cdot$ is the inner product.
Assuming that every sentence $u \in V^\ast$ has length at most $m$ and that
we have a word embedding $E: V \to \bb{R}^k$, we can parametrize the set of functions
$\bb{R}^{m k} \to \bb{R}^{\size{X}}$ using a neural network $K: m k \to \size{X}$
and use softmax to get a probability distribution over classes:
$$f(x \, \vert \, u ) = \tt{softmax}_k(I_\theta(K)(E^\ast(u))) \in \cal{D}(X)$$
where $u = v_1 v_2 \dots v_m \in V^\ast$ is an utterance and $E^\ast(u)$ is the
feature vector for $u$, obtained by concatenating the word embeddings
$E^\ast(u) = \tt{concat}(E(v_1), E(v_2), \dots, E(v_n))$. In this approach,
$K$ plays the role of a black box which takes in a list of words and ouputs
a class. For instance in a sentiment analysis task, we can set
$X = \set{\text{positive}, \text{negative}}$ and $m = 3$ and expect that
``not so good'' is classified as ``negative''.
\begin{equation*}
    \input{figures/not-so-good.tikz}
\end{equation*}

A second task we will be interested in is \emph{language modelling},
the task of predicting a word in a text given the previous words.
It takes as input a corpus, i.e. a set of strings of words $C \sub V^\ast$,
and outputs $f: V^\ast \to \cal{D}(V)$ a function which ouputs a distribution
for next word $f(u) \in \cal{D}(V)$ given a sequence of previous words $u \in V^\ast$.
Language modelling can be seen as an instance of the classification task, where the
corpus $C$ is turned into a dataset of pairs $D$ such that $(u, x) \in D$ whenever
$ux$ is a substring of $C$. Thus language modelling is a \emph{self-supervised}
task, i.e. the supervision is not annotated by humans but directly generated from text.
Neural networks were first used for language modelling by Bengio et al. \cite{bengio2003},
who trained a single (deep) neural network taking a fixed number of words as input and
predicting the next, obtaining state of the art results at the time.

For both of these tasks, feed-forward neural networks perform poorly compared
to the recurrent architectures which we are about to study. This is because
feed-forward networks take no account of the sequential nature of the input.

\subsection{Recurrent networks}

The most popular architectures currently used modelling language are
\emph{recurrent neural networks} (RNNs). They were introduced in NLP by Mikolov et al.
in 2010 \cite{mikolov2010} and have since become widely used, see
\cite{demulder2015} for a survey.

RNNs were introduced by Ellman \cite{elman1990} to process sequence input data.
They are defined by the following recursive equations:
\begin{equation}
    h_t = \sigma(W_h x_t + U_h h_{t-1} + b_h) \quad y_t = \sigma(W_yh_t + b_y)
\end{equation}
where $t \in \bb{N}$ is a time variable, $h_t$ denotes the encoder hidden vector,
$x_t$ the input vector, $y_t$ the output, $W_h$ and $W_y$ are matrices of weights
and $b_h, b_y$ are bias vectors. We may draw these two components as diagrams
in $\bf{NN}$:
\begin{equation*}
    \input{figures/recurrent-network.tikz}
\end{equation*}

\begin{remark}
    The diagrams in the remainder of this section should be read from left/top to
    bottom/right. One may obtain the corresponding string diagram by bending the left wires to the top and the right wires to the bottom. In the diagram above we
    have omitted the subscript $t$ since it is determined by the left-to-right
    reading of the diagram and we use the labels $h, x, y$ for dimensions, i.e.
    objects of $\bf{NN}$.
\end{remark}

The networks above are usually called \emph{simple} recurrent networks since
they have only one layer.
More generally $X$ and $Y$ could be any neural networks of type
$X : x \oplus h \to h$ and $Y : h \to y$, often called recurrent network and
decoder respectively.

\begin{definition}
    A recurrent neural network is a pair of neural networks $X : x \oplus h \to h$
    and $Y : h \to y$ in $\bf{NN}$ for some dimensions $x, h, y \in \bb{N}$.
\end{definition}

The data of RNNs defined above captures precisely the data of a functor from a
regular grammar as we proceed to show.
Fix a finite vocabulary $V$ and consider the regular grammar $RG$ with three symbols
$s_0, h , s_1$ and transitions $h \xto{w} h$ for each word $w \in V$.
\begin{equation*}
    \input{figures/recurrent-grammar.tikz}
\end{equation*}
Note that $RG$ parses any string of words, i.e. the language generated by $RG$
is $\cal{L}(RG) = V^\ast$. The derivations in $RG$ are sequences:
\begin{equation}\label{eq-recurrent-derivation}
    \input{figures/recurrent-derivation.tikz}
\end{equation}
Recurrent neural networks induce functors from this regular grammar to $\bf{NN}$.

\begin{proposition}
    Recurrent neural networks $(X: x \oplus h \to h, Y: h \to y)$ induce
    functors $RN : RG \to \bf{NN}$ such that
    $\size{V} = x$, $RN(h) = h$, $RN(s_1) = y$, $RN(s_0) = 0$ and
    $RN(s_0 \to h) = \vec{0}$.
\end{proposition}
\begin{proof}
    Given an RNN $(X, Y)$ we can build a functor with
    $RN(h \xto{w} h) = (\ket{w} \oplus \tt{id}_h) \cdot X$ where $\ket{w}$ is
    the one-hot encoding of word $w \in V$ as a vector of dimension $n = \size{V}$,
    and $RN(h \to s_0) = Y$.
\end{proof}
\begin{remark}
    Note that a functor $RN : RG \to \bf{NN}$ induces a finite family of neural
    networks $RN(h \xto{w} h)$ indexed by $w \in V$.
    We do not think that it is always possible to construct a recurrent network
    $X : \size{V} \oplus RN(h) \to RN(h)$ such that
    $RN(h \xto{w} h) = (\ket{w} \oplus \tt{id}_n) \cdot X$, hinting that functors
    $RN: RG \to \bf{NN}$ are a larger class of processes than RNNs,
    but we were unable to find a counterexample.
\end{remark}

Given a recurrent neural network $RN: G \to \bf{NN}$, the image under $RN$ of
the derivation \ref{eq-recurrent-derivation} is the following diagram in $\bf{NN}$.

\begin{equation*}
    \input{figures/recurrent-nn.tikz}
\end{equation*}

It defines a network that runs through the input words $w_i$ with the recurrent
network and uses the decoder to output a prediction of dimension $y$.

In order to classify sentences in a set of classes $X$, one may choose
$y = \size{X}$, so that the implementation of the diagram
above is a vector of size $\size{X}$ giving scores for the likelihood that
$w_0 w_1 \dots w_k$ belongs to class $c \in X$.

For language modelling, we may set $y = \size{V}$.
Given a string of words $u \in V^\ast$ with derivation
$g_u: s_0 \to s_1 \in \bf{C}(RG)$
and an implementation $I_\theta : \bf{NN} \to \bf{Set}_\bb{R}$,
we can compute the distribution over the next words
$\tt{softmax}(I_\theta(RN(g_u))) \in \cal{D}\size{V}$.

Another task that recurrent neural networks allow to tackle is \emph{machine
translation}. This is done by composing a encoder recurrent network
$X: \size{V} \oplus h \to h$ with a decoder recurrent network $K : h \to h \oplus \size{V}$,
as in the following diagram.
\begin{equation}
    \input{figures/recurrent-translation.tikz}
\end{equation}

One of the drawbacks of standard RNNs is that they don't have memory.
For instance in the translation architecture above, the last encoder state
is unfolded into all the output decoder states. This often results in a translation
that loses accuracy as the sentences grow bigger. \emph{Long-short term memory}
networks (LSTMs) were introduced in order to remedy this problem \cite{hochreiter1997}.
They are a special instance of recurrent neural networks, where the encoder states
are split into a \emph{hidden state} and a \emph{cell state}. The cell state
allows to store information for longer durations and their architecture was
designed to avoid vanishing gradients \cite{gers2000}. This makes them particularly
suited to NLP applications, see \cite{otter2019} for example applications.

Moreover, RNNs are well suited for processing data coming in sequences but they fail
to capture more structured input data. If we want to process syntax trees and
more complex grammatical structures, what we need is a recursive neural network.

\subsection{Recursive networks}

Recursive neural networks (RvNN) generalise RNNs by allowing neural networks
to recur over complex structures. They were introduced
in the 1990s by Goller and Kuchler \cite{goller1996} for the classification
of logical terms, and further generalised by Sperduti et al.
\cite{sperduti1997, frasconi1998} who pioneered their applications in chemistry
\cite{bianucci2000, micheli2004}.
RvNNs were introduced in NLP by Socher et al. \cite{socher2013a}, who obtained
state-of-the-art results in the task of sentiment analysis using tree-shaped
RvNNs. They have since been applied to several tasks including word sense
disambiguation \cite{cheng2015} and logical entailment \cite{bowman2015}.

In \cite{sperduti1997}, Sperduti and Starita provide a structured way of mapping
labelled directed graphs onto neural networks.
The main difficulty in their formalisation appears when the graph in the domain
has cycles, in which case they give a procedure for unfolding it into a directed acyclic graph.
We review their formalisation for the case of directed acyclic graphs (DAGs).
A DAG is given by a set of vertices $N$ and a set of
edges $E \sub N \times N$ such that the transitive closure of $E$ does not contain
loops. In any DAG $(N, E)$, the parents of a vertex $v$ are defined by
$\tt{pa}(v) = \set{v' \in N \, \vert \, (v', v) \in E}$ and the children by
$\tt{ch}(v) = \set{v' \in N \, \vert \, (v, v') \in E}$.
A $\Sigma$-labelled DAG is a DAG $(N, E)$ with a labelling function $\phi : N \to \Sigma$.

Suppose we have a set $X$ of $\Sigma$-labelled DAGs which we want to classify over $k$
classes, given a dataset of pairs $D \sub X \times [k]$.
The naive way to do this is to encode every graph in $D$ as a vector
of fixed dimension $n$ and learning the parameters of a neural network
$K : n \to k$. Sperduti and Starita propose instead to assign a
\emph{recursive neuron} to each vertex of the graph, and connecting them according
to the topological structure of the graph.
Given a labelled DAG $(N, E, \phi)$ with an assignment $\theta: E + \Sigma \to \bb{R}$
of weights to every edge in $E$ and every label in $\Sigma$,
the recursive neuron assigned to a vertex $v$ in a DAG $(N, E)$ is the function
defined by the following recursive formula:
\begin{equation}\label{eq-recursive-neuron}
    o(v) = \sigma( \theta(\phi(v)) + \sum_{v' \in \tt{pa}(v)} \theta(v', v) o(v'))
\end{equation}
Note that this is a simplification of the formula given by Sperduti et al. where
we assume that a single weight is assigned to every edge and label, see \cite{sperduti1997}
for details.
They further assume that the graph has a \emph{sink}, i.e. a vertex which can
be reached from any other vertex in the graph. This allows to consider the
output of the neuron assigned to the supertarget as the score of the graph,
which they use for classification.

We show that the recursive neuron (\ref{eq-recursive-neuron}) defines a functor
from a monoidal grammar to the category of neural networks $\bf{NN}$.
Given any set of $\Sigma$-labelled DAGs $X$, we may construct a monoidal signature
$G = \Sigma + E^\ast \xleftarrow{\phi + \tt{in}} N \xto{\tt{out}} E^\ast$ where $N$
is the set of all vertices appearing in a graph in $X$ and $E$ is the set of all
edges appearing in a graph in $X$, with $\tt{in}, \tt{out}: N \to E^\ast$ listing
the input and output edges of vertex $v$ respectively and $\phi: N \to \Sigma$ is
the labelling function. Then the recursive neuron
$o$ given above defines a functor $O : \bf{MC}(G + \tt{swap}) \to \bf{Set}_\bb{R}$
from the free monoidal category generated by $G$ and swaps to $\bf{Set}_\bb{R}$.
The image of a vertex $v : l \oplus \vec{e} \to \vec{e'}$ is a function
$O(v) : \bb{R}^{\size{\vec{e}} + 1} \to \bb{R}^{\size{\vec{e'}}}$ given by:
$O(v)(x) = \tt{copy}(o(v)(x))$ where $\tt{copy}: \bb{R} \to \bb{R}^{\size{\vec{e'}}}$
is the diagonal map in $\bf{Set}_\bb{R}$.
Note that any DAG $H \in X$ gives rise to a morphism in $\bf{MC}(G + \tt{swap})$
given by connecting the boxes (vertices) according to the topological structure
of the graph, using swaps if necessary.
Note that the functor $O$ factors through the category $\bf{NN}$ of neural
networks since all the components of Equation \ref{eq-recursive-neuron} are
generators of $\bf{NN}$, thus $O$ defines a mapping from DAGs in $X$ to
neural networks.
Generalising from this to the case where vertices in the graph can be assigned
multiple neurons, we define recursive neural networks as follows.

\begin{definition}
    A recursive network model is a monoidal functor $F: G \to \bf{NN}$ for a
    monoidal grammar $G$, such that $F(w) = 1$ for $w \in V \sub G_0$.
    Given a choice of parameters $\theta :  W \to \bb{R}$,
    the semantics of a parsed sentence $g : u \to s$ in $\bf{MC}(G)$
    is given by $I_\theta(F(g)) \in \bb{R}^{F(s)}$.
\end{definition}
\begin{remark}
    Although it is in general useful to consider non-planar graphs as the input
    of a recursive neural network, in applications to linguistics one can usually
    assume that the graphs are planar. In order to recover non-planar graphs
    from the definition above it is sufficient to add a swap to the signature $G$.
\end{remark}

With this definition at hand, we can look at the applications of RvNNs in
linguistics.
The recursive networks for sentiment analysis of \cite{socher2013a}
are functors from a context-free grammar to neural networks. In this case
RvNNs are shown to capture correctly the role of negation in changing the
sentiment of a review. This is because the tree structure induced by the context-free
grammar captures the part of the phrase which is being negated, as in the following
example.
\begin{equation*}
    \input{figures/not-very-good.tikz}
\end{equation*}
As shown in \cite[Figure 4]{socher2013a} the recurrent network automtically learned
to assign a negative sentiment to the phrase above, even if the sub-phrase
``very good'' is positive. Generalising from this, researchers have shown that
recursive networks perform well on natural language entailment tasks \cite{bowman2015}.

\subsection{Attention is all you need?}\label{sec-attention}

The great advantage of neural networks is also their worst handicap. Since they
are able to approximate almost any function, we have no guarantees
as to what the resulting function will look like, an issue famously put as
the ``black box problem''.
The baffling thing is that neural networks seem to work best when their underlying
structure is unconstrained and neurons are arranged in a deep fully connected
network. We can see this pattern in the recent progression that
neural language models have made: from structured networks to attention.

Attention layers were added to recurrent neural networks by Bahdanau et al. in
2014 \cite{bahdanau2014}.
The architecture of Bahdanau et al. is composed of two recurrent neural
networks, an encoder $f$ and a decoder $g$, connected by an attention mechanism
as in the following diagram:
\begin{equation}\label{eq-bahdanau}
    \input{figures/recurrent-attention.tikz}
\end{equation}
where $x_i$ is the $i$th input word in the domain language, $y_i$ is the $i$th
output word in the target language, the $h_i$s and $s_i$s are the hidden states
of the encoder and decoder RNN respectively:
$$ h_i = f(x_i, h_{i-1}) \qquad s_i = g(s_{i-1}, y_{i-1}, c_i)$$
and $c_i$ is the context vector calculated from the input $\vec{x}$, the
hidden states $\vec{h}$ and the last decoder hidden state $s_{i-1}$ as
follows:
\begin{equation}\label{eq-attention0}
    c_i = \sum_{j=1}^{n} \alpha_{ij} h_j
\end{equation}
where
\begin{equation}\label{eq-attention1}
    \alpha_{ij} = (\tt{softmax}_n(\vec{e_{i}}))_j  \quad (\vec{e_{i}})_j = a(s_{i-1}, h_j)
\end{equation}
where $n = \size{\vec{x}}$ and $a$ is a (learned) feed-forward neural network.
The coefficients $\alpha_{ij}$ are called
\emph{attention weights}, they provide information as to how much input word $x_j$
is relevant for determining the output word $y_i$.
In the picture above, we used a comb notation for $\tt{Attention}$ informally to
represent the recursive relation between context vectors $c$ and hidden
states $h$ and $s$, capturing the flow of information in the architecture
of Bahdanau et al. This diagram should be read from top-left to bottom-right.
At each time step $t$, the attention mechanism computes the
context vector $c_t$ from the last decoder hidden state $s_{t-1}$ and all the
encoder hidden states $\vec{h}$. Infinite comb diagrams such as the one above may be formalised as monoidal streams over the category of neural networks \cite{dilavore2022}. A similar notation is used in Chapter 3.
Note that Bahdanau et al. model the encoder $f$ as a bidirectional RNN \cite{schuster1997} which produces hidden states $h_i$ that depend both on the previous and the following words. For simplicity we have depicted $f$ as a standard RNN.

In 2017, Vaswani et al. published the paper ``Attention is all you need''
\cite{vaswani2017} which introduced \emph{transformers}. They showed that the recurrent
structure is not needed to obtain state-of-the-art results in machine translation.
Instead, they proposed a model built up from three simple components:
positional encoding, attention and feed-forward networks, composed as in the
following diagram:
\begin{equation}
    \input{figures/attention-is-all.tikz}
\end{equation}
where
\begin{enumerate}
    \item $W_Q, W_K, W_V$ are (learned) linear functions, and $D$ is a (learned)
    decoder neural network.
    \item the positional encoding $\tt{pos}$ supplements the word vector $w_i$
    of dimension $n$ with another $n$ dimensional vector $\tt{pos}(i) \in \bb{R}^n$
    where $\tt{pos} : \bb{N} \to \bb{R}^n$ is given by:
    \begin{equation*}
    \tt{pos}(i)_j =
    \begin{cases*}
      \tt{sin}(\frac{j}{m^{2k/n}}) & if $j = 2k$ \\
      \tt{cos}(\frac{j}{m^{2k/n}}) & if $j = 2k + 1$
    \end{cases*}
    \end{equation*}
    where $m \in \bb{N}$ is a hyperparameter which determines the phase of the
    sinusoidal function (Vaswani et al. choose $m = 1000$ \cite{vaswani2017}),
    \item and attention is defined by the following formula:
    \begin{equation}\label{eq-attention2}
        \tt{Attention}(Q, K, V) = \tt{softmax}(\frac{Q K^T}{\sqrt{d_K}}) V
    \end{equation}
    where $Q$ and $K$ are vectors of the same dimension $d_K$ called \emph{query} and
    \emph{keys} respectively and $V$ is a vector called \emph{values}.
\end{enumerate}
Note that the positional encoding is needed since otherwise the network would
have no information about the position of words in a sentence. Using these
sinusoidal encodings turns absolute into relative position \cite{vaswani2017}.
Comparing this definition \ref{eq-attention2} of attention with the one used by
Bahdanau et al. \cite{bahdanau2014}, we note that in Equations \ref{eq-attention0}
and \ref{eq-attention1} the hidden states $\vec{h}$ and $\vec{s}$
play the role of keys $K$ and queries $Q$ respectively,
and the values $V$ are again taken to be encoder hidden states $\vec{h}$.
The main difference is that instead of a deep neural network (denoted
$a$ above) the queries and keys are pre-processed with \emph{linear} operations
($W_K$, $W_Q$ and $W_V$ above).
This architecture is the basis of BERT \cite{devlin2019} and its extensions
such as GPT-3 which use billions of parameters to achieve state of the art
results in a wide range of NLP tasks.

Reasoning about what happens inside these models and explaining their behaviour
with linguistic analysis is hard.
Indeed, the same architectures where sentences are replaced by images and
words by parts of that image give surprisingly accurate results in the image
recognition task \cite{dosovitskiy2020}, suggesting that linguistic principles
are insufficient for analysing these algorithms.
However, viewing deep neural networks as end-to-end black-box learners, it becomes
interesting to open the box and analyse the output of the intermediate layers
with the aim of understanding the different features that the network learns in
the process. Along these lines, researchers have found that neural network models
automatically encode linguistic features such as grammaticality \cite{coenen2019}
and dependency parsing \cite{alishahi2019}.
One possible line of future research would be to study what happens in the linear
world of keys, values and queries. One may take the latest forms of attention
\cite{vaswani2017} as the realisation that accurate predictions can be obtained
from a linear process $frac{Q K^T}{\sqrt{d_K}}$ by using a single $\tt{softmax}$
activation. We will give a bayesian interpretation of this activation function
in Chapter 3.

\subsection{Neural networks in DisCoPy}

We now show how to interface DisCoPy with Tensorflow/Keras neural networks
\cite{chollet2015keras, tensorflow2015-whitepaper}. In order to do this,
we need to define objects, identities, composition and tensor for Keras models.

The objects of our category of neural networks will be instances of \py{PRO},
a subclass of \py{monoidal.Ty} initialised by a dimension \py{n}.
A morphism from \py{PRO(n)} to \py{PRO(k)} is a neural network with input shape \py{(n, )}
and output shape \py{(k, )} (we only deal with flat shapes for simplicity).
A \py{neural.Network} is initialised by providing domain and codomain dimensions
together with a Keras model of that type.
Composition is easily implemented using the call method of Keras models.
For \py{tensor}, we first need to split the domain using \py{keras.layers.Lambda},
then we act on each subspace independently with \py{self} and \py{other}, and finally we
concatenate the outputs. Identities are simply Keras models with \py{outputs = inputs},
and we include a static method for constructing dense layer models.

\begin{python}\label{listing:neural.Network}
{\normalfont The category of Keras models}

\begin{minted}{python}
from discopy import monoidal, PRO
import tensorflow as tf
from tensorflow import keras

class Network(monoidal.Box):
    def __init__(self, dom, cod, model):
        self.model = model
        super().__init__("Network", dom, cod)

    def then(self, other):
        inputs = keras.Input(shape=(len(self.dom),))
        output = self.model(inputs)
        output = other.model(output)
        composition = keras.Model(inputs=inputs, outputs=output)
        return Network(self.dom, other.cod, composition)

    def tensor(self, other):
        dom = len(self.dom) + len(other.dom)
        cod = len(self.cod) + len(other.cod)
        inputs = keras.Input(shape=(dom,))
        model1 = keras.layers.Lambda(
            lambda x: x[:, :len(self.dom)],)(inputs)
        model2 = keras.layers.Lambda(
            lambda x: x[:, len(self.dom):],)(inputs)
        model1 = self.model(model1)
        model2 = other.model(model2)
        outputs = keras.layers.Concatenate()([model1, model2])
        model = keras.Model(inputs=inputs, outputs=outputs)
        return Network(PRO(dom), PRO(cod), model)

    @staticmethod
    def id(dim):
        inputs = keras.Input(shape=(len(dim),))
        return Network(dim, dim, keras.Model(inputs=inputs, outputs=inputs))

    @staticmethod
    def dense_model(dom, cod, hidden_layer_dims=[], activation=tf.nn.relu):
        inputs = keras.Input(shape=(dom,))
        model = inputs
        for dim in hidden_layer_dims:
            model = keras.layers.Dense(dim, activation=activation)(model)
        outputs = keras.layers.Dense(cod, activation=activation)(model)
        model = keras.Model(inputs=inputs, outputs=outputs)
        return Network(PRO(dom), PRO(cod), model)
\end{minted}
\end{python}

As an application, we use a \py{monoidal.Functor} to construct a Keras model
from a context-free grammar parse, illustrating the tree neural networks
of \cite{socher2013a}.

\begin{python}\label{listing:neural.Example}
{\normalfont Tree neural networks}

\begin{minted}{python}
from discopy.monoidal import Ty, Box, Functor

n, a, np, neg, star = Ty('N'), Ty('A'), Ty('NP'), Ty('Neg'), Ty('*')
not_, very, good = Box('not', star, neg), Box('very', star, a), Box('good', star, n)
P1, P2 = Box('P1', a @ n, np), Box('P2', neg @ np, np)
diagram = not_ @ very @ good >> Id(neg) @ P1 >> P2

dim = 5
ob = lambda x: PRO(dim)
ar = lambda box: Network.dense_model(len(box.dom) * dim, len(box.cod) * dim)
F = Functor(ob, ar, ob_factory=PRO, ar_factory=Network)
keras.utils.plot_model(F(diagram).model)
\end{minted}

\begin{center}
    \includegraphics[scale=0.12]{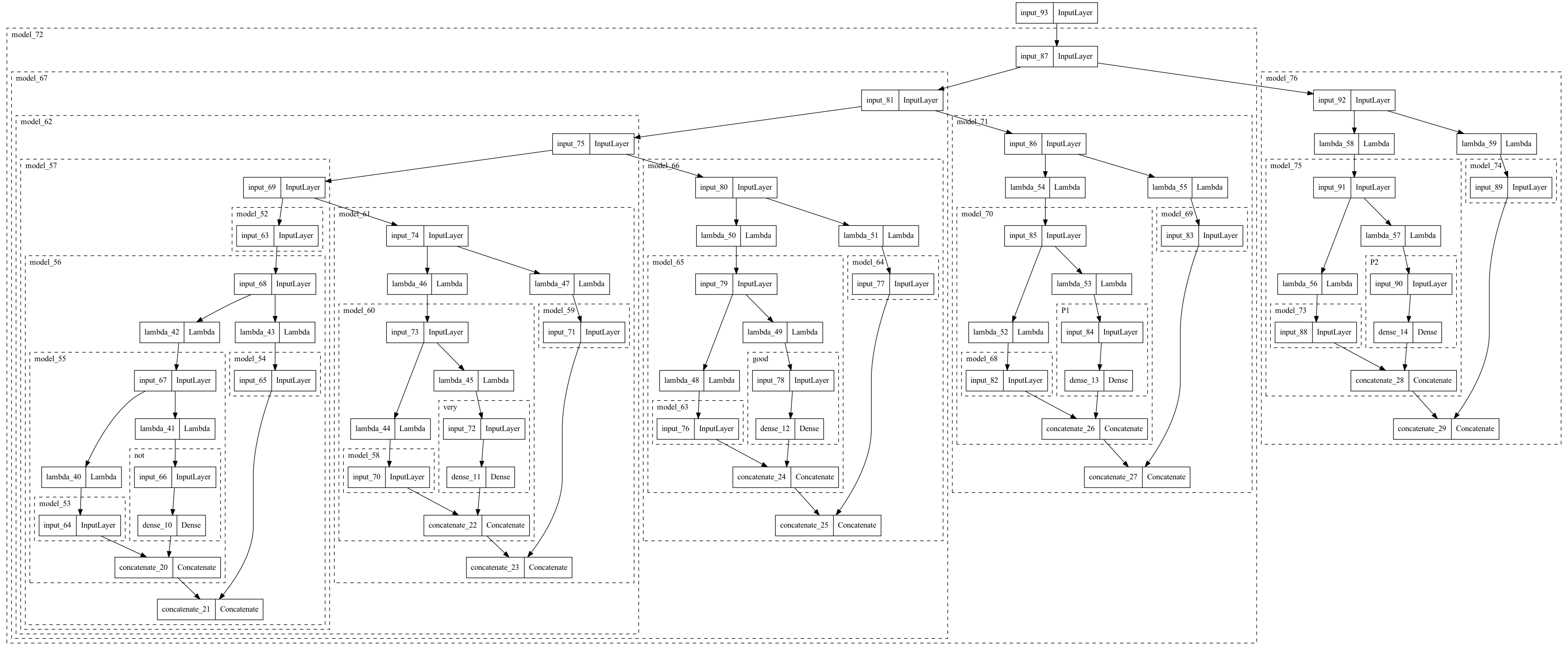}
\end{center}
\end{python}

\section{Relational models}\label{sec-rel-model}

The formal study of \emph{relations} was initiated by De Morgan in the mid
19th century \cite{demorgan1847}.
It was greatly developed by Peirce who only published small fragments of his
\emph{calculus of relations} \cite{peirce1897}, although much of it was
popularised in the influential work of Schroder \cite{schroder1890}.
In the first half of the twentieth century, this calculus was often
disregarded in favour of Frege's and Russell's approach to logic \cite{anellis2012},
until it was revived by Tarski \cite{tarski1941} who developed it into the rich
field of model theory.

With the advent of computer science, the calculus of relations came to be
recognized as a convenient framework for storing and accessing data, leading to
the development of \emph{relational databases} in the 1970s \cite{Codd70}.
SQL queries were introduced by Chamberlin and Boyce in 1976
\cite{ChamberlinBoyce76} and they are still used for accessing databases today.
\emph{Conjunctive queries} are an important subclass of SQL, corresponding to
the Select-Where-From fragment. They were introduced by Chandra and Merlin
\cite{ChandraMerlin77} who showed that their evaluation in
a relational database is an NP-complete problem, spawning a large field
of studies in the complexity of constraint satisfaction problems \cite{DalmauEtAl02}.
The blend of algebra and logic offered by the theory of relations
is particularly suited to a categorical formalisation and it has motivated
the work of Carboni and Walters \cite{carboni1987} as well as Brady and
Trimble \cite{brady2000} and Bonchi et al. \cite{bonchi2017b, BonchiEtAl18}
among others.

In this section, we start by reviewing the theory of relational databases
and conjunctive queries from a categorical perspective.
We then introduce relational models by transposing the definitions into a
linguistic setting.
This allows us to transfer results from database theory to linguistics
and define $\tt{NP-complete}$ \emph{entailment} and \emph{question answering}
problems. The correspondence of terminology between databases, category theory
and linguistics is summarized in the table below.

\begin{center}
\begin{tabular}{ |c|c|c| }
 \hline
 Databases & Algebra & Linguitics\\
 \hline
 Relational database & Cartesian bicategory & Relational model \\
 \hline
 attributes & objects & basic types\\
 schema & signature & lexicon\\
 query & morphism & sentence\\
 instance & functor & model\\
 containment & preorder enrichment & entailment \\
 \hline
\end{tabular}
\end{center}

\subsection{Databases and queries}\label{app-database}

We introduce the basic notions of relational databases, starting with an example.

\begin{example}\label{ex-1600}
\begin{center}
    \begin{tabular}{ |c|c|c| }
     \hline
     reader & book & writer \\
     \hline
     Spinoza & De Causa & Bruno\\
     Shakespeare & World of Wordes & Florio \\
     Florio & De Causa & Bruno\\
     Leibniz & Tractatus & Spinoza\\
     \hline
    \end{tabular}
\end{center}

Consider the structure of the table above, which we denote by $\rho$.
There is a set of \emph{attributes}
$A =\set{\text{reader}, \text{book}, \text{writer}}$ which name the
columns of the table and a set of \emph{data values} $D_a$ for each attribute
$a \in A$,
$$D_\text{r} = D_\text{w} =
\{\text{Spinoza, Shakespeare, Leibniz, Bruno, Florio}\}$$
$$D_{b} = \{\text{De Causa}, \text{World of Wordes}, \text{Tractatus}\}$$
A row of the table is a tuple $t \in \prod_{a \in A} D_a$,
which assigns a particular value $t_a$ to each attribute $a \in A$, e.g.
$(\text{Leibniz, Tractatus, Spinoza})$.
The table then consists in a set of tuples, i.e. a \emph{relation}
$\rho \sub \prod_{a \in A} D_a$.
\end{example}

A relational database is a collection of tables (relations), organised by a
\emph{schema}. Given a set of attributes $A$, a schema $\Sigma$ is a set of
relational symbols, together with a domain function $\tt{dom}: \Sigma \to A^\ast$.
The schema serves to specify the set of names $\Sigma$ for the tables in
a database together with the type of their columns.
For example we may have $\rho \in \Sigma$ for the table above with
$\tt{dom}(\rho) = (\text{reader}, \text{book}, \text{writer})$.
We have already encountered this type of structure in \ref{section-hypergraph},
where we used the term \emph{hypergraph signature} instead of schema.

\begin{definition}
    A relational database $K$ with schema $\Sigma$ is an assignment
    of each attribute $a \in A$ to a corresponding set of data values $D_a$,
    and an assignment of each symbol $R \in \Sigma$ to a relation
    $K(R) \sub \prod_{a \in \tt{dom}(R)} D_a$.
\end{definition}

Instead of working directly with relational databases, it is often
convenient to work with a simpler notion known as a \emph{relational structure}.
The schema is replaced by a \emph{relational signature}, which is a set of
symbols $\Sigma$ equipped with an \emph{arity} function $\tt{ar} :\Sigma \to \N$.

\begin{definition}[Relational structure]
    A relational structure $K$ over a signature $\Sigma$, also called a
    $\Sigma$-structure, is given by a set $U$ called the \emph{universe} and
    an interpretation $K(R) \sub U^{\tt{ar}(R)}$ for every symbol $R \in \Sigma$.
    We denote by $\cal{M}_\Sigma$ the set of $\Sigma$-structures with
    \emph{finite} universe $U(K)$.

    Given two $\Sigma$-structures $K, K'$, a homomorphism $f : K \to K'$ is a
    function $f : U(K) \to U(K')$ such that
    $\forall \ R \in \Sigma \s \forall \ \vec{x} \in U^{\tt{ar}(R)} \s \cdot
    \s \vec{x} \in K(R) \implies f(\vec{x}) \in K'(R)$.
\end{definition}
\begin{remark}
    Note that relational structures are the same as relational databases with
    only one attribute. Attributes $a \in A$ can be recovered by encoding them
    as predicates $a \in \Sigma$ of arity $1$ and one may take the universe to
    be the union of the sets of data values $U = \cup_{a \in A} D_a$.
\end{remark}

We consider the problem of finding a homomorphism between relational structures.

\begin{definition}
    \begin{problem}
    \problemtitle{$\tt{Homomorphism}$}
    \probleminput{$K, K' \in \cal{M}_\Sigma$}
    \problemoutput{$f : K \to K'$}
  \end{problem}
\end{definition}

\begin{proposition}\cite{GareyJohnson90}
  $\tt{Homomorphism}$ is $\tt{NP-complete}$.
\end{proposition}
\begin{proof}
  Membership may be shown to follow from Fagin's theorem: homomorphisms are defined by an existential second-order logic formula.
  Hardness follows by reduction from graph homomorphism: take $\Sigma = \set{\bullet}$ and $\tt{ar}(\bullet) = 2$ then a $\Sigma$-structure is a graph.
\end{proof}

The most prominent query language for relational databases is SQL \cite{ChamberlinBoyce76}.
\emph{Conjunctive queries} form a subset of SQL (corresponding to
the Select-Where-From fragment) with a convenient mathematical formulation. We
define conjunctive queries and the corresponding $\tt{Evaluation}$ and
$\tt{Containment}$ problems.
Let $\cal{X}$ be a (countable) set of variables, $\Sigma$ a relational signature
and consider the logical formulae generated by the following context-free grammar:
$$
\phi \s ::= \s \top \s\vert\s x = x' \s\vert\s \phi \land \phi \s\vert\s
\exists \ x \cdot \phi \s\vert\s R(\vec{x})
$$
where $x, x' \in \cal{X}$, $R \in \Sigma$ and $\vec{x} \in \cal{X}^{\tt{ar}(R)}$.
Let us denote the variables of $\phi$ by $\tt{var}(\phi) \sub \cal{X}$, its free
variables by $\tt{fv}(\phi) \sub \tt{var}(\phi)$ and its atomic formulae by
$\tt{atoms}(\phi) \sub \coprod_{R \in \Sigma} \text{var}(\phi)^{\tt{ar}(R)}$,
i.e. an atomic formula is given by $R(x_1, \dots, x_{\tt{ar}(R)})$ for some
variables $x_i \in \cal{X}$.

This fragment is called \emph{regular logic} in the category-theory literature
\cite{fong2018}. It yields conjuntive queries via the \emph{prenex normal form}.
\begin{definition}
    Conjunctive queries $\phi \in \cal{Q}_\Sigma$ are the prenex normal form
    $\phi = \exists\ x_0 \cdots \exists \ x_k \cdot \phi'$ of regular logic
    formulae, for the bound variables $\set{x_0, \dots, x_k} = \tt{var}(\phi)
    \setminus \tt{fv}(\phi)$ and $\phi' = \bigwedge \tt{atoms}(\phi)$.
    We denote by :
    $$\cal{Q}_\Sigma(k) = \{\phi \in \cal{Q}_\Sigma\, \vert\, \tt{fv}(\phi)= k\}$$
    the set of conjunctive queries with $k$ free variables.
\end{definition}

Given a structure $K \in \cal{M}_\Sigma$, let
$\tt{eval}(\phi, K) = \set{v \in U(K)^{\tt{fv}(\phi)} \s \vert \s  (K, v) \vDash \phi}$
where the satisfaction relation $(\vDash)$ is defined in the usual way.

\begin{definition}
    \begin{problem}
    \problemtitle{$\tt{Evaluation}$}
    \probleminput{$\phi \in \cal{Q}_\Sigma, \quad K \in \cal{M}_\Sigma$}
    \problemoutput{$\tt{eval}(\phi, K) \sub U(K)^{\tt{fv}(\phi)}$}
  \end{problem}
\end{definition}

\begin{definition}
    \begin{problem}
    \problemtitle{$\tt{Containment}$}
    \probleminput{$\phi, \phi' \in \cal{Q}_\Sigma$}
    \problemoutput{$\phi \sub \phi' \s\equiv\s \forall \ K \in \cal{M}_\Sigma \ \cdot \ \tt{eval}(\phi, K) \sub \tt{eval}(\phi', K)$}
  \end{problem}
\end{definition}

\begin{example}\label{ex-1600-query}
    Following from \ref{ex-1600} let  $U := D_{w} \cup D_b$ and fix the schema
    $\Sigma = \{ \text{read, wrote}\} \cup U$ with
    $\tt{ar}(w) = 1$ for $w \in U$ and
    $\tt{ar}(\text{read}) =  \tt{ar}(\text{wrote}) = 2$. Consider the relational
    structure $K \in \cal{Q}_\Sigma$ with universe $U$ and
    $K(w) = \set{w} \sub U$ for $w \in \Sigma - \set{\text{read, wrote}}$ and
    $K(\text{read}) \sub U \times U$ given by the first two columns of table $\rho$,
    $K(\text{wrote}) \sub U \times U$ given by the second two columns of $\rho$.
    The following is a conjunctive query with no free variables:
    $$\phi = \exists y, z \, \cdot \, \text{read}(x, y)  \land \text{Bruno}(y) \land
    \text{wrote}(x, z)$$
    Since $\phi$ has one free variables, $\tt{eval}(\phi, K) \sub U$.
    With $K$ as defined above there are three valuations
    $v : \tt{var}(\phi) = \set{x, y, z} \to U$ such that $(v, K) \vDash \phi$,
    yielding  $\tt{eval}(\phi, K) = \set{\text{Spinoza, Florio, Galileo}} \sub U$.
\end{example}

\begin{definition}[Canonical structure]
  Given a query $\phi \in \cal{Q}_\Sigma$, the canonical structure
  $CM(\phi) \in \cal{M}_\Sigma$ is given by $U(CM(\phi)) = \tt{var}(\phi)$ and
  $CM(\phi)(R) = \set{\vec{x} \in \tt{var}(\phi)^{\tt{ar}(R)} \ \vert \ R(\vec{x}) \in \tt{atoms}(\phi)}$
  for $R \in \Sigma$.
\end{definition}

This result was used by Chandra and Merlin to reduce from $\tt{Homomorphism}$
to both $\tt{Evaluation}$ and $\tt{Containment}$.

\begin{theorem}[Chandra-Merlin \cite{ChandraMerlin77}]\label{chandra-merlin}
  The problems $\tt{Evaluation}$ and $\tt{Containment}$ are logspace equivalent
  to $\tt{Homomorphism}$, hence $\tt{NP-complete}$.
\end{theorem}
\begin{proof}
  Given a query $\phi \in \cal{Q}_\Sigma$ and a structure $K \in \cal{M}_\Sigma$,
  query evaluation $\tt{eval}(\phi, K)$ is given by the set of homomorphisms $CM(\phi) \to K$.
  Given $\phi, \phi' \in \cal{M}_\Sigma$, we have $\phi \sub \phi'$ iff there
  is a homomorphism $f : CM(\phi) \to CM(\phi')$ such that
  $f(\tt{fv}(\phi)) = \tt{fv}(\phi')$.
  Given a structure $K \in \cal{M}_\Sigma$, we construct $\phi \in \cal{Q}_\Sigma$
  with $\tt{fv}(\phi) = \varnothing$, $\tt{var}(\phi) = U(K)$ and $\tt{atoms}(\phi) = K$.
\end{proof}

\subsection{The category of relations}

Let us now consider the structure of the category of relations.
A relation $R: A \to B$ is a subset $R \sub A \times B$ or equivalently
a predicate $R: A \times B \to \bb{B}$, we write $aRb$ for the logical
statement $R(a, b) = 1$.
Given $S: B \to C$, the composition $R;S : A \to C$  is defined as follows:
$$a R;S c \, \iff \, \exists b \in B \cdot a R b \land b S c \, .$$
Under this composition relations form a category denoted $\bf{Rel}$.

We can also construct the category of relations by considering the powerset
monad $\cal{P}: \bf{Set} \to \bf{Set}$ defined on objects by
$\cal{P}(X)= \set{S \sub X}$ and on arrows by
$f: X \to Y$ by  $\cal{P}(f): \cal{P}(X) \to \cal{P}(Y): S \mapsto f(S)$.
A relation $R : A \to B$ is the same as a function $R: A \to \cal{P}(B)$ and in fact
$\bf{Rel} = \bf{Kl}(\cal{P})$ is the Kleisli category of the powerset monad.

Any function $f: A \to B$ induces a relation
$I(f) = \set{(x, f(x))\, |\, x \in A} \sub A \times B$,
sometimes called the graph of $f$.
The tensor product of relations $R : A \to B$ and $T : C \to D$ is denoted
$R \otimes T: A \times C \to B \times D$ and defined by:
$$ (a, c) R \otimes T (b, d) \, \iff \, a R b \land c T d\, .$$
Equipped with this tensor product and unit the one-element set $1$,
$\bf{Rel}$ forms a symmetric monoidal category, with symmetry lifted from
$\bf{Set}$. We can thus use the graphical language of monoidal categories
for reasoning with relations.

Note that each object $A \in \bf{Rel}$ is self-dual, as witnessed by
the morphisms $\tt{cup}_A: A \times A \to 1$ and $\tt{cap}_A: 1 \to A \times A$,
defined by:
$$(a, a')\tt{cup}_A \iff (a = a') \iff \tt{cap}_A(a, a')$$
These denoted graphically as cups and caps and satisfy the snake equations \ref{axioms-rigid}.
Thus $\bf{Rel}$ is a compact-closed category.
Moreover, every object $A \in \bf{Rel}$ comes equipped with morphisms
$\Delta : A \to A \times A$ and $\nabla : A \times A \to A$ defined by:
$$a \Delta (a', a'')\, \iff \, (a = a' = a'') \, \iff (a, a') \nabla a''\, .$$
Together with the unit $\eta : 1 \to A$ and the counit $\epsilon: A \to 1$,
defined by $\eta = A = \epsilon$, the tuple $(\Delta, \epsilon, \nabla, \eta)$
satisfies the axioms of special commutative frobenius algebras,
making $\bf{Rel}$ a hypergraph category in the sense of Fong and Spivak \cite{fong2018a}.
Note that $\tt{cup} = \nabla ; \epsilon$ and $\tt{cap} = \eta ; \Delta$, moreover it is easy to show that the snake equations follow from the axioms of
special commutative Frobenius algebras.
Finally, we can equip the hom-sets $\bf{Rel}(A, B)$ with a preorder structure
given by:
$$R \leq S \iff (a R b \implies a S b)\, .$$
In category theory, this situation is known as a \emph{preorder enrichment}.
Equipped with this preorder enrichment, $\bf{Rel}$ forms a
\emph{Cartesian bicategory} in the sense of Carboni and Walters \cite{carboni1987}.

\subsection{Graphical conjunctive queries}

Bonchi, Seeber and Sobocinski \cite{BonchiEtAl18} introduced graphical
conjunctive queries (GCQ), a graphical calculus where query evaluation and
containment are captured by the axioms of the \emph{free Cartesian bicategory}
$\bf{CB}(\Sigma)$ generated by a relational signature $\Sigma$.
Cartesian bicategories were introduced by Carboni and Walters \cite{carboni1987}
as an axiomatisation of categories of relations, they are hypergraph categories
where every hom-set has a partial order structure akin to subset inclusion
between relations. We review the correspondence of Bonchi et al.
\cite{BonchiEtAl18}, between conjunctive queries and morphisms of free cartesian
bicategories. We refer to Appendix \ref{app-database} for an introduction
to relational datatabases with examples.

\begin{definition}[Cartesian bicategory]\cite{carboni1987}\label{def-CB}
    A cartesian bicategory $\bf{C}$ is a hypergraph category enriched in preorders
    and where the preorder structure interacts with the hypergraph structure
    as follows:
    \begin{equation}
        \scalebox{0.8}{\input{figures/cartesian-bicategory.tikz}}
    \end{equation}
    for all objects $a, b \in \bf{C}_0$ and morphisms $R : a \to b$.
    A morphism of Cartesian bicategories is a strong monoidal functor which
    preserves the partial order, the monoid and the comonoid structure.
\end{definition}

We recall the basic notions of relational databases and conjunctive queries.
A \emph{relational signature} is a set of symbols $\Sigma$
equipped with an \emph{arity} function $\tt{ar} :\Sigma \to \N$. This is
used to define a relational structure as a mathematical abstraction of a databases.

\begin{definition}[Relational structure]
    A relational structure $K$ over a signature $\Sigma$, also called a
    $\Sigma$-structure, is given by a set $U$ called the \emph{universe} and
    an interpretation $K(R) \sub U^{\tt{ar}(R)}$ for every symbol $R \in \Sigma$.
    We denote by $\cal{M}_\Sigma$ the set of $\Sigma$-structures with
    \emph{finite} universe $U(K)$.

    Given two $\Sigma$-structures $K, K'$, a homomorphism $f : K \to K'$ is a
    function $f : U(K) \to U(K')$ such that
    $\forall \ R \in \Sigma \s \forall \ \vec{x} \in U^{\tt{ar}(R)} \s \cdot
    \s \vec{x} \in K(R) \implies f(\vec{x}) \in K'(R)$.
\end{definition}

Let $\cal{X}$ be a (countable) set of variables, $\Sigma$ a relational signature
and consider the logical formulae generated by the following context-free grammar:
$$
\phi \s ::= \s \top \s\vert\s x = x' \s\vert\s \phi \land \phi \s\vert\s
\exists \ x \cdot \phi \s\vert\s R(\vec{x})
$$
where $x, x' \in \cal{X}$, $R \in \Sigma$ and $\vec{x} \in \cal{X}^{\tt{ar}(R)}$.
Let us denote the variables of $\phi$ by $\tt{var}(\phi) \sub \cal{X}$, its free
variables by $\tt{fv}(\phi) \sub \tt{var}(\phi)$ and its atomic formulae by
$\tt{atoms}(\phi) \sub \coprod_{R \in \Sigma} \text{var}(\phi)^{\tt{ar}(R)}$,
i.e. an atomic formula is given by $R(x_1, \dots, x_{\tt{ar}(R)})$ for some
variables $x_i \in \cal{X}$.
This fragment is called \emph{regular logic} in the category-theory literature
\cite{fong2018}. It yields conjuntive queries via the \emph{prenex normal form}.

\begin{definition}
    Conjunctive queries $\phi \in \cal{Q}_\Sigma$ are the prenex normal form
    $\phi = \exists\ x_0 \cdots \exists \ x_k \cdot \phi'$ of regular logic
    formulae, for the bound variables $\set{x_0, \dots, x_k} = \tt{var}(\phi)
    \setminus \tt{fv}(\phi)$ and $\phi' = \bigwedge \tt{atoms}(\phi)$.
    We denote by :
    $$\cal{Q}_\Sigma(k) = \{\phi \in \cal{Q}_\Sigma\, \vert\, \tt{fv}(\phi)= k\}$$
    the set of conjunctive queries with $k$ free variables.
\end{definition}

\begin{proposition}\label{structure-functor-bijection}
    There is a bijective correspondence between relational structures
    with signature $\sigma: \Sigma \to \N = \set{x}^\ast$
    and monoidal functors $K: \Sigma \to \bf{Rel}$ such that $K(x) = U$.
\end{proposition}
\begin{proof}
    Given a schema $\tt{dom} : \Sigma \to A^\ast$, the data for a monoidal
    functor $K : \Sigma \to \bf{Rel}$ is an assignment of each $ a \in A$ to
    a set of data-values $D_a = K(a)$ and of each symbol $R \in \Sigma$ to a
    relation $K(R)\sub \prod_{a \in \tt{dom}(R)} D_a$. This is precisely
    the data of a relational database. Relational structures are a sub-example
    with $A = \set{x}$.
\end{proof}

Bonchi, Seeber and Sobocinski show that queries can be represented as diagrams
in the free cartesian bicategory, and that this translation is semantics preserving.
Let $\bf{CB}(\Sigma)$ be the free Cartesian bicategory generated by one
object $x$ and arrows $\set{ R : 1 \to x^\tt{ar}(R)}_{R \in \Sigma}$,
see \cite[def.~21]{BonchiEtAl18}.

\begin{proposition}(\cite[prop.~9,10]{BonchiEtAl18})\label{translation}
  There is a two-way translation between formulas and diagrams:
  $$\Theta : \cal{Q}_\Sigma \leftrightarrows \bf{CB}(\Sigma): \Lambda$$
  which preserves the semantics, i.e.
  such that for all $\phi, \phi' \in \cal{Q}_\Sigma$ we have
  $\phi \sub \phi' \iff \Theta(\phi) \leq \Theta(\phi')$,
  and for all arrows $d, d' \in \bf{CB}(\Sigma)$, $d \leq d' \iff \Lambda(d) \sub \Lambda(d')$.
\end{proposition}
\begin{proof}
  The translation is defined by induction from the syntax of regular logic
  formulae to that of GCQ diagrams and back. Note that given
  $\phi \in \cal{Q}_\Sigma$ with $\size{\tt{fv}(\phi)} = n$, we
  have $\Theta(\phi) \in \bf{CB}(\Sigma)(0, n)$ and similarly we have
  $\tt{fv}(\Lambda(d)) = m + n$ for $d \in \bf{CB}(\Sigma)(m, n)$, i.e.
  open wires correspond to free variables.
\end{proof}

\begin{example}
    The translation works as follows. Given a morphism in $\bf{CB}(\Sigma)$,
    normalized according to \ref{prop-hyp-normal-form},
    the spiders are interpreted as variables and the boxes as relational symbols.
    For example, assuming $A, B, C \in \Sigma$, the following morphism
    $f: x \to x \in \bf{CB}(\Sigma)$
    \begin{equation*}
        \input{figures/hypergraph-to-query.tikz}
    \end{equation*}
    is mapped to the query:
    $$\Lambda(f) = \exists x_1, x_2 \, \cdot \, A(x_0, x_1) \land B(x_0, x_1, x_2) \land C(x_1, x_3)$$
    The query from Example \ref{ex-1600-query}, is mapped by $\Theta$ to the
    diagram:
    \begin{equation*}
        \input{figures/query-to-diagram.tikz}
    \end{equation*}
\end{example}

\begin{proposition}\label{isomorphisms}
  Let $[\bf{CB}(\Sigma), \bf{Rel}]$ denote the set of morphisms of
  Cartesian bicategories, there are bijective correspondences between closed diagrams
  in $\bf{CB}(\Sigma)$, formulas in $\cal{Q}_\Sigma$ with no free variables
  and models with signature $\Sigma$.
  $$\bf{CB}(\Sigma)(0,0)
  \s\stackrel{(1)}{\simeq}\s \set{\phi \in \cal{Q}_\Sigma \ \vert\ \tt{fv}(\phi) = \varnothing}
  \s\stackrel{(2)}{\simeq}\s \cal{M}_\Sigma
  \s\stackrel{(3)}{\simeq}\s [\bf{CB}(\Sigma), \bf{Rel}]$$
\end{proposition}
\begin{proof}
  (1) follows from theorem \ref{translation}, (2) from theorem \ref{chandra-merlin}
  and (3) follows from proposition \ref{structure-functor-bijection} since
  any monoidal functor $\Sigma \to \bf{Rel}$ induces a morphism of Cartesian
  bicategories $\bf{CB}(\Sigma) \to \bf{Rel}$.
\end{proof}

\subsection{Relational models}\label{section-relational-models}

We have seen that relational databases are functors
$K: \Sigma \to \bf{Rel}$ from a relational signature $\Sigma$. It is natural to
to generalise this notion by considering functors $G \to \bf{Rel}$ where
$G$ is a formal grammar.
Any of the formal grammars studied in Chapter 1 may be used to build a relational
model. However, it is natural to pick a grammar $G$ that we can easily interpret
in $\bf{Rel}$. In other words, we are interested in grammars which have common
structure and properties with the category of relations.
Recall that $\bf{Rel}$ is compact-closed
with the diagonal and its transpose as cups and caps.
This makes rigid grammars particularly suited for
relational semantics, since we can interpret cups and caps using the compact
closed structure of $\bf{Rel}$.

\begin{definition}[Relational model]
    A relational model is a rigid monoidal functor $F : G \rightarrow \bf{Rel}$
    where $G$ is a rigid grammar.
\end{definition}

We illustrate relational models with an example.

\begin{example}[Truth values]
    Let us fix the vocabulary $V = U +
    \set{\text{read, wrote}}$, where $U = D_{w} \cup D_b \cup D_r$
    is the set of data values from Example \ref{ex-1600}.
    Consider the pregroup grammar defined by the following lexicon:
    $$ \Delta(x) = \set{n} \, , \quad
    \Delta(\text{read}) = \Delta(\text{wrote}) = \set{n^r s n^l}$$
    for all $x \in U \sub V$. We build a functor $F: \Delta \to \bf{Rel}$, defined
    on objects by $F(n) = U$, and $F(w) = 1 = F(s)$ for all $w \in V$,
    on proper nouns by $F(x \to n) = \set{x} \sub U$ for $x \in U \sub V$ and on
    verbs as follows:
    \begin{equation*}
        \input{figures/rel-pregroup-functor.tikz}
    \end{equation*}
    where $\rho : 1 \to U \otimes U \otimes U \in \bf{Rel}$ is the table (relation)
    from Example \ref{ex-1600}.
    Interpreting cups and caps in $\bf{RC}(\Delta)$ with their counterparts in
    $\bf{Rel}$, we can evaluate the semantics of the sentence
    $g :\text{Spinoza read Bruno} \to s$:
    \begin{equation*}
        \input{figures/rel-functor-1.tikz}
    \end{equation*}
    obtaining a morphisms $F(g) : 1 \to 1 \in \bf{Rel}$, which is simply a
    \emph{truth value} given by the evaluation of the following query:
    $$ F(g) = \top \iff \exists x, y \in U \, \cdot \, F(\text{Spinoza})(x) \land
    F(\text{read})(x, y) \land F(\text{Bruno})(y) $$
    which is true with our definition of $F$.
\end{example}

From this example, we see that relational models can be used to give a truth
theoretic semantics by interpreting the sentence type $s$ as the unit of the
tensor in $\bf{Rel}$. We can also use these models to answer questions.

\begin{example}[Questions]\label{1600-questions}
    Add a question type $q \in B$ and the question word $\text{Who} \in V$ to
    the pregroup grammar above with $\Delta(\text{Who}) = q\, s^l \, n$. Then
    there is a grammatical question $g_q : \text{Who read De Causa} \to s$
    in $\bf{RC}(\Delta)$ given by the following diagram:
    \begin{equation*}
        \input{figures/rel-functor-2.tikz}
    \end{equation*}
    Let $F(q) = U$ and $F(\text{Who} \to q\, s^l \, n) =
    \tt{cap}_U \sub U \otimes U$. Then evaluating the question above in $F$
    yields $F(g_q) = \set{\text{Spinoza, Florio}} \sub U$.
\end{example}

As shown by Sadrzadeh et al. \cite{sadrzadeh2014}, we may give semantics to the
relative pronoun ``that'' using the Frobenius algebra in $\bf{Rel}$.

\begin{example}[Relative pronouns]\label{ex:rel-pron}
    Add the following lexical entries:
    $$\Delta(\text{that}) = \set{n^r \, n \, s^l \, n^{ll}, n^r \, n \, s^l \, n}\, , \quad \Delta(\text{a}) = \set{d}\, ,
    \quad \Delta(\text{book}) = \set{d^r n} \, .$$.
    Then there is a grammatical question
    $g_q' : \text{Who read a book that Bruno wrote} \to q$
    in $\bf{RC}(\Delta)$ given by the following diagram:
    \begin{equation*}
        \input{figures/rel-functor-3.tikz}
    \end{equation*}
    Let $F(d) = 1$, $F(\text{a} \to d) = \top$,
    $F(\text{that} \to n^r \, n \, s^l \, n) =
    \nu \cdot \delta \cdot (\delta \otimes \tt{id}_U): 1 \to U^{3}$
    (the spider with 3 outputs and 0 inputs) and
    $F(\text{book} \to d^r \, n) = D_{b} \sub U$. Then evaluating the question
    above in $F$ yields $F(g_q') = \set{\text{Spinoza, Florio, Galileo}} \sub U$,
    as expected.
\end{example}

The first linguistic problem that we consider is the task of computing the
semantics of a sentence $u \in \cal{L}(G)$ for a pregroup grammar $G$ in a given
relational model $F: G \to \bf{Rel}$. Throughout this section and the next, we
assume that $G$ is a rigid grammar.

\begin{definition}
    \begin{problem}
      \problemtitle{$\tt{RelSemantics}(G)$}
      \probleminput{$g \in \bf{RC}(G)(u, s)$, $F: G \to \bf{Rel}$}
      \problemoutput{$F(g)$}
    \end{problem}
\end{definition}

Since $\bf{Rel}$ is a cartesian bicategory, any monoidal functor from
$G$ to $\bf{Rel}$ must factor through a free cartesian bicategory.

\begin{lemma}\label{lemma-factorisation}
    Let $G= (V, B, \Delta, s)$ be a pregroup grammar.
    Any relational model $F: \bf{RC}(G) \to \bf{Rel}$ factors through a free
    cartesian bicategory $\bf{RC}(G) \to \bf{CB}(\Sigma) \to \bf{Rel}$, where
    $\Sigma$ is obtained from $\Delta$ via the map $P(B) \to B^\ast$ which
    forgets the adjoint structure.
\end{lemma}
\begin{proof}
    This follows from the universal property of the free cartesian bicategory.
\end{proof}

This lemma allows to reduce the problem of computing the semantics of sentences in
$\bf{Rel}$ to $\tt{Evaluation}$, thus proving its membership in $\tt{NP}$.

\begin{proposition}\label{prop-rel-np}
  There is a logspace reduction from $\tt{RelSemantics}(G)$ to conjunctive query
  $\tt{Evaluation}$, hence $\tt{RelSemantics} \in \tt{NP}$.
\end{proposition}
\begin{proof}
  The factorisation $K \circ L = F$ of lemma~\ref{lemma-factorisation} and
  the translation $\Lambda$ of theorem~\ref{translation} are in logspace,
  they give a query $\phi = \Lambda (L(r)) \in \cal{Q}_\Delta$ such that
  $\tt{eval}(\phi, K) = F(r)$.
\end{proof}

The queries that arise from a pregroup grammar are a particular subclass
of conjunctive queries. This leads to the question: what is the complexity
of $\tt{Evaluation}$ for this class of conjunctive queries?
We conjecture that these queries have bounded treewidth, i.e. that
they satisfy the tractability condition for the CSP dichotomy theorem \cite{Bulatov17}.

\begin{conjecture}
  For any pregroup grammar $G$, $\tt{RelSemantics}(G)$ is poly-time computable
  in the size of $(u, s) \in List(V) \times P(B)$ and in the size of the
  functor $F$.
\end{conjecture}

\subsection{Entailment and question answering}

We have seen that relational models induce functors $L: G \to \bf{CB}(\Sigma)$,
turning sentences into conjunctive queries.
Thus we can test whether a sentence $u \in \cal{L}(G)$ entails a second sentence
$u' \in \cal{L}(G)$ by checking containment of the corresponding queries.
More generally, we may consider models in a \emph{finitely presented} cartesian
bicategory $\bf{C}$, i.e. a cartesian bicategory equipped with a finite set of
\emph{existential rules} of the form
$\forall \ x_0 \ \cdots \ \forall \ x_k \ \cdot \ \phi \to \phi'$ for
$\phi, \phi' \in \bf{CB}(\Sigma)$ with
$\tt{fv}(\phi) = \tt{fv}(\phi') = \set{x_0, \dots, x_k}$. These are also called
tuple-generating dependencies in database theory \cite{Thomazo13}. They will
allow us to model more interesting forms of entailment in natural language.

\begin{definition}[CB model]
    A CB model for a rigid grammar $G$ is a monoidal functor $L : G \to \bf{C}$
    where $\bf{C}$ is a finitely presented Cartesian bicategory.
\end{definition}

\begin{example}
    Take $\bf{C}$ to be the Cartesian bicategory generated by the signature
    $\Sigma = \set{\text{Leib, Spi, infl, calc, phil, \dots}}$ as
    1-arrows with codomain given by the function $\tt{ar} : \Sigma \to \N$ and
    the following set of 2-arrows:
    \begin{equation*}
        \input{figures/existential-rules.tikz}
    \end{equation*}
    The composition of 2-arrows in $\bf{C}$ allows us to compute entailment, e.g.:
    \begin{equation*}
        \input{figures/derivation.tikz}
    \end{equation*}
    where the second and third inequations follow from the axioms of definition~\ref{def-CB},
    the first and last from the generating 2-arrows.

    Starting from the pregroup grammar $G$ defined in \ref{section-relational-models},
    and adding the following lexical entries:
    $$\Delta(\text{influenced}) = \Delta(\text{discovered}) = \set{n^r s n^l},
    \: \Delta(\text{calculus}) = \set{n}, \: \Delta(\text{philosopher}) = \set{d^r n} .$$
    We may construct a functor $L : G \to \bf{CB}(\Sigma)$ given on objects by
    $L(w) = L(s) = 1$ and $L(n) = x$, and on arrows by sending every lexical
    entry to the corresponding symbol in $\Sigma$ except for the question
    word ``who'' which is interpreted as a cap and the functional word ``that''
    which is interpreted as a spider with three outputs. Then one may check that
    the image of the sentence ``Spinoza influenced a philosopher that
    discovered calculus''  is grammatical in $G$ and that the corresponding
    pregroup reduction is mapped via $L$ to the last diagram in the derivation
    above.
\end{example}

\begin{definition}
    \begin{problem}
    \problemtitle{$\tt{Entailment}$}
    \probleminput{$r \in \bf{RC}(G)(u, s), \quad r' \in \bf{RC}(G)(u', s), \quad L : \bf{RC}(G) \to \bf{C}$}
    \problemoutput{$L(r) \leq L(r')$}
  \end{problem}
\end{definition}

\begin{proposition}\label{prop-entailment}
  $\tt{Entailment}$ is undecidable for finitely presented Cartesian bicategories.
  When $\bf{C}$ is freely generated (i.e. it has no existential rules),
  the problem reduces to conjunctive query $\tt{Containment}$.
\end{proposition}
\begin{proof}
    Entailment of conjunctive queries under existential rules is undecidable,
    see \cite{BagetMugnier02}. When $\bf{C} = \bf{CB}(\Sigma)$ is freely
    generated by a relational signature $\Sigma$, i.e. with no existential rules,
    theorem~\ref{translation} yields a logspace reduction to
    $\tt{Containment}$: $\tt{Entailment} \in \tt{NP}$.
\end{proof}

We now consider the following computational problem: given a natural language
corpus and a question, does the corpus contain an answer?
We show how to translate a corpus into a relational database so that question
answering reduces to query evaluation.

In order to translate a corpus into a relational database, it is not sufficient
to parse every sentence independently, since the resulting queries will have
disjoint sets of variables. The extra data that we need is a
\emph{coreference resolution}, which allows to link the common entities
mentioned in these sentences. In \ref{section-coreference}, we defined
a notion of pregroup grammar with coreference $G$ which allows to represent a
corpus of $k$ sentences as one big diagram $C \in \bf{Coref}(G)$, assuming that
both the pregroup parsing and the coreference resolution have been performed.
In order to interpret a pregroup grammar with coreference $G= (V, B, \Delta, I, R, s)$
in a cartesian bicategory $\bf{C}$, it is sufficient to fix a CB model from the pregroup
grammar $(V, B, \Delta, I, s)$ into $\bf{C}$ and choose an image for the
reference types in $R$. Then the coreference resolution is interpreted
using the Frobenius algebra in $\bf{C}$.

Fix a pregroup grammar with coreference $G$ and a CB model
$L : \bf{Coref}(G) \to \bf{CB}(\Sigma)$ with $L(s) = 0$,
i.e. grammatical sentences are mapped to closed formulae.
We assume that $L(q) = L(a)$ for $q$ and $a$ the question and answer
types respectively, i.e. both are mapped to queries with the same number of
free variable. Lexical items such as ``influence'' and ``Leibniz'' are mapped
to their own symbol in the relational signature $\Sigma$, whereas functional
words such as relative pronouns are sent to the Frobenius algebra of
$\bf{CB}(\Sigma)$, see \ref{section-relational-models} or \cite{sadrzadeh2013}.

We describe a procedure for answering a question given a corpus.
Suppose we are given a corpus $u \in V^\ast$ with $k$ sentences. Parsing and
resolving the coreference yields a morphism $C \in \bf{Coref}(G)(u, s^k)$. Using $L$ we
obtain a relational database from $C$ given by the canonical structure
induced by the corresponding query $K(C) = CM(L(C))$,
we denote by $E := \tt{var}(L(C))$ the universe of this relational structure.
Given a parsed question $g: v \to q \in \bf{Coref}(G)$,
we can answer the question $g$ by evaluating it in the model $K(C)$.

\begin{definition}
    \begin{problem}
    \problemtitle{$\tt{QuestionAnswering}$}
    \probleminput{$C \in \bf{Coref}(G)(u, s^k), g \in \bf{Coref}(G)(v, q)$}
    \problemoutput{$\tt{Evaluation}(K, L(g)) \sub E^{\tt{fv}(L(g))} \quad$
                   where $K = CM(L(C))$}
  \end{problem}
\end{definition}

\begin{proposition}\label{prop-qa}
  $\tt{QuestionAnswering}$ is $\tt{NP-complete}$.
\end{proposition}
\begin{proof}
  Membership follows immediately by reduction to $\tt{Evaluation}$.
  Hardness follows by reduction from $\tt{Evaluation}$. Indeed fix any
  relational structure $K$ and query $\phi$, using Proposition \ref{prop-svo-argument},
  we may build a corpus $C \in \bf{Coref}(G)$ and a question $g \in \bf{Coref}(G)$ such
  that $L(C) = K$ and $L(g) = \phi$.
\end{proof}

Note that this is an asymptotic result. In practice, the questions we may ask
are small compared to the size of the corpus. This would make the problem
tractable since $\tt{Evaluation}$ is only $\tt{NP}$-complete in the
combined size of database and query, but it becomes polytime computable when the
query is fixed \cite{Thomazo13}.

\section{Tensor network models}\label{sec-tensor-network}

Tensors arose in the work of Ricci and Levi-Civita in the end of the 19th
century \cite{ricci1900}. They were adopted by Einstein \cite{einstein1916},
who used repeated indices to denote their compositions, and applied to
quantum mechanics by Heisenberg \cite{kramers1925} to describe the possible
states of quantum systems.

In the 1970s, Penrose introduced a diagrammatic notation for manipulating tensor
expressions \cite{penrose1971}: wires represent vector spaces, nodes represent
multi-linear maps between them.
This work was one of the main motivations behind Joyal and Street's
graphical calculus for monoidal categories \cite{joyal1991geometry}, which
was later adopted in the development of Categorical Quantum Mechanics (CQM)
\cite{abramsky2008}.
The same notation is widely used in the Tensor Networks (TN) community
\cite{eisert2013, biamonte2015, desrosiers2019}.
Until recently, CQM and TN remained separated fields because they were interested
in different aspects of these graphical networks. On the one hand, rewriting
and axiomatics. On the other, fast methods for
tensor contraction and complexity theoretic guarantees.
These two lines of research are now seeing a fruitful exchange of ideas as
category theorists become more applied and vice-versa.
For instance, rewriting strategies developed in the context
of categorical quantum mechanics can be used to speed-up quantum computations
\cite{kissinger2020}, or solve satisfiability and counting problems
\cite{debeaudrap2020, townsend-teague2021a}.

The tools and methods developed by these communities are finding many
applications in artificial intelligence.
Tensor networks are widely used in machine learning, supported by efficient
contraction tools such as Google's TensorNetwork library \cite{efthymiou2019},
and are beginning to be applied to natural language processing
\cite{pestun2017a, zhang2019}.
Distributional Compositional models of meaning \cite{DisCoCat08, DisCoCat11}
(DisCoCat) arise naturally from Categorical Quantum Mechanics \cite{abramsky2008, coecke2017a},
In a nutshell, CQM provides us with graphical calculi to reason about tensor
networks, and DisCoCat provides a way of mapping natural language to
these calculi so that semantics is computed by tensor contraction.

In this section, we establish a formal connection between tensor networks and
functorial models. This allows us to transfer complexity and tractability results
from the TN literature to the DisCoCat models of meaning. In particular, we show
that DisCoCat models based on dependency grammars can be computed in polynomial
time. We end by discussing an extension of these tensor-based models where
bubbles are used to represent non-linear operations on tensor networks.

\subsection{Tensor networks}\label{section-tensor-networks}

In this section, we review the basic notions of tensor networks.
We take as a starting point the definition of tensor networks used in the
complexity theory literature on TNs \cite{arad2010, ogorman2019, gray2020}.

Let us denote an undirected graph by $(V, E)$ where $V$ is a finite set of vertices
and  $E \sub \set{\set{u, v} \, \vert \, u, v \in V}$ is a set of undirected
edges. The incidence set of a vertex $v$ is $I(v) = \set{e \in E \, \vert \, v \in e}$
and the degree of $v$ is the number of incident edges $\tt{deg}(v) = \size{I(v)}$.
An order $n$ tensor $T$ of shape $(d_1, \dots, d_n)$ with $d_i \in \bb{N}$
is a function $T : [d_1] \otimes \dots \otimes [d_n] \to \bb{S}$ where
$[d_i] = \set{1, \dots, d_i}$ is the ordered set with $d_i$ elements and
$\bb{S}$ is a semiring of numbers (e.g. $\bb{B}, \bb{N}, \bb{R}, \bb{C}$),
see \ref{sec-concrete}.

\begin{definition}[Tensor network]
    A tensor network $(V, E, T)$ over a semiring $\bb{S}$ is a
    undirected graph $(V, E)$ with edge weights $\tt{dim}: E \to \bb{N}$
    and a set of tensors $T = \set{T_v}_{v \in V}$
    such that $T_v$ is a tensor of order $\tt{deg}(v)$ and shape
    $(\tt{dim}(e_0), \dots, \tt{dim}(e_{\tt{deg}(v)}))$ for $e_i \in I(v) \sub E$.
    Each edge $e \in E$ corresponds to an index $i \in [\tt{dim}(e)]$
    along which the adjacent tensors are to be contracted.
\end{definition}

The contration of two tensors $T_0 : [m] \otimes [d] \to \bb{S}$ and
$T_1: [d] \otimes [l] \to \bb{S}$ along their common dimension $[d]$ is a
tensor $T_0 \cdot T_1: [k] \otimes [l] \to \bb{S}$ with:
$$T_0\cdot T_1(i, j) = \sum_{k \in [d]} T_0(i, k) T_1(k, l)$$
If $T_0: [m] \to \bb{S}$ and $T_1: [l] \to \bb{S}$ do not have a shared dimension
then we still denote by
$T_0 \cdot T_1: [m] \otimes [l] \to \bb{S}$, the outer (or tensor) product of $T_0$
and $T_1$ given by $T_0\cdot T_1(i, j) = T_0(i) T_1(j)$. Note that this is
sufficient to define the product $T_0 \cdot T_1$ for tensors of arbitary shape
since $(d_0, \dots, d_n)$ shaped tensors are in one-to-one correspondence with
order $1$ tensors of shape $(d_0 d_1 \dots d_n)$.

We are interested in the \emph{value} $\tt{contract}(V, E, T)$ of a tensor network
which is the number in $\bb{S}$ obtained by contracting the tensors in
$(V, E, T)$ along their shared edges. We may define this by first looking at the
order of contractions, called bubbling in \cite{aharonov2007, arad2010}.

\begin{definition}[Contraction order]\cite{ogorman2019}
    A contraction order $\pi$ for $(V, E, T)$ is a total order of its vertices,
    i.e. a bijection $\pi : [n] \to V$ where $n = \size{V}$.
\end{definition}

Given a contraction order $\pi$ for $(V, E, T)$ we get an algorithm for
computing the value of $(V, E, T)$, given by $\tt{contract}(V, E, T) = A_{n}$
where:
$$A_0 = 1 \in \bb{S} \qquad A_i = A_{i-1} \cdot T_{\pi(i)}$$
One may check that the value $A_n$ obtained is independent of the choice of
contraction order $\pi$, although the time and space required to compute the
value may very well depend on the order of contractions \cite{ogorman2019}. We can now
define the general problem of contracting a tensor network.

\begin{definition}
    \begin{problem}
      \problemtitle{$\tt{Contraction}(\bb{S})$}
      \probleminput{A tensor network $(V, E, T)$ over $\bb{S}$}
      \problemoutput{$\tt{contract}(V, E, T)$}
    \end{problem}
\end{definition}

\begin{proposition}
    $\tt{Contraction}(\bb{S})$ is $\tt{\#P}$-complete for
    $\bb{S} = \bb{N}, \bb{R}^+, \bb{R}, \bb{C}$,
    it is $\tt{NP}$-complete for $\bb{S} = \bb{B}$.
\end{proposition}
\begin{proof}
    $\tt{\#P}$-hardness was proved in \cite{biamonte2015} by reduction
    from $\#\tt{SAT}$. $\tt{NP}$-completeness when $\bb{S} = \bb{B}$ follows by
    equivalence with conjunctive query evaluation.
\end{proof}

Even though contracting tensor networks is a hard problem in general, it
becomes tractable, in several restricted cases of interest.

Let us look a bit more closely at the process of contraction.
Recall that given two $n\otimes n$ matrices $M_0, M_1$ the time complexity of matrix
multiplication is in general $n^3$. This is the simplest instance of a tensor
contraction which can be depicted graphically as follows:
\begin{equation*}
    \input{figures/matrix-multiplication.tikz}
\end{equation*}
The standard way to compute this contraction is with two for loops over the outgoing
wires for each basis vector in the middle wire, resulting in the cubic complexity.
There are Strassen-like algorithms with a better asymptotic runtime, but we do
not consider these here.
Now suppose $M_0$ and $M_1$ are tensors with indices of dimension $n$ connected
by one edge, with $k_0$ and $k_1$ outgoing edges respectively:
\begin{equation*}
    \input{figures/tensor-multiplication.tikz}
\end{equation*}
Then in order to contract the middle wire, we need $k_0 + k_1$ for loops for each
basis vector in the middle wire, resulting in the complexity $n^{k_0 + k_1 + 1}$.
In particular if we have $0$ edges connecting the tensors, i.e. we are taking
the outer product, then the time complexity is $O(n^{k_0 + k_1})$.
Similarly, if we have $k_2$ parallel edges connecting the tensors,
the time complexity is $O(n^{k_0 + k_1 + k_2})$.
We can speed up the computation by \emph{slicing} \cite{gray2020}, i.e. by
parallelizing the computation over $k_2 \cdot n$ GPUs, where each computes the
for loops for a basis vector of the middle wires, reducing the time complexity to
$n^{k_0 + k_1}$. This is for instance enabled by Google's Tensor Network library
\cite{efthymiou2019}.

For a general tensor network $(V, E, T)$, we have seen that the contraction strategy
can be represented by a contraction order $\pi: [\size{V}] \to V$, this may be
viewed as a \emph{bubbling} on the graph $(V, E)$:
\begin{equation*}
    \input{figures/bubbles.tikz}
\end{equation*}
At each contraction step $i$, we are contracting the tensor
$A_{i-1}$ enclosed by the bubble with $T_{\pi(i)}$. From the considerations
above we know that the time complexity will depend on the number of edges
crossing the bubble at each time step. For instance in the bubbling above, there
are at most $3$ edges crossing a bubble at each time step,
We can capture this complexity using the
notion of \emph{bubblewidth} \cite{aharonov2007}, also known as the cutwidth
\cite{ogorman2019}, of a tensor network.

\begin{definition}[Bubblewidth]
    The bubblewidth of a contraction order $\pi: [n] \to V$ for a graph $(V, E)$ is given by:
    $$\tt{bubblewidth}(\pi) = \tt{max}_{j \in [n]}
    \size{ \{ \set{\pi(j), \pi(k)} \in E \, \vert \, i \leq j < k \}})$$
    The bubblewidth of a tensor network $(V, E)$ is the minimum bubblewidth of a
    contraction order $\pi$ for $(V, E)$:
    $$BW(V, E) = \tt{min}_{\pi : [n] \to V}(\tt{bubblewidth}(\pi))$$
\end{definition}

Interestingly, the bubblewidth of a graph can be used to bound from above the much better
studied notions of treewidth and pathwidth. Let us denote by $PW(V, E)$ and
$TW(V, E)$ the path width and tree width of a graph $(V, E)$.

\begin{proposition}[Aharonov et al. \cite{aharonov2007}]\label{prop-bubble-width}
    $$TW(V, E) \leq PW(V, E) \leq 2 BW(V, E)$$
\end{proposition}

The following proposition is a weaker statement then the results of O'Gorman
\cite{ogorman2019}, that we can prove using the simple considerations about
contraction time given above.

\begin{proposition}
    Any tensor network $(V, E, T)$ can be contracted in $O(n^{BW(d) + a})$ time
    where $n = \tt{max}_{e \in E}(\tt{dim}(e))$ is the maximum dimension of
    the edges, $BW(d)$ is the bubble width of $d$ and
    $a = \tt{max}_{v \in V}(\tt{deg}(v))$ is the maximum order of the tensors in $T$.
\end{proposition}
\begin{proof}
    At each contraction step $i$, we are contracting the tensor
    $A_{i-1}$ enclosed by the bubble with $T_{\pi(i)}$. From the considerations
    above we know that the time complexity of such a contraction is at most
    $O(n^{w_i + \tt{deg}(\pi(i))})$ where $w_i$ is the number of
    edges crossing the $i$-th bubble and $\tt{deg}(\pi(i))$ is the degree of the
    $i$-th vertex in the contraction order $\pi$. The overall time complexity
    is the sum of these terms, which is at most of the order $O(n^{BW(d) + a})$.
\end{proof}

Therefore tensor networks of bounded bubblewidth may be evaluated efficiently.

\begin{corollary}\label{prop-efficient-contraction}
    $\tt{Contraction}(\bb{S})$ can be computed in poly-time if the input
    tensor networks have bounded bubblewidth and bounded vertex degree.
\end{corollary}

\subsection{Tensor functors}

We now show that tensor networks may be reformulated in categorical language as
diagrams in a free compact-closed category equipped with a functor into the
category of matrices over a commutative semiring $\bf{Mat}_\bb{S}$,
defined in \ref{sec-concrete}.
From this perspective, the contraction order $\pi$ provides a way of turning the
compact-closed diagram into a premonoidal one, where the order of contraction is
the linear order of the diagram.

\begin{proposition}\label{traslation-tns-compact-closed}
    There is a semantics-preserving translation between tensor networks $(V, E, T)$
    over $\bb{S}$ and pairs $(F, d)$ of a diagram $d: 1 \to 1 \in \bf{CC}(\Sigma)$ in a
    free compact closed category and a monoidal functor $F: \Sigma \to \bf{Mat}_\bb{S}$.
    This translation is semantics-preserving in the sense that
    $\tt{contract}(V, E, T) = F(d)$.
\end{proposition}
\begin{proof}
   Given a tensor network $(V, E, T)$ we build the signature $\Sigma = V \xto{\tt{cod}} E^\ast$
   where $\tt{cod}(v) = I(v)$ is (any ordering of) the incidence set of vertex $v$,
   i.e. vertices correspond to boxes and edges to generating objects.
   The set of tensors $\set{T_v}_{v \in V}$ induces a monoidal functor
   $F: \Sigma \to \bf{Mat}_\bb{S}$ given on objects $e \in E$ by $F(e) = \tt{dim}(e)$
   and on boxes $v \in V$ by $F(v_i) = T_i : 1 \to \tt{dim}(e_0)
   \otimes \dots \otimes \tt{dim}(e_{\tt{deg(v)}})$ for $e_i \in \tt{cod}(v_i) = I(v)$.
   We build a diagram $d: 1 \to 1 \in \bf{CC}(\Sigma)$
   by first tensoring all the boxes in $\Sigma$ and then composing it with a morphism
   made up only of cups and swaps, where there is a cap connecting an output
   port of box $v \in \Sigma$ with an output port of $v' \in \Sigma$,
   whenever $\set{v, v'} \in E$.
   \begin{equation*}
       \input{figures/tn-to-diagram.tikz}
   \end{equation*}
   One can check that $F(d) = \tt{contract}(V, E, T)$ since composing
   $T_0 \otimes T_1: 1 \to m \otimes n \otimes n \otimes l$ with $\tt{id}_m \otimes
   \tt{cup}_n \otimes \tt{id}_l$ in $\bf{Mat}_\bb{S}$ corresponds to contracting the tensors $T_0$
   and $T_1$ along their common dimension $n$.
   For the other direction, it is easy to see that any closed diagram in $\bf{CC}(\Sigma)$
   induces a graph with vertices $\Sigma$ and edges given by the struture of the diagram.
   Then a functor $F: \Sigma \to \bf{Mat}_\bb{S}$ yields precisely the data of the
   tensors $\set{T_{v}}_{v \in \Sigma}$ where $T_v = F(v)$.
\end{proof}

It is often useful to allow a special type of vertex in the tensor network.
These are called \emph{COPY tensors} in the tensor network literature
\cite{biamonte2015, ogorman2019, gray2020}, where they are used for optimizing
tensor contraction.
They also appear throughout categorical quantum mechanics \cite{coecke2017a}
--- and most prominently in the ZX calculus \cite{vandewetering2020}
--- where they are called \emph{spiders}. We have seen in \ref{section-hypergraph}
that diagrams with spiders are morphisms in a free hypergraph category
$\bf{Hyp}(\Sigma)$. Since $\bf{Hyp}(\Sigma) \simeq \bf{CC}(\Sigma + \tt{Frob})/\cong$
(see \ref{section-hypergraph}), Proposition \ref{traslation-tns-compact-closed}
can be used to show that there is a semantics-preserving translation between
tensor networks with COPY tensors and pairs $(F, d)$ of a diagram $d \in \bf{Hyp}(\Sigma)$
and a functor $F: \Sigma \to \bf{Mat}_\bb{S}$.

\begin{example}[Conjunctive queries as TNs]
    Conjunctive queries are tensor networks over $\bb{B}$.
    Tensor networks also subsume probabilistic graphical models by taking the
    underlying category to be $\bf{Prob}$ or $\bf{Mat}_{\bb{R}^+}$ \cite{glasser2019}.
\end{example}

We now consider the problem of evaluating closed diagrams $d: 1 \to 1 \in \bf{CC}(\Sigma)$
using a tensor functor $F: \Sigma \to \bf{Mat}_\bb{S}$

\begin{definition}
    \begin{problem}
      \problemtitle{$\tt{FunctorEval}(\bb{S})$}
      \probleminput{$\Sigma$ a monoidal signature, $d: 1 \to 1 \in \bf{CC}(\Sigma)$, $F: \Sigma \to \bf{Mat}_\bb{S}$,}
      \problemoutput{$F(d)$}
    \end{problem}
\end{definition}

\begin{proposition}\label{prop-functor-eval}
    $\tt{FunctorEval}(\bb{S})$ is equivalent to $\tt{Contraction}(\bb{S})$
\end{proposition}
\begin{proof}
    This follows from Proposition \ref{traslation-tns-compact-closed}.
\end{proof}

The notion of contraction order has an interesting categorical counterpart.
It gives a way of turning a compact-closed diagram into a premonoidal one, such
that the width of the diagram is the bubblewidth of the contraction order.

\begin{proposition}\label{traslation-tns-premonoidal}
    There is a semantics-preserving translation between tensor networks $(V, E, T)$
    over $\bb{S}$ with a contraction order $\pi : [\size{V}] \to V$ and
    pairs $(F, d)$ of a diagram $d: 1 \to 1 \in \bf{PMC}(\Sigma + \tt{swap})$ in a
    free premonoidal category with swaps and a monoidal functor
    $F: \Sigma \to \bf{Mat}_\bb{S}$. This translation is semantics-preserving
    in the sense that $\tt{contract}(V, E, T) = F(d)$. Moreover, we have that
    $$\tt{bubblewidth}(\pi) = \tt{width}(d)$$
    with $\tt{width}(d)$ as defined in \ref{prop-combinatorial-encoding}.
\end{proposition}
\begin{proof}
    We can use Proposition \ref{traslation-tns-compact-closed} to turn $(V, E, T)$
    into a compact closed diagram $d: 1 \to 1 \in \bf{CC}(\Sigma)$ over a signature
    $\Sigma$ with only output types for each box. Given a contraction order
    $\pi: [n] \to V$, we modify the signature $\Sigma$ by transposing a output
    port $e = \set{v, v'} \in E$ of $v \in V$ into an input whenever
    $\pi^{-1}(v) > \pi^{-1}(v')$ forming a signature $\Sigma_\pi$
    Then we can construct a premonoidal diagram
    $$\pi(d) = \circ_{i =1}^k (\tt{perm}_i \otimes b_i) \in \bf{PMC}(\Sigma_\pi + \tt{swap})$$
    where $\tt{perm}_i \in \bf{PMC}(\tt{swap})$ is a morphism containing only swaps
    and identities.
    This is done by ordering the boxes according to $\pi$ and pushing all the wires
    to the left, as in the following example:
    \begin{equation}\label{eq-bubbling-diagram}
        \input{figures/bubbling.tikz}
    \end{equation}
    Note that no cups or caps need to be added since the input/output types of
    the boxes in the signature have been changed accordingly. The bubblewidth of
    $\pi$ is then precisely the width of the resulting premonoidal diagram.
\end{proof}

We may define the bubblewidth of a diagram $d \in \bf{CC}(\Sigma)$ as the
bubblewidth of the corresponding graph via Proposition \ref{traslation-tns-compact-closed}.
Also we define the dimension of a functor $F: \Sigma \to \bf{Mat}_\bb{S}$ as follows:
$$\tt{dim}(F) = \tt{max}_{x \in \Sigma_0}(F(x))$$
We may now derive the consequences of Proposition \ref{prop-efficient-contraction}
in this categorical context.

\begin{proposition}\label{prop-efficient-functor-eval}
    $\tt{FunctorEval}(\bb{S})$ can be computed in polynomial time if the input
    diagrams $d$ have bounded bubblewidth and the input functors $F$ have
    bounded dimension.
\end{proposition}
\begin{proof}
    This follows from the conjunction of Proposition \ref{prop-efficient-contraction}
    and the reduction from $\tt{FunctorEval}$ to $\tt{Contraction}$ of
    Proposition \ref{prop-functor-eval}.
\end{proof}

\subsection{DisCoCat and bounded memory}\label{section-disco}

DisCoCat models were introduced by Coecke et al. in 2008 \cite{DisCoCat08, DisCoCat11}.
In the original formalism, the authors considered functors from a pregroup grammar
to the category of finite dimensional real vector spaces and linear maps
(i.e. $\bf{Mat}_\bb{R}$). In this work, we follow \cite{coecke2018c, defelice2020}
in treating DisCoCat models as functors into categories of matrices over any semiring.
These in particular subsume the relational models studied in \ref{sec-rel-model} by taking $\bb{S} = \bb{B}$.
As shown in Section \ref{sec-concrete}, $\bf{Mat}_\bb{S}$ is a compact-closed
category for any commutative semiring. Therefore, the most suitable choice of syntax
for this semantics are \emph{rigid grammars}, for which we have a canonical
way of interpreting the cups and caps.
Thus, in this section, the grammars we consider are either pregroup grammars
\ref{section-pregroup}, or dependency grammars \ref{section-dependency},
or pregroup grammar with crossed dependencies \ref{section-pregroup} or
with coreference \ref{section-coreference}.

\begin{definition}
    A DisCoCat model is a monoidal functor $F: G \to \bf{Mat}_\bb{S}$ for a
    rigid grammar $G$ and a commutative semiring $\bb{S}$. The semantics of
    a sentence $g: u \to s \in \bf{RC}(G)$ is given by its image $F(g)$.
    We assume that $F(w) = 1$ for $w \in V$ so that
\end{definition}

Distributional models $F: G \to \bf{Mat}_\bb{R}$ can be constructed by
counting co-occurences of words in a corpus \cite{Grefenstette11}. The
image of the noun type $n \in B$ is a vector space where the inner product
computes noun-phrase similarity \cite{sadrzadeh2013}.
When applied to question answering tasks, distributional models can be used
to compute the distance between a question and its answer \cite{coecke2018c}.

\begin{example}
    As an example we may take the pregroup lexicon:
    $$\Delta(\text{Socrate}) = \Delta(\text{spettu}) = \set{n} \quad
    \Delta(\text{j\'e}) = \set{n^rsn^l, n^r n^r s} $$
    And we may define a DisCoCat model $F: \Delta \to \bf{Mat}_\bb{R}$
    with $F(n) = 2$, $F(s) = 1$ and $F(\text{Socrate}, n) = (1, 0) : 1 \to 2$,
    $F(\text{spettu}, n) = (-1, 0)$ and $F(\text{j\'e},n^rsn^l) = F(\text{j\'e}, n^r n^r s)= \tt{cap}_2 : 1 \to 2 \otimes 2$
    in $\bf{Mat}_\bb{R}$. Then the following grammatical sentences:
    \begin{equation*}
        \input{figures/socrate-je-spettu.tikz}
    \end{equation*}
    evaluate to a scalar in $\bb{R}$ given by the matrix multiplication:
    \begin{equation*}
    F(\text{``Socrate j\'e spettu''}) =
    \begin{bmatrix}
      1 & 0
    \end{bmatrix}
    \left[\begin{array}{cc}
      1 & 0 \\
      0 & 1 \\
    \end{array}\right]
    \begin{bmatrix}
      -1  \\
      0  \\
    \end{bmatrix}
    = -1 = F(\text{``Socrate spettu j\'e''})
    \end{equation*}
\end{example}

We are interested in the complexity of the following semantics problem.

\begin{definition}
    \begin{problem}
      \problemtitle{$\tt{TensorSemantics}(G, \bb{S})$}
      \probleminput{$g \in \bf{RC}(G)(u, s)$, $F: G \to \bf{Mat}_\bb{S}$}
      \problemoutput{$F(g)$}
    \end{problem}
\end{definition}

Given a sentence diagram $g \in \cal{L}(G)$, the pair $(g, F)$ forms a
tensor network. Computing the semantics of $g$ in $F$ amounts to contracting this
tensor network. Therefore $\tt{TensorSemantics}(G, \bb{S}) \in \tt{\#P}$.
For a general rigid grammar, this problem is $\#\tt{P}$-hard, since it subsumes
the $\tt{FunctorEval}(\bb{S})(\Sigma)$ problem.
When $G$ is a pregroup grammar with coreference, we can construct any
tensor network as a list of sentences connected by the coreference,
making the $\tt{TensorSemantics}(G, \bb{S})$ problem $\tt{\#P-complete}$.

\begin{proposition}\label{prop-tensor-semantics}
    $\tt{TensorSemantics}(G, \bb{S})$ where $G$ is a pregroup grammar with coreference
    is equivalent to $\tt{Contraction}(\bb{S})$.
\end{proposition}
\begin{proof}
    One direction follows by reduction to $\tt{FunctorEval}(\bb{S})$ since any model
    $F: G \to \bf{Mat}_\bb{S}$ factors through a free compact-closed category.
    For the other direction, fix any pair $(F, d)$ of a diagram $d: x \to y \in \bf{CC}(\Sigma)$
    and a functor $F: \Sigma \to \bf{Mat}_\bb{S}$. Using Proposition
    \ref{prop-svo-argument} we can build a corpus $C: u \to s^k \in \bf{Coref}(G)$,
    where the connectivity of the tensor network is encoded a coreference
    resolution, so that the canonical functor $\bf{Coref}(G) \to \bf{CC}(\Sigma)$
    maps $C$ to $d$.
\end{proof}

From a linguistic perspective, a contraction order for a grammatical sentence
$g: u \to s$ gives a reading order for the sentence
and the bubblewidth is the maximum number of tokens (or basic types)
that the reader should hold in memory in order to parse the sentence.
Of course, in natural language there is a natural reading order from left to
right which induces a ``canonical'' bubbling of $g$.
In light of the famous psychology experiments of Miller
\cite{miller1956}, we expect that the short-term memory required to parse a
sentence is bounded, and more precisely that $BW(g) = 7 \pm 2$.
For pregroup diagrams generated by a dependency grammar, it is easy to show
that the bubblewidth is bounded.

\begin{proposition}
    The diagrams in $\bf{RC}(G)$ generated by a dependency grammar $G$ have
    bubblewidth bounded by the maximum arity of a rule in $G$.
\end{proposition}
\begin{proof}
    Dependency relations are acyclic, and the bubblewidth for an acyclic graph
    is smaller than the maximum vertex degree of the graph, which is equal to the
    maximum arity of a rule in $G$, i.e. the maximum nuber of dependents for a
    symbol plus one.
\end{proof}

Together with Proposition \ref{prop-efficient-functor-eval}, we deduce that the problem
of computing the tensor semantics of sentences generated by a dependency grammar
is tractable.

\begin{corollary}\label{prop-tensor-dep}
    If $G$ is a dependency grammar then $\tt{TensorSemantics}(G)$ can be computed in
    polynomial time.
\end{corollary}

For a general pregroup grammar we also expect the generated diagrams to have bounded
bubblewidth, even though they are not acyclic. For instance, the cyclic
pregroup reductions given in \ref{ex-cyclic-pregroup} have constant bubblewidth
$4$, which is obtained by chosing a contraction order from the middle to the sides,
even though the naive contraction order of \ref{ex-cyclic-pregroup} from left to
right has unbounded bubblewidth. We end with a cojecture, since we were unable to
show pregroup reductions have bounded bubblewidth in general.

\begin{conjecture}
    Diagrams generated by a pregroup grammar $G$ have bounded bubble width and
    thus $\tt{TensorSemantics}(G, \bb{S})$ can be computed in polynomial time.
\end{conjecture}

%

\subsection{Bubbles}

We end this section by noting that a simple extension of tensor network models
allows to recover the full expressive power of neural networks.
We have seen that contraction strategies for tensor networks can be represented
by a pattern of bubbles on the diagram. These bubbles did not have any semantic
interpretation and they were just used as brackets, specifying the order of
contraction. We could however give them semantics, by interpreting bubbles as operators
on tensors. As an example, consider the following diagram, where each
box $W_i$ is interpreted as a matrix:
\begin{equation*}
    \scalebox{0.8}{\input{figures/neural-nets-revisited.tikz}}
\end{equation*}
Suppose that each bubble acts on the tensor it encloses by applying a
non-linearity $\sigma$ to every entry. Then we see that this diagram specifies
a neural network of depth $k$ and width $n$ where $n$ is the maximum dimension
of the wires. With this representation of neural networks, we have more control
over the structure of the network.
\begin{equation*}
    \scalebox{0.8}{\input{figures/caesar-crossed.tikz}}
\end{equation*}
Suppose we have a pregroup grammar $G_0$ and context-free grammar $G_1$ over the
same vocabulary $V$ and let $u \in \cal{L}(G_0) \cap \cal{L}(G_1)$ with
two parses $g_0 : u \to s$ in $\bf{RC}(G_0)$ and $g_1 : u \to s$ in $\bf{O}(G_1)$.
Take the skeleton of the context-free parse $g_1$, i.e. the tree without labels.
This induces a pattern of bubbles, as described in Example \ref{peirce-alpha}
on Peirce's alpha.
Fix a tensor network model $G_0 \to \bf{Mat}_\bb{S}$ for the pregroup grammar
with $F(w) = 1$ for $w \in V$ and choose an activation function
$\sigma: \bb{S} \to \bb{S}$. Then $\sigma$ defines a unary operator on homsets
$\cal{S}_\sigma : \prod_{x, y \in \bb{N}}\bf{Mat}_\bb{S}(x, y) \to
\prod_{x, y \in \bb{N}}\bf{Mat}_\bb{S}(x, y)$ called a \emph{bubble}, see
\cite{Haydon2020} or \cite{toumi2021a} for formal definitions. We can
combine the pregroup reduction with the pattern of bubbles as in the example
above.
We may then compute the semantics of $u$ as follows. Starting from the inner-most
bubble, we contract the tensor network enclosed by it and apply the activation
$\sigma$ to every entry in the resulting tensor. This gives a new bubbled network
where the inner-most bubble has been contracted into a box. And we repeat this process.
The procedure described above gives a way of controlling the non-linearities
applied to tensor networks from the grammatical structure of the sentence.
While this idea is not formalised yet, it points to a higher diagrammatic
formalism in which diagrams with bubbles are used to control tensor computations.
This diagrammatic notation can be traced back to Peirce and was recently used
in \cite{Haydon2020} to model negation in first-order logic and in
\cite{toumi2021a} to model differentiation of quantum circuits. We will also
use it in the next section to represent the non-linearity of knowledge-graph
embeddings and in Chapter 3 to represent the softmax activation function.

\section{Knowledge graph embeddings}\label{sec-embedding}

Recent years have seen the rapid growth of web-scale knowledge bases such as Freebase \cite{bollacker2008}, DBpedia \cite{lehmann2014} and Google's Knowledge Vault \cite{dong2014}. These resources of structured knowledge enable a wide
range of applications in NLP, including semantic parsing \cite{berant2013},
entity linking \cite{hakimov2012} and question answering \cite{bordes2014}.
Knowledge base \emph{embeddings} have received a lot of attention in the statistical
relational learning literature \cite{wang2017}. They approximate a knowledge
base continuously given only partial access to it, simplifying the querying
process and allowing to predict missing entries and relations --- a task
known as \emph{knowledge base completion}.

Knowledge graph embedding are of particular interest to us since they provide a
link between the Boolean world of relations and the continuous world of tensors.
They will thus provide us with a connection between the relational models of
Section \ref{sec-rel-model} and the tensor network models of Section \ref{sec-tensor-network}.
We start by defining the basic notions of knowledge
graph embeddings. Then we review three factorization models for KG embedding,
focusing on their expressivity and their time and space complexity. This
will lead us progressively from the category of relations to the category of
matrices over the complex numbers with its convenient factorization properties.

\subsection{Embeddings}

Most large-scale knowledge bases encode information according to the Resource
Description Framework (RDF), where data is stored as triples
$(\text{head}, \text{relation}, \text{tail})$ (e.g. $(\text{Obama}, \text{BornIn}, \text{US})$).
Thus a \emph{knowledge graph} is just a set of triples $K \sub
\cal{E} \times \cal{R} \times \cal{E}$ where $\cal{E}$ is a set of entities
and $\cal{R}$ a set of relations. This form of knowledge representation can
be seen as an instance of relational databases where all the relations are
\emph{binary}.

\begin{proposition}
    Knowledge graphs are in one-to-one correspondence with functors $\Sigma \to \bf{Rel}$
    where $\Sigma = \cal{R}$ is a relational signature containing only symbols of arity two.
\end{proposition}
\begin{proof}
    This is easy to see, given a knowledge graph $K \sub \cal{E} \times \cal{R} \times \cal{E}$,
    we can build a functor $F: \cal{R} \to \bf{Rel}$ defined on objects
    by $F(1) = \cal{E}$ and on arrows by
    $F(r) = \set{(e_0, e_1) \vert (e_0, r, e_1) \in K} \sub \cal{E} \times \cal{E}$.
    Similarly any functor $F: \Sigma \to \bf{Rel}$ defines a knowledge
    graph $K = \set{(\pi_1F(r), r, \pi_2F(r)) \, \vert\, r \in \Sigma}
    \sub F(1) \times \Sigma \times F(1)$.
\end{proof}

Higher-arity relations can be encoded into the graph through
a process known as \emph{reification}. To reify a $k$-ary relation $R$, we form
$k$ new binary relations $R_i$ and a new entity $e_R$ so that
$R(e_1, \dots, e_{k})$ is true iff $\forall i$ we have $R_i(e_R, e_i)$.
Most of the literature on embeddings focuses on knowledge graphs, and the
results which we present in this section follow this assumption.
However, some problems with reification have been pointed out in the literature,
and current research is aiming at extending the methods to knowledge
\emph{hypergraphs} \cite{fatemi2020}, a direction which we envisage also
for the present work.

Embedding a knowledge graph consists in the following learning problem.
Starting from a knowledge graph $K: \cal{E} \times \cal{R} \times \cal{E}
\to \bb{B}$, the idea is to approximate $K$ by a scoring function
$X : \cal{E} \times \cal{R} \times \cal{E} \to \bb{R}$
such that $\norm{\sigma(X) - K}$ is minimized where $\sigma: \bb{R} \to \bb{B}$
is any \emph{activation} function.

The most popular activation functions for knowledge graph embeddings are
approximations of the $\tt{sign}$ function which takes a real number to its
sign $\pm 1$, where $-1$ is interpreted as the Boolean $0$ (false) and $1$ is
interpreted as the Boolean $1$ (true).
In machine learning applications, one needs a differentiable version of
$\tt{sign}$ such as $\tt{tanh}$ or the sigmoid function.
In this section, we are mostly interested in the expressive power of knowledge
graph embeddings, and thus we define ``exact'' embeddings as follows.

\begin{definition}[Embedding]
    An \emph{exact} embedding for a knowledge graph $K \sub \cal{E} \times \cal{R} \times \cal{E}$,
    is a tensor $X : \cal{E} \times \cal{R} \times \cal{E} \to \bb{R}$ such
    that $\tt{sign}(X) = K$.
\end{definition}

Even though $\tt{sign}$ is not a homomorphism of semirings between $\bb{R}$ and
$\bb{B}$, it allows to define a notion of sign rank which has interesting
links to measures of learnability such as the VC-dimension \cite{alon2016}.

\begin{definition}[Sign rank]
    Given $R \in \bf{Mat}_\bb{B}$ the sign rank of $R$, denoted
    $\tt{rank}_\pm (R)$ is given by:
    $$ \tt{rank}_\pm (R) = \tt{min}\{\tt{rank}(X)\,\vert \,\tt{sign}(X)= R \, , \,
    X \in \bf{Mat}_\bb{R}\}$$
\end{definition}

The sign rank of a relation $R$ is often much lower than its rank. In fact,
we don't know any relation $R: n \to n$ with $\tt{rank}_\pm(R) > \sqrt{n}$.
The identity Boolean matrix $n \to n$ has sign rank $3$ for any $n$
\cite{alon2016}. Therefore any permutation of the rows and columns of an
identity matrix has sign rank $3$, an example is the relation ``is married to''.
For the factorization models studied in this
section, this means that the dimension of the embedding is potentially much
smaller than then the number of entities in the knowledge graph.

Several ways have been proposed for modeling the scoring function $X$,
e.g. using translations \cite{bordes2013a} or neural networks \cite{socher2013b}.
We are mostly interested in \emph{factorization} models where
$X: \cal{E} \times \cal{R} \times \cal{E} \to \bb{R}$ is treated as a tensor with
three outputs, i.e. a state in $\bf{Mat}_\bb{R}$. Assuming that $X$ admits a
factorization into smaller matrices allows to reduce the search space for
the embedding while decreasing the time required to predict missing entries.
The space complexity of the embedding is the amount of memory required
to store $X$. The time complexity of an embedding is
the time required to compute $\sigma(X \ket{s, v, o})$ given a triple.
These measures are particularly relevant when it comes to embedding integrated
large-scale knowledge graphs.
The problem then boils down to finding a good \emph{ansatz} for such a
factorization, i.e. an assumed factorization shape that reduces the time and space
complexity of the embedding.

\subsection{Rescal}

The first factorization model for knowledge graph embedding --- known as
$\tt{Rescal}$ --- was proposed by Nickel et al. in 2011 \cite{nickel2011}.
It models the real tensor $X$ with the following ansatz:
\begin{equation}\label{2.3-rescal}
    \input{figures/embeddings/rescal.tikz}
\end{equation}
where $E: \size{\cal{E}} \to n$ is the embedding for entities,
$W : n \otimes \size{\cal{R}} \otimes n \to 1$ is a real tensor,
and $n$ is a hyper-parameter determining the dimension of the embedding.

The well-known notion of rank can be used to get bounds on the dimension of
an embedding model. It is formalised diagrammatically as follows.

\begin{definition}[Rank]
    The rank of a tensor $X: 1 \to \otimes_{i=1}^n d_i$ in $\bf{Mat}_\bb{S}$
    is the smallest dimension $k$ such that $X$ factors as
    $X = \Delta_k^n ; \otimes_{i=1}^n E_i$ where $E_i : k \to d_i$ and
    $\Delta_k^n$ is the spider with no inputs and $n$ outputs of dimension $k$.
    \begin{equation*}
        \input{figures/embeddings/rank.tikz}
    \end{equation*}
\end{definition}

The lowest the rank, the lowest the search space for a factorization.
In fact if $X$ is factorized as in \ref{2.3-rescal} then
$\tt{rank}(X) = \tt{rank}(W) \leq \size{\cal{R}} n^2$.
The time and space complexity of the embedding are thus both quadratic in
the dimension of the embedding $n$.

Since any three-legged real tensor $X$ can be factorized as in \ref{2.3-rescal}
for some $n$, $\tt{Rescal}$ is a very expressive model and performs well.
However, the quadratic complexity is a limitation which can be avoided by
looking for different ansatze.

\subsection{DistMult}

The idea of $\tt{DistMult}$ \cite{yang2015} is to factorize the tensor $X$ via
\emph{joint orthogonal diagonalization}, i.e. the entities are embedded in a
vector space of dimension $n$ and relations are mapped to diagonal matrices on
$\bb{R}^n$. Note that the data required to specify a diagonal matrix $n \to n$
is just a vector $v: 1 \to n$ which induces a diagonal matrix $\Delta(v)$
by composition with the frobenius algebra $\Delta : n \to n \otimes n$ in
$\bf{Mat}_{\bb{R}}$.
$\tt{DistMult}$ starts from the assumption --- or ansatz ---
that $X$ can is factorized as follows:
\begin{equation}\label{distmult}
    \input{figures/embeddings/distmult.tikz}
\end{equation}
where $E: \size{\cal{E}} \to n$ and $W: \size{\cal{R}} \to n$ and
$\Delta : n \otimes n \otimes n \to 1$ is the spider with three legs.
Note that $\tt{rank}_\pm(K) \leq \tt{rank}(X) = \tt{rank}(\Delta) = n$,
so that both the time and space complexity of $\tt{DistMult}$ are linear $O(n)$.
$\tt{DistMult}$ obtained the state of the art on $\tt{Embedding}$ when it
was released. It performs especially well at modeling \emph{symmetric}
relations such as the transitive verb ``met'' which satisfies
``x met y'' iff ``y met x''. It is not a coincidence that $\tt{DistMult}$ is
good in these cases since a standard linear algebra result says that
a matrix $Y: n \to n$ can be orthogonally diagonalized if and only if $Y$
is a symmetric matrix. For tensors of order three we have the following
genralization, with the consequence that $\tt{DistMult}$ models \emph{only}
symmetric relations.

\begin{proposition}
    $X : 1 \to \size{\cal{E}} \times \size{\cal{R}} \times \size{\cal{E}}$ is jointly orthogonally
    diagonalizable if and only if it is a family of symmetric matrices indexed by $R$.
\end{proposition}
\begin{proof}
   \begin{equation*}
       \input{figures/embeddings/distmult-symmetry.tikz}
   \end{equation*}
\end{proof}

This called for a more expressive factorization method, allowing to represent
\emph{asymmetric} relations such as ``love'' for which we cannot assume that
``x loves y'' implies ``y loves x''.

\subsection{ComplEx}

Trouillon et al. \cite{Trouillon17} provide a factorization of any real
tensor $X$ into complex-valued matrices, that allows to model symmetric and asymmetric
relations equally well. Their model, called $\tt{Complex}$ allows to embed
any knowledge graph $K$ in a low-dimensional complex vector space.
We give a diagrammatic treatment of their results, working in the category
$\bf{Mat}_\bb{C}$, allowing us to improve their results.
The difference between working with real-valued
and complex-valued matrices is that the latter have additional structure on
morphisms given by conjugation. This is pictured, at the diagrammatic level,
by an asymmetry on the boxes which allows to represent the operations of adjoint,
transpose and conjugation as rotations or reflexions of the diagram. This graphical
gadget was introduced by Coecke and Kissinger in the context of quantum mechanics
\cite{coecke2017a}, we summarize it in the following picture:
\begin{equation}\label{conjugate-structure}
    \input{figures/conjugate-structure.tikz}
\end{equation}

The key thing that Trouillon et al. note is that any real square
matrix is the real-part of a diagonalizable complex matrix, a consequence of
Von Neumann's \emph{spectral theorem}.

\begin{definition}
    $Y : n \to n$ in $\bf{Mat}_\bb{C}$ is unitarily diagonalizable if it
    factorizes as:
    \begin{equation}
        \input{figures/embeddings/unitarily-diagonalizable.tikz}
    \end{equation}
    where $w : 1 \to n$ is a state and $E : n \to n$ is a unitary.
\end{definition}

\begin{definition}
    $Y$ is normal if it commutes with its adjoint: $Y Y^\dagger = Y^\dagger Y$.
    \begin{equation*}
        \input{figures/embeddings/normal.tikz}
    \end{equation*}
\end{definition}

\begin{theorem}[Von Neumann \cite{vonneumann1929}]
    $Y$ is diagonalizable if and only if $Y$ is normal.
\end{theorem}

\begin{proposition}[Trouillon \cite{Trouillon17}]
    Suppose $Y: n \to n$ in $\bf{Mat}_\bb{R}$ is a real square
    matrix, then $Z = Y + i Y^T$ is a normal matrix and $Re(Z) = Y$.
    Therefore there is a unitary $E$ and a diagonal complex matrix $W$ such
    that $Y = Re(E W E^\dagger)$. Graphically we have:
    \begin{equation}\label{trouillon}
        \input{figures/embeddings/trouillon.tikz}
    \end{equation}
    where the red bubble indicates the real-part operator.
\end{proposition}

Note that $\tt{rank}(A + B) \leq \tt{rank}(A) + \tt{rank}(B)$
which implies that $\tt{rank}(Z) = \tt{rank}(Y + i Y^T) \leq 2 \tt{rank}(Y)$.

\begin{corollary}
    Suppose $Y: n \to n$ in $\bf{Mat}_\bb{R}$ and $\tt{rank}(X) = k$,
    then there is $E : n \to 2k$ and $W: 2k \to 2k$ diagonal in
    $\bf{Mat}_\bb{C}$ such that $Y = Re(E W E^\dagger)$
\end{corollary}

Given a binary relation $R: \size{\cal{E}} \to \size{\cal{E}}$, the sign-rank
gives a bound on the dimension of the embedding.

\begin{corollary}
    Suppose $R: \size{\cal{E}} \to \size{\cal{E}}$ in $\bf{Mat}_\bb{B}$ and $\tt{rank}_\pm(R) = k$,
    then there is $E : \size{\cal{E}} \to 2k$ and $W: 2k \to 2k$ diagonal in
    $\bf{Mat}_\bb{C}$ such that $R = \tt{sign}(Re(E W E^\dagger))$
\end{corollary}

The above reasoning works for a single binary relation $R$.
However, given knowledge graph $K \sub \cal{E} \times \cal{R} \times \cal{E}$
and applying the reasoning above we will get $\size{\cal{R}}$
different embeddings $E_r: \size{\cal{E}} \to n_r$, one for each relation
$r \in \cal{R}$,
thus obtaining only the following factorization:
$K = \tt{sign}(\sum_{r \in \cal{R}} \Delta \circ (E_r \otimes W(r) \otimes E_r)$.
Which means that the overall complexity of the embedding is $O(\size{\cal{R}}n)$
where $n = \tt{max}_{r \in \cal{R}}(n_r)$.
The problem becomes: can we find a single embedding
$E: \size{\cal{E}} \to n$ such
that all relations are mapped to diagonal matrices over the same $n$?
In their paper, Trouillon et al. sketch a direction for this simplification
but end up conjecturing that such a factorization does \emph{not} exist.
We answer their conjecture negatively, by proving that for any real tensor $X$
of order 3, the complex tensor $Z = X + i X^T$ is jointly unitarily
diagonalizable.

\begin{definition}
    $X : n\, \otimes\, m \to n$ in $\bf{Mat}_\bb{C}$ is simultaneously unitarily
    diagonalizable if there is a unitary $E : n \to n$ and $W: m \to n$
    such that:
    \begin{equation}
        \input{figures/embeddings/sim-unitarily-diagonalizable.tikz}
    \end{equation}
\end{definition}

\begin{definition}
    $X: n\, \otimes\, m \to n$ is a commuting family of normal matrices if:
    \begin{equation}
        \input{figures/embeddings/sim-normal.tikz}
    \end{equation}
\end{definition}

The following is an extension of Von Neumann's theorem to the multi-relational
setting.

\begin{theorem}\cite{horn2012}
    $X : n\, \otimes\, m \to n$ in $\bf{Mat}_\bb{C}$ is a commuting family of
    normal matrices if and only if it is simultaneously unitarily diagonalizable.
\end{theorem}

Our contribution is the following result.

\begin{proposition}
    For any real tensor $X : n\, \otimes\, m \to n$ the complex tensor $Z$ defined
    by:
    \begin{equation*}
        \input{figures/embeddings/define-Z.tikz}
    \end{equation*}
    is a commuting family of normal matrices.
\end{proposition}
\begin{proof}
    This follows by checking that the two expressions below are equal, using
    only the snake equation:
    \begin{equation*}
        \scalebox{0.8}{\input{figures/embeddings/normal-Z-proof.tikz}}
    \end{equation*}
\end{proof}

\begin{corollary}
    For any real tensor $X: 1 \to  n\, \otimes\, m \otimes \, n$ there is a
    unitary $E: n \to n$ and $W: n \to m$ in $\bf{Mat}_\bb{C}$ such that:
    \begin{equation}
        \input{figures/embeddings/final-trouillon.tikz}
    \end{equation}
    where the bubble denotes the real-part operator.
\end{corollary}

\begin{proposition}
    For any knowledge graph $K \sub \cal{E} \times \cal{R} \times \cal{E}$ of
    sign rank $k = \tt{rank}_\pm(K)$, there is $E: \size{\cal{E}} \to 2k$ and
    $W : \size{\cal{R}} \to 2k$ in $\bf{Mat}_\bb{C}$ such that
    $$K = \tt{sign}(Re(\Delta \circ (E \otimes W \otimes E^\ast))) .$$
\end{proposition}

This results in an improvement of the bound on the size of the factorization
by a factor of $\size{\cal{R}}$.
Although this improvement is only incremental, the diagrammatic
tools used here open the path for a generalisation of the factorization result
to higher-order tensors which would allow to model higher-arity relations.
This may require to move into quaternion valued vector spaces where (at least)
ternary symmetries can be encoded.

Moreover, as we will show in Section \ref{sec-quantum-model}, quantum computers
allow to speed up the multiplication of complex valued tensors.
An implementation of $\tt{ComplEx}$ on a quantum computer was proposed by
\cite{ma2019b}. The extent to which quantum computers can be used to speed up
knowledge graph embeddings is an interesting direction of future work.

\section{Quantum models}\label{sec-quantum-model}

Quantum computing is an emerging model of computation which promises speed-up
on certain tasks as compared to classical computers. With the steady growth of
quantum hardware, we are approaching a time when quantum computers perform tasks
that cannot be done on classical hardware with reasonable resources.
Quantum Natural Language Processing (QNLP) is a recently proposed model
which aims to meet the data-intensive requirements of NLP algorithms with
the computational power of quantum hardware \cite{zeng2016, coecke2020a, meichanetzidis2020a}.
These models arise naturally from the categorical approaches to linguistics
\cite{DisCoCat11} and quantum mechanics \cite{abramsky2008}.
They can in fact be seen as instances of the tensor network models
studied in \ref{sec-tensor-network}. We have tested them
on noisy intermediate-scale quantum computers \cite{meichanetzidis2020a, lorenz2021},
obtaining promising results on small-scale datasets of around 100 sentences.
However, the crucial use of post-selection in these models, is a limit to their
scalability as the number of sentences grows.

In this section, we study the complexity of the quantum models based on a mapping
from sentences to pure quantum circuits \cite{meichanetzidis2020a, lorenz2021}.
Building on the work of Arad and Landau \cite{arad2010} on
the complexity of tensor network contraction, we show that the \emph{additive} approximation
(with scale $\Delta = 1$) of quantum models is a complete problem for $\tt{BQP}$,
the class of problems which can be solved in polynomial time by a quantum computer
with a bounded probability of error.
Note that this approximation scheme has severe limits when the amplitude we want
to approximate is small compared to the scale $\Delta$.
Thus the results may be seen as a negative statement about the first generation of
QNLP models.
However, specific types of linguistic structure (such as trees) may allow to
reduce the post-selection and thus the approximation scale.
Moreover, this puts QNLP on par with well-known $\tt{BQP}$-complete
problems such as approximate counting \cite{bordewich2009} and the evaluation
of topological invariants \cite{freedman2000a, freedman2002} to set a
standard for the next generations of QNLP models.

The development of DisCoPy was driven and motivated by the implementation of
these QNLP models. The \py{quantum} module of DisCoPy is described in \cite{toumi2022}, it features interfaces with the \py{tensor} module for
classical simulation, with PyZX \cite{kissinger2019} for optimization and with
tket \cite{sivarajah2020} for compilation and evaluation on quantum hardware.
In order to run quantum machine learning routines, we also developed diagrammatic tools for automatic differentiation \cite{toumi2021a}. The pipeline for performing
QNLP experiments with DisCoPy has been packaged in Lambeq \cite{kartsaklis2021}
which provides state-of-the-art categorial parsers and optimised classes for training
QNLP models.

\subsection{Quantum circuits}\label{subsec-quantum-computing}

In this section, we give a brief introduction to the basic ingredients of
quantum computing, and define the categories $\bf{QCirc}$ and $\bf{PostQCirc}$
of quantum circuits and their post-selected counterparts. A proper introduction
to the field is beyond the scope of this thesis and we point the reader to
\cite{coecke2017a} and \cite{vandewetering2020} for diagrammatic treatments.

The basic unit of information in a quantum computer is the \emph{qubit}, whose
\emph{state space} is the Bloch sphere, i.e. the set of vectors
$\psi  = \alpha \ket{0} + \beta \ket{1} \in \bb{C}^2$ with norm
$\norm{\psi} = \size{\alpha}^2 + \size{\beta}^2 = 1$.
Quantum computing is inherently probabilistic, we never observe the coefficients
$\alpha$ and $\beta$ directly, we only observe the probabilities that outcomes
$0$ or $1$ occur. These probabilities are given by the \emph{Born rule}:
$$P(i) = \size{\bra{i}\ket{\psi}}^2$$
with $i \in \set{0, 1}$ and where $\bra{i}\ket{\psi}$ is the inner product of
$\ket{i}$ and $\ket{\psi}$ in $\bb{C}^2$, also called \emph{amplitude}.
Note that the requirement that $\psi$ be of norm $1$ ensures that these
probabilities sum to $1$.
The \emph{joint state} of $n$ qubits is given by the tensor product
$\psi_1 \otimes \dots \otimes \psi_n$ and thus lives in $\bb{C}^{2n}$,
a Hilbert space of dimension $2^n$.
The evolution of a quantum system composed of $n$ qubits is described by a
\emph{unitary} linear map $U$ acting on the space $\bb{C}^{2n}$, i.e. a
linear map satisfying $U U^\dagger = \tt{id} = U^\dagger U$.
Where $\dagger$ (dagger) denotes the conjugate transpose.
Note that the requirement that $U$ be unitary ensures that
$\norm{U \psi} = \norm{\psi}$ and so $U$ sends quantum states to quantum states.

The unitary map $U$ is usually built up as a \emph{circuit} from some
set of basic \emph{gates}. Depending on the generating set of gates, only
certain unitaries can be built. We say that the set of gates is \emph{universal},
when any unitary can be obtained using a finite sequence of gates from this set.
The following is an example of a universal gate-set:
\begin{equation}\label{eq-universal-gate-set}
    \tt{Gates} = \set{CX, \, H,\, Rz(\alpha),\, \tt{swap}}_{\alpha \in [0, 2\pi]}
\end{equation}
where $CX$ is the controlled $X$ gate acting on $2$ qubits and defined on the
basis of $\bf{C}^{4}$ by:
\begin{equation*}
    CX (\ket{00}) = \ket{00} \quad CX (\ket{01}) = \ket{01} \quad
    CX (\ket{10}) = \ket{11} \quad CX (\ket{11}) = \ket{10}
\end{equation*}
$Rz(\alpha)$ is the $Z$ phase, acting on $1$ qubit as follows:
\begin{equation*}
    Rz(\alpha)(\ket{0}) = \ket{0} \qquad Rz(\alpha)(\ket{1}) = e^{i\alpha}\ket{1} \, .
\end{equation*}
$H$ is the Hadamard gate, acting on $1$ qubit as follows:
\begin{equation*}
    H(\ket{0}) = \frac{1}{\sqrt{2}}(\ket{0} + \ket{1}) \qquad
    H(\ket{1}) = \frac{1}{\sqrt{2}}(\ket{0} - \ket{1}) \, .
\end{equation*}
and the $\tt{swap}$ gate acts on $2$ qubits and is defined as usual.

We define the category of quantum circuits $\bf{QCirc} = \bf{MC}(\tt{Gates})$ as
a free monoidal category with objects natural numbers $n$ and arrows generated
by the gates in (\ref{eq-universal-gate-set}).

\begin{definition}[Quantum circuit]
    A quantum circuit is a diagram $c : n \to n \in \bf{MC}(\tt{Gates}) = \bf{QCirc}$
    where $n$ is the number of qubits, it maps to a corresponding unitary
    $U_c : (\bb{C}^2)^{\otimes n} \to (\bb{C}^2)^{\otimes n} \in \bf{Mat}_\bb{C}$.
\end{definition}

\begin{example}\label{ex-quantum-circuits}
    An example of quantum circuit is the following:
    \begin{equation}
        \input{figures/quantum-circuit.tikz}
    \end{equation}
    where we denote the $Rz(\alpha)$ gate using the symbol $\alpha$ and the
    Hadamard gate using a small white box.
    Note that different quantum circuits $c$ and $c'$ may correspond to the same
    unitary, we write $c \sim c'$ when this is the case.
    For instance the swap gate is equivalent to the following composition:
    \begin{equation*}
        \input{figures/quantum-circuit-swap.tikz}
    \end{equation*}
\end{example}

When performing a quantum computation on $n$ qubits, the quantum system is
usually prepared in an all-zeros states $\ket{\bf{0}} = \ket{0}^{\otimes n}$,
then a quantum circuit $c$ is applied to it and measurements are performed on
the resulting state $U_c\ket{\bf{0}}$, yielding a probability distribution over
bitstrings of length $n$.
We can then post-select on certain outcomes, or predicates over the resulting
bitstring, by throwing away the measurements that do not satisfy this predicate.
Diagrammatically, state preparations and post-selections are drawn using
the generalised Dirac notation \cite{coecke2017a} as in the following
post-selected circuit:
\begin{equation}
    \input{figures/post-selected-circuit.tikz}
\end{equation}

The category of post-selected quantum circuits is obtained from $\bf{QCirc}$
by allowing state preparations $\ket{0}, \ket{1}$, and post-selections
$\bra{0}, \bra{1}$.
$$ \bf{PostQCirc} = \bf{MC}(\tt{Gates} + \set{\bra{i}, \ket{i}}_{i \in {0, 1}}
+ \set{a: 0 \to 0}_{a \in \bb{R}})$$
Note that we also allow arbitrary scalars $a \in \bb{R}$, seen as boxes
$0 \to 0$ in $\bf{PostQCirc}$ to rescale the results of measurements, this is needed
in order to interpret pregroup grammars in $\bf{PostQCirc}$, see \ref{section-quantum-models}.
Post-selected quantum circuits map functorially into linear maps:
$$I: \bf{PostQCirc} \to \bf{Mat}_\bb{C}$$
by sending each gate in $\bf{Gates}$ to the corresponding unitary, state preparations
and post-selections to the corresponding states and effects in $\bf{Mat}_\bb{C}$.
This makes post-selected quantum circuits instances of tensor networks over $\bb{C}$.

\begin{proposition}\label{prop-quantum-computing}
    For any morphism $d: 1 \to 1$ in $\bf{PostQCirc}$, there exists a quantum circuit
    $c \in \bf{QCirc}$ such that $I(d) = a \cdot \bra{\bf{0}}U_c\ket{\bf{0}}$
    where $a = \prod_i U(a_i)$ is the product of the scalars appearing in $d$.
\end{proposition}
\begin{proof}
    We start by removing the scalars $a_i$ from the diagram $d$ and multiplying them into
    $a = \prod_i U(a_i)$. Then we pull all the kets to the top of the diagram and
    all the bras to the bottom, using $\tt{swap}$ if necessary.
\end{proof}

\subsection{Quantum models}\label{section-quantum-models}

Quantum models are functors from a syntactic category to the
category of post-selected quantum circuits $\bf{PostQCirc}$.
These were introduced in recent papers \cite{meichanetzidis2020, coecke2020a},
and experiments were performed on IBM quantum computers \cite{meichanetzidis2020a, lorenz2021}.

\begin{definition}[Quantum model]
    A quantum (circuit) model is a functor $F: G \to \bf{PostQCirc}$ for a grammar $G$,
    the semantics of a sentence $g \in \cal{L}(G)$ is given by a distribution over
    bitstrings $b \in  \set{0, 1}^{F(s)}$ obtained as follows:
    $$ P(b) = \size{\bra{\bf{b}}F(g)\ket{\bf{0}}}^2$$
\end{definition}

Quantum models define the following computational problem:

\begin{definition}
    \begin{problem}
      \problemtitle{$\tt{QSemantics}$}
      \probleminput{$G$ a grammar, $g: u \to s \in \bf{MC}(G)$,
      $F: G \to \bf{PostQCirc}$, $b \in  \set{0, 1}^{F(s)}$}
      \problemoutput{$\size{\bra{\bf{b}}I(F(g))\ket{\bf{0}}}^2$}
    \end{problem}
\end{definition}
\begin{remark}
    Note that quantum language models do \emph{not} assume that the semantics
    of words be a unitary matrix. Indeed a word may be interpreted as a unitary
    with some outputs post-selected. However, mapping words to unitaries is justified
    in many cases of interest in linguistics. For instance, this implies that
    there is no loss of information about ``ball'' when an adjective such as
    ``big'' is applied to it.
\end{remark}

Any of the grammars studied in Chapter 1 may be intepreted with a quantum model.
For monoidal grammars (including regular and context-free grammars), this is
simply done by interpreting every box as a post-selected quantum circuit.
For rigid grammars, we need to choose an interpretation for the cups and caps.
Since the underlying category $\bf{Mat}_\bb{C}$ is compact-closed, there is a
canonical way of interpreting them using a $CX$ gate and postselection, and
re-scaling the result by a factor of $\sqrt{2}$:
\begin{equation}
    \input{figures/quantum-circuit-cup.tikz}
\end{equation}
In order to interpret cross dependencies we can use the $\tt{swap}$ gate, see
Example \ref{ex-quantum-circuits}. Finally, in order to interpret the Frobenius
algebras in a pregroup grammar with coreference, we use the following mapping:
\begin{equation*}
    \input{figures/quantum-circuit-frobenius.tikz}
\end{equation*}
One can check the the image under $I$ of the circuit above maps to the Frobenius
algebra in $\bf{Mat}_\bb{C}$.

In practice, quantum models are learned from data by choosing a parametrized quantum
circuit in $\bf{PostQCirc}$ for each production rule and then tuning the parameters
(i.e. the phases $\alpha$ in $Rz(\alpha)$) appearing in this circuit, see
\cite{meichanetzidis2020a, lorenz2021} for details on ansatze and training.

\subsection{Additive approximations}

Since quantum computers are inherently probabilistic, there is no deterministic
way of computing a function $F: \{ 0, 1 \}^\ast \to \bb{C}$ encoded in the amplitudes
of a quantum state. Rather what we obtain is an \emph{approximation} of $F$.
In many cases, the best we can hope for is an \emph{additive} approximation,
which garantees to generate a value within the range
$[F(x) - \Delta(x) \epsilon, F(x) + \Delta(x)\epsilon]$ where
$\Delta: \set{0, 1}^\ast \to \bb{R}^+$ is an approximation scale.
This approximation scheme has been found suitable to describe the performance of
quantum algorithms for contracting tensor network \cite{arad2010},
counting approximately \cite{bordewich2009}, and computing topological invariants
\cite{freedman2000a} including the Jones polynomial \cite{freedman2002}.

\begin{definition}[Additive approximation]\cite{arad2010}
    A function $F : \set{0, 1}^\ast \to \bb{C}$ has an additive approximation $V$
    with an approximation scale $\Delta : \set{0, 1}^\ast \to \bb{R}^+$
    if there exists a probabilistic algorithm that given any instance
    $x \in \set{0, 1}^\ast$ and $\epsilon > 0$, produces a complex number
    $V(x)$ such that
    \begin{equation}\label{eq-additive-approximation}
        P(\size{ V(x) - F(x)} \geq \epsilon \Delta (x)) \leq \frac{1}{4}
    \end{equation}
    in a running time that is polynomial in $\size{x}$ and $\epsilon^{-1}$.
\end{definition}
\begin{remark}
    The error signal $\epsilon$ is usually inversely proportional to a polynomial
    in the run-time $t$ of the algorithm, i.e. we have $\epsilon = \frac{1}{\tt{poly}(t)}$.
\end{remark}

An algorithm $V$ satisfying the definiton above is a solution of the
following problem defined for any function $F: \set{0, 1}^\ast \to \bb{C}$
and approximation scale $\Delta: \set{0, 1}^\ast \to \bb{R}^+$.

\begin{definition}
    \begin{problem}
      \problemtitle{$\tt{Approx}(F, \Delta)$}
      \probleminput{$x \in \set{0, 1}^\ast$ , $\epsilon > 0$}
      \problemoutput{$v \in \bb{C}$ such that
      $P(\size{ v - F(x)} \geq \epsilon \Delta(x)) \leq \frac{1}{4}$}
    \end{problem}
\end{definition}

Compare this to the definition of a multiplicative approximation, also known
as a fully polynomial randomised approximation scheme \cite{jerrum2004}.

\begin{definition}[Multiplicative approximation]
    A function $F : \set{0, 1}^\ast \to \bb{C}$ has a multiplicative approximation $V$
    if there exists a probabilistic algorithm that given any instance $x \in \set{0, 1}^\ast$
    and $\epsilon > 0$, produces a complex number $V(x)$ such that
    \begin{equation}\label{eq-multiplicative-approximation}
        P(\size{ V(x) - F(x)} \geq \epsilon \size{F(x)}) \leq \frac{1}{4}
    \end{equation}
    in a running time that is polynomial in $\size{x}$ and $\epsilon^{-1}$.
\end{definition}

\begin{remark}
    This approximation is called multiplicative because $V(x)$ is guaranteed to be
    within a multiplicative factor $(1 \pm \epsilon)$ of the optimal value $F(x)$.
    For instance the inequality in (\ref{eq-multiplicative-approximation})
    implies that $\size{F(x)}(1 - \epsilon) \leq \size{V(x)} \leq \size{F(x)}(1 + \epsilon))$
    with probability bigger than $\frac{3}{4}$.
\end{remark}

Note that any multiplicative approximation is also additive with approximation
scale $\Delta(x) = \size{F(x)}$.
However, the converse does not hold. An additive approximation scheme allows
for the approximation scale $\Delta(x)$ to be exponentially larger than
the size of the output $\size{F(x)}$, making the approximation
(\ref{eq-additive-approximation}) quite loose since the error parameter $\epsilon$
may be bigger than the output signal $V(x)$.

\subsection{Approximating quantum models}\label{section-approx-quantum-models}

In this section, we show that the problem of additively approximating the
evaluation of sentences in a quantum model is in $\tt{BQP}$ and that it is
$\tt{BQP}$-complete in special cases of interest. The argument is based on
Arad and Landau's work \cite{arad2010} on the quantum approximation of
tensor networks. We start by reviewing their results and end by demonstrating the
consequences for quantum language models.

Consider the problem of approximating the contraction of tensor networks $T(V, E)$
using a quantum computer.
Arad and Landau \cite{arad2010} show that this problem can be solved in
polynomial time, up to an additive accuracy with a scale
$\Delta$ that is related to the norms of the swallowing operators.

\begin{proposition}[Arad and Landau \cite{arad2010}]
    Let $T(V, E)$ be a tensor network over $\bb{C}$ of dimension $q$ and maximal
    node degree $a$, let $\pi: [\size{V}] \to V$ be a contraction order for $T$ and let
    $\set{A_i}_{i \in \set{1, \dots, k}}$
    be the corresponding set of swallowing operators. Then for any
    $\epsilon > 0$ there exists a quantum algorithm that runs in $k \cdot \epsilon^{-2}
    \cdot \tt{poly}(q^a)$ time and outputs a complex number $r$ such that:
    $$P( \size{\tt{value}(T) - r} \geq \epsilon \Delta) \leq \frac{1}{4}$$
    with
    $$\Delta(T) = \prod_{i=1}^k\norm{A_i}$$
\end{proposition}
\begin{proof}
    Given a tensor network with a contration order $\pi$,
    the swallowing operators $A_i$ are linear maps.
    For each of them, we can construct a unitary $U_i$ acting on a larger space
    such that post-selecting on some of its outputs yields $A_i$. Composing these
    we obtain a unitary $U_c = U_1 \cdot U_2 \dots U_{\size{V}}$ represented by
    a poly-size quantum circuit $c \in \bf{QCirc}$ such that
    $\tt{value}(T) = \bra{\bf{0}}U_c\ket{\bf{0}}$. In order to compute an approximation
    of this quantity we can use an $H$-test, defined by the following circuit:
    \begin{equation}
        \input{figures/h-test.tikz}
    \end{equation}
    where the white boxes are hadamard gates and the subdiagram in the middle
    denotes the controlled unitary $U$.
    It can be shown that the probability $r$ of measuring $0$ on the
    ancillary qubit is equal to $Re(\bra{\bf{0}}U_c\ket{\bf{0}})$. A slightly
    modified version of the $H$-test computes $Im(\bra{\bf{0}}U_c\ket{\bf{0}})$.
    Arad and Landau \cite{arad2010} show that this process can be done in polynomial
    time and that the result of measuring the ancillary qubit gives an additive
    approximation of $\tt{value}(T)$ with approximation scale
    $\Delta(T) = \prod_{i=1}^k\norm{F(d_i)}$.
\end{proof}

\begin{corollary}
    The problem $\tt{Approx}(\tt{Contraction}(\bb{C}), \Delta)$ with $\Delta$
    as defined above is in $\tt{BQP}$.
\end{corollary}

\begin{remark}
    From the Cauchy-Schwartz inequality we have that:
    $$\size{T} \leq \prod_{i=1}^k\norm{A_i} = \Delta(T) \, . $$
    In fact we have no guarantee that the approximation scale $\Delta(T)$ is not
    exponentially larger than $\size{F(d)}$.
    This is a severe limitation, since the approximations we
    get from the procedure defined above can have an error larger than the value
    we are trying to approximate.
\end{remark}

We now consider the problem of approximating the amplitude of a post-selected
quantum circuit $\size{\bra{\bf{0}}U_c\ket{\bf{0}}}^2$.
Note that this is an instance of the problem studied in the
previous paragraph, since post-selected quantum circuits are instances of tensor
networks, so that this problem belong to the class $\tt{BQP}$.
For this subclass of tensor networks the approximation scale $\Delta(c)$ can
be shown to be constant and equal to $1$, with the consequence --- shown by
Arad and Landau \cite{arad2010} --- that the additive approximation of
post-selected quantum circuits is $\tt{BQP-hard}$.

In order to show $\tt{BQP}$-hardness of a problem $F$, it is sufficient to show
that an oracle which computes $F$ can be used to perform universal quantum
computation with bounded error. More precisely, for any quantum circuit
$u \in \bf{QCirc}$ denote by $p_0$ the probability of obtaining outcome
$0$ when measuring the last qubit of $c$. To perform universal quantum
computation, it is sufficient to distinguish between the cases where
$p_0 < \frac{1}{3}$ and $p_0 > \frac{2}{3}$ for any quantum circuit $c$.
Thus a problem $F$ is $\tt{BQP}$-hard if for any
circuit $c$, there is a poly-time algorithm $V$ using $F$ as an oracle that
returns $YES$ when $p_0 < \frac{1}{3}$ with probability bigger than $\frac{3}{4}$
and $NO$ when $p_0 > \frac{2}{3}$ with probability bigger than $\frac{3}{4}$. .

\begin{proposition}\cite{arad2010}\label{prop-circuit-bqp}
    The problem of finding an additive approximation of
    $\size{\bra{\bf{0}}U_c\ket{\bf{0}}}^2$ with scale $\Delta =1$
    where $U_c$ is the unitary induced by a quantum circuit $c \in \bf{QCirc}$
    is $\tt{BQP}$-complete.
\end{proposition}
\begin{proof}
    Membership follows by reduction to $\tt{Approx}(\tt{Contraction}(\bb{S}), \Delta)$,
    since post-selected quantum circuits are instances of tensor networks.
    To show hardenss, fix any quantum circuit $c$ on $n$ qubits and denote by
    $p_0$ the probability of obtaining outcome $0$ on the last qubit of $c$.
    We can construct the quantum circuit:
    \begin{equation}
        \input{figures/amplitude-circuit.tikz}
    \end{equation}
    it is easy to check that $\bra{\bf{0}}U_q\ket{\bf{0}} =  p_0$.
    Since $q$ is a circuit, there is a natural contraction order given by the order of
    the gates, moreover the corresponding swallowing operators $U_{q_i}$ are
    unitary and thus $\norm{U_{q_i}} = 1$, also $\norm{U(\ket{\bf{0}})} =
    \norm{U(\bra{\bf{0}})} = 1$ and therefore the approximation scale is
    constant $\Delta(q) = 1$.
    Suppose we have an oracle $V$ that computes an approximation of
    $\bra{\bf{0}}U_q\ket{\bf{0}}$ with constant scale $\Delta(q) = 1$. Then
    for any circuit $c \in \bf{QCirc}$, we construct the circuit $q$
    and apply $V$ to get a complex number $V(q)$ such that:
    $$P(\size{V(q) - p_0} \geq \epsilon) \leq \frac{1}{4}
    \implies P(p_0 - \epsilon \leq \size{V(q)} \leq p_0 + \epsilon) \geq \frac{3}{4}$$
    Setting $\epsilon < \frac{1}{6}$, we see that if $\size{V(q)} < \frac{1}{6}$
    then $p_0 < \frac{1}{3}$ with probability $\geq \frac{3}{4}$ and similarly
    $\size{V(q)} > \frac{5}{6}$ implies $p_0 > \frac{1}{3}$ with probability
    $\geq \frac{3}{4}$. Note that this would not be possible if we didn't
    know that the approximation scale $\Delta(q)$ is bounded.
    Thus $V$ can be used to distinguish between the cases where
    $p_0 < \frac{1}{3}$ and $p_0 > \frac{2}{3}$ with probability of success
    greater than $\frac{3}{4}$.
    Therefore the oracle $V$ can be used to perform universal quantum computation
    with bounded error. And thus $V$ is $\tt{BQP}$-hard.
\end{proof}

\begin{corollary}
    $\tt{Approx}(\tt{FunctorEval}(\bb{C})(I), \Delta = 1)$ where
    $I: \bf{PostQCirc} \to \bf{Mat}_\bb{C}$ is the
    functor defined in paragraph \ref{subsec-quantum-computing} is a
    $\tt{BQP}$-complete problem.
\end{corollary}
\begin{proof}
    Membership follows by reduction to $\tt{Approx}(\tt{Contraction}(\bb{C}), \Delta)$
    since post-selected quantum circuits are instances of tensor networks.
    Since $\bf{QCirc} \injects \bf{PostQCirc}$ the problem of Proposition
    \ref{prop-circuit-bqp} reduces to
    $\tt{Approx}(\tt{FunctorEval}(I), \Delta = 1)$, thus showing hardness.
\end{proof}

Note that the value $v= \size{\bra{\bf{0}}U_c\ket{\bf{0}}}$ may be
exponentially small in the size of the circuit $c$, so that the approximation scale
$\Delta = 1$ is still not optimal.
In this case we would need exponentially many samples from the quantum computer
approximate $v$ up to multiplicative accuracy.

We end by showing $\tt{BQP}$-completeness for the problem of approximating the
semantics of sentences in a quantum model with approximation scale $\Delta = 1$.

\begin{definition}
    \begin{problem}
      \problemtitle{$\tt{QASemantics} = \tt{Approx}(\tt{QSemantics}, \Delta = 1)$}
      \probleminput{$G$ a monoidal grammar, $g: u \to s \in \bf{MC}(G)$,
       $F: G \to \bf{PostQCirc}$, $b \in  \set{0, 1}^{F(s)}$, $\epsilon > 0$}
      \problemoutput{$v \in \bb{C}$ such that
      $P(\size{ v - \bra{\bf{b}}I(F(g))\ket{\bf{0}}} \geq \epsilon) \leq \frac{1}{4}$}
    \end{problem}
\end{definition}

\begin{proposition}\label{prop-bqp}
    There are pregroup grammars $G$, such that the problem
    $\tt{QASemantics}(G)$ is $\tt{BQP}$-complete.
\end{proposition}
\begin{proof}
    Membership follows by reduction to the problem of Proposition
    \ref{prop-circuit-bqp},
    since the semantics of a grammatical reduction $g: u \to s$ in $\bf{RC}(G)$
    is given by the evaluation of $\size{\bra{\bf{b}}I(F(g))\ket{\bf{0}}}^2$ which
    corresponds to evaluating $\size{\bra{\bf{0}}U_c\ket{\bf{0}}}^2$ where $c$ is
    defined by Proposition \ref{prop-quantum-computing}.
    To show hardness, let $G$ have only one word $w$ of type $s$, fix any unitary $U$,
    and define $F: G \to \bf{PostQCirc}$ by $F(s) = 0$ and
    $F(w) = \bra{\bf{0}}U\ket{\bf{0}}$, then
    evaluating the semantics of $w$ reduces to the problem
    of Proposition \ref{prop-circuit-bqp} and is thus $\tt{BQP-hard}$.
\end{proof}

Note that we were able to show completeness using the fact that the functor
$F$ is in the input of the problem. It is an open problem to show that
$\bf{QASemantics}(G)(F)$ is $\tt{BQP}$-complete for fixed choices of grammar $G$
and functors $F$. In order to show this, one may need to assume that $G$ is
a pregroup grammar with coreference and show that there is a functor $F$ such that
any post-selected quantum circuit can be built up using fixed size circuits (corresponding
to words) connected by GHZ states (corresponding to the spiders encoding the coreference),
adapting the argument from Proposition \ref{prop-functor-eval}.

Moreover, as discussed above, this approximation scheme is limited by the fact that the
approximation scale $\Delta$ may be too big when the output
$\bra{\bf{b}}I(F(g))\ket{\bf{0}}$ is exponentially small. This is particularly
significant at the beginning of training, when $F$ is initialised as a random
mapping to circuits, see \cite{meichanetzidis2020a}. This seems to be an
inherent problem caused by the use of post-selection in the model, although
methods to reduce the post-selection have been proposed, e.g. the snake removal
scheme from \cite{meichanetzidis2020}.

One avenue to overcome this limitation, is to consider a different typ of models,
defined as functors $G \to \bf{CPM}(\bf{QCirc})$ where $\bf{CPM}(\bf{QCirc})$ is the
category of completely positive maps induced by quantum circuits as defined in
\cite{heunen2019}. In this category, post-selection is not allowed since every
map is causal. The problem with these models however is that we cannot interpret
the cups and caps of rigid grammars. We may still be able to interpret monoidal
grammars, as well as acyclic pregroup reductions such as those induced by a
dependency grammar. Exploring the complexity of these models, and testing them
on quantum hardware, is left for future work.


\section{DisCoPy in action}\label{sec-discopy}

We now give an example of how DisCoPy can be used to solve a concrete task.
We define a relational model (\ref{sec-rel-model}) and then
learn a smaller representation of the data as a tensor model
(\ref{sec-tensor-network}). Since the sentences we will deal with are all of
the form subject-verb-object, this means we will perform a knowledge
ambedding task in the sense of \ref{sec-embedding}.
We start by defining a simple pregroup grammar with 3 nouns and 2 verbs.

\begin{python}\label{listing:discopy.language}
{\normalfont Subject-verb-object language.}

\begin{minted}{python}
from discopy.rigid import Ty, Id, Box, Diagram

n, s = Ty('n'), Ty('s')
make_word = lambda name, ty: Box(name, Ty(), ty)
nouns = [make_word(name, n) for name in ['Bruno', 'Florio', 'Shakespeare']]
verbs = [make_word(name, n.l @ s @ n.r) for name in ['met', 'read']]
grammar = Diagram.cups(n, n.l) @ Id(s) @ Diagram.cups(n.r, n)
sentences = [a @ b @ c >> grammar for a in nouns for b in verbs for c in nouns]
sentences[2].draw()
\end{minted}
\begin{center}
    \includegraphics[scale=0.40]{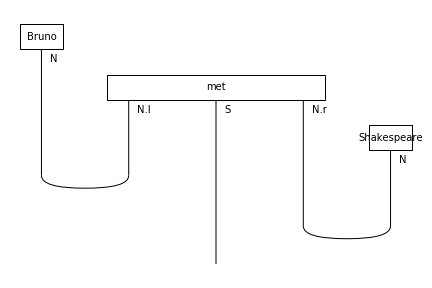}
\end{center}
\end{python}

We can now build a relational model for this language as a \py{tensor.Functor}.

\begin{python}\label{listing:discopy.relations}
{\normalfont Relational model in DisCoPy.}

\begin{minted}{python}
from discopy.tensor import Dim, Tensor, Functor, Spider
import jax.numpy as np
Tensor.np = np

ob = {n: Dim(3), s: Dim(1)}
def mapping(box):
    if box.name == 'Bruno':
        return np.array([1, 0, 0])
    if box.name == 'Florio':
        return np.array([0, 1, 0])
    if box.name == 'Shakespeare':
        return np.array([0, 0, 1])
    if box.name == 'met':
        return np.array([[1, 0, 1], [0, 1, 1], [1, 1, 1]])
    if box.name == 'read':
        return np.array([[1, 0, 0], [0, 1, 1], [0, 1, 1]])
    if box.name == 'who':
        return Spider(0, 3, Dim(3)).array
T = Functor(ob, mapping)
assert T(sentences[2]).array == [0.]
\end{minted}

We use \py{float} numbers for simplicity, but one may use \py{dtype=bool} instead.
Note the special intepretation of ``who'' as a Frobenius spider,
see \ref{section-relational-models}.
\end{python}

Relational models can be used to evaluate any conjunctive query over words.
We can generate a new pregroup reduction using \py{lambeq} \cite{kartsaklis2021}
and evaluate it in \py{T}.

\begin{python}\label{listing:discopy.query}
{\normalfont Evaluating a conjunctive query in a relational model.}

\begin{minted}{python}
from lambeq import BobcatParser
parser = BobcatParser()
diagram = parser.sentence2diagram('Bruno met Florio who read Shakespeare.')
diagram.draw()
assert T(diagram) = [1.0]
\end{minted}
\begin{center}
    \includegraphics[scale=0.40]{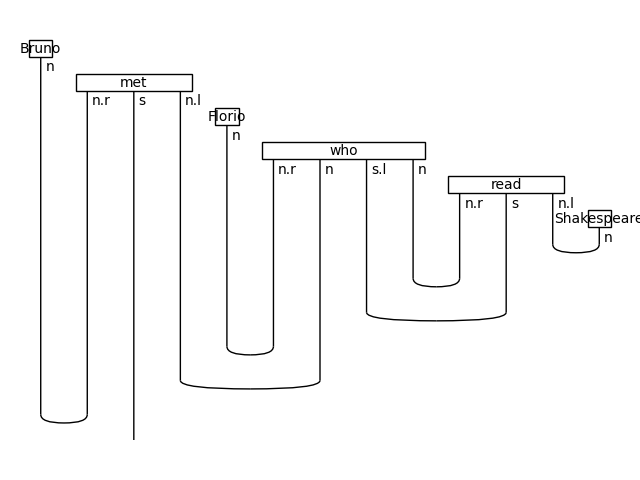}
\end{center}
\end{python}

We now show how to embed the three-dimensional data defined by \py{T}
as a two-dimensional \py{tensor.Functor} with \py{float} entries.
We start by parametrising two-dimensional tensor functors.

\begin{python}\label{listing:discopy.params}
{\normalfont Parametrising a tensor functor.}

\begin{minted}{python}
import numpy

def p_mapping(box, params):
    if box.name == 'Bruno':
        return np.array(params[0])
    if box.name == 'Florio':
        return np.array(params[1])
    if box.name == 'Shakespeare':
        return np.array(params[2])
    if box.name == 'met':
        return np.array([params[3], params[4]])
    if box.name == 'read':
        return np.array([params[5], params[6]])
    if box.name == 'who':
        return Spider(0, 3, Dim(2)).array

ob = {n: Dim(2), s: Dim(1)}
F = lambda params: Functor(ob, lambda box: p_mapping(box, params))
params0 = numpy.random.rand(6, 2)
assert F(params0)(sentences[2]).array != [0.]
\end{minted}
\end{python}

We obtain a prediction by evaluating \py{F(params)} on a sentence and taking
\py{sigmoid} to get a number between 0 and 1.
We can then define the loss of a functor as the mean squared difference
between its predictions and the true labels given by \py{T}. Of course,
other activation and loss functions may be used.

\begin{python}\label{listing:discopy.loss}
{\normalfont Defining the loss function for a knowledge embedding task.}

\begin{minted}{python}
def sigmoid(x):
    sig = 1 / (1 + np.exp(-x))
    return sig

evaluate = lambda F, sentence: sigmoid(F(sentence).array)

def mean_squared(y_true, y_pred):
    return np.mean((np.array(y_true) - np.array(y_pred)) ** 2)

loss = lambda params: mean_squared(*zip(\
    *[(T(sentence).array, evaluate(F(params), sentence)) for sentence in sentences]))
\end{minted}
\end{python}

The Jax package \cite{jax2018github} supports automatic differentiation \py{grad}
and just-in-time compilation \py{jit} for \py{jax} compatible \py{numpy} code.
Since the code for \py{Tensor} is compatible, we can directly use Jax to
compile a simple update function for the functor's parameters.
We run the loop and report the results obtained.

\begin{python}\label{listing:discopy.learning}
{\normalfont Learning functors with Jax.}

\begin{minted}{python}
from jax import grad, jit
from time import time

step_size = 0.1

@jit
def update(params):
    return params - step_size * grad(loss)(params)

epochs, iterations = 7, 30
params = numpy.random.rand(6, 2)
for epoch in range(epochs):
    start = time()
    for i in range(iterations):
        params = update(params)

    print("Epoch {} ({:.3f} milliseconds)".format(epoch, 1e3 * (time() - start)))
    print("Testing loss: {:.5f}".format(loss(params)))

y_true = [T(sentence).array for sentence in sentences]
y_pred = [0 if evaluate(F(final_params), sentence) < 0.5 else 1
          for sentence in sentences]
print(classification_report(y_true, y_pred))
\end{minted}
\begin{center}
    \includegraphics[scale=0.65]{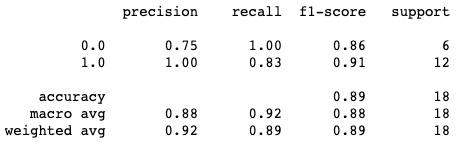}
\end{center}
\end{python}

And voil\`a! In just a few lines of code we have learnt a compressed 2D
tensor representation of the data of a 3D relational model.
We have executed a simple statistical relational learning routine \cite{wang2017},
which can in fact be seen as a version of the knowledge graph embedding model
$\tt{Rescal}$ \cite{nickel2011}, see \ref{sec-embedding}, we review the theory of knowledge graph
embeddings from a diagrammatic perspective, we show the advantages of moving
into the complex numbers for embedding relations.
In fact, we implemented and tested a similar embedding method using quantum computers \cite{meichanetzidis2020a, lorenz2021} which naturally handle complex tensors,
see \ref{sec-quantum-model}.
Large-scale diagram parsing is now possible thanks to Lambeq \cite{kartsaklis2021},
as we showed above. There remains work to do on optimization and
batching for tensor models, which would allow to scale these experiments to
real-world corpora.

\chapter{Games for Pragmatics}
\renewcommand{\chaptermark}[1]{\markboth{#1}{}}

The threefold distinction between syntax, semantics and pragmatics may be
traced back to the semiotics of Peirce an his trilogy between sign, object and
interpretant. According to Peirce \cite{peirce1865}, these three aspects
would induce three different approaches to semiotics, renewing the medieval
\emph{trivium}: formal grammar, logic and formal rhetoric.
In his studies \cite{stalnaker1970}, Stalnaker gives a finer
demarcation between pragmatics and semantics through the concept of \emph{context}:
``It is a semantic problem to specify the rules for matching up sentences
of a natural language with the propositions that they express.
In most cases, however, the rules will not match sentences directly with
propositions, but will match sentences with propositions relative to features
of the context in which the sentence is used.
Those contextual features are part of the subject matter of pragmatics.''

In his \emph{Philosophical Investigations} \cite{wittgenstein1953},
Wittgenstein introduces the concept of \emph{language-game} (\emph{Sprachspiel})
as a basis for his theory of meaning.
He never gives a general definition, and instead proceeds by enumeration of
examples: ``asking, thanking, cursing, greeting, praying''.
Thus, depending on the language-game in which it is played, the same utterance
``Water!'' can be interpreted as the answer to a question, a request to a waiter
or the chorus of a song. From the point of view of pragmatics, language-games
provide a way of capturing the notion of \emph{context}, isolating a particular
meaning/use of language within an environment defined by the game.

Since Lewis' work on conventions \cite{lewis1969}, formal \emph{game theory} has
been used to model the variability of the meaning of sentences and their dependence
on context ~\cite{franke2009, monroe2015, benz2018}. These theoretical enquiries have also
been supported by psycholinguistic experiments such as those of Frank and
Goodman \cite{frank2012}, where a Bayesian game-theoretic model is used to
predict the behaviour of listeners and speakers in matching words with their referents.

In parallel to its use in pragmatics, game theory has been proven
significant in designing machine learning tasks.
It is at the heart of multi-agent reinforcement learning \cite{tuyls2005}, where
decision-makers interact in a stochastic environment. It is also used to
improve the performance of neural network models, following the seminal
work of Goodfellow et al. \cite{goodfellow2014a}, and is
beginning to be applied to natural language processing tasks
such as dialogue generation \cite{li2016a, li2017}, knowledge graph embedding
\cite{cai2018, xiong2018} and word-sense disambiguation \cite{tripodi2019}.

The aim of this chapter is to develop formal and diagrammatic tools to
model language games and NLP tasks.
More precisely, we will employ the theory of \emph{lenses}, which have been
developed as a model for the dynamics of data-accessing programs \cite{pickering2017}
and form the basis of the recent applications of category theory to both game theory
\cite{GhaniHedges18, bolt2019} and machine learning \cite{fong2019e, cruttwell2021}.
The results are still at a preliminary stage, but the diagrammatic representation
succeeds in capturing a wide range of pragmatic scenarios.

In Section \ref{section-tasks}, we argue that probabilistic models are best
suited for analysing pragmatics and NLP tasks and show how the deterministic
models studied in Chapter 2 may be turned into probabilistic ones. In Section
\ref{section-tools}, we introduce lenses and show how they capture the dynamics
of probabilistic systems and stochastic environments. In Section \ref{section-agents},
we show how parametrization may be used to capture agency in these
environments and discuss the notions of optimum, best response and equilibrium
for multi-agent systems. Throughout, we give examples illustrating how these
concepts may be used in both pragmatics and NLP and in Section
\ref{section-examples}, we give a more in-depth analysis of three examples,
building on recent proposals for modelling language games with category theory
\cite{HedgesLewis18, defelice2020a}.

\section{Probabilistic models}\label{section-tasks}
Understanding language requires resolving large amounts of
vagueness and ambiguity as well as inferring interlocutor's intents and beliefs.
Uncertainty is a central feature of pragmatic interactions which has led
linguists to devise probabilistic models of language use, see \cite{franke2016}
for an overview.

Probabilistic modelling is also used throughout machine learning and NLP in particular.
For example, the language modelling task is usually formulated as the task
of learning the conditional probability $p(x_k \, \vert \, x_{k-1}, \dots, x_{1})$
to hear the word $x_k$ given that words $x_1, \dots, x_{k-1}$ have been heard.
We have seen in Section \ref{sec-neural-net} how neural networks induce
probabilistic models using the $\tt{softmax}$ function.
The language of probability theory is particularly suited to formulating
NLP tasks, and reasoning about them at a high level of abstraction.

Categorical approaches to probability theory and Bayesian reasoning have made
much progress in recent years. In their seminal paper \cite{cho2019}, Cho and Jacobs
identified the essential features of probabilistic resoning using string diagrams,
including marginalisation, conditionals, disintegration and Bayesian inversion.
Some of their insights derived from Fong's thesis \cite{fong2013}, where Bayesian
networks are formalised as functors into categories of Markov kernels.
Building on Cho and Jacobs, Fritz \cite{fritz2020} introduced
\emph{Markov categories} as a synthetic framework for probability theory,
generalising several results from the dominant measure-theoretic approach to
probability into this high-level diagrammatic setting. These algebraic tools
are powering interesting developments in applied category theory, including
diagrammatic approaches to causal inference \cite{jacobs2018} and Bayesian game
theory \cite{bolt2019}.

In this section, we review the basic notions of probability theory from a categorical
perspective and show how they can be used to reason about discriminative
and generative NLP models.

\subsection{Categorical probability}

We introduce the basic notions of categorical probability theory, following
\cite{cho2019} and \cite{fritz2020}. For simplicity, we work in the setting of
\emph{discrete} probabilities, although most of the results we use are formulated
diagrammatically and are likely to generalise to any Markov category.

Let $\cal{D} : \bf{Set} \to \bf{Set}$ be the discrete distribution monad defined on
objects by:
$$ \cal{D}(X) = \set{p: X \to \bb{R}^+ \, \vert \, d \text{has finite support and}
\, \sum_{x \in X} p(x) = 1 }$$
and on arrows $f: X \to Y$ by:
$$\cal{D}(f) : \cal{D}(X) \to \cal{D}(Y) : (p : X \to \bb{R}^+) \mapsto (y \in Y \mapsto
\sum_{x\in f^{-1}(y)} p(x) \in \bb{R}^+)$$
We can construct the Kleisli category for the distribution monad
$\bf{Prob} = \bf{Kl}(\cal{D})$ with objects sets and arrows discrete conditional
distributions $f: X \to \cal{D}(Y)$ we denote by $f(y \vert x) \in \bb{R}^+$ the
probability $f(x)(y)$. Composition of $f: X \to \cal{D}(Y)$ and $g: Y \to \cal{D}(Z)$
is given by:
$$ X \xto{f} \cal{D}(Y) \xto{\cal{D}(g)} \cal{D}\cal{D}(Z) \xto{\mu_Z} \cal{D}(Z) $$
where $\mu_Z : \cal{D}\cal{D}(Z) \to \cal{D}(Z)$ flattens a distribution of distributions by
taking sums. Explicitly we have:
$$f \cdot g (z \vert x) = \sum_y g(z \vert y) f(y \vert x) \in \bb{R}^+$$
We may think of the objects of $\bf{Prob}$ as \emph{random variables} and
the arrows as \emph{conditional distributions}.

The category $\bf{Prob}$ has interesting structure. First of all, it is a
symmetric monoidal category with $\times$ as monoidal product.
This is not a cartesian product since $\cal{D}(X \times Y) \neq \cal{D}(X) \times \cal{D}(Y)$.
The unit of the monoidal product $\times$ is the singleton set $1$ which is
\emph{terminal} in $\bf{Prob}$, i.e. for any set $X$ there is only one map
$\tt{del}_X: X \to \cal{D}(1) \simeq 1$ called \emph{discard}. Terminality of the monoidal
unit means that if we discard the output of a morphism we might as well have
discarded the input, a property often interpreted as \emph{causality} \cite{heunen2019}.

There is a commutative comonoid structure $\tt{copy}_X: X \to X \times X$ on each
object $X$ with counit $!_X$. Also there is a embedding of $\bf{Set}$ into
$\bf{Prob}$ which gives the deterministic maps. These are characterized by the
following property
$$ \tt{copy}_Y \circ f = (f \times f) \circ \tt{copy}_X \iff f: X \to Y \text{is deterministic}$$

Morphisms $p : 1 \to X$ in $\bf{Prob}$ are simply distributions $p \in \cal{D}(X)$.
Given a joint distribution $p : 1 \to X \times Y$ we can take \emph{marginals}
of $p$ by composing with the discard map:
\begin{equation*}
    \input{figures/marginals.tikz}
\end{equation*}

In $\bf{Prob}$ the disintegration theorem holds.
\begin{proposition}[Disintegration]\label{prop-disintegration} \cite{cho2019}
    For any joint distribution
    $p\in \cal{D}(X \times Y)$, there are channels $c: X \to \cal{D}(Y)$ and
    $c^\dagger: Y \to \cal{D}(X)$ satisfying:
    \begin{equation}\label{eq-disintegration}
        \input{figures/disintegration.tikz}
    \end{equation}
\end{proposition}
\begin{proof}
    The proof is given in \cite{cho2019} in the case of the full subcategory of
    $\bf{Prob}$ with objects finite sets, i.e. for the category $\bf{Stoch}$
    of stochastic matrices. Extending this proof to the infinite discrete case is
    simple because for any distribution $p \in \cal{D}(X \times Y)$ we may construct
    a stochastic vector over the support of $p$, which is finite by definition
    of $\cal{D}$. So we can disintegrate $p$ in $\bf{Stoch}$ and then extend it
    to $\bf{Prob}$ by assigning probability $0$ to elements outside of the support.
\end{proof}

\begin{example}
    Taking $X = \set{1, 2}$ and $Y = \set{A, B}$ an example of disintegration
    is the following:
    \begin{equation*}
        \left(
        \begin{tabular}{ |c|c| }
         \hline
         1 & $1/8$ \\
         \hline
         2 & $7/8$ \\
         \hline
        \end{tabular}
        \, , \,
        \begin{tabular}{ |c|c|c| }
         \hline
          & A & B\\
         \hline
         1 & $1$ &  $0$ \\
         \hline
         2 & $3/7$ & $4/7$\\
         \hline
        \end{tabular}
        \right)
        \leftarrow
        \begin{tabular}{ |c|c|c| }
         \hline
          & A & B\\
         \hline
         1 & $1/8$ &  $0$ \\
         \hline
         2 & $3/8$ & $1/2$\\
         \hline
        \end{tabular}
        \rightarrow
        \left(
        \begin{tabular}{ |c|c|c| }
         \hline
          & A & B\\
         \hline
         1 & $1/4$ &  $0$ \\
         \hline
         2 & $3/4$ & $1$\\
         \hline
        \end{tabular}
        \, , \,
        \begin{tabular}{ |c|c| }
         \hline
          A & B\\
         \hline
          $1/2$ &  $1/2$ \\
         \hline
        \end{tabular}
        \right)
    \end{equation*}
\end{example}

Given a \emph{prior} $p \in \cal{D}(X)$ and a channel $c: X \to \cal{D}(Y)$,
we can integrate them to get a joint distribution over $X$ and $Y$ and then
disintegrate over $Y$ to get the channel $c^\dagger: Y \to \cal{D}(X)$. This
process is known as \emph{Bayesian inversion} and $c^\dagger$ is
called a Bayesian inverse of $c$ along $p$. These satisfy the following equation,
which can be derived from \ref{eq-disintegration}.
\begin{equation}\label{eq-bayes}
    \input{figures/Bayes-law.tikz}
\end{equation}
Interpreting this diagrammatic equation in $\bf{Prob}$ we get that:
$$ c(y \vert x) p(x) = c^\dagger(x \vert y) \sum_x (c(y \vert x) p(x))$$
known as Bayes law.

The category $\bf{Prob}$ satisfies a slightly stronger notion of disintegration
which doesn't only apply to states or joint distributions but to channels directly.

\begin{proposition}[Conditionals]\cite{fritz2020}\label{prop-conditionals}
    The category $\bf{Prob}$ has conditionals,
    i.e. for any morphism $f: A \to \cal{D}(X \times Y)$ in $\bf{Prob}$ there is
    a morphism $f\vert_X : A \times X \to \cal{D}(Y)$ such that:
    \begin{equation*}
        \input{figures/conditionals.tikz}
    \end{equation*}
\end{proposition}
\begin{proof}
    As for Proposition \ref{prop-disintegration}, this was proved in the case
    of $\bf{Stoch}$ by Fritz \cite{fritz2020} and it is simple to extend the
    proof to $\bf{Prob}$.
\end{proof}

\subsection{Discriminators}

The first kind of probabilistic systems that we consider are \emph{discriminators}.
We define them in general as probabilistic channels that take sentences in a language
$\cal{L}(G)$ and produce distributions over a set of features $Y$.
These can be seen as solutions to the general classification task of assigning
sentences in $\cal{L}(G)$ to classes in $Y$.

\begin{definition}[Discriminator]
    A discriminator for a grammar $G$ in a set of features $Y$ is a probabilistic
    channel:
    $$ c: \cal{L}(G) \to \cal{D}(Y) $$
\end{definition}
\begin{remark}
    Throughout this chapter we assume that parsing for the chosen grammar $G$
    can be done efficiently. WFor simplicity, we assume that we are given a function:
    $$\tt{parsing}: \cal{L}(G) \to \coprod_{u \in V^\ast} \bf{G}(u, s)$$
    where $\bf{G}$ is the category of derivations for the grammar $G$.
    This could also be made a probabilistic channel with minor modifications
    to the results of this section.
\end{remark}

In the previous chapters, we defined NLP models as functors $F: G \to \bf{S}$
where $G$ is a grammar and $\bf{S}$ a semantic category. The aim for this section
is to show that, in most instances, we can turn these models into probabilstic
discriminators. We will do this in two steps. Assuming that parsing can be
performed efficiently, it is easy to show that functorial models $F: G \to \bf{S}$
induce \emph{encoders}, i.e. deterministic functions
$\cal{L}(G) \to S$ which assign to every sentence in the language
$\cal{L}(G) \sub V^\ast$ a compressed semantic representation in the sentence
space $S = F(s)$.
In order to build a discriminator $c$ from an encoder $f: \cal{L}(G) \to S$
the only missing piece of structure is an \emph{activation} function
$\sigma: S \to \cal{D}(Y)$, mapping semantic states to distributions over classes.

\emph{Softmax} is a useful activation function which allows to turn real valued vectors
and tensors into probability distributions.
$$ \tt{softmax}_X: \bb{R}^X \to \cal{D}(X)$$
$$ \tt{softmax}(\vec{x})_i = \frac{e^{x_i}}{\sum_{i=1}^n e^{x_i}} $$
Thus, when the sentence space is $S = \bb{R}^Y$, $\tt{softmax}$ allows to turn encoders
$f: X \to \bb{R}^Y$ into discriminators $c = \tt{softmax}\circ f :  X \to \cal{D}(Y)$.
In fact all discriminators arise from encoders by post-composition with $\tt{softmax}$
as the following proposition shows.

\begin{proposition}\label{prop-softmax}
    Given a channel $c: X \to \cal{D}(Y)$ and a prior distribution
    $p \in \cal{D}(Y)$ there is a function $f : X \to \bb{R}^Y$ such that
    $$ X \xto{f} \bb{R}^Y \xto{\tt{softmax}} \cal{D}(Y)
    \, =\,  X \xto{c} \cal{D}(Y)$$
\end{proposition}
\begin{proof}
    In order to prove the existence of $f$, we construct the log likelihood
    function. Given $c: X \to \cal{D}(Y)$ there is a Bayesian inverse
    $c^\dagger : Y \to \cal{D}(X)$, then we can define the \emph{likelihood}
    $l: X \times Y \to \bb{R}^+$ by:
    $$ l(x, y) = c^\dagger(x \vert y) p(y)$$
    and the \emph{log likelihood} is given by
    $$ f(x)  = \tt{log}_Y(l(x, y)) \in \bb{R}^Y$$
    where $\tt{log}_Y : (\bb{R}^+)^Y \to \bb{R}^Y$ is the logarithm applied to each
    entry. Then one can prove using Bayes law that:
    \begin{equation}\label{bayes-likelihood}
        \tt{softmax}(f) = \tt{softmax}(\tt{log}(l(x, y))) = \frac{l(x, y)}{\sum_{y'} l(x, y')} = \frac{c^\dagger(x \vert y) p(y)}{\sum_{y'}c^\dagger(x \vert y')} = c(y \vert x)
    \end{equation}
\end{proof}

Note that many functions $f$ may induce the same channel $c$ in this way.
The encoder $f : X \to (\bb{R}^+)^Y$ given in the proof is the
log of the likelihood function $l: X \times Y \to \bb{R}^+$. In these
instances, the encoder is well behaved probabilistically, since it satisfies the
version \ref{bayes-likelihood} of Bayes equation. However, in most gradient-based
applications of stochastic methods, encoders $X \to (\bb{R})^Y$ such as those
built from a neural network tend to achieve a better performance.

We already saw in Section \ref{sec-neural-net}, that $\tt{softmax}$ can be used to
turn neural networks into probabilistic models. We rephrase this in the
following definition.

\begin{definition}
    Given any recursive neural network $F: G \to \bf{NN}(W)$ for a monoidal grammar
    $G$ with $F(s) = n$and $F(v) = 0$ for $v \in V \sub G_0$, and parameters
    $\theta: W \to \bb{R}$, we can build
    a discriminator $\tilde{F} : \cal{L}(G) \to \cal{D}([n])$
    as the following composition:
    $$\tilde{F} =
    \cal{L}(G) \xto{\tt{parsing}} \coprod_{u \in V^\ast} \bf{MC}(G)(u, s)
    \xto{F} \bf{NN}(W)(0, n) \xto{I_\theta} \bb{R}^n \xto{\tt{softmax}} \cal{D}([n])$$
    where $I_\theta: \bf{NN}(W) \to \bf{Set}_\bb{R}$ is the functor defined
    in\ref{sec-neural-net}.
\end{definition}

Softmax induces a \emph{function} $\cal{S}: \bf{Mat}_\bb{R} \to \bf{Prob}$
defined on objects by $n \mapsto [n]$ for $n \in \bb{N}$ and on arrows by:
$$ n \xto{M} m \quad \mapsto \quad [n] \xto{\cal{S}(M)} \cal{D}([m]) = [n] \xto{\ket{-}} \bb{R}^n \xto{M\cdot} \bb{R}^m \xto{\tt{softmax}} \cal{D}([m]) $$
Where $\ket{-}$ is the \emph{one-hot encoding} which maps each element of
$[n]$ to the corresponding basis vector in $\bb{R}^n$. Note that the mapping
$\cal{S}$ is \emph{not functorial}, but it is surjective as can readily be shown
from Proposition \ref{prop-softmax}.

\begin{definition}
    Given any tensor network model $F: G \to \bf{Mat}_\bb{R}$ for a rigid grammar $G$
    with $F(s) = n$ and $F(w) = 1$ for $w \in V \sub G_0$ we can build a
    discriminator $\tilde{F}: \cal{L}(G) \to \cal{D}([n])$ as the following composition:
    $$\tilde{F} = \cal{L}(G) \xto{\tt{parsing}}
    \coprod_{u \in V^\ast} \bf{RC}(G)(u, s) \xto{F} \bf{Mat}_\bb{R}(1, n) \xto{\cal{S}} \cal{D}([n])$$
\end{definition}

\begin{example}[Student]\label{ex-student}
      Consider a student who is faced with a question in $\cal{L}(G, q)$,
      and has to guess the answer. This can be seen as a probabilistic
      channel $\cal{L}(G, q) \to \cal{D}(A)$ where $A$ is a set of possible answers.
      We can build this channel using a tensor network model
      $F: G \to \bf{Mat}_\bb{R}$ with $F(q) = \bb{R}^A$ which induces a discriminator
      $\tilde{F}: \cal{L}(G, q) \to \cal{D}(A)$, as defined above.
      The image under $\tilde{F}$ of the grammatical question ``Who was the emperor
      of France'' in $\cal{L}(G, q)$ is given by the following diagram, representing
      a distribution in $\cal{D}(A)$:
      \begin{equation*}
          \input{figures/student.tikz}
      \end{equation*}
      where the bubble represents the unary operator on hom-sets $\cal{S}$.
\end{example}

\begin{definition}
    Given any relational model $F: G \to \bf{Rel}$ for a rigid grammar $G$
    with $F(s) = Y$ and $F(w) = 1$ for $w \in V \sub G_0$
    we can build a discriminator $\tilde{F}: \cal{L}(G) \to \cal{D}(Y)$
    as follows:
    $\tilde{F} = \cal{L}(G) \xto{\tt{parsing}} \coprod_{u \in V^\ast} \bf{RC}(G)(u, s)
    \xto{F} \bf{Rel}(1, Y) \simeq \cal{P}(Y) \xto{\tt{uniform}} \cal{D}(Y)$
    where $\tt{uniform}$ takes a subset of $Y$ to the uniform distribution over
    that subset.
\end{definition}

\begin{example}[Literal listener]\label{ex-hearer}
    Consider a listener who hears a word in $W$ and needs to choose which object
    in $R$ the word refers to. This example is based on  Bergen et al. \cite{bergen2016}.
    Suppose we have a relation (called lexicon in \cite{bergen2016})
    $\phi: W \times R \to \bb{B}$, where
    $\phi(w, r) = 1$ if $w$ could refer to $r$ and $0$ otherwise.
    We can treat $\phi$ as a \emph{likelihood} $\phi: W \times R \to \bb{R}^+$.
    and model a \emph{literal listener} as a classifier $l: W \to \cal{D}(R)$ defined by:
    $$l(w \vert r) \propto \phi(w, r)$$
    This is the same as taking $l(r) \in \cal{D}(W)$ to be the uniform distribution over
    the words that could refer to $r$ according to $\phi$.
    Thus the literal listener does not take the context into account.
    We can extend this model by allowing noun phrases to be uttered,
    instead of single words, i.e. $W = \cal{L}(G, n)$ for some rigid grammar $G$ with
    noun type $n \in G_0$. Then we can replace the relation $\phi: W \times R \to \bb{B}$
    by a \emph{relational model} $F: G \to \bf{Rel}$ with $F(n) = R$, so that
    any grammatical noun phrase $g: u \to n$ in $\bf{RC}(G)$ is mapped to a subset
    $F(g) \sub R$ of the objects it could refer to. Then the listener is modeled by:
    $$l(g \vert r) \propto \braket{F(g)}{r}$$
    where $\braket{F(g)}{r} = 1$ if $r \in F(g)$ and is zero otherwise.
    This corresponds to taking $l = \tilde{F}$ as defined in the proposition above.
    As we will see in Example \ref{ex-speaker} and Section \ref{bayesian-pragmatics},
    we can use a literal listener to define a pragmatic speaker
    which reasons about the literal interpretation of its words.
\end{example}

\subsection{Generators}

In order to solve NLP tasks, we need \emph{generators} along-side discriminators.
These are probabilistic channels which produce a sentence in the language
given some semantic information in $S$. They are used throughout NLP, and most
notably in neural language modelling \cite{bengio2003} and machine translation
\cite{bahdanau2014}.

\begin{definition}[Generator]
    A generator for $S$ in $G$ is a probabilistic channel:
    $$ c: S \to \cal{D}\cal{L}(G)$$
\end{definition}

Bayesian probability theory allows to exploit the symmetry between
encoders and decoders. Recall that Bayes law relates a conditional distribution
$c: X \to \cal{D}(Y)$ with its Bayesian inverse $c^\dagger: Y \to \cal{D}(X)$
given a prior $p: 1 \to \cal{D}(X)$:
$$c^\dagger(x \vert y) = \frac{c(y \vert x) p(x)}{\sum_{x'}c(y \vert x') p(x')}$$
Thus Bayesian inversion $\dagger$ can be used to build a decoder
$c^\dagger : Y \to \cal{D}\cal{L}(G)$ from an encoder
$c: \cal{L}(G) \to \cal{D}(Y)$ given a prior over the language $p \in \cal{D}\cal{L}(G)$,
and viceversa.

\begin{example}[Speaker]\label{ex-speaker}
    Consider a speaker who is given an object in $R$ (e.g. a black chess piece), and
    needs to produce a word in $W$ to refer to it (e.g. "bishop"). We can model
    it as a generator $s: R \to \cal{D}(W)$.
    Following from Example \ref{ex-hearer}, assume the speaker has access to a
    relation $\phi: R \times W \to \bb{B}$. Then she knows what the literal interpretation
    of her words is, i.e. she can build a channel $l: W \to \cal{D}(R)$ corresponding
    to a literal speaker. In order to perform her task $s =R \to \cal{D}(W)$,
    she can take the Bayesian inverse of the literal listener $s = l^\dagger$.
    We will see in Section \ref{bayesian-pragmatics}, that this strategy for the
    teacher yields a Nash equilibrium in a collaborative game with a listener.
\end{example}

In practice, it is often very expensive to compute the Bayesian inverse of a
channel. Expecially when one does not have a finite table of probabilities
$X \times Y \to \bb{R}^+$ but rather a log-likelihood function
$X \to \bb{R}^Y$ represented e.g. as a neural network.  Thus, in most
cases, we must resort to other tools to build generators.
Recurrent neural networks are the simplest such tool.
Indeed, given a recurrent network $g: n \to n \oplus \size{V}$ in $\bf{NN}$
we can build a generator $\bb{R}^n \to \cal{D}(V^\ast)$, by picking an initial encoder state $x \in \bb{R}^n$ and simply iterate the recurrent network by
composing it along the encoder space $n$, see Section \ref{sec-neural-net} for examples.

\begin{example}[Translator/Chatbot]\label{ex-chatbot}
    The sequence-to-sequence (Seq2Seq) model of Bahdanau et al. \cite{bahdanau2014},
    surveyed at the end of Section \ref{sec-neural-net} is composed of recurrent
    encoder and decoder networks connected by an attention mechanism, see (\ref{eq-bahdanau}).
    It was originally used to build a translator
    $V^\ast \to \cal{D}(V'^\ast)$ for two vocabularies $V$ and $V'$.
    Taking $V' = V$ we can use Seq2Seq to build a chatbot
    $V^\ast \to \cal{D}(V^\ast)$. We will use these in \ref{ex-dialogue}.
\end{example}

From a categorical perspective, we do not understand generators as well as
discriminators. If generators arise from functors, discriminators should arise
from a dual notion, i.e. a concept of cofunctor, but we were unable to work out
what these should be.
Recently, Toumi and Koziell-Pipe \cite{toumi2021} introduced
a notion of functorial language model which allows to generate missing words
in a grammatical sentence using DisCoCat models. Indeed, if we assume that
the grammatical structure is given, then we may generate sentences with that
structure using an activation function. We give an example in the case of
relational models.

\begin{example}[Teacher]\label{ex-examiner}
    Consider a teacher who knows an answer (e.g. ``Napoleon''), and needs
    to produce a question with that answer (e.g. ``Who was emperor of France?'').
    Suppose that the teacher has access to a relational model $F: G \to \bf{Rel}$
    for a grammar $G$ with a question type $q$ and a noun type $n$ with
    $F(q) = F(n)$ (i.e. answers are nouns).
    Following \cite{coecke2018c, toumi2021}, we may repackage the functor $F$
    via an embedding $E : N \to F(n)$ in $\bf{Rel}$ where
    $N = \set{w \in V \, \vert \, (w, n) \in G}$ is the set of nouns, so
    that $F(w \to n) = E\ket{w}$. Fixing a grammatical structure $g: u \to q$,
    the teacher can generate questions with that structure, by evaluating
    the structure in her model $F$ and then uniformly choosing which nouns
    to use in the question. This process is shown in the following diagram:
    \begin{equation*}
        \input{figures/teacher.tikz}
    \end{equation*}
    where the bubble indicates the $\tt{uniform}$ operator on hom-sets of $\bf{Rel}$.
    Read from bottom to top, the diagram above represents a channel $N \to \cal{D}(N \times N)$
    which induces a channel $N \to \cal{D}(\cal{L}(G, q))$. However, taking
    the $\tt{uniform}$ function is not a pragmatic choice for the teacher. Indeed,
    given ``Napoleon'' as input, the channel above is equally likely to choose
    the question ``Who was emperor of France?'' as ``Who was citizen of France?''.
    We will see in Section \ref{section-adversarial-qa},
    that the pragmatics of posing questions may be captured by an adversarial
    scenario in which the teacher aims to ask hard questions to a student who
    tries to answer them.
\end{example}

\section{Bidirectional tools}\label{section-tools}

The concept of \emph{reward} is central in both game theory and machine learning.
In the first, it is formalised in terms of a \emph{utility function} and allows
to define the optimal strategies and Nash equilibria of games. In the second, where it
is captured (pessimistically) using a \emph{loss function}, it defines an
objective that the learning system must minimise.
Reward is also a central concept in reinforcement learning
where one considers probabilistic processes which run through a state-space
while collecting rewards \cite{howard1960, gouberman2014}.
In this section we show that the information flow of rewards in a dynamic
probabilistic system can be captured in a suitable category of \emph{lenses}.
The bidirectionality of lenses allows to represent the action-reaction structure
of game-theoretic and machine learning systems.

\subsection{Lenses}\label{section-lenses}

Lenses are bidirectional data accessors which were introduced in the context of
the view-update problem in database theory \cite{bohannon2006, johnson2012},
although they have antecedents in Godel's ``Dialectica interpretation''.
They are widely used in functional programming \cite{pickering2017}
and have recently received a lot of attention in the applied category
theory community, due to their applications to machine learning
\cite{fong2019e, fong2019d} and game theory \cite{GhaniHedges18}.
There are several variants on the definition of lenses available in the literature.
These were largely unified by Riley who introduced \emph{optics} as generalisation
of lenses from $\bf{Set}$ to any symmetric monoidal category \cite{riley2018a}.
We use the term lens instead of optics for this general notion.

\begin{definition}[Lens]\cite{riley2018a}
    Let $X, Y, O, S$ be objects in a symmetric monoidal category $\bf{C}$.
    A \emph{lens} or \emph{optic} $[f, v]: (X, S) \to (Y, O)$ in $\bf{C}$ is
    given by the following data:
    \begin{itemize}
        \item an object $M \in \bf{C}$,
        \item a forward part $f : X \to M \otimes  Y$ called ``get'',
        \item a backward part $v : M \otimes O \to S$ called ``put''.
    \end{itemize}
\end{definition}

Stripped out of its set-theoretic semantics, a lens is simply seen as a pair of
morphisms arranged as in the following two equivalent diagrams.

\begin{equation}\label{diagram-optics}
    \input{figures/optics.tikz}
\end{equation}

Two lenses $[f, v], [f', v']: (X, S) \to (Y, R)$ are said to
be \emph{equivalent}, written $[f, v] \cong [f', v']$, if there is
a morphisms $h: M \to M'$ satisfying
\begin{equation}\label{eq-diagram-lens}
    \input{figures/equivalent-lenses.tikz}
\end{equation}
The quotient of the set of lenses by the equivalence relation $\cong$ can be
expressed as a coend formula \cite{riley2018}.
Here we omit the coend notation for simplicity.
This quotient is needed in order to show that composition of lenses
is associative and unital. Indeed, as shown by Riley \cite{riley2018},
equivalence classes of lenses under $\cong$ form a symmetric monoidal category
denoted $\bf{Lens}(\bf{C})$.
Sequential composition for $[f, u]: (X, S) \to (Y, O)$ and $[g, v]: (Y, O) \to (Z, Q)$
is given by $[g, v] \circ [f, u] = [(\tt{id}_{M_f} \otimes g) \circ f,
v \circ (\tt{id}_{M_f} \otimes u)]$ with $M_{g \circ f} = M_f \times M_g$,
diagrammatically we have:
\begin{equation}
    \input{figures/composition-lenses.tikz}
\end{equation}
and tensor product $[f, v]  \otimes [g, u]$ is given by:
\begin{equation}
    \input{figures/tensor-lenses.tikz}
\end{equation}
Moreover, there are cups allowing to turn a covariant wire in the contravariant
direction.
\begin{equation}\label{eq-diagram-cup}
    \input{figures/teleological.tikz}
\end{equation}
\begin{remark}
    This diagrammatic notation is formalised in~\cite{hedges2017}, where categories
    endowed with this structure are called \emph{teleological}.
    Note that $\bf{Lens}_\bf{C}$ is not compact-closed, i.e. we can only turn wires
    from the covariant to the contravariant direction.
    When we draw a vertical wire as in \ref{eq-diagram-lens}, we actually mean
    a cup as in \ref{eq-diagram-cup}, this makes the notation more compact.
\end{remark}

We interpret lenses along the lines of compositional game theory \cite{GhaniHedges18}.
A lens is a process which makes
\emph{observations} in $X$, produces \emph{actions} or moves in $Y$, then it
gets some \emph{reward} or utility in $R$ and gives back information as to its
degree of \emph{satisfaction} or coutility in $S$. Of course, this interpretation
in no way exhausts the possible points of vue on lenses. For instance in
\cite{fong2019e} and \cite{cruttwell2021}, one interprets lenses as smooth functions
turning inputs in $X$ into outputs in $Y$, and then backpropagating the \emph{error}
in $R = \Delta Y$ to an error in the input $S = \Delta X$.

We are particularly interested in lenses over the category $\bf{Prob}$ of
conditional probability distributions, which we call \emph{stochastic lenses}.
These have been characterised in the context of causal inference
\cite{jacobs2018}, where they are called \emph{combs}, as morphisms of
$\bf{Prob}$ satisfying a causality condition.

\begin{definition}[Comb]\cite{jacobs2018}\label{def-combs}
    A comb is a stochastic map $c: X \otimes R \to Y \otimes S$ satisfying for some
    $c': X \to Y$, the following equation:
    \begin{equation}\label{eq-comb}
        \input{figures/comb-equation.tikz}
    \end{equation}
\end{definition}

Combs have an intuitive diagrammatic representation from which they take their
name.
\begin{equation}
    \input{figures/combs.tikz}
\end{equation}
With this diagram in mind, the condition \ref{eq-comb} reads: ``the contribution
from input $R$ is only visible via output $S$ ''\cite{jacobs2018}.

\begin{proposition}\label{prop-lenses-combs}
    In $\bf{Prob}$, combs $X \otimes R \to Y \otimes S$ are in one-to-one correspondence with lenses $(X, S) \to (Y, R)$.
\end{proposition}
\begin{proof}
    The translation works as follows:
    \begin{equation}
        \input{figures/lenses-combs.tikz}
    \end{equation}
    where $c'$ is the morphism defined in \ref{def-combs} and
    $c\vert_{X \otimes Y}$ is the conditional defined by Proposition
    \ref{prop-conditionals}. Note that $c\vert_{X \otimes Y}$ is not unique in general. However any such choice yields equivalent lenses since the category
    $\bf{Prob}$ is \emph{productive}, see \cite[Theorem 7.2]{dilavore2022} for a slightly more general result. Indeed, this proposition is the base case for the inductive proof of \cite[Theorem 7.2]{dilavore2022}.
\end{proof}

Why should we use lenses $(X, S) \to (Y, R)$ instead of combs?
The difference between them is in the composition. Indeed, composing lenses
corresponds to nesting combs as in the following diagram:
\begin{equation*}
    \input{figures/composing-combs.tikz}
\end{equation*}
As we will see, the composition of lenses allows to define a notion of
feedback for probabilistic systems which correctly captures their dynamics.

\subsection{Utility functions}

We analyse a special type of composition in $\bf{Lens}(\bf{Prob})$:
between a lens and its environment, also called \emph{context} in the open games
literature \cite{bolt2019}. This is a comb in $\bf{Lens}(\bf{C})$ that first
produces an initial state in $X$, then receives a move in $Y$ and produces a
\emph{utility} or reward in $R$.

\begin{definition}[Context]\label{def-context-lenses}
    The context for lenses of type $(X, S) \to (Y, R)$ over $\bf{C}$ is a comb
    $c$ in $\bf{Lens}(\bf{C})$ of the following shape:
    \begin{equation}
        \input{figures/lenses-context.tikz}
    \end{equation}
\end{definition}

\begin{proposition}\label{prop-contexts-factor}
    Contexts $C$ for lenses of type $(X, S) \to (Y, R)$ in $\bf{Prob}$ are in
    one-to-one correspondence with pairs $[p, k]$ where $p \in \cal{D}(X)$ is a
    distribution over observations called \emph{prior} and
    $k: X \times Y \to \cal{D}(R)$ is a channel called \emph{utility function}.
\end{proposition}
\begin{proof}
    Since the unit of the tensor is terminal in $\bf{Prob}$,
    a context $C$ of the type above is given by a pair of morphisms,
    $[p', k']$ as in the following diagram:
    \begin{equation}
        \input{figures/lenses-prob-context.tikz}
    \end{equation}
    Moreover, by the disintegration theorem, the joint distribution $p'$ above
    may be factored into a prior distribution $p$ over $X$ and a channel
    $c: X \to \cal{D}(M)$. Therefore the context above is
    equivalent to to a context induced by a distribution over starting state
    $p \in \cal{D}(X)$ and a utility function $k: X \times Y \to R$
    as in the following diagram:
    \begin{equation}
        \input{figures/lenses-prob-context-2.tikz}
    \end{equation}
\end{proof}

As we will see in Section \ref{section-games}, this notion of utility captures
the notion from game theory where utility functions assign a reward in $R$
given the move of the player in $Y$ and the moves of the other players, see
also \cite{bolt2019}.

\subsection{Markov rewards}

Another interesting form of composition in $\bf{Lens}(\bf{Prob})$ arises when
considering the notion of a Markov reward process (MRP) \cite{howard1960}.
MRPs are dynamic probabilistic systems which transit through a state space $X$ while
collecting rewards in $\bb{R}$.

\begin{definition}[MRP]\label{def:MRP}
    A Markov reward process with state space $X$ is given by the following data:
    \begin{enumerate}
        \item a transition function $T: X \to \cal{D}(X)$,
        \item a payoff function $R: X \to \bb{R}$,
        \item and a discount factor $\gamma \in [0, 1)$.
    \end{enumerate}
\end{definition}

This data defines lens $[T, R_\gamma]: (X, \bb{R}) \to (X, \bb{R})$ in $\bf{Prob}$
drawn as follows:
\begin{equation}
    \input{figures/markov-reward.tikz}
\end{equation}
where $R_\gamma: X \times \bb{R} \to \cal{D}(\bb{R})$ is given by the one-step
discounted payoff:
$$ R_\gamma(x, r) = R(x) + \gamma r$$
Intuitively, the MRP observes the state $x \in X$ which it is in and collects
a reward $R(x)$, then uses the transition $T$ to move to the next state. Given
an expected future reward $r$, the MRP computes the current value given
by summing the current reward with the expected future reward discounted by
$\gamma$. This process is called \emph{Markov}, because the state at time step
$t + 1$ only depends on the state a time step $t$, i.e. $x_{t+1}$ is sampled
from the distribution $T(x_t)$.

Thus an MRP is an endomorphism $[T, R_\gamma]: (X, \bb{R}) \to (X, \bb{R})$
in the category of stochastic lenses.
We can compose $[T, R_\gamma]$ with itself $n$ times to
get a new lens $[T, R_\gamma]^n$ where the forward part is given by iteration
of the transition function and the backward part is given by the $n$-step
discounted payoff:
$${R_\gamma^n}^-(x, r) = \sum_{i =1}^{n-1} \gamma^iR(T^i(x)) + \gamma^n r$$
Since $0\leq \gamma < 1$, this expression converges in the limit as $n \to \infty$
if $R$ and $T$ are deterministic. When $R$ and $T$ are stochastic, one can show
that the expectation $\bb{E}({R_\gamma^n}^-)$ converges to the \emph{value} of the
MRP which yields a measure of the total reward expected from this process.
\begin{equation} \label{eq-discounted-sum}
  \tt{value}([T, R_\gamma])(x) = \bb{E}(\sum_{i=1}^\infty \gamma^iR(T^i(x)))
  = \lim_{n \to \infty} \bb{E}(\sum_{i=1}^n \gamma^iR(T^i(x)))
\end{equation}
We can represent this as an effect $v = \tt{value}([T, R_\gamma]):
(X, \bb{R}) \to (1, 1)$ in $\bf{Lens}(\bf{Prob})$, which simply consists in a
probabilistic channel $v : X \to \bb{R}$.
This effect $v$ satifies the following Bellman fixed point equation in
$\bf{Lens}(\bf{Prob})$, characterising it as the iteration of $[T, R_\gamma]$,
see \cite{gouberman2014}.
$$v(x) = \bb{E}(R(x) + \gamma v(T(x)))$$
which we may express diagrammatically as:
\begin{equation}\label{bellman-expectation}
    \input{figures/bellman-expectation.tikz}
\end{equation}
where $\bb{E}$ denotes the conditional expectation operator, given for any channel
$f: X \to \cal{D}(\bb{R})$ by the function $\bb{E}(f): X \to \bb{R}$ defined
by $\bb{E}(f)(x) = \sum_{y \in \bb{R}} y f(y \vert x)$
(note that there only finitely many non-zero terms in this discrete setting).
The value $v$ of the reward process is often estimated by running Monte Carlo
methods which iterate the transition function while collecting rewards.
Note that \emph{not} all stochastic lenses $(X, \bb{R}) \to (X, \bb{R})$ have
an effect $(X, \bb{R}) \to (1, 1)$ with which they satisfy Equation \ref{bellman-expectation}.
In fact it is sufficient to set $T = \tt{id}_X$, $R(x) = 1$ for all $x \in X$ and
$\gamma > 1$ in order to get a counter-example. It would be interesting to characterise
the stochastic lenses satisfying \ref{bellman-expectation} algebraically, e.g.
are they closed under composition? This would in fact be true if the conditional
expectation operator $\bb{E}$ was functorial.
Given the results of \cite{chaput2009} and \cite{adachi2018}, we expect that conditional expectation can be made a functor by restriciting the objects of $\bf{Prob}$.

\section{Cybernetics}\label{section-agents}

Parametrization is the process of representing functions
$X \to Y$ via a parameter space $\Pi$ with a map $\Pi \to (X \to Y)$.
It is a fundamental tool in both machine learning and game theory, since it
allows to define a notion of \emph{agency}, through the choice of parameters.
For example, players in a formal game are parametrized over a set of
\emph{strategies}: there is a function $X \to Y$, turning observations into
moves, for any strategy in $\Pi$. In reinforcement learning, the agent is
parametrized by a set of \emph{policies}, describing how to turn states into
actions. We show that parametrized lenses are suitable for representing these
systems and give examples relevant for NLP.

\subsection{Parametrization}

Categories allow to distinguish between two types of parametrization. Let
$\bf{S}$ be a semantic category with a forgetful functor $\bf{S} \injects \bf{Set}$

\begin{definition}[External parametrization]
    An \emph{external} parametrization $(f, \Pi)$ of morphisms $X \to Y$ in
    a category $\bf{S}$, also called a \emph{family} of morphisms indexed by $\Pi$,
    is a function $f : \Pi \to \bf{S}(X, Y)$. These form a category
    $\bf{Fam}(\bf{S})$  with composition defined for $f : \Pi_0 \to \bf{S}(X, Y)$
    and $g : \Pi_1 \to \bf{S}(Y, Z)$ by $g \circ f : \Pi_0 \times \Pi_1 \to \bf{S}(X, Z)$
    with $g \circ f(\pi_0, \pi_1) = f(\pi_0) \circ g(\pi_1)$.
\end{definition}

\begin{definition}[Internal parametrization]
    An \emph{internal} parametrization $(f, \Pi)$ of morphisms $X \to Y$ in a
    monoidal category $\bf{S}$ is a morphism $f : \Pi \otimes X \to Y$ in $\bf{S}$.
    These form a category $\bf{Para}(\bf{S})$ with composition defined for
    $f : \Pi_0 \otimes X \to Y$ and $g : \Pi_1 \otimes Y \to Z$ by
    $g \circ f : \Pi_0 \otimes \Pi_1 \otimes X \to Z$ with
    $g \circ f = f(\pi_0) \circ g(\pi_1)$.
\end{definition}

Which parametrization should we prefer, $\bf{Fam}$ or $\bf{Para}$?
It depends on context. Internal parametrization
is usually a stricter notion, because it imposes that the parametrization be a
morphism of $\bf{S}$ and not simply a function.
For example, $\bf{Para}(\bf{Smooth})$ embeds in
$\bf{Fam}(\bf{Smooth})$ but the embedding is not full, i.e. there are external
parametrizations defined by non-differentiable functions.
In fact, the $\bf{Para}$ construction was introduced in the context of
gradient-based learning \cite{cruttwell2021}, where it is very desirable that
the parametrization be differentiable.
External parametrizations are mostly used in the compositional game theory
literature \cite{GhaniHedges18}, since they are more flexible and allow to define
a notion of \emph{best response} (see \ref{def-best-response}).
However they are also inextricably linked to $\bf{Set}$, making them less
desirable from a categorical perspective, since it is harder to prove results
about $\bf{Fam}(\bf{S})$ given knowledge of $\bf{S}$.

Note that the two notions coincide for the category of sets and functions.
Indeed, since $\bf{Set}$ is cartesian closed, we have that:
$$ \Pi \to \bf{Set}(X, Y) \iff  \Pi \to (X \to Y) \iff \Pi \times X \to Y$$
and therefore $\bf{Para}(\bf{Set}) \simeq \bf{Fam}(\bf{Set})$.
Even though $\bf{Prob}$ is not cartesian closed, these notions again coincide.

\begin{proposition}
    Internal and external parametrizations coincide in
    $\bf{Prob}$, i.e. $\bf{Fam}(\bf{Prob}) \simeq \bf{Para}(\bf{Prob})$.
\end{proposition}
\begin{proof}
    This follows by the following derivation in $\bf{Set}$.
    $$ \Pi \to \bf{Prob}(X, Y) \iff  \Pi \to (X \to \cal{D}(Y)) \iff
    \Pi \times X \to \cal{D}(Y)$$
\end{proof}

In fact, the derivation above holds in any Kleisli category for a commutative
strong monad.

\begin{proposition}
    For any commutative strong monad $M: \bf{Set} \to \bf{Set}$, the Kleisli
    category $\bf{Kl}(M)$ is monoidal and we have:
    $$\bf{Fam}(\bf{Kl}(M)) \simeq \bf{Para}(\bf{Kl}(M))$$
    i.e. the notions of internal and external parametrization for Kleisli
    categories coincide.
\end{proposition}

\subsection{Open games}\label{section-games}

We now review the theory of open games \cite{GhaniHedges18, bolt2019}.
Starting from the definition of lenses, open games are obtained in two steps.
First an open game is a family of lenses parametrized by a set of \emph{strategies}:
for each strategy the forward part of the lens says how the agent turns observations
into moves and the backward part says how it computes a payoff given the outcome
of its actions. Second, this family of lenses is equipped with a \emph{best response}
function indicating the optimal strategies for the agent in a given context.

\begin{definition}[Family of lenses]
    A family of lenses over a symmetric monoidal category $\bf{C}$ is a morphism
    in $\bf{Fam}(\bf{Lens}(\bf{C}))$. Explicitly, a family of lenses
    $P: (X, S) \to (Y, R)$ is given by a function
    $P: \Pi \to \bf{Lens}(\bf{C})((X, S), (Y, R))$ for some set of parameters
    $\Pi$.
\end{definition}

Recall the definition of \emph{context} for lenses given in \ref{def-context-lenses}.
Let us use the notation $\bb{C}((X, S), (Y, R))$ for the set of contexts of lenses
$(X, S) \to (Y, R)$.

\begin{definition}[Best response] \label{def-best-response}
     A best response function for a family of lenses $P: (X, S) \to (Y, R)$ is
     a function of the following type:
     $$B : \bb{C}((X, S), (Y, R)) \to \cal{P}(P)$$
     Thus $B$ takes as input a context and outputs a predicate
     on the set of parameters (or strategies).
\end{definition}

\begin{definition}[Open game]\cite{bolt2019}
    An open game $\G: (X, R) \to (Y, O)$ over a symmetric monoidal category
    $\bf{C}$ is a family of lenses $\set{[\pi, v_\pi]}_{\pi \in P_\G}$ in $\bf{C}$
    equipped with a best response function $B_\G$.
\end{definition}
\begin{proposition}\cite{bolt2019}
    Open games over $\bf{C}$ form a symmetric monoidal category denoted
    $\bf{Game}(\bf{C})$.
\end{proposition}
\begin{proof}
    See \cite{bolt2019} for details on the proof and the composition of best
    responses in this general case.
\end{proof}

The category $\bf{Game}(\bf{C})$ admits a graphical calculus developed
in~\cite{hedges2017} which is the same as the one for lenses described in
Section \ref{section-lenses}.. Each morphism is represented as a box with
covariant wires for observations in $X$ and moves in $Y$, and contravariant
wires for rewards in $R$ (called utilities in \cite{GhaniHedges18})
and satisfaction in $S$ (called coutilities in \cite{GhaniHedges18}).
\begin{equation}
    \input{figures/generic-og.tikz}
\end{equation}
These boxes can be composed in parallel,
indicating that the players make a move simultaneously, or in sequence,
indicating that the second player can observe the first player's move.
\begin{equation}
    \input{figures/compose-og.tikz}
\end{equation}

Of particular interest to us is the category of open games over $\bf{Prob}$. Since
we focus on this category for the rest of the chapter we will use the
notation $ \bf{Game} := \bf{Game}(\bf{Prob})$.
In this category the best response function is easier to specify, since contexts
factor as in Proposition \ref{prop-contexts-factor}.

\begin{proposition}
    A context for an open game in $\bf{Game}(\bf{Prob})$
    $(X, S) \to (Y, R)$ is given by a stochastic lens $[p, k]: (1, X) \to (Y, R)$
    or explicitly by a pair of channels $p: 1 \to \cal{D}(M \times X)$ and
    $k: M \times Y \to \cal{D}(R)$.
\end{proposition}
\begin{proof}
    This follows from Proposition \ref{prop-contexts-factor}.
\end{proof}

With this simpler notation for contexts as pairs $[p, k]$ we can define
the composition of best responses for open games.

\begin{definition}\cite[Definition 3.18]{bolt2019}
    Given open games $(X, S) \xto{\G} (Y, R) \xto{\cal{H}} (Z, O)$ the composition
    of their best responses is given by the cartesian product:
    $$B_{\G \cdot \cal{H}}([p, k]) = B_{\G}([p, \cal{H}\cdot k])
    \times B_{\cal{H}}([p \cdot \G,k]) \sub \Pi_\G \times \Pi_\cal{H}$$
    where $\Pi_\G$ and $\Pi_\cal{H}$ are the corresponding sets of strategies.
\end{definition}

We can now define a notion of utility-maximising agent

\begin{definition}[Utility-maximising agent]
    An (expected) utility maximising agent is a morphism $c: (X, 1) \to (Y, \bb{R})$
    in $\bf{Game}$ given by the following data:
    \begin{enumerate}
        \item $\pi \in \Pi= X \to \cal{D}(Y)$ is the set of strategies.
        \item $f_\pi^+: X \to \cal{D}(A) : x \mapsto \pi(x)$
        \item $f_\pi^-: \bb{R} \to 1$ is the discard map.
        \item $B: ([p: 1 \to M \times X,\, k: M \times Y \to \bb{R}]) =
                \tt{argmax}_{\pi \in P}(\bb{E}(p ; \tt{id}_M \otimes f_\pi ; k))$
    \end{enumerate}
    where $\tt{argmax}: \bb{R}^\Pi \to \cal{P}(\Pi)$.
\end{definition}

Note that it is sufficient to specify the input type $X$ and output type $Y$ in
order to define a utility-maximising agent $(X, 1) \to (Y, \bb{R})$ in
$\bf{Game}$.

\begin{example}[Classifier utility]\label{ex-classifier-utility}
    We model a classifier $X \to \cal{D}(Y)$ as a utility-maximising agent
    $c: (X, 1) \to (Y, \bb{R})$. Assume that $Y$ has a metric $d: Y \times Y \to \bb{R}$.
    Given a dataset of pairs, i.e. a distribution $K \in \cal{D}(X \times Y)$,
    we can define a context $[K, k]$ for $c$ by setting $k = -d$,
    as shown in the following pair of equivalent diagrams:
    \begin{equation*}
        \input{figures/classifier-utility.tikz}
    \end{equation*}
    where $(p, K\vert_X)$ is the disintegration of $K$ along $X$.
    This yields a distribution over the real numbers $\bb{R}$ corresponding to
    the utility that classifier $c$ gets in its approximation of $K\vert_X$.
    The best response function for $c$ in this context is given by:
    $$B([K, k]) = \tt{argmax}_{c}\bb{E}(- K ; \tt{id}_Y \otimes c ; d)$$
    Maximising this expectation corresponds to minimising the average distance
    between the true label $y \in Y$ and the predicted label $c(x)$ for $(x, y)$
    distributed as in the dataset $K \in \cal{D}(X \times Y)$.
\end{example}

\begin{example}[Generator/Discriminator]
    Fix a grammar $G$, a feature set $Y$ with a distance function $d: Y \times Y \to \bb{R}$
    and a disribution $p \in \cal{D}(Y)$.
    We model a generator $G: Y \to \cal{D}\cal{L}(G)$ and a discriminator
    $D: \cal{L}(G) \to \cal{D}(Y)$ as utility-maximising agents and compose
    them as in the following diagram:
    \begin{equation*}
        \input{figures/generator-discriminator.tikz}
    \end{equation*}
    Where $- : \bb{R} \to \bb{R}$ is multiplication by $-1$. This yields an
    adversarial (or zero-sum) game between the generator and the discriminator.
    Assuming $G$ and $D$ are utility-maximising agents, the best response function
    for this closed game is given by:
    $$\tt{argmax}_{G}(\bb{E}_{y \in p}(d(y, D(G(y)))) \times \tt{argmax}_{D}(- \bb{E}_{y \in p}(d(y, D(G(y)))))$$
    where we use the notation $\bb{E}_{y \in p}$ to indicate the expectation
    (over the real numbers) where $y$ is distributed according to $p \in \cal{D}(Y)$.
    Thus we can see that the game reaches an equilibrium when the discriminator
    $D$ can always invert the generator $G$, and the generator cannot choose
    any strategy to change this.
    Implementing $G$ and $D$ as neural networks yields the generative adversarial
    architecture of Goodfellow et al. \cite{goodfellow2014a} which is used to solve
    a wide range of NLP tasks, see \cite{wang2019} for a survey.
    In Section \ref{section-adversarial-qa}, we will see how this architecture
    can be used for question answering.
\end{example}

In the remainder of this chapter, we will work interchangeably in the category
of families of stochastic lenses $\bf{Fam}(\bf{Lens}(\bf{Prob}))$ or in
$\bf{Game}(\bf{Prob})$. The only difference between these categories is that
the latter has morphisms equipped with a best response function and a
predefined way of composing them. As we will see, this definition of best response
is not suited to formalising repeated games and the category of families of
stochastic lenses gives a more flexible approach.

\subsection{Markov decisions}

We now study Markov decision processes (MDP) \cite{howard1960} as parametrized
stochastic lenses. These are central in reinforcement learning, they model
a situation in which a single agent makes decisions in a stochastic environment
with the aim of maximising its expected long-term reward. An MDP is simply an
MRP \ref{def:MRP} parametrised by a set of actions.

\begin{definition}[MDP]
    A Markov decision process with states in $X$ and actions in $A$ is given
    by the following data:
    \begin{enumerate}
        \item a transition $T: A \times X \to \cal{D}(X)$
        \item a reward function $R: A \times X \to \cal{D}(\bb{R})$
        \item a discount factor $\gamma \in [0, 1)$
    \end{enumerate}
\end{definition}

A \emph{policy} is a function $\pi : X \to \cal{D}(Y)$ which represents a strategy
for the agent: it yields a distribution over the actions $\pi(x) \in \cal{D}(Y)$ for each
state $x \in X$.
Given a policy $\pi \in \Pi = X \to \cal{D}(Y)$, the MDP induces an MRP which runs
through the state space $X$ by choosing actions
according to $\pi$ and collecting rewards.
Thus, we can formalise an MDP as a family of MRPs
parametrized by the policies $\Pi$, i.e. a morphism
$P_\pi : (X, \bb{R}) \to (X, \bb{R})$ in $\bf{Fam}(\bf{Lens}(\bf{Prob}))$,
given by the following composition:
\begin{equation}
    \input{figures/markov-decision.tikz}
\end{equation}

The aim of the agent is to maximise its expected discounted reward.
Given a policy $\pi$, the MRP $P_\pi$ induces a value function
$v_\pi = \tt{value}(P_\pi) : X \to \bb{R}$, as defined in Section \ref{section-tools}.
Again, we have that the \emph{Bellman expectation equation} holds:
\begin{equation}\label{eq-bellman-expectation}
    \input{figures/bellman-equation.tikz}
\end{equation}
This fixed-point equation gives a way of approximating the value function by
iterating the transitions and rewards under a given policy.
The aim is then to find the \emph{optimal policy} $\pi^\ast$ at starting state
$x_0$ by maximising the corresponding value function $v_\pi$:
$$ \pi^\ast =  \tt{argmax}_{\pi \in P}(v_\pi(x_0)).$$
Note that the this optimal policy is not captured by the open games formalism.
If we took the MDP $P_\pi$ to be a morphism in $\bf{Game}$, i.e. if we equipped
it with a best response $B$ as defined in \ref{def-best-response}, then the
sequential composition $P_\pi \cdot P_\pi$ would have a best response function
with target $\Pi \times \Pi$, i.e. the strategy for an MDP with two stages
would be a pair of policies for the first and second stage, contradicting
the idea that the MDP tries to find the best policy to use at all stages of the
game. Working in the more general setting of parametrized stochastic lenses
we can still reason about optimal strategies, although it would be interesting
to have a compositional characterization of these.

\begin{example}[Dialogue agents]\label{ex-dialogue}
    In \cite{li2016a}, Li et al. use reinforcement learning to model dialogue
    agents. For this they define a Markov decision process with the
    following components. Fix a set of words $V$.
    The states are given by pairs $(p, q)$ where $q \in V^\ast$ is the current
    sentence uttered (to which the agent needs to reply) and $p \in V^\ast$ is
    the previous sentence uttered by the agent. An action is given by a sequences of words
    $a \in V^\ast$. The policy is modeled by a Seq2Seq model with attention based on
    Bahdanau et al. \cite{bahdanau2014}, which they pretrain on a dataset of dialogues.
    Rewards are computed using three reward functions $R_1$, $R_2$ and $R_3$,
    where $R_1$ penalises dull responses $a$ (by comparing them with a manually
    constructed list), $R_2$ penalizes semantic similarity between consecutive
    statements of the agent $p$ and $a$, and $R_3$ rewards semantic coherence between
    $p, q$ and $a$. These three rewards are combined using a weighted sum
    $\lambda\tt{sum}(r_1, r_2, r_3) = \lambda_1r_1 + \lambda_2r_2 + \lambda_3r_3$.
    We can depict the full architecture as one diagram:
    \begin{equation*}
        \input{figures/dialogue-agent.tikz}
    \end{equation*}
\end{example}

\subsection{Repeated games}

The definition of MDP given above captures a \emph{single} agent interacting
in a fixed environment with no agency. In real-world situations however,
the environment consists itself in a collection of agents which also make choices
so as to maximise their expected reward.

In order to allow many agents interacting in an environment, we break up the
definition of MDPs above, and isolate the part that is making decisions from
the \emph{environment} part.
We pack the transition $T$, reward $R$ and discount factor $\gamma$ into one
lens called environment and given by:
$$E = [T, R_\gamma]: (X \times A, \bb{R}) \to (X , \bb{R})$$
where $R_\gamma : X \times A \times \bb{R} \to \bb{R}$ is defined by
$$R_\gamma(x, a, r) = R(x, a) + \gamma r$$
This allows to isolate the decision part, seen as an open game of the following type:
$$D : (X, 1) \to (A, \bb{R}) $$
which represents an agent turning states in $X$ into actions in $A$ and receiving
rewards in $\bb{R}$. Explicitly we have the following definition.
\begin{definition}[Decision process]
    A decision process $D$ with state space $X$, action space $A$ and discount
    factor $\gamma$ is a utility maximising agent $D: (X, 1) \to (A, \bb{R})$
\end{definition}
Composing $D$ with the environment $E$ as in the
following diagram, we get back precisely the definition of MDP.
\begin{equation}
    \input{figures/decision-environment.tikz}
\end{equation}
We can now consider a situation in which many decision processes
interact in an environment.

Stochastic games, a.k.a Markov games, were introduced by Shapley in the 1950s
\cite{shapley1953}. They can be seen as a generalisation of MDPs, even though
they appeared before the work of Bellman \cite{bellman1957}.
\begin{definition}[Stochastic game]
    A stochastic game is given by the following data:
    \begin{enumerate}
        \item A number $k$ of players, a set of states $X$ and a discount factor $\gamma$,
        \item a set of actions $A_i$ for each player $i \in \{ 0, \dots, k - 1\}$,
        \item a transition function $T: X \times A_0 \times \dots \times A_k \to \cal{D}(X)$,
        \item a reward function $R: X \times A_0 \times \dots \times A_k \to \bb{R}^k$.
    \end{enumerate}
\end{definition}
The game is played in stages. At each stage, every player observes the
current state in $X$ and chooses an action in $A_i$, then ``Nature'' changes the
state of the game using transition $T$ and every player is rewarded according to
$R$. It is easy express this repeated game as a parallel composition of
decision processes $D_i$ followed by the environment $[T, R_\gamma]$ in $\bf{Game}$.
As an example, we give the diagram for a two-player zero-sum stochastic game.
\begin{equation}
    \input{figures/zero-sum-stoch-game.tikz}
\end{equation}
In this case, one can see directly from the diagram, that the zero-sum stochastic
game induces an MDP with action set $\Pi_0 \times \Pi_1$ given by pairs of
policies from players $0$ and $1$. A consequence of this is that one can prove
the Shapley equations \cite{renault2019}, the analogue
of the Bellman equation of MDPs, for determining the equilibria of the game.

Note that stochastic games only allow players to make moves in parallel, i.e.
at every stage the move of player $i$ is independent of the other players' moves.
The generality provided by this compositional approach allows to consider repeated
stochastic games in which moves are made in sequence, as in the following example.

\begin{example}[Wittgenstein's builders]\cite{wittgenstein1953}
    Consider a builder and an apprentice at work on the building site.
    The builder gives orders to his apprentice, who needs to implement them
    by acting on the state of the building site.
    In order to model this language game we start by fixing a grammar $G$
    for orders, e.g. the regular grammar defined by the following labelled graph:
    \begin{equation*}
        \input{figures/builder-order.tikz}
    \end{equation*}
    This defines a language for orders $\cal{L}(G)$ given by the paths
    $s_0 \to o$ in the graph above.
    The builder observes the state of the building site $S$ and produces an
    order in $\cal{L}(G)$, we can model it as a channel $\pi_B: S \to \cal{D}\cal{L}(G)$.
    We assume that the builder has a project $P$, i.e. a subset of the possible
    configurations of the building site that he finds satisfactory $P \sub S$.
    We can model these preferences with a utility function $u_P : S \to \bb{R}$,
    which computes the distance between the current state of the building site
    and the desired states in $P$. The builder wants to get the project
    done as soon as possible and he discounts the expected future rewards by a
    discount factor $\gamma \in [0, 1)$. This data defines a
    stochastic lens given by the following composition:
    \begin{equation*}
        \input{figures/builder-def.tikz}
    \end{equation*}
    The apprentice receives an order in $\cal{L}(G)$ and uses it to modify
    the state of the building site $S$. We can model this as a probabilistic channel
    $\pi_A: \cal{L}(G) \times S \to \cal{D}(S)$. This data defines the following
    stochastic lens:
    \begin{equation*}
        \input{figures/apprentice-def.tikz}
    \end{equation*}
    We assume that the apprentice is satisfied when the builder is, i.e.
    the utilities of the builder are fed directly into the apprentice.
    This defines a repeated game, with strategy profiles given by pairs
    $(\pi_B, \pi_A)$, given by the following diagram:
    \begin{equation*}
        \input{figures/builders.tikz}
    \end{equation*}
    This example can be extended and elaborated in different directions.
    For example, the policy for the apprentice $\pi_A$ could be modeled as a
    functor $F: G \to \bf{Set}$ with $F(s_0) = 1$ and $F(o) = S \to \cal{D}(S)$
    so that grammatical orders in $\cal{L}(G)$ are mapped to channels that change
    the configuration of the building site.
    Another aspect is the builder's strategy $\pi_B : S \to \cal{L}(G, o)$
    which could be modeled as a generation process from a probabilistic
    context-free grammar. Also the choice of state space $S$ is interesting,
    it could be a discrete minecraft-like space, or a continuous space in which
    words like ``cut'' have a more precise meaning.
\end{example}

\section{Examples}\label{section-examples}

\subsection{Bayesian pragmatics}\label{bayesian-pragmatics}

In this section we study the \emph{Rational Speech Acts} model of pragmatic reasoning
\cite{franke2009, frank2012, goodman2013, bergen2016}.
The idea, based on Grice's conversational implicatures \cite{grice1967}, is to
model speaker and listener as rational agents who choose words attempting
to be informative in context.
Implementing this idea involves the interaction of game theory and Bayesian inference.
While this model has been criticised on the ground of attributing excessive
rationality to human speakers \cite{gatt2013}, it has received support by
psycholinguistic experiments on children and adults \cite{frank2012} and has
been applied successfully to a referential expression generation task
\cite{monroe2015}.

Consider a collaborative game between speaker and listener.
There are some objects or \emph{referents} in $R$ lying on the table.
The \emph{speaker} utters a word in $W$ referring to one of the objects.
The \emph{listener} has to guess which object the word refers to.
We define this reference game by the following diagram in $\bf{Game}(\bf{Prob})$.
\begin{equation}\label{referring}
    \input{figures/referring.tikz}
\end{equation}
Where $p$ is a given prior over the referents (encoding the probability that an object $r \in R$
would be referred to) and $\Delta(r, r') = 1$ if $r = r'$ and $\Delta(r, r') = 0$ otherwise.
The strategies for the speaker and listener are given by:
$$ P_0 = R \to \cal{D}(W) \qquad  P_1 = W \to \cal{D}(R)$$
The speaker is modeled by a utility-maximising agent with strategies
$\pi_0: R \to \cal{D}(W)$ and best response in context
$[p \in \cal{D}(R \times W), l: R \times W \to \bb{B}]$ given by the
$\pi : R \to \cal{D}(W)$ in the argmax of
$\bb{E}(l \circ (\pi \otimes \tt{id}_R) \circ p)$.
Similarly for the listener with the roles of $R$ and $W$ interchanged.
Composing speaker and listener according to (\ref{referring}) we obtain a
closed game for which the best response is a predicate over the
strategy profiles $(\pi_0: R \to \cal{D}(W), \pi_1: W \to \cal{D}(R))$ indicating the
subset of Nash equilibria.
\begin{equation}\label{eq-argmax-referring}
    \tt{argmax}_{\pi_0, \pi_1} (\bb{E}( \Delta \circ ((\pi_1 \circ \pi_0)
    \otimes \tt{id}_R) \circ \tt{copy} \circ p)
\end{equation}
If we assume that the listener uses Bayesian inference to recover the speaker's intended
referent, then we are in a Nash equilibrium.

\begin{proposition}\label{prop-nash-equilibrium}
    If $\pi_1: W \to \cal{D}(R)$ is a Bayesian inverse of $\pi_0: R \to \cal{D}(W)$
    along $p \in \cal{D}(R)$ then $(\pi_0, \pi_1)$ is a Nash equilibrium for (\ref{referring}).
\end{proposition}
\begin{proof}
    Since $\pi_1$ is a Bayesian inverse of $\pi_0$ along $p$, Equation (\ref{eq-bayes})
    implies the following equality:
    \begin{equation*}
        \input{figures/bayesian-proof.tikz}
    \end{equation*}
    By definition of $\Delta$, it is easy to see that the expectation of the
    diagram above is $1$. Therefore it is in the argmax of (\ref{eq-argmax-referring}),
    i.e. it is a Nash equilibrium for (\ref{referring}).
\end{proof}
\begin{remark}
    We do not believe that all equilibria for this game arise through Bayesian
    inversion, but were unable to find a counterexample.
\end{remark}

\begin{example}\label{ex-pointy-round}
    As an example suppose the referents are shapes $R = \set{Bouba, Kiki}$ and
    the words are $W = \set{\text{pointy}, \text{round}}$. A literal listener
    would use the following channel $\pi_0 : W \to \cal{D}(R)$ to assign words
    to their referents.
    \begin{equation*}
        \begin{tabular}{ |c|c|c| }
         \hline
          & Bouba & Kiki\\
         \hline
         \text{round} &  $1$ & $0$ \\
         \hline
         \text{pointy} & $1/2$ & $1/2$\\
         \hline
        \end{tabular}
    \end{equation*}
    A pragmatic speaker would use the Bayesian inverse of this channel to refer
    to the letters.
    \begin{equation*}
        \begin{tabular}{ |c|c|c| }
         \hline
          & Bouba & Kiki\\
         \hline
         \text{round} &  $2/3$ & $0$ \\
         \hline
         \text{pointy} & $1/3$ & $1$ \\
         \hline
        \end{tabular}
    \end{equation*}
    A pragmatic listener would be inclined to use the Bayesian inverse of the
    pragmatic speaker's channel.
    \begin{equation*}
        \begin{tabular}{ |c|c|c| }
         \hline
          & Bouba & Kiki\\
         \hline
         \text{round} &  $1$ & $0$ \\
         \hline
         \text{pointy} & $1/4$ & $3/4$\\
         \hline
        \end{tabular}
    \end{equation*}
    And so on. We can see that Bayesian inversion correctly captures the pragmatics
    of matching words with their referents in this restricted context.
\end{example}

Frank and Goodman \cite{frank2012} model the conditional
distribution $\pi_0 : R \to \cal{D}(W)$ with a likelihood based on an
information-theoretic measure known as \emph{surprisal} or \emph{information content}.
$$ \pi_0(r)(w) = \frac{\size{R(w)}}{\sum_{w' \in W(r)} \size{R(w')}} $$
where $R(w)$ is the set of objects to which word $w$ could refer to and
$W(r)$ is the set of words that could refer to object $r$.
Then given a prior $p$ over the set of objects, the Bayesian inverse
$\pi_1 = \pi_0^\dagger$ can be computed using Bayes rule.
In their experiment, the participants were split in three
groups: speaker, salience and listeners. They were asked questions
to test respectively for the likelihood $\pi_0$, the prior $p$ and the
posterior predictions $\pi_1$ of their model.

Note that the experiment of Frank and Goodman involved three referents
(a blue square, a blue circle and a green square) and four words
(blue, green, square and circle). The sets
$R(w)$ and $W(r)$ were calculated by hand, and computing the Bayesian inverse
of $\pi_0$ in this case was an easy task. However, as the number of possible
referents and words grows, Bayesian inversion quickly becomes computationally
intractable without some underlying compositional structure mediating it.

\subsection{Adversarial question answering}\label{section-adversarial-qa}

In this section we define a game modelling an interaction between a teacher
and a student. The teacher poses questions that the student tries
to answer. We assume that the student is incentivised to answer questions
correctly, whereas the teacher is incentivised to ask hard questions, resulting
in an \emph{adversarial} question answering game (QA).
For simplicity, we work in a deterministic setting, i.e. we work in the category
of in $\bf{Game}(\bf{Set})$.
We first give a syntactic definition of QA as a diagram,
we then instantiate the definition with respect to monoidal grammars and
functorial models, which allows us to compute the Nash equilibria for the game.

Let us fix three sets $C$, $Q$, $A$ for corpora (i.e. lists of facts),
questions and answers respectively. Let $U$ be a set of utilities, which
can be taken to be $\bb{R}$ or $\bb{B}$.
A \textit{teacher} $\game{\teacher}{C}{1}{Q \times A}{U}{}$ is a
utility-ma\-xi\-mi\-sing player where each strategy represents a
function turning facts from the corpus into pairs of questions and answers.
A \textit{student} $\game{\student}{Q}{1}{A}{U}{}$ is a utility-maximising player
where each strategy represents a way of
turning questions into answers.
A \textit{marker} is a strategically trivial open game
$\game{\marker}{A \times A}{U \times U}{1}{1}{}$
with trivial play function and a coplay function defined as
$\Cf_\marker (a_\teacher, a_\student) = (-\dist(a_{\student}, a_{\teacher}),
\dist(a_{\student}, a_{\teacher}))$
where $\dist: A \times A \to U$ is a given metric on $A$.
Finally, we model a \textit{corpus} as a strategically trivial
open game $\game{f}{1}{1}{C}{1}{}$ with play function given by
$\Pf_f (*) = f \in C$.
All these open games are composed to obtain a \textit{question answering game}
in the following way.
\begin{equation}\label{qa-game}
    \input{figures/qa.tikz}
\end{equation}
Intuitively, the teacher produces a question from the corpus and gives it
to the student who uses his strategy to answer. The marker will receive
the correct answer from the teacher together with the answer that the student produced,
and output two utilities. The utility of the teacher will be the distance
between the student's answer and the correct answer; the utility of the
student will be the exact opposite of this quantity. In this sense,
question answering is a zero-sum game.

We now instantiate the game defined above with respect to a
pregroup grammar $G = (B, V, D, s)$ with a fixed question type $z \in B$.
The strategies of the student are given by relational models $\sigma: G \to \bf{Rel}$
with $\sigma(z) = 1$, so that given a question $q: u \to z$,
$\sigma(q) \in \cal{P}(1) = \bb{B}$ is the student's answer. In practice,
the student may only have a subset of models available to him so we set
$ \Sigma_\student \subseteq \{ \sigma: G \to \bf{Rel}\, \colon\, \sigma(z) = 1\} $.

We assume that the teacher has the particularly simple role of picking
a question-answer pair from a set of possible ones, i.e.
we take the corpus $C$ to be a list of question-answer pairs $(q, a)$ for
$q: u \to z$ and $a \in A$. For simplicity, we assume $q$ is a yes/no question
and $a$ is a boolean answer, i.e. $Q = \cal{L}(G, z)$ and $A = \bb{B}$.
The strategies of the teacher are given by indices
$\Sigma_\teacher = \set{0, 1, \dots n}$, so that the play function
$\Pf_\teacher :  \Sigma_\teacher \times (Q \times A)^\ast \to Q \times A$
picks the question-answer pair indicated by the index.
The marker will compare the teacher's answer $a$ with the student's
answer $\sigma(q) \in \bb{B}$ using the metric
$\dist: A \times A \to \bb{B}:: (a_0, a_1) \mapsto (a_0 = a_1)$ and output
boolean utilities in $U = \bb{B}$.
Plugging these open games as in (\ref{qa-game}),
we can compute the set of equilibria by composing the
equilibrium functions of its components.
\begin{equation*}
  \Ef_{\G}= \{ (j, \sigma) \in \Sigma_\teacher \times \Sigma_{\student}\,
  \colon\, j \in \amax_{i \in \Sigma_\teacher} a_{i} \neq \sigma(q_{i}) \land
  \sigma \in \amax_{\sigma \in \Sigma_\student} (a_{j} = \sigma(q_{j}))\}
\end{equation*}
Therefore, in a Nash equilibrium, the teacher will ask the question that the
student, even with his best guess, is going to answer in the worst way.
The student, on the other hand, is going to answer as correctly as possible.

We analyse the possible outcomes of this game.
\begin{enumerate}
  \item There is a pair $(q_{i}, a_{i})$ in $C$ that the student cannot
    answer correctly, i.e. $ \forall \sigma \in \Sigma_\student\, \colon\, \sigma(q_{i}) \neq a_{i}$. Then $i$ is a winning strategy for the teacher
    and $(i, \sigma)$ is a Nash equilibrium, for any choice of strategy
    $\sigma$ for the student.
    If no such pair exists, then we fall into one of the following cases.
  \item The corpus is consistent --- i.e. $\exists \sigma: G \to \bf{Rel}$
    such that $\forall i \cdot \sigma(q_i) = a_i$ --- and the student has access to the model $\sigma$ that answers all the possible questions correctly.
    Then, the strategy profile $(j, \sigma)$ is a Nash equilibrium and a
    winning strategy for the student for any choice $j$ of the teacher.
  \item For any choice $i$ of the teacher, the student has a model $\sigma_i$
    that answers $q_i$ correctly. And viceversa, for any strategy
    $\sigma$ of the student there is a choice $j$ of the teacher such that
    $\sigma(q_j) \neq a_j$. Then the set $\Ef_\G$ is empty, there is
    no Nash equilibrium.
\end{enumerate}

To illustrate the last case, consider a situation where the corpus
$C = \set{(q_0, a_0), (q_1, a_1)}$ has only
two elements and the student has only
two models $\Sigma_\student = \set{\sigma_0, \sigma_1}$ such that
$\sigma_i(q_i) = a_i$ for $i \in \set{0, 1}$ but $\sigma_0(q_1) \neq a_1$ and
$\sigma_1(q_0) \neq a_0$. Then we're in a \emph{matching pennies} scenario,
both the teacher and the student have no incentive to choose any one of their
strategies and there is no Nash equilibrium. This problem can be ruled out if
we allowed the players in the game to have \emph{mixed strategies}, which can
be achieved with minor modifications of the open game formalism \cite{ghani2019}.

\subsection{Word sense disambiguation}\label{section-wsd}

Word sense disambiguation (WSD) is the task of choosing the correct sense of
a word in the context of a sentence. WSD has been argued to
be an AI-complete task in the sense that it can be used to simulate other NLP
task \cite{navigli2009}.
In \cite{tripodi2019}, Tripodi and Navigli use methods from evolutionary
game theory to obtain state-of-the-art results in the WSD task.
The idea is to model WSD as a collaborative game between words
where strategies are word senses. In their model, the interactions between
words are weighted by how close the words are in a piece of text. In this
section, we propose a compositional alternative, where the interaction
between words is mediated by the grammatical structure of the sentence they
occur in. Concretely, we show that how to build a functor $J: G_W \to \bf{Game}$
given a functorial model $F: G_V \to \bf{Prob}$ where $G_W$ is a
grammar for words and $G_V$ a grammar for word-senses. Given a sentence
$u \in \cal{L}(G_V)$, the functor $J$ constructs a
collaborative game where the Nash equilibrium is given by a choice of sense
for each word in $u$ that maximises the score of the sentence $u$ in $F$.

We work with a \emph{dependency grammar} $G \sub (B + V) \times B^\ast$ where
$B$ is a set of basic types and $V$ a set of words, see Definition
\ref{def-dependency-grammar}. Recall from Proposition \ref{prop-pregroup-trees},
that dependency relations are acyclic, i.e. we can always turn the dependency
graph into a tree as in the following example:
\begin{equation*}
    \input{figures/wsd-dependency.tikz}
\end{equation*}
This means that any sentence parsed with a dependency grammar induces a directed
acyclic graph of dependencies. We may represent these parses in the free monoidal
category $\bf{MC}(\tilde{G})$ where $\tilde{G} = B^\ast \leftarrow V \to B$ is
the signature with boxes labelled by words $v \in V$ with inputs given by the
symbols that depend on $v$ and a single output given by the symbol on which
$v$ depends, as shown in Proposition \ref{prop-pregroup-trees}.

Taking dependency grammars seriously, it is sensible to interpret them directly
in the category $\bf{Prob}$ of conditional distributions. Indeed, a functor
$F: \tilde{G} \to \bf{Prob}$ is defined on objects by a choice of feature set
$F(x)$ for every symbol $x \in B$ and on words $v : y_1y_2 \dots y_n \to x$
by a conditional distribution $F(v): F(y_1) \times F(y_2) \dots F(y_n) \to \cal{D}(F(x))$
indicating the probability that word $v$ has a particular feature
given that the words it depends on have particular features. Thus a parsed sentence
is sent to a \emph{Bayesian network} where the independency constraints of the
network are induced by the dependency structure of the sentence. One may prove
formally that functors $\tilde{G} \to \bf{Prob}$ induce Bayesian networks
using the work of Fong \cite{fong2013}.

We are now ready to describe the WSD procedure. Fix a set $W$ of words and
$V$ of word-senses with a dependency grammar $G_V$
for senses and a relation $\Sigma \sub W \times V$, assigning to each word $w \in W$
the set of its senses $\Sigma(w) \sub V$.
Composing $G_V$ with $\Sigma$, we get a grammar for words $G_W$.
Assume we are given a functorial model
$F: G_V \to \bf{Prob}$ with $F(s) = 2$, i.e. for any grammatical
sentence $g: u \to s$, its image $F(g) \in \cal{D}(2) = [0, 1]$ quantifies how likely it is
that the sentence is true. We use $F$ to define a functor
$J: G_W \to \bf{Game}({\bf{Prob}})$ such that the Nash equilibrium
of the image of any grammatical sentence is an assignement of each word to a
sense maximising the overall score of the sentence.
On objects $J$ is defined by $J(b) = (F(b), 2)$ for any $b \in B$.
Given a word $w : y_1 \dots y_n \to x$ in $G_W$ its image under $J$ is given by
an open game $J(w) : (F(y_1) \times F(y_2) \dots F(y_n), 2^n) \to (F(x), 2)$
with strategy set $\Sigma(w)$ (i.e. strategies for word $w \in W$
are given by its senses $v \in \Sigma(w) \sub V$) defined for every strategy
$v \in \Sigma(w)$ by the following lens:
\begin{equation*}
    \input{figures/wsd-word.tikz}
\end{equation*}
The best response function in context $[p, k]$ is given by:
$$ B_{J(w, t)}(k) = \tt{argmax}_{v \in \Sigma(w)}(\bb{E}(p ; J(w)(v) ; k))$$
Then given a grammatical sentence $g \in \bf{MC}(G_W)$
we get a closed game $J(g) : 1 \to 1$  with equilibrium given by:
$$ B_{J(g)} = \tt{argmax}_{v_i \in \Sigma(w_i)}(F(g[w_i := v_i]))$$
where $u[w_i := v_i]$ denotes the sentence obtained by replacing each word
$w_i \in W$ with its sense $v_i \in \Sigma(w_i) \sub V$. Computing this
argmax corresponds precisely to choosing a sense for each word such that the
probability that the sentence is true is maximised.

\begin{example}
  As an example take the following dependency grammar $G$ with $\tilde{G}$ given
  by the following morphisms:
  \begin{equation}
      \text{Bob}: 1 \to n \: , \: \text{draws} = n \otimes n  \to s \: , \:
      \text{a}: 1 \to d \: , \:
      \text{card} : d \to n \: , \: \text{diagram}: d \to n
  \end{equation}
  The relation $\Sigma \sub W \times V$ between words and senses is given by:
  \begin{equation*}
    \Sigma(\text{Bob}) = \{\text{Bob Coecke},\: \text{Bob Ciaffone}\} \quad
    \Sigma(\text{draws}) = \{\text{draws (pull)},\: \text{draws (picture)}\}
  \end{equation*}
  and $\Sigma(x) = \set{x}$ for $x \in \{\text{card}, \text{diagram}, \text{a}\}$.
  The functor $F: G_V \to \bf{Prob}$ is defined on objects by
  $F(d) = 1$ (i.e. determinants are discarded), $F(s) = 2$ and
  $F(n) = \{\text{Bob Coecke}, \text{Bob Ciaffone}, \text{card}, \text{diagram}\}$
  and on arrows by:
  \begin{equation*}
    F(x)(y) = \begin{cases}
                  1 & x = y\\
                  0 & \text{otherwise}
                \end{cases}\, .
  \end{equation*}
  for $x \in V \setminus \{ \text{draws (pull)},\, \text{draws (picture)}\}$.
  The image of the two possible senses of ``draw'' are given by the
  following table:
  \begin{equation*}
      \begin{tabular}{ |c|c|c|c| }
       \hline
       subject & object & draws (picture) & draws (pull) \\
       \hline
       \text{Bob Coecke} & \text{card} & $0.1$ & $0.3$  \\
       \hline
       \text{Bob Ciaffone} & \text{card} & $0.1$ & $0.9$ \\
       \hline
       \text{Bob Coecke} & \text{diagram} & $0.9$ & $0.2$  \\
       \hline
       \text{Bob Ciaffone} & \text{diagram} & $0.1$ & $0.1$ \\
       \hline
      \end{tabular}
  \end{equation*}
  Note that a number in $[0, 1]$ is sufficient to specify a distribution in $\cal{D}(2)$.
  We get a corresponding functor $J: G_W \to \bf{Game}(\bf{Prob})$
  which maps ``Bob draws a diagram'' as follows:
  \begin{equation*}
      \input{figures/wsd-game.tikz}
  \end{equation*}
  Composing the channels according to the structure of the diagram we get a distribution
  in $\cal{D}(2)$ parametrized over the choice of sense for each word.
  According to the table above, the expectation of this distribution is maximised
  when the strategy of the word ``Bob'' is the sense  ``Bob Coecke'' and the strategy
  of ``draws'' is the sense ``draws (picture)''.
\end{example}

\pagestyle{plain}
\printbibliography[heading=bibintoc,title=References]

\end{document}